\pdfoutput=1
\documentclass[11pt]{article}

\usepackage[utf8]{inputenc}
\usepackage[T1]{fontenc}
\usepackage[margin=1in]{geometry}
\usepackage{url}
\usepackage{booktabs}
\usepackage{amsfonts}
\usepackage{nicefrac}
\usepackage{microtype}
\usepackage{natbib}
\usepackage{algorithm}
\usepackage{algorithmic}
\usepackage{wrapfig}
\usepackage{tikz}
\usepackage{amsmath,amssymb}
\usetikzlibrary{arrows.meta,decorations.markings,calc,angles,quotes,positioning,
                shapes.geometric,backgrounds}

\graphicspath{{figures/}{./}}

\definecolor{teachercol}{RGB}{205,40,40}
\definecolor{student1col}{RGB}{40,90,200}
\definecolor{student2col}{RGB}{230,140,20}
\definecolor{leapcol}{RGB}{110,110,115}
\definecolor{phase1col}{RGB}{45,130,90}
\definecolor{phase2col}{RGB}{150,60,150}

\def\figscale{5.75} 
\def\refscale{5.4}
\pgfmathsetmacro{\sR}{\figscale/\refscale}

\pgfmathsetmacro{\fSm}{9.0*\sR}
\pgfmathsetmacro{\fSmLs}{1.2*\fSm}
\pgfmathsetmacro{\fFt}{8.0*\sR}
\pgfmathsetmacro{\fFtLs}{1.2*\fFt}
\pgfmathsetmacro{\fSc}{7.2*\sR}
\pgfmathsetmacro{\fScLs}{1.2*\fSc}

\newcommand{\figfsmall}{\fontsize{\fSm pt}{\fSmLs pt}\selectfont}
\newcommand{\figffoot}{\fontsize{\fFt pt}{\fFtLs pt}\selectfont}
\newcommand{\figfscript}{\fontsize{\fSc pt}{\fScLs pt}\selectfont}

\pgfmathsetmacro{\lwTraj}{1.4*\sR}
\pgfmathsetmacro{\lwTeacher}{1.5*\sR}
\pgfmathsetmacro{\lwLeap}{0.6*\sR}
\pgfmathsetmacro{\lwBox}{0.4*\sR}

\pgfmathsetmacro{\isepReg}{2.0*\sR}
\pgfmathsetmacro{\isepSm}{1.5*\sR}
\pgfmathsetmacro{\isepXSm}{1.2*\sR}
\pgfmathsetmacro{\saddleSep}{1.6*\sR}
\pgfmathsetmacro{\starSep}{1.2*\sR}
\pgfmathsetmacro{\dotSep}{0.9*\sR}
\pgfmathsetmacro{\legSaddleSep}{1.0*\sR}

\usepackage{macro_wth}

\title{Gradient Descent on Two ReLU Neurons: \\ Global Landscape and Bifurcation Dynamics}
\newcommand{\equalcontrib}{\textsuperscript{*}}
\author{%
  \begin{tabular}{c@{\hspace{2em}}c}
    Binghua Li\equalcontrib & Mengzhe Li\equalcontrib\\
    University of California, Berkeley & \'Ecole Normale Sup\'erieure Paris-Saclay\\
    \small\texttt{binghuali@berkeley.edu} & \small\texttt{mengzhe.li\_mengzhe@ens-paris-saclay.fr}\\[2ex]
    Denny Wu & Tianhao Wang\\
    New York University \& Flatiron Institute & University of California, San Diego\\
    \small\texttt{dwu@flatironinstitute.org} & \small\texttt{tianhaowang@ucsd.edu}
  \end{tabular}
}
\date{}

\begin{document}

\maketitle
\begingroup
\renewcommand{\thefootnote}{\fnsymbol{footnote}}
\footnotetext[1]{Equal contribution.}
\endgroup

\begin{abstract}
Understanding how gradient descent learns features remains a central challenge in neural network theory. We study this question in a minimal multi-index setting that already exhibits rich multi-phase gradient dynamics: a well-specified two-neuron ReLU teacher-student model under isotropic Gaussian inputs, trained by population gradient flow from small initialization. We show that the dynamics are organized by two directions: the ``easy'' bisector direction, which carries the leading signal, and the ``hard'' splitting direction, which governs specialization. To characterize the loss structure and learning behavior, we analyze the population squared-loss landscape and show that every critical point with both student weights nonzero is either a global minimum or a saddle. We then track the gradient-flow trajectory, showing that student neurons first collapse toward a bisector saddle before escaping and specializing to teacher neurons. Our results provide both a landscape and dynamics account of the multi-phase symmetry-breaking behavior that arises in a simple multi-index model and under standard algorithmic and architectural choices.
\end{abstract}

\section{Introduction}

A central promise of gradient-based training of neural networks is \textit{feature learning}: rather than fitting the target function on a fixed feature space, neural networks can extract features by aligning their parameters with the latent structure of the data or target function. A recent line of theoretical works has made this intuition precise by analyzing teacher-student models under Gaussian inputs, where the target function is a \textit{multi-index model} \citep{abbe2022merged,damian2022neural,ba2022high} which depends on a small number of input directions (e.g., a shallow neural network with $O_d(1)$ neurons). In these settings, it has been demonstrated that gradient descent training of a shallow ``student'' neural network discovers the relevant subspace, and achieves a sample complexity superior to fixed features methods (or neural networks in the ``lazy'' regime \citep{chizat2019lazy}).  

However, most existing theory becomes tractable only after imposing simplifying assumptions that deviate from standard algorithmic and architectural choices. Examples include $(i)$ near-orthogonal teacher neurons \citep{li2020learning,ren2024learning}, $(ii)$ local convergence conditions \citep{zhou2021local,zhu2025gradient}, $(iii)$ activation functions with high \textit{information exponent} \citep{oko2024learning,csimcsek2024learning,ren2025emergence}, and $(iv)$ modified algorithms such as loss preprocessing \citep{damian2022neural,zhang2025neural} and spherical/Stiefel gradients \citep{arous2021online,arous2024stochastic,abbe2023sgd}. 
While these assumptions enable clean analyses, they do not capture the complex \textit{multi-phase} learning dynamics present in standard ReLU networks \citep{boursier2022gradient,wang2023understanding,min2025neural}.
A notable multi-index example exhibiting such nontrivial dynamics is the ``well-specified'' setting in which both the student and teacher are two-layer neural networks with matching activation functions and more than one neuron: for commonly used nonlinearities such as ReLU, the optimization trajectory exhibits symmetry-breaking and bifurcation behavior, with neurons initially moving toward an ``average'' direction before specializing to different teacher directions. While this plateau-specialization phenomenon has appeared in the statistical physics literature \citep{saad1995line,goldt2019dynamics}, it is not captured by recent analyses that rely on the above simplifications. To our knowledge, a rigorous and quantitative understanding of the loss landscape and multi-phase training dynamics remains incomplete and largely open.

\subsection{Our contributions}

We study arguably the simplest multi-index model setting that exhibits this bifurcation phenomenon: learning a two-neuron ReLU network under isotropic Gaussian inputs via population gradient flow, starting from small initialization. Specifically, the teacher model is
\begin{equation}
f^*(x)=\mathrm{ReLU}(\inner{v_1}{x})+\mathrm{ReLU}(\inner{v_2}{x}), \qquad x\sim \cN(0,I_d),
\label{eq:target}
\end{equation}
where $\mathrm{ReLU}(z) = \max\{z,0\}$, $\|v_1\|_2=\|v_2\|_2=1$ and $\angle(v_1,v_2)=\alpha\in(0,\frac{\pi}{2}]$, and the student has the same architecture with trainable weights $w_1,w_2$.

\begin{wrapfigure}{r}{0.44\textwidth} 
\vspace{-5.2mm}
\centering
\begin{tikzpicture}[
    scale=\figscale,
    >=Stealth,
    line cap=round,
    line join=round,
    every node/.append style={font=\figfsmall},
    traj1/.style={student1col, line width=\lwTraj pt},
    traj2/.style={student2col, line width=\lwTraj pt},
    arr1/.style={traj1, postaction={decorate},
        decoration={markings,
            mark=at position 0.34 with {\arrow[traj1]{Stealth}},
            mark=at position 0.84 with {\arrow[traj1]{Stealth}}}},
    arr2/.style={traj2, postaction={decorate},
        decoration={markings,
            mark=at position 0.34 with {\arrow[traj2]{Stealth}},
            mark=at position 0.84 with {\arrow[traj2]{Stealth}}}},
]
 
\def\angA{60}
\def\halfA{30}
\def\rs{0.62}
\def\Rteacher{1.0}
\pgfmathsetmacro{\vxtip}{cos(\halfA)*\Rteacher}
 
\def\boxL{-0.04}
\def\boxR{0.96}
\def\boxBot{-0.50}
\def\boxTop{0.50}
 
\coordinate (O)  at (0,0);
\coordinate (v1) at (\halfA:\Rteacher);
\coordinate (v2) at (-\halfA:\Rteacher);
\coordinate (ws) at (\rs,0);
\coordinate (W1) at (0.62, 0.010);
\coordinate (W2) at (0.62,-0.010);
 
\begin{scope}[on background layer]
    \fill[phase1col!9, rounded corners=2pt]
        (\boxL,\boxBot) rectangle (\rs,\boxTop);
    \fill[phase2col!9, rounded corners=2pt]
        (\rs,\boxBot) rectangle (\boxR,\boxTop);
    \draw[phase1col!28, line width=\lwBox pt, rounded corners=2pt]
        (\boxL,\boxBot) rectangle (\rs,\boxTop);
    \draw[phase2col!28, line width=\lwBox pt, rounded corners=2pt]
        (\rs,\boxBot) rectangle (\boxR,\boxTop);
\end{scope}
 
\draw[leapcol, dashed, line width=\lwLeap pt] (0,0) -- (\boxR,0);
\node[leapcol, anchor=south west, inner sep=\isepXSm pt]
    at (\boxR+0.012, 0.005) {$v_1 + v_2$};
\node[leapcol, font=\figffoot\bfseries, anchor=north west, inner sep=\isepXSm pt]
    at (\boxR+0.012,-0.005) {(leap-1)};
 
\draw[leapcol, dashed, line width=\lwLeap pt]
    (\vxtip,\boxBot) -- (\vxtip,\boxTop);
\node[leapcol, anchor=south, inner sep=\isepReg pt]
    at (\vxtip, 0.510) {$v_1 - v_2$\,{\figffoot\bfseries(leap-2)}};
 
\draw[->, teachercol!78, line width=\lwTeacher pt, dotted] (O) -- (v1);
\draw[->, teachercol!78, line width=\lwTeacher pt, dotted] (O) -- (v2);
\node[star, star points=5, fill=teachercol, draw=teachercol,
      inner sep=\starSep pt, scale=0.55] at (v1) {};
\node[star, star points=5, fill=teachercol, draw=teachercol,
      inner sep=\starSep pt, scale=0.55] at (v2) {};
\node[teachercol, anchor=north west, inner sep=\isepSm pt]
    at ($(v1)+(0.020,-0.005)$) {$v_1$};
\node[teachercol, anchor=south west, inner sep=\isepSm pt]
    at ($(v2)+(0.020, 0.005)$) {$v_2$};
 
\node[regular polygon, regular polygon sides=3, fill=black, draw=black,
      inner sep=\saddleSep pt] at (ws) {};
\node[anchor=north, font=\figffoot, inner sep=\isepReg pt]
      at ($(ws)+(0,-0.012)$) {$w_s$};
 
\node[circle, fill=black, inner sep=\dotSep pt] at (O) {};
\node[anchor=south east, font=\figffoot, inner sep=\isepSm pt]
      at ($(O)+(-0.005,0.010)$) {$0$};
 
\draw[arr1] (0,0)
    .. controls (0.20, 0.000) and (0.45, 0.0015) .. (W1)
    .. controls (0.83, 0.0205) and (0.84, 0.30)  .. (v1);
 
\draw[arr2] (0,0)
    .. controls (0.20, 0.000) and (0.45,-0.0015) .. (W2)
    .. controls (0.83,-0.0205) and (0.84,-0.30)  .. (v2);
 
\node[phase1col, font=\figfsmall\bfseries, anchor=south]
      at ({(\boxL+\rs)/2}, -0.405) {Phase 1};
\node[phase1col, font=\figffoot, anchor=south]
      at ({(\boxL+\rs)/2}, -0.470) {(collapse)};
 
\node[phase2col, font=\figfsmall\bfseries, anchor=south]
      at ({(\rs+\vxtip)/2}, -0.405) {Phase 2};
\node[phase2col, font=\figffoot, anchor=south]
      at ({(\rs+\vxtip)/2}, -0.470) {(split)};
 
\begin{scope}[shift={(-0.07,-0.55)}]
    \draw[traj1] (0.02,0) -- (0.08,0);
    \node[font=\figfscript, anchor=west, inner sep=\isepXSm pt]
        at (0.08,0) {student $w_1$};
    \draw[traj2] (0.35,0) -- (0.41,0);
    \node[font=\figfscript, anchor=west, inner sep=\isepXSm pt]
        at (0.41,0) {student $w_2$};
    \draw[->, teachercol, line width=\lwTeacher pt, dotted] (0.69,0) -- (0.77,0);
    \node[font=\figfscript, anchor=west, inner sep=\isepXSm pt]
        at (0.77,0) {teacher $v_i$};
    \node[regular polygon, regular polygon sides=3, fill=black,
          inner sep=\legSaddleSep pt] at (1.03,0) {};
    \node[font=\figfscript, anchor=west, inner sep=\isepSm pt]
        at (1.05,0) {saddle $w_s$};
\end{scope}
 
\end{tikzpicture}
\caption{\small Illustration of bifurcation dynamics. }
\label{fig:intro}
\vspace{-5mm}
\end{wrapfigure}
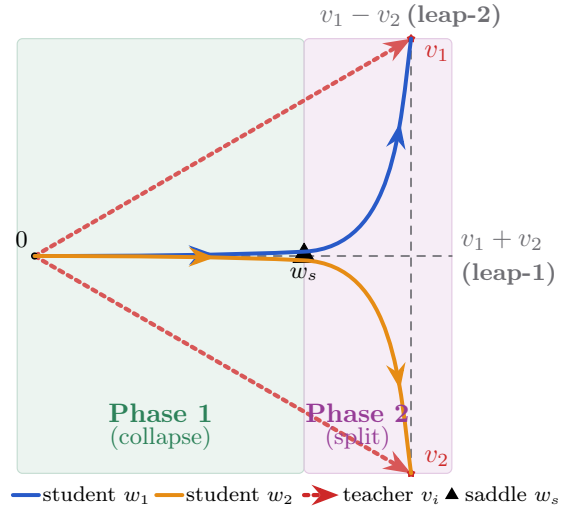

As we will see, the geometry of this model is organized by two directions in the teacher plane: the bisector ($v_1+v_2$), which carries the dominant signal, and the difference ($v_1-v_2$), which governs specialization. In the language of leap complexity for Gaussian multi-index models \citep{abbe2023sgd,bietti2023learning}, the target function \eqref{eq:target} has \textit{leap exponent} $2$, where the bisector and difference vectors specify the ``easy'' leap-1 direction and the ``hard'' leap-2 direction, respectively. 
This suggests a two-stage picture for gradient-based learning: the student neurons are first attracted toward the leap-1 bisector, and only later separate toward the individual teacher directions along the leap-2 direction. Our goal is to make this collapse-then-splitting picture quantitative by studying both the population squared loss landscape and the gradient-flow dynamics from small initialization. Specifically,

\begin{itemize}[leftmargin=*,topsep=0mm,itemsep=1mm]

    \item \textbf{Landscape analysis.} \Cref{sec:landscape} studies the loss geometry and the critical-point structure for both the correlation loss and the MSE loss. We reduce the stationarity conditions to a three-dimensional problem and isolate the remaining in-plane and out-of-plane critical families. 
    We establish that every critical point of the population loss with both student weights nonzero is either a global minimum or a saddle, so none of these critical points is a spurious local minimum. This resolves a two-neuron conjecture from \citet{tian2017analytical} in the acute angle ($\alpha\le\pi/2$) setting.

    \item \textbf{Small-initialization dynamics.} \Cref{sec:training_dynamics} studies the gradient-flow trajectory from small initialization and characterizes the symmetry breaking and specialization behavior --- see Figure~\ref{fig:intro}. We show that for sufficiently small initialization, both student neurons stay close to the bisector for fixed finite time and approach the symmetric saddle. We then analyze escape from that saddle, including global convergence to a minimum for typical initializations, as well as a quantitative escape-time bound that is logarithmic in the ambient dimension $d$ with a structured initialization.
\end{itemize}

\subsection{Further related work}

\paragraph{Learning multi-index models.} Recent works have shown shallow neural networks can efficiently learn low-dimensional latent structure via gradient descent \citep{abbe2022merged,damian2022neural,ba2022high,bietti2022learning,glasgow2023sgd,dandi2024two}. Specifically, the sample complexity has been shown to be superior to that of kernel methods and to be governed by the \textit{information exponent} \citep{dudeja2018learning,arous2021online} for single-index models and the \textit{leap exponent} \citep{abbe2023sgd,bietti2023learning} for multi-index models. 
Our target function \eqref{eq:target} is an example of a leap-2 multi-index model, but, crucially, it also contains a leap-1 subspace, which complicates the SGD dynamics. In fact, most prior works studying gradient-based feature learning either directly assume link functions with information exponent 2 \citep{li2020learning,martin2024impact,ren2024learning,arous2025learning} or higher \citep{oko2024learning,csimcsek2024learning,ren2025emergence}, or remove the low-leap components through algorithmic modifications such as deflation / projection \citep{ge2021understanding,abbe2023sgd} and loss preprocessing \citep{damian2022neural,zhang2025neural}. Consequently, these analyses do not fully characterize the learning dynamics of two ReLU neurons with standard gradient-based training.

\paragraph{Training dynamics of ReLU networks.} Restricting to shallow neural networks with ReLU activations, prior works have studied landscape properties such as the existence of spurious minima \citep{safran2018spurious,safran2021effects}; the concurrent work of \cite{huang2026sharp} gives a summary-statistics characterization of the population minima and proposes an ansatz to simplify the fixed-point equations. 
For gradient-based training dynamics, while the single-neuron case is relatively well-understood \citep{tian2017analytical,soltanolkotabi2017learning,yehudai2020learning,xu2023over}, when the target network contains more than one neuron, existing dynamical analyses either assume local convergence, thereby bypassing the bifurcation behavior \citep{zhou2021local,zhu2025gradient}, or require exponential width to invoke mean-field global convergence \citep{chizat2022sparse,akiyama2021learnability}. 

The two-phase trajectory illustrated in Figure~\ref{fig:intro} also relates to the ``plateau'' phenomenon observed in the statistical mechanics literature \citep{biehl1996transient,fukumizu2000local}. Specifically, the specialization phenomenon has been investigated in \cite{goldt2019dynamics,oostwal2021hidden} using the ODE analysis originated from \cite{saad1995line,riegler1995line}. 
However, to our knowledge, a complete understanding of the landscape and rigorous characterization of the learning dynamics have not yet been established even in the two-neuron setting \eqref{eq:target} studied here. 

\section{Preliminary}\label{sec:preliminary}

\paragraph{Notation.}
For a nonzero vector $w\in\RR^d$, we denote its direction by $\bar w\coloneqq w/\|w\|_2$, and write orthogonal projection onto the subspace perpendicular to $w$ as $P_w^\perp\coloneqq I-\bar w\bar w^\top\in\RR^{d\times d}$. 
The angle between two nonzero vectors $u,v\in\RR^d$ is
$
\angle(u,v) \coloneqq \arccos({\inner{u}{v}}/{(\|u\|_2\,\|v\|_2)}).
$
For vectors $a_1,\ldots,a_k$, $\operatorname{span}\{a_1,\ldots,a_k\}$ denotes the linear subspace generated by them.
We use $\phi(z) \coloneqq \mathrm{ReLU}(z) = \max\{0,z\}$ for the activation function. 

\subsection{Problem setup}

We use the teacher $f^*$ defined in \eqref{eq:target}. The student has trainable weights $w_1,w_2\in\RR^d$ and output
\begin{align}\label{eq:two_relu_student}
f(x) &= \phi(\inner{w_1}{x}) + \phi(\inner{w_2}{x}).
\end{align}
The student is trained by gradient flow on the population squared loss under isotropic Gaussian inputs:
\begin{equation}\label{eq:population_squared_loss_and_gf}
L(w_1,w_2) \coloneqq \frac{1}{2}\,\EE_{x\sim\cN(0,I_d)}\bigl[(f(x)-f^*(x))^2\bigr],\quad \frac{\diff w_i}{\diff t} = -\nabla_{w_i} L(w_1,w_2), \quad i\in\{1,2\},
\end{equation}
starting from an arbitrary nonzero initialization
$w(0)=(w_1(0),w_2(0))$ with $w_1(0)\ne0$ and $w_2(0)\ne0$.
We define its initialization scale by
\begin{equation}\label{eq:initialization-scale}
r_{\mathrm{init}}
\coloneqq
\frac{1}{2}\bigl(\|w_1(0)\|_2+\|w_2(0)\|_2\bigr).
\end{equation}
For independent Gaussian initialization $w_i(0)\sim\cN(0,\sigma^2I_d)$, $i=1,2$, both norms are of order $\sigma\sqrt d$ with high probability.
Our main focus is the small-initialization regime, which for the Gaussian initialization above corresponds to taking $\sigma$ sufficiently small.

We refer to $\operatorname{span}\{v_1,v_2\}$ as the \emph{teacher plane}. A student weight is \emph{in-plane} if it lies in the teacher plane and \emph{out-of-plane} if it has a nonzero component orthogonal to that plane. 
Within the teacher plane, we highlight two important directions: the \emph{bisector} and the \emph{splitting direction},
\begin{equation}\label{eq:two-directions}
\bar v \coloneqq \frac{v_1+v_2}{\|v_1+v_2\|_2},
\qquad
v_\Delta \coloneqq \frac{v_1-v_2}{\|v_1-v_2\|_2}.
\end{equation}
The pair $(\bar v,v_\Delta)$ forms an orthonormal basis of the teacher plane.
We define the initial planar splitting amplitude by
\begin{equation}\label{eq:initial-planar-splitting-amplitude}
a_{\mathrm{init}}
\coloneqq
\bigl|\langle w_1(0)-w_2(0),v_\Delta\rangle\bigr|.
\end{equation}
For $C\ge1$, we say that the initialization is $C$-balanced if
\begin{equation}\label{eq:c-balanced}
\max_{i=1,2}\|w_i(0)\|_2\le C\min_{i=1,2}\|w_i(0)\|_2.
\end{equation}
We use the angle shorthand
\begin{equation}\label{eq:angle_shorthand}
\theta_{ij}\coloneqq\angle(w_i,v_j),\quad i,j\in\{1,2\},
\qquad
\gamma\coloneqq\angle(w_1,w_2),
\end{equation}
where $\theta_{ij}$ is the angle between student neuron $i$ and teacher neuron $j$, and $\gamma$ is the angle between the two student directions.

As we will see, in the small-initialization regime, the dynamics follow the two-stage collapse-then-splitting picture from the introduction. As illustrated in \cref{fig:intro}, both student directions first collapse toward the bisector direction $\bar v$ and then split along the direction $v_\Delta$, eventually specializing toward the individual teacher directions $v_1$ and $v_2$. The next subsection explains the teacher-side hierarchy behind this behavior.

\subsection{Leap complexity of the teacher}\label{subsec:teacher_leap_complexity}

We use the standard Hermite-support notion of leap complexity for Gaussian multi-index models \citep{abbe2023sgd,bietti2023learning}; here only the following two-coordinate version is needed. 
Let $\{h_k\}_{k\ge0}$ denote the normalized probabilists' Hermite polynomials.
Suppose $U,W$ are independent standard Gaussians and $g\in L^2(\cN(0,I_2))$. Its Hermite expansion in $L^2(\cN(0,I_2))$ is
\begin{equation}\label{eq:two_coordinate_hermite_expansion}
g(U,W)=\sum_{m,n\ge0}\widehat g_{m,n}h_m(U)h_n(W).
\end{equation}
Along the ordering $U$ then $W$, a leap is the first Hermite degree at which a new coordinate appears. Thus the leap to $U$ is the smallest $m>0$ with $\widehat g_{m,0}\neq0$; once $U$ is included, the leap to $W$ is the smallest $n>0$ for which some $\widehat g_{m,n}$ is nonzero.

For the two-ReLU teacher network defined in \eqref{eq:target}, let $z_1\coloneqq\inner{\bar v}{x}$ and $z_2\coloneqq\inner{v_\Delta}{x}$ be the projections of the input onto the bisector and splitting directions defined in \eqref{eq:two-directions}.
Because the teacher lies in $\operatorname{span}\{\bar v,v_\Delta\}$, $f^*(x)$ depends on $x$ only through these two coordinates, and hence we can write $f^*(x)=g(z_1,z_2)$. Also, since $(z_1,z_2)\sim\cN(0,I_2)$, we may Hermite expand $g$ with respect to coordinates $(z_1,z_2)$.
The following proposition shows that although the nonlinearity $\phi$ has information/leap exponent 1, the sum of two ReLU neurons is a leap-2 function. 
\begin{proposition}[Leap complexity of teacher]\label{prop:teacher_leap_profile}
The two-ReLU teacher network $f^*$ in \eqref{eq:target} has leap exponent 2, where the bisector coordinate $z_1$ has leap 1 and the splitting coordinate $z_2$ has leap 2.
\end{proposition}
The proof is in \cref{subsec:teacher-leap-proof}. This leap structure gives useful intuition for the two-stage dynamics shown in \cref{fig:intro}. Specifically, the leap at each subspace determines the gradient magnitude and hence the order of learning: the leap-1 bisector $\bar v$ gives a linear signal visible at random initialization, while the leap-2 splitting direction $v_\Delta$ appears in a smaller quadratic signal.
Consequently, near small initialization the student--student interaction is negligible, and each student neuron is driven by its correlation with the teacher. At a random initialization in high dimension, each student direction is nearly orthogonal to the teacher plane, so the degree-one part of this correlation, which points toward $\bar v$, dominates. We show in \cref{thm:corr-loss-global-min} that, over unit-norm student weights, the full correlation term is minimized at the bisector, which results in a collapse toward $\bar v$ before the higher-order signal along $v_\Delta$ drives splitting and specialization. The ensuing sections will make this intuition precise through landscape and gradient-flow analysis.

\section{Landscape analysis}\label{sec:landscape}

We begin with a landscape analysis, with the goal of classifying all stationary points of the population squared loss with both student weights nonzero as global minima or saddles. We first consider the correlation loss landscape, which admits a simple characterization but cannot drive neuron specialization, and then analyze the more challenging squared loss setting.

\subsection{Warm-up: landscape of the correlation loss}

We first work with a more general setting where the teacher is a width-$k$ ReLU network, 
\begin{equation}\label{eq:general_relu_teacher}
f^*(x)=\sum_{i=1}^k \phi(\inner{v_i}{x}), \qquad k\geq2.
\end{equation}
The student has $m$ ReLU neurons, $f(x)=\sum_{j=1}^m \phi(\inner{w_j}{x})$, with trainable weights $w_1,\ldots,w_m\in\RR^d$.
The setup in \Cref{sec:preliminary} corresponds to the special case $k=m=2$.

By the positive homogeneity of ReLU, the population squared loss expands as
\[
\frac12\EE[(f(x)-f^*(x))^2]
=\frac12\EE[f^*(x)^2]
-\sum_{j=1}^m \|w_j\|_2 \cdot \EE[f^*(x)\phi(\inner{\bar w_j}{x})]
 + O\bigg(\sum_{i,j=1}^m \|w_i\|_2 \|w_j\|_2\bigg).
\]
The first term is constant in the student weights, while the last term is second order in their norms. Thus, for small student weights, the first-order decrease of the loss is governed by the correlations between the teacher output $f^*(x)$ and the normalized student activations $\phi(\inner{\bar w_j}{x})$. 
We refer to this leading $L^2$ correlation term as the correlation loss,
\begin{align}\label{eq:correlation-loss}
    L_{\mathrm{corr}}(w_1,\ldots,w_m) \coloneqq - \sum_{j=1}^m \EE_{x\sim \cN(0,I_d)}\bigl[f^*(x)\phi(\inner{w_j}{x})\bigr]. 
\end{align}
For the width-two result below, let
\(\alpha_\star\in(\pi/2,\pi)\) be the unique solution of
\begin{equation}\label{eq:width-two-angle-threshold}
    \tan\frac{\alpha_\star}{2}
    =
    \pi-\frac{\alpha_\star}{2}.
\end{equation}
Numerically, \(\alpha_\star\approx127.521476^\circ\).
While many existing works on gradient-based feature learning employ a correlation-loss approximation to simplify the SGD dynamics \citep{abbe2023sgd,damian2023smoothing,lee2024neural}, the following theorem shows that the collapsed configuration at the normalized average teacher direction is the unique minimizer on the unit sphere, and hence cannot recover the two teacher neurons.
\begin{theorem}\label{thm:corr-loss-global-min}
Consider the teacher model $f^*$ in \eqref{eq:general_relu_teacher}.
Suppose that $v_1,\ldots,v_k$ are unit vectors satisfying
$\angle(v_i,v_j)=\alpha$ for all $i\neq j$, and that either
$0<\alpha\leq\pi/2$, or $k=2$ and
$0<\alpha\leq\alpha_\star$. Let
$\bar{v}=(\sum_{i=1}^k v_i)/\|\sum_{i=1}^k v_i\|_2$.
Then, for every student width $m$, $(\bar v,\ldots,\bar v)$ is the unique
global minimum of $L_{\mathrm{corr}}$ among all unit-norm student weights.
\end{theorem}
The proof is in \cref{app:correlation-loss-proof}.
\begin{remark}\label{rem:corr-angle-range}
For $k=2$, the threshold $\alpha_\star$ is sharp: above it, the bisector
is no longer a local minimizer of \(L_{\mathrm{corr}}\) in
\eqref{eq:correlation-loss} (\cref{lem:corr-two-teachers}). For fixed $k>2$ and $d\geq k$,
$(\bar v,\ldots,\bar v)$ remains the unique global minimizer when
$\cos\alpha$ is negative and sufficiently close to zero. This follows from a
perturbation argument at $\alpha=\pi/2$. There, for a single student, $\bar v$
is a strict minimizer of $L_{\mathrm{corr}}$ on the unit sphere, with
positive definite Hessian on the sphere, and $L_{\mathrm{corr}}$ is bounded
away from its minimum outside any neighborhood of $\bar v$. Both properties persist
under small changes of $\alpha$. The size of this range of $\alpha$ may
depend on $k$ and $d$, and we do not determine how far it extends.
\end{remark}
The following two observations connect the theorem to the training dynamics.
\begin{itemize}
\item This result aligns with the first stage of the dynamics observed in \cref{fig:intro}: near small initialization, the squared loss is dominated by the correlation term, which is minimized when all student directions equal the averaged teacher direction $\bar v$.
In the two-neuron setting of \eqref{eq:target}, this averaged direction is exactly the bisector, explaining why the student neurons initially move toward this direction. This result is close in spirit to \citet{csimcsek2024learning}, which studies correlation-loss gradient in the high-information-exponent setting, whereas our ReLU setting has information exponent 1.
\item On the other hand, since the correlation loss $L_{\mathrm{corr}}$ ignores the student-student interaction terms when $\|w_j\|_2$ becomes non-negligible, it cannot explain why collapsed students eventually repel and specialize.
To capture this effect, we have to characterize the full squared-loss landscape.
\end{itemize}

\subsection{Landscape of the squared loss}

Now we turn to the squared loss \eqref{eq:population_squared_loss_and_gf}. To simplify the analysis, we write the first-order condition in angular variables. Recall the angle shorthand from~\eqref{eq:angle_shorthand}: $\theta_{ij}=\angle(w_i,v_j)$ and $\gamma=\angle(w_1,w_2)$.
\begin{lemma}[Characterization of stationary points]\label{lem:critical_eq_from_flow}
Define $a_1\coloneqq \sin\theta_{11}+\sin\theta_{12}-\pi\|w_1\|_2-\sin\gamma\|w_2\|_2$ and $a_2\coloneqq \sin\theta_{21}+\sin\theta_{22}-\pi\|w_2\|_2-\sin\gamma\|w_1\|_2$.
Then for $w_1,w_2\neq0$, $(w_1,w_2)$ is a stationary point of the squared loss $L$ in \eqref{eq:population_squared_loss_and_gf} if and only if
\begin{equation}\label{eq:critical_eq}
\begin{aligned}
(\pi-\theta_{11})v_1+(\pi-\theta_{12})v_2+a_1\bar w_1-(\pi-\gamma)\|w_2\|_2\bar w_2&=0,\\
(\pi-\theta_{21})v_1+(\pi-\theta_{22})v_2+a_2\bar w_2-(\pi-\gamma)\|w_1\|_2\bar w_1&=0.
\end{aligned}
\end{equation}
\end{lemma}
The proof is in \cref{subsubsec:stationary-equations-proof}. With the above
first-order condition in hand, we classify all stationary points with
\(w_1\neq0\) and \(w_2\neq0\), up to permutation of student neurons.

\paragraph{Additional notation.}
Before stating the main landscape theorem, we introduce some quantities that will be used to describe the stationary points. 
Recall that $\alpha=\angle(v_1,v_2)\in(0,\pi/2]$ is the angle between the teacher neurons.
Define
\begin{align}\label{eq:def_r}
    r_+\coloneqq \frac{2}{\pi}\left(\left(\pi-\frac{\alpha}{2}\right)\cos\frac{\alpha}{2}+\sin\frac{\alpha}{2}\right),
    \qquad
    r_-\coloneqq \frac{2}{\pi}\left(\sin\frac{\alpha}{2}-\frac{\alpha}{2}\cos\frac{\alpha}{2}\right).
\end{align}
For \(0<\kappa<\pi/2\), define
\begin{align}
    \theta_\alpha(\kappa)
    &=
    \arccos\left(\cos\kappa \cdot \cos\frac{\alpha}{2}\right),
    \label{eq:out-of-plane-angle}\\
    r_\alpha(\kappa)
    &=
    \frac{2\sin\theta_\alpha(\kappa)}
    {\sin(2\kappa)+2\kappa},
    \label{eq:out-of-plane-radius}\\
    F_\alpha(\kappa)
    &=
    \left(\pi-\theta_\alpha(\kappa)\right)\cos\frac{\alpha}{2}
    -
    (\pi-2\kappa)r_\alpha(\kappa)\cos\kappa.
    \label{eq:out-of-plane-root-function}
\end{align}

\begin{lemma}[Unique out-of-plane root]
\label{lem:out-of-plane-root-uniqueness}
For every \(\alpha\in(0,\pi)\), the equation
\(F_\alpha(\kappa)=0\) has exactly one solution in \((0,\pi/2)\), which we
denote by \(\kappa_\alpha\).
\end{lemma}

The proof is in \cref{subsubsec:out-of-plane-root-uniqueness-proof}.

Now we are ready to state the main landscape theorem for the squared loss.

\begin{theorem}[Landscape of the squared loss]\label{thm:squared_loss_landscape}
Let $L$ be the population squared loss of the two-neuron ReLU model in
\eqref{eq:population_squared_loss_and_gf}, and let \((w_1,w_2)\) be a
stationary point of \(L\) with \(w_1\neq0\) and \(w_2\neq0\). 
Then, after possibly swapping $w_1$ and $w_2$, exactly one of the following holds.
\begin{enumerate}[leftmargin=*]
    \item \textbf{Global minimum:} $w_1=v_1$ and $w_2=v_2$; this is a global minimum.
    \item \textbf{In-plane saddles:} $w_1,w_2\in \operatorname{span}\{v_1,v_2\}$ satisfy one of the following three forms:
    \begin{align}\label{eq:in-plane-saddle-form}
    \begin{cases}
        w_1=s_1\bar v,\quad w_2=s_2\bar v, & \text{for some $s_1,s_2>0$ with $s_1+s_2=r_+$};\\
        w_1=-s_1\bar v,\quad w_2=-s_2\bar v, & \text{for some $s_1,s_2>0$ with $s_1+s_2=r_-$};\\
        w_1=r_+\bar v,\quad w_2=-r_-\bar v.
    \end{cases}
    \end{align}
    Every such point has a strictly negative Hessian direction.
    \item \textbf{Out-of-plane saddles:} necessarily \(d\geq3\), and for
    some unit vector
    \(\xi\perp\operatorname{span}\{v_1,v_2\}\),
    \begin{equation}\label{eq:out-of-plane-saddle-form}
    \begin{aligned}
        w_1=r_\alpha(\kappa_\alpha)\left(\cos\kappa_\alpha\cdot \bar v+\sin\kappa_\alpha\cdot\xi\right),
        \qquad
        w_2=r_\alpha(\kappa_\alpha)\left(\cos\kappa_\alpha \cdot \bar v-\sin\kappa_\alpha\cdot\xi\right).
    \end{aligned}
    \end{equation}
    Every such point has a strictly negative Hessian direction.
\end{enumerate}
Consequently, the squared loss $L$ has no spurious local minima on $\{(w_1,w_2):w_1\neq0,\ w_2\neq0\}$.
\end{theorem}

To the best of our knowledge, this is the first complete classification of all stationary points with both student weights nonzero (under Euclidean geometry) in a well-specified shallow ReLU teacher–student model with more than one teacher neuron. We make the following remark on the significance of this landscape result.

\begin{remark}
\cite{tian2017analytical} conjectured that the unconstrained 
population squared loss for $k=2$ ReLU teacher-student under isotropic Gaussian 
inputs has no spurious local minima. Tian's own analysis was conditional on a numerically-verified inequality (Conjecture~1 of \citealp{tian2017analytical}) 
and resolved only the case where both student vectors lie inside or
outside $\mathrm{Cone}(v_1, v_2)$. 
\citet{wu2018no} made progress by proving the absence of spurious local minima 
for \emph{orthogonal} teachers on the \emph{unit-norm manifold}. Our Theorem~\ref{thm:squared_loss_landscape} settles this open problem in the acute angle setting: it provides an unconditional, complete classification of every critical point of the unconstrained squared loss with both student weights nonzero for general $\alpha \in (0, \pi/2]$, showing that each is either the global minimum or a saddle.
\end{remark}

We prove \cref{thm:squared_loss_landscape} in three steps. The first is a dimension reduction: any solution $(w_1,w_2)$ of \eqref{eq:critical_eq} is contained in the teacher plane together with at most one orthogonal direction. The next two steps analyze the resulting three-dimensional problem and the planar cases that remain.

\subsubsection{Step 1: reduce to three dimensions}
The dimension reduction is a simple consequence of the characterization in \cref{lem:critical_eq_from_flow}.
It is useful to decompose each student weight into the part lying in the teacher plane $\operatorname{span}\{v_1,v_2\}$ and the part orthogonal to that plane. For a solution of \eqref{eq:critical_eq}, the two orthogonal parts cannot vary independently: the system \eqref{eq:critical_eq} forces them to be linearly dependent, as summarized in the following proposition.

\begin{proposition}\label{prop:3D}
Let $w_1,w_2\neq0$ satisfy \eqref{eq:critical_eq}. 
For $i=1,2$, decompose
$w_i=w_i^\parallel+w_i^\perp$, where
$w_i^\parallel\in\operatorname{span}\{v_1,v_2\}$ and
$w_i^\perp\perp\operatorname{span}\{v_1,v_2\}$. Then
$w_1^\perp$ and $w_2^\perp$ are linearly dependent.
Consequently,
$
    \dim\operatorname{span}\{v_1,v_2,w_1,w_2\}\leq3.
$
\end{proposition}

The proof is in \cref{subsubsec:three-dimensional-reduction-proof}.
Consequently, for any stationary point $(w_1,w_2)$, the vectors $w_1,w_2$ are contained in the teacher plane together with at most one orthogonal direction.

\subsubsection{Step 2: reduce the nonplanar case to two dimensions}

Given the previous reduction, there are two possible types of stationary points: either both $w_1$ and $w_2$ lie in the teacher plane, or at least one of them lies outside the teacher plane.
We next show that for the latter case, the weights lie in the plane spanned by the bisector and an orthogonal direction.

\begin{lemma}[Three-dimensional symmetry reduction]\label{lem:3Dsym}
Let $w_1,w_2\neq0$ satisfy \eqref{eq:critical_eq}. 
If at least one of $w_1,w_2$ lies outside the teacher plane, then $w_1,w_2\in\operatorname{span}\{\bar v,\xi\}$ for some unit vector $\xi\perp\operatorname{span}\{v_1,v_2\}$.
\end{lemma}

The proof is in \cref{subsubsec:splitting-coordinate-elimination}.
Geometrically, \cref{lem:3Dsym} says that for a stationary point, if one weight is not contained in
the teacher plane, both weights have zero component along the splitting direction \(v_\Delta\).
Thus the
three-dimensional problem reduces to two planar problems, one in
$\operatorname{span}\{v_1,v_2\}$ and one in
$\operatorname{span}\{\bar v,\xi\}$.

\subsubsection{Step 3: solve the 2D cases}

\paragraph{Out-of-plane solutions.}
We first consider solutions in
$\operatorname{span}\{\bar v,\xi\}$ that do not lie in the teacher plane. By
\cref{lem:perp_plane_symmetry}, such a solution is symmetric with respect to the
teacher plane: after possibly replacing $\xi$ by $-\xi$, the two student weights have the same component in $\bar v$ and opposite components in $\xi$.
This corresponds to the out-of-plane saddle form in
\eqref{eq:out-of-plane-saddle-form}. 
The stationarity condition therefore reduces to a scalar equation for the angle away from the teacher plane.

\begin{lemma}[Out-of-plane stationary points]\label{lem:critical_point_perp_plane}
Let \(\alpha\in(0,\pi/2]\), suppose $d\geq 3$, and let \(\xi\) be any unit vector
perpendicular to \(\operatorname{span}\{v_1,v_2\}\).
For any nonzero $w_1,w_2 \in \operatorname{span}\{\bar v,\xi\}$ such that $\langle w_1,\xi\rangle\neq0$ or $\langle w_2,\xi\rangle\neq0$,
the pair \((w_1,w_2)\) is a stationary point of the squared loss \eqref{eq:population_squared_loss_and_gf} if and only if, after possibly swapping
the students or replacing \(\xi\) by \(-\xi\), it has the form
\eqref{eq:out-of-plane-saddle-form}.
Furthermore, every such stationary point has a strictly negative Hessian
direction and thus is a saddle.
\end{lemma}

The proof is in \cref{subsec:out-of-plane-stationary-points}. For \(d=2\), the out-of-plane branch is vacuous. For \(d\geq3\), the out-of-plane stationary points form a family of saddles indexed by the unit vectors in
\(\operatorname{span}\{v_1,v_2\}^\perp\).

\paragraph{In-plane solutions.}
We now turn to the other planar case where $w_1,w_2\in\operatorname{span}\{v_1,v_2\}$. 
This case contains the global minima where the student weights match the teacher weights up to permutation; the only other stationary points are the three bisector-supported families listed below.

\begin{lemma}[In-plane stationary points]\label{lem:critical_points_in_plane}
Let $r_+$ and $r_-$ be as defined in \eqref{eq:def_r}.
Let $w_1,w_2\neq0$ satisfy $w_1,w_2\in\operatorname{span}\{v_1,v_2\}$. Then $(w_1,w_2)$ is a stationary point of the squared loss \eqref{eq:population_squared_loss_and_gf} if and only if, after possibly swapping $w_1$ and $w_2$, either $(w_1,w_2)$ is a global minimum with $w_1=v_1,w_2=v_2$ or $(w_1,w_2)$ is a saddle satisfying one of the three forms in \eqref{eq:in-plane-saddle-form}.
\end{lemma}

The in-plane classification is proved
in \cref{subsec:in-plane-stationary-points}. 
Combining the above results yields the complete classification of stationary points in \cref{thm:squared_loss_landscape}.
The symmetric point
$w_1=w_2=(r_+/2)\bar v$ in the positive-bisector saddle family is the saddle
shown in \cref{fig:intro}. It remains to analyze how gradient flow moves
through this landscape from small initialization.

\section{Training dynamics}\label{sec:training_dynamics}

Next we study the gradient-flow trajectory from small initialization. Our analysis divides into two phases: \textit{Phase 1} shows that student neurons rapidly align with the bisector $\bar v$ and then their norms grow, and \textit{Phase 2} studies what happens near the saddle. We then present global convergence guarantees and escape-time estimates. For concise notation, we write $r_s=r_+/2$ where $r_+$ is defined in \cref{thm:squared_loss_landscape}, $w_s=r_s\bar v$, and $w_{\mathrm{saddle}}\coloneqq(w_s,w_s)$ for the symmetric positive-bisector saddle.

\subsection{Phase 1: alignment with the bisector}\label{subsec:phase1}
We first study how the gradient flow approaches the saddle point $w_{\mathrm{saddle}}$.
This involves two subphases: The direction dynamics drives both student neurons toward $\bar v$, and then the radial dynamics grows their norms to the saddle radius.

\paragraph{Subphase 1: Directional alignment.}
By direct calculation, the student directions along gradient flow admit the following evolution:
\[
    \begin{aligned}
    \frac{\diff\bar w_i}{\diff t}
    &=
    \frac{1}{2\pi\|w_i\|_2}
    P_{\bar w_i}^{\perp}
    \big(
        \underbrace{(\pi-\theta_{i1})v_1+(\pi-\theta_{i2})v_2}_{\text{teacher force}}
        -
        \underbrace{(\pi-\gamma)w_{3-i}}_{\text{student interaction}}
    \big),
    \qquad i=1,2.
    \end{aligned}
\]
The projected teacher force is the directional part of the teacher-student correlation term analyzed in \cref{thm:corr-loss-global-min}. Note that this term dominates the student interaction when $\|w_{3-i}\|_2$ is very small. Thus the teacher force controls the angular motion while the weights are small.
Moreover, the projected teacher force points toward the bisector by \cref{lem:phase-one-teacher-force-alignment}:
\[
    \textstyle\left\langle
        \bar v,
        P_{\bar w_i}^{\perp}
        \big((\pi-\theta_{i1})v_1+(\pi-\theta_{i2})v_2\big)
    \right\rangle
    \ge 0,
\]
with equality only at $\bar w_i=\bar v$ or $\bar w_i=-\bar v$. Thus, except at these two directions, the teacher force increases the alignment of the student neurons with $\bar v$. 
The factor $1/\|w_i\|_2$ makes this aligning very fast near small initialization, while the norm barely changes.

\paragraph{Subphase 2: Radial growth toward the saddle.}
Subphase 1 brings both student directions rapidly close to $\bar v$ while their norms are still small.
The subsequent radial-growth stage takes only a fixed amount of time. Over this interval, the direction dynamics can amplify the angular errors by only a bounded factor. Therefore, for sufficiently small initialization, the alignment established in Subphase 1 persists until the trajectory reaches the saddle neighborhood.

To understand the radial dynamics in the second subphase, suppose first that $w_i=r_i\bar v$ exactly. The norm equations then reduce to
\[
    \dot r_i=r_s-\frac12(r_1+r_2),
    \qquad i=1,2.
\]
Since $r_1+r_2$ is initially small, both derivatives are positive and the two norms grow. Moreover,
\[
    \frac{\diff}{\diff t}(r_1+r_2)
    =r_+-(r_1+r_2),
    \qquad
    \frac{\diff}{\diff t}(r_1-r_2)=0.
\]
Hence the total radius is attracted to $r_+=2r_s$, while the radial imbalance remains fixed. Small, balanced initialization makes this imbalance small, so both norms grow nearly together toward $r_s$. The alignment and balance estimates show that the full trajectory follows this reduced dynamics up to small errors and therefore enters a neighborhood of $w_{\mathrm{saddle}}$.

Combining the two subphases gives the following finite-time approach result.

\begin{lemma}[Phase 1: finite-time approach to the bisector saddle]
\label{lem:phase1_bisector_saddle}
Fix $d\geq2$ and unit teacher weights $v_1,v_2\in\RR^d$ with
$\alpha\coloneqq\angle(v_1,v_2)\in(0,\pi/2]$. Consider the two-student
population gradient flow \eqref{eq:population_squared_loss_and_gf} under
isotropic Gaussian inputs, and write $w(t)=(w_1(t),w_2(t))$.
Fix an absolute constant $C\geq 1$ and an angular margin
$\eta\in(0,\pi/2)$. For every target radius $\rho>0$, there are
$t_0=t_0(\alpha,\rho)<\infty$ and $r_*=r_*(\alpha,C,\eta,\rho)>0$ such that the
following holds. 
Recall the initialization scale $r_{\mathrm{init}}$ from
\eqref{eq:initialization-scale} and the balance condition
\eqref{eq:c-balanced}. If $0<r_{\mathrm{init}}\leq r_*$,
the initialization is $C$-balanced, and
$\angle(\bar w_i(0),-\bar v)\geq\eta$ for $i=1,2$, then the gradient-flow
solution exists on $[0,t_0]$, both student weights remain nonzero on this interval,
and there is a time $t_{\rm in}\leq t_0$ such that
\[
    \|w(t_{\rm in})-w_{\mathrm{saddle}}\|_2<\rho.
\]
\end{lemma}
The proof is in \cref{proof:phase-one}. The dependence on
$C$ and $\eta$ enters the allowed initialization scale; the time bound is
independent of these two parameters. For Gaussian initialization, the
following corollary separates the small-scale condition from the
dimension-dependent probability of balance and angular separation.

\begin{corollary}[Phase 1 under Gaussian initialization]
\label{cor:phase-one-gaussian-initialization}
In the setting of \cref{lem:phase1_bisector_saddle}, fix $C>1$,
$\eta\in(0,\pi/2)$, $\rho>0$, and let $t_0$, $r_*$ be the constants
in that lemma. Suppose
$w_1(0),w_2(0)\overset{\mathrm{iid}}{\sim}\cN(0,\sigma^2I_d)$, where $0< \sigma\le r_*/(2\sqrt d)$.
There exists $c=c(C,\eta)>0$ such that, with probability at least
$1-6e^{-cd}$, the gradient-flow solution exists on $[0,t_0]$, both student
weights remain nonzero on this interval, and
\[
    \|w(t_{\rm in})-w_{\mathrm{saddle}}\|_2<\rho
    \qquad\text{for some }t_{\rm in}\in[0,t_0].
\]
\end{corollary}
The proof is in \cref{proof:phase-one}.
The next subsection studies how gradient flow escapes this saddle and separates the student neurons.

\subsection{Phase 2: escape from the bisector saddle}

Before turning to the local escape mechanism, we establish a qualitative
global convergence result for Lebesgue-almost every initialization, with no
restriction on its scale or structure.

\begin{theorem}[Unrestricted almost-everywhere global convergence]
\label{thm:generic_global_convergence}
Fix $d\ge2$ and unit teachers with
$\alpha=\angle(v_1,v_2)\in(0,\pi/2]$. For Lebesgue-almost every
initialization $(w_1(0),w_2(0))\in(\RR^d)^2$, the population
gradient-flow trajectory \eqref{eq:population_squared_loss_and_gf} is defined
for every $t\ge0$, both student weights remain nonzero at every finite time,
and
\[
    (w_1(t),w_2(t))
    \longrightarrow
    (v_1,v_2)
    \quad\text{or}\quad
    (v_2,v_1).
\]
\end{theorem}
The proof is in \cref{proof:thm-generic-global-convergence}. Note that while
\cref{thm:generic_global_convergence} establishes global convergence, it does
not quantify the convergence rate or the escape behavior.
The following subsections characterize how the trajectory escapes the saddle under additional assumptions. 

\subsubsection{Escape directions near the bisector saddle}\label{subsubsec:without_init}
Near $w_{\mathrm{saddle}}$, Taylor expansion of gradient flow gives
\begin{align}\label{eq:taylor_expansion_near_saddle}
    \frac{\diff}{\diff t}\big(w(t)-w_{\mathrm{saddle}}\big)
    &=
    -\nabla^2L(w_{\mathrm{saddle}})\big(w(t)-w_{\mathrm{saddle}}\big)
    +o\big(\|w(t)-w_{\mathrm{saddle}}\|_2\big).
\end{align}
Thus negative Hessian directions are pushed away from the saddle, positive directions shrink, and the kernel remains flat.
Recall the landscape theorem \Cref{thm:squared_loss_landscape} identifies the positive-bisector saddle family; here we focus on $w_{\mathrm{saddle}}=r_s(\bar v,\bar v)$.
Complete the orthonormal pair $(\bar v,v_\Delta)$ from \eqref{eq:two-directions} to an orthonormal basis $\bar v,v_\Delta,u_3,\ldots,u_d$ of $\mathbb R^d$.
The Hessian directions at this saddle are as follows. 
\begin{lemma}[Hessian eigenspaces at the bisector saddle]\label{lem:hessian_cal}
    In the setting of \cref{thm:squared_loss_landscape}, the Hessian of the loss at the bisector saddle $w_{\mathrm{saddle}}$ has the following eigenspace structure:
    \begin{itemize}[leftmargin=*,itemsep=0.5mm,topsep=0.4mm]
        \item The negative eigenspace of $\nabla^2L(w_{\mathrm{saddle}})$ is
            $
                \operatorname{span}\{(v_\Delta,-v_\Delta),(u_j,-u_j):3\leq j\leq d\}.
            $
            Moreover, $(v_\Delta,-v_\Delta)$ has eigenvalue
            $
                -\frac{2\sin(\alpha/2)}{\pi r_s},
            $
            while each $(u_j,-u_j)$ has eigenvalue
            $
                -\frac{\sin(\alpha/2)}{\pi r_s}.
            $
        
        \item The positive eigenspace of $\nabla^2L(w_{\mathrm{saddle}})$ is
            $
                \operatorname{span}\{(\bar v,\bar v),(v_\Delta,v_\Delta),(u_j,u_j):3\leq j\leq d\}.
            $
        \item The kernel of $\nabla^2L(w_{\mathrm{saddle}})$ is $\operatorname{span}\{(\bar v,-\bar v)\}$.
    \end{itemize}
\end{lemma}
The proof is in \cref{subsec:bisector-saddle-hessian}.
We observe that the planar antisymmetric mode $(v_\Delta,-v_\Delta)$ is the most unstable direction: its eigenvalue is twice as negative as the out-of-plane antisymmetric eigenvalues. 
Thus, if a trajectory enters a sufficiently small saddle neighborhood with homogeneous coordinates in the Hessian eigenspace, the planar split would grow fastest along $(v_\Delta,-v_\Delta)$ and dominate the escape.
This gives the geometric intuition for the observed in-plane splitting in Figure~\ref{fig:intro}.

Write $w_-(t)$ for the projection of $w(t)-w_{\mathrm{saddle}}$ onto the negative eigenspace in the lemma, and let $\epsilon$ be the \textit{smallest value of $\|w_-(t)\|_2$ over the gradient flow trajectory}.
Near the saddle, keeping only the linear term in the Taylor expansion \eqref{eq:taylor_expansion_near_saddle} gives the approximation
\[
    \frac{\diff}{\diff t}\|w_-(t)\|_2^2
    \approx
    2\inner{w_-(t)}{-\nabla^2L(w_{\mathrm{saddle}})w_-(t)}
    \asymp
    \|w_-(t)\|_2^2 .
\]
Thus, after $\|w_-(t)\|_2$ reaches its minimum, it will grow exponentially.
This suggests that the escape time from the saddle is on the order of $\log(1/\epsilon)$ (see a rigorous version in \cref{thm:componentwise_thin_escape}).
However, this calculation does not fully describe the escape path.
In particular, it leaves two questions open:
\begin{enumerate}
    \item \textit{Which negative direction carries the escape?} The component $w_-(t)$ includes both the planar split $(v_\Delta,-v_\Delta)$ and the out-of-plane splits $(u_j,-u_j)$, $j\geq3$.
    \item \textit{How large is $\epsilon$?} Can this quantity be decided by properties of the initialization?
\end{enumerate}
The next subsection answers the first question in a deterministic way: a condition on the signs and sizes of the initial planar and out-of-plane coordinates keeps the trajectory away from the out-of-plane saddles.

\subsubsection{Exclusion of out-of-plane saddles}\label{subsubsec:homo_init}
Recall the orthonormal basis $\{\bar v,v_\Delta,u_3,\ldots,u_d\}$ extending the teacher-plane directions.
For each student weight $i\in\{1,2\}$, we decompose
\[
    w_i(t)=z_{i,1}(t)\bar v+z_{i,2}(t)v_\Delta
    +\sum_{j=3}^d z_{i,j}(t)u_j.
\]
In words, $z_{i,1}$ is the bisector coordinate, $z_{i,2}$ is the planar splitting coordinate, and $z_{i,j}$, $j\geq3$, are the out-of-plane coordinates.

\begin{assumption}\label{assumption:homogeneous_initialization}
Fix $C_\perp>0$. The planar splitting coordinates have opposite signs, $z_{1,2}(0)z_{2,2}(0)<0$, and every out-of-plane coordinate is dominated by the corresponding planar splitting coordinate: $|z_{i,j}(0)|\le C_\perp|z_{i,2}(0)|$ for every $i\in\{1,2\}$ and $3\leq j\leq d$.
\end{assumption}

Notice that only the split inside the teacher plane receives an additive force from the teacher directions. The out-of-plane coordinates have no such source: each is multiplied by a scalar and coupled to the corresponding coordinate of the other student. Hence the conditions of \cref{assumption:homogeneous_initialization} continue to hold along the flow, so the trajectory stays away from saddles with zero planar split and nonzero out-of-plane split. The theorem below makes this comparison precise. Since the loss is differentiable only where both student weights are nonzero, the solution is defined up to the first time a weight reaches zero; we write $[0,T_{\max})$ for this maximal interval, with $T_{\max}=\infty$ if no weight reaches zero. 
By \cref{thm:generic_global_convergence}, $T_{\max}=\infty$ for Lebesgue-almost every initialization.

\begin{theorem}[Separation from out-of-plane saddles]\label{thm:homogeneous_excludes_out_of_plane}
    Under \cref{assumption:homogeneous_initialization}, $|z_{1,j}(t)-z_{2,j}(t)|\leq C_\perp|z_{1,2}(t)-z_{2,2}(t)|$ for every $j\geq3$ and $0\leq t<T_{\max}$.
    Let $\mathcal S_{\perp}$ denote the out-of-plane saddle family in \cref{thm:squared_loss_landscape}.
    Then there exists a constant $c>0$ (depending on $\alpha$, $C_\perp$, and $d$) such that $\|w(t)-\widehat w\|_2\geq c$ for every $\widehat w\in\mathcal S_{\perp}$ and $0\leq t<T_{\max}$.
\end{theorem}

\begin{remark}
The theorem first gives the comparison estimate
\[
    |z_{1,j}(t)-z_{2,j}(t)|
    \leq C_\perp|z_{1,2}(t)-z_{2,2}(t)|,
    \qquad j\geq3,\ 0\leq t<T_{\max}.
\]
This estimate keeps the trajectory away from the out-of-plane saddle family,
whose saddles have zero planar split and nonzero out-of-plane split. The
opposite signs of the planar splitting coordinates also rule out convergence
to the in-plane saddle families; see \cref{lem:signed-in-plane-saddle-exclusion}. Combining this deterministic saddle exclusion with the
zero-weight avoidance and trajectory-convergence results used in
\cref{thm:generic_global_convergence} gives global convergence for
Lebesgue-almost every initialization satisfying
\cref{assumption:homogeneous_initialization}, without invoking the
stable-manifold saddle-avoidance step; see
\cref{cor:signed-transverse-global-convergence}.

For independent Gaussian initialization $w_i(0)\sim\cN(0,\sigma^2I_d)$,
$i=1,2$, the coordinates $z_{1,2}(0)$ and $z_{2,2}(0)$ are independent
centered Gaussians, so the sign condition of
\cref{assumption:homogeneous_initialization} holds with probability $1/2$.
Since $z_{1,2}(0)$ and $z_{2,2}(0)$ are almost surely nonzero, the
domination condition holds almost surely with
$C_\perp=\max\{|z_{i,j}(0)|/|z_{i,2}(0)|:i=1,2,\ 3\le j\le d\}$ when
$d\ge3$; when $d=2$ there are no out-of-plane coordinates and any
$C_\perp>0$ works. Thus the theorem applies with
probability $1/2$, with a constant $C_\perp$ that depends on the
initialization.
\end{remark}

\subsubsection{Logarithmic escape under structured initialization}

Phase 1 explains how small initializations reach a neighborhood of the
bisector saddle. The results above then describe the subsequent dynamics
qualitatively: almost every trajectory converges to a global minimum, while
under \cref{assumption:homogeneous_initialization} the out-of-plane coordinates remain
controlled by the planar splitting coordinate, ruling out convergence to
out-of-plane saddles. Neither result, however, quantifies how long symmetry
breaking takes. We now study this local time scale under a structured
initialization condition, which holds with constant probability at small random initialization.

We use the coordinates \(z_{i,j}\) of \cref{subsubsec:homo_init} and write
$\mu_\alpha\coloneqq2\sin(\alpha/2)/(\pi r_s)$ for the planar escape rate
from \cref{lem:hessian_cal}. For $\varepsilon>0$, let
\[
    T_\varepsilon
    \coloneqq\inf\bigg\{
      t\ge0:\sum_{i=1}^2\sum_{j=2}^d z_{i,j}(t)^2
      \ge\varepsilon^2
    \bigg\}
\]
be the first time the total squared distance of the student weights from
the bisector axis reaches $\varepsilon^2$. The next theorem shows that
$T_\varepsilon$ is governed by the initial planar splitting and, up to a
bounded error, depends logarithmically on $a_{\rm init}$ and $r_{\rm init}$.

\begin{theorem}[Logarithmic escape governed by the initial planar splitting]
\label{thm:componentwise_thin_escape}
Fix $\alpha\in(0,\pi/2]$ and an absolute constant $C>1$. For every
sufficiently small $\varepsilon>0$, there are constants $c\in(0,1)$ and
$K<\infty$, depending only on $\alpha,C,\varepsilon$, such that the
following holds for every $d\ge2$.
Consider the gradient flow in \eqref{eq:population_squared_loss_and_gf} from
a $C$-balanced initialization (see \eqref{eq:c-balanced}) satisfying
\[
    r_{\rm init}\sqrt d\le c,
    \qquad
    r_{\rm init}\sqrt d
    \left(
      1+\log\frac{\sqrt d}{r_{\rm init}}+\log d
    \right)
    \le c.
\]
Assume also that
\[
    a_{\rm init}\ge\frac{r_{\rm init}}{C\sqrt d},
    \qquad
    z_{1,1}(0),z_{2,1}(0)>0,
    \qquad
    z_{1,2}(0)z_{2,2}(0)<0.
\]
Then $T_\varepsilon$ is finite and satisfies
\[
    \left|
      T_\varepsilon-
      \frac1{\mu_\alpha}
      \log\frac{r_{\rm init}^{\mu_\alpha}}{a_{\rm init}}
    \right|
    \le K.
\]
\end{theorem}

The proof is in \cref{subsubsec:structured-escape-local-passage}.

\begin{remark}[Dimension dependence and i.i.d.~initialization]
When the initial planar splitting is of the size of a single coordinate,
that is, $a_{\rm init}\asymp r_{\rm init}/\sqrt d$, the theorem gives
\[
    T_\varepsilon
    =\left(\mu_\alpha^{-1}-1\right)\log\frac1{r_{\rm init}}
      +\frac1{2\mu_\alpha}\log d+O(1).
\]
For example, if $r_{\rm init}=d^{-\gamma}$ for a fixed $\gamma>1/2$, then
the two upper bounds on $r_{\rm init}$ in the theorem hold for all
sufficiently large $d$, and $T_\varepsilon=\Theta(\log d)$.

Independent Gaussian initialization is a natural example. If
$w_i(0)\sim\cN(0,\sigma^2I_d)$ independently for $i=1,2$, then with
constant probability the initialization satisfies the other conditions of
\cref{thm:componentwise_thin_escape}, namely $C$-balance,
$a_{\rm init}\ge r_{\rm init}/(C\sqrt d)$ and the sign conditions, and
\[
    r_{\rm init}\asymp\sigma\sqrt d,
    \qquad
    a_{\rm init}\asymp\sigma
      \asymp\frac{r_{\rm init}}{\sqrt d}.
\]
On this event, for $0<\sigma\le1$, the two upper bounds on $r_{\rm init}$
hold once
$
    \sigma d\bigl(1+\log(1/\sigma)+\log d\bigr)
$
is sufficiently small, and the theorem gives
\[
    T_\varepsilon
    =\left(\mu_\alpha^{-1}-1\right)\log\frac1\sigma
      +\frac12\log d+O_{\alpha,C,\varepsilon}(1).
\]
If $\sigma=d^{-\beta}$ for a fixed $\beta>1$, then these
bounds hold for all sufficiently large $d$, and
$T_\varepsilon=\Theta(\log d)$.
\end{remark}

\paragraph{Empirical validation.}
The calculation above shows that a planar splitting of the size of a single
coordinate contributes an explicit logarithmic term in $d$; for Gaussian
initialization at fixed $\sigma$, the formula suggests a slope of $1/2$ with
respect to $\log d$.
The random-initialization experiments in
\cref{fig:escape_and_global_per_alpha} support this prediction: across the
tested values of $\alpha$, $T_\varepsilon$ grows approximately linearly with
$\log d$, while
$T_{\mathrm{global}}-T_\varepsilon$ remains nearly constant.
This suggests a two-stage picture in which the ambient dimension primarily
affects the time needed to break the bisector symmetry, whereas the subsequent
descent to a global minimizer occurs on an essentially dimension-independent
time scale.

\begin{figure}[!htb]
    \centering
    \includegraphics[width=\linewidth]{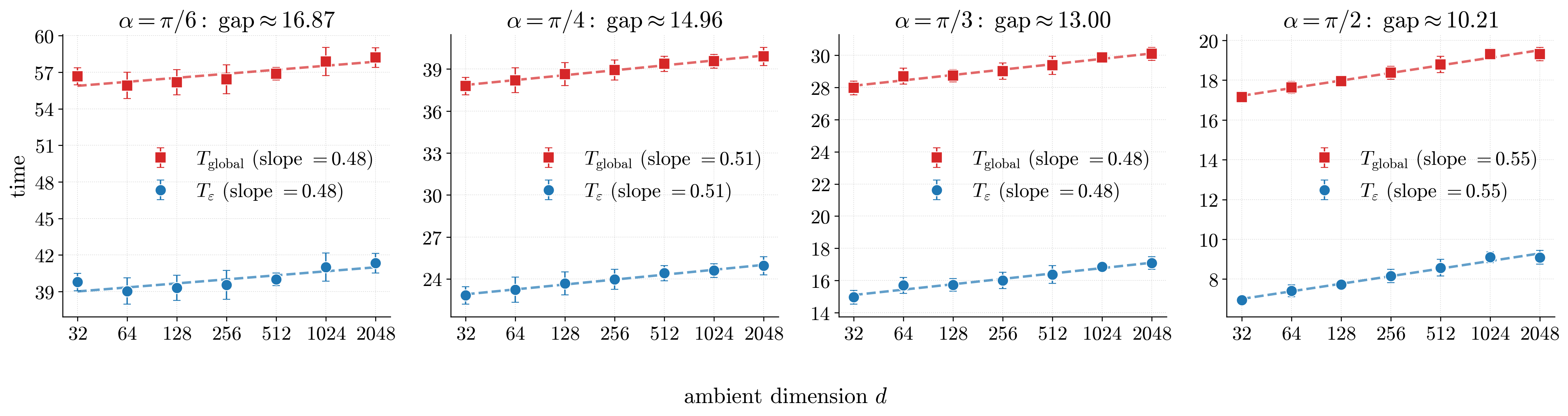}
    \vspace{-5mm}
    \caption{\small Escape time \(T_\varepsilon\) and global-convergence time
    \(T_{\mathrm{global}}\) vs.\ ambient dimension \(d\), one panel per
    \(\alpha\), for random initialization at fixed \(\sigma=10^{-4}\). Each
    point is the mean over \(5\) seeds, and error bars are \(\pm 1\) s.d.
    The time \(T_\varepsilon\) is when the total squared distance of the two
    student weights from the bisector axis,
    \(\sum_{i=1}^2\sum_{j=2}^d z_{i,j}(t)^2\), first reaches
    \(\varepsilon^2\) (\(\varepsilon=0.1\)); \(T_{\mathrm{global}}\) is when
    the student state first reaches distance \(0.05\) from the global minimizer
    set. Dashed lines are log-\(d\) fits; fitted slopes \(\in[0.48,0.55]\)
    match the \(\tfrac12\log d\) prediction above, and the gap
    \(T_{\mathrm{global}}-T_\varepsilon\) shown in the panel titles is almost
    $d$-independent.
    }
    \label{fig:escape_and_global_per_alpha}
    \vspace{-1.8mm}
\end{figure}

\section{Conclusion}

We studied population gradient flow for a two-neuron ReLU teacher-student model, a minimal multi-index setting that already shows staged high-dimensional feature learning. The loss has no spurious local minima with both student weights nonzero, but it contains a bisector saddle that organizes training. From small initialization, gradient flow first learns the easy bisector direction, then escapes along the harder splitting direction. The logarithmic escape time identifies where dimension enters: in the leap-2 symmetry-breaking phase, rather than in the final descent after specialization begins.
Several directions remain open, including: $(i)$ obtaining an end-to-end escape-time analysis from random initialization with high probability; $(ii)$ extending the landscape and dynamical characterization to other teacher-student settings and high-leap multi-index functions.

\bigskip

\subsection*{Acknowledgements}

The authors thank Margalit Glasgow, Jason D.~Lee, and Theodor Misiakiewicz for discussions and feedback. ChatGPT and Claude were used to polish the writing, to check and reorganize proofs, and to produce figures from the numerical experiments.

\bigskip

\bibliography{reference}

\begin{thebibliography}{48}
\expandafter\ifx\csname natexlab\endcsname\relax\def\natexlab#1{#1}\fi
\expandafter\ifx\csname url\endcsname\relax
  \def\url#1{\texttt{#1}}\fi
\expandafter\ifx\csname urlprefix\endcsname\relax\def\urlprefix{URL }\fi

\bibitem[{Abbe et~al.(2022)Abbe, Boix-Adser{\`a} and
  Misiakiewicz}]{abbe2022merged}
\textsc{Abbe, E.}, \textsc{Boix-Adser{\`a}, E.} and \textsc{Misiakiewicz, T.}
  (2022).
\newblock The merged-staircase property: a necessary and nearly sufficient
  condition for {SGD} learning of sparse functions on two-layer neural
  networks.
\newblock In \textit{Conference on Learning Theory}. PMLR.

\bibitem[{Abbe et~al.(2023)Abbe, Boix-Adser{\`a} and
  Misiakiewicz}]{abbe2023sgd}
\textsc{Abbe, E.}, \textsc{Boix-Adser{\`a}, E.} and \textsc{Misiakiewicz, T.}
  (2023).
\newblock {SGD} learning on neural networks: leap complexity and
  saddle-to-saddle dynamics.
\newblock In \textit{The Thirty Sixth Annual Conference on Learning Theory}.
  PMLR.

\bibitem[{Akiyama and Suzuki(2021)}]{akiyama2021learnability}
\textsc{Akiyama, S.} and \textsc{Suzuki, T.} (2021).
\newblock On learnability via gradient method for two-layer {ReLU} neural
  networks in teacher-student setting.
\newblock In \textit{International Conference on Machine Learning}.

\bibitem[{Ba et~al.(2022)Ba, Erdogdu, Suzuki, Wang, Wu and Yang}]{ba2022high}
\textsc{Ba, J.}, \textsc{Erdogdu, M.~A.}, \textsc{Suzuki, T.}, \textsc{Wang,
  Z.}, \textsc{Wu, D.} and \textsc{Yang, G.} (2022).
\newblock High-dimensional asymptotics of feature learning: How one gradient
  step improves the representation.
\newblock \textit{Advances in Neural Information Processing Systems}
  \textbf{35} 37932--37946.

\bibitem[{Bah et~al.(2022)Bah, Rauhut, Terstiege and
  Westdickenberg}]{bah2022learning}
\textsc{Bah, B.}, \textsc{Rauhut, H.}, \textsc{Terstiege, U.} and
  \textsc{Westdickenberg, M.} (2022).
\newblock Learning deep linear neural networks: Riemannian gradient flows and
  convergence to global minimizers.
\newblock \textit{Information and Inference: A Journal of the IMA} \textbf{11}
  307--353.

\bibitem[{Ben~Arous et~al.(2025)Ben~Arous, Erdogdu, Vural and
  Wu}]{arous2025learning}
\textsc{Ben~Arous, G.}, \textsc{Erdogdu, M.~A.}, \textsc{Vural, N.~M.} and
  \textsc{Wu, D.} (2025).
\newblock Learning quadratic neural networks in high dimensions: {SGD} dynamics
  and scaling laws.
\newblock \textit{Advances in Neural Information Processing Systems}
  \textbf{38} 162874--162961.

\bibitem[{Ben~Arous et~al.(2026)Ben~Arous, Gerbelot and
  Piccolo}]{arous2024stochastic}
\textsc{Ben~Arous, G.}, \textsc{Gerbelot, C.} and \textsc{Piccolo, V.} (2026).
\newblock Stochastic gradient descent in high dimensions for multi-spiked
  tensor {PCA}.
\newblock \textit{Communications on Pure and Applied Mathematics} \textbf{79}
  2291--2369.

\bibitem[{Ben~Arous et~al.(2021)Ben~Arous, Gheissari and
  Jagannath}]{arous2021online}
\textsc{Ben~Arous, G.}, \textsc{Gheissari, R.} and \textsc{Jagannath, A.}
  (2021).
\newblock Online stochastic gradient descent on non-convex losses from
  high-dimensional inference.
\newblock \textit{Journal of Machine Learning Research} \textbf{22} 1--51.

\bibitem[{Biehl et~al.(1996)Biehl, Riegler and W{\"o}hler}]{biehl1996transient}
\textsc{Biehl, M.}, \textsc{Riegler, P.} and \textsc{W{\"o}hler, C.} (1996).
\newblock Transient dynamics of on-line learning in two-layered neural
  networks.
\newblock \textit{Journal of Physics A: Mathematical and General} \textbf{29}
  4769--4780.

\bibitem[{Bietti et~al.(2025)Bietti, Bruna and
  Pillaud-Vivien}]{bietti2023learning}
\textsc{Bietti, A.}, \textsc{Bruna, J.} and \textsc{Pillaud-Vivien, L.} (2025).
\newblock On learning {Gaussian} multi-index models with gradient flow part
  {I}: General properties and two-timescale learning.
\newblock \textit{Communications on Pure and Applied Mathematics} \textbf{78}
  2354--2435.

\bibitem[{Bietti et~al.(2022)Bietti, Bruna, Sanford and
  Song}]{bietti2022learning}
\textsc{Bietti, A.}, \textsc{Bruna, J.}, \textsc{Sanford, C.} and \textsc{Song,
  M.~J.} (2022).
\newblock Learning single-index models with shallow neural networks.
\newblock \textit{Advances in neural information processing systems}
  \textbf{35} 9768--9783.

\bibitem[{Bolte et~al.(2007)Bolte, Daniilidis and Lewis}]{bolte2007lojasiewicz}
\textsc{Bolte, J.}, \textsc{Daniilidis, A.} and \textsc{Lewis, A.} (2007).
\newblock The {{\L}ojasiewicz} inequality for nonsmooth subanalytic functions
  with applications to subgradient dynamical systems.
\newblock \textit{SIAM Journal on Optimization} \textbf{17} 1205--1223.

\bibitem[{Boursier et~al.(2022)Boursier, Pillaud-Vivien and
  Flammarion}]{boursier2022gradient}
\textsc{Boursier, E.}, \textsc{Pillaud-Vivien, L.} and \textsc{Flammarion, N.}
  (2022).
\newblock Gradient flow dynamics of shallow {ReLU} networks for square loss and
  orthogonal inputs.
\newblock \textit{Advances in Neural Information Processing Systems}
  \textbf{35} 20105--20118.

\bibitem[{Chizat(2022)}]{chizat2022sparse}
\textsc{Chizat, L.} (2022).
\newblock Sparse optimization on measures with over-parameterized gradient
  descent.
\newblock \textit{Mathematical Programming} \textbf{194} 487--532.

\bibitem[{Chizat et~al.(2019)Chizat, Oyallon and Bach}]{chizat2019lazy}
\textsc{Chizat, L.}, \textsc{Oyallon, E.} and \textsc{Bach, F.} (2019).
\newblock On lazy training in differentiable programming.
\newblock \textit{Advances in neural information processing systems}
  \textbf{32}.

\bibitem[{Damian et~al.(2022)Damian, Lee and Soltanolkotabi}]{damian2022neural}
\textsc{Damian, A.}, \textsc{Lee, J.} and \textsc{Soltanolkotabi, M.} (2022).
\newblock Neural networks can learn representations with gradient descent.
\newblock In \textit{Conference on Learning Theory}. PMLR.

\bibitem[{Damian et~al.(2023)Damian, Nichani, Ge and Lee}]{damian2023smoothing}
\textsc{Damian, A.}, \textsc{Nichani, E.}, \textsc{Ge, R.} and \textsc{Lee,
  J.~D.} (2023).
\newblock Smoothing the landscape boosts the signal for {SGD}: Optimal sample
  complexity for learning single index models.
\newblock \textit{Advances in Neural Information Processing Systems}
  \textbf{36} 752--784.

\bibitem[{Dandi et~al.(2024)Dandi, Krzakala, Loureiro, Pesce and
  Stephan}]{dandi2024two}
\textsc{Dandi, Y.}, \textsc{Krzakala, F.}, \textsc{Loureiro, B.},
  \textsc{Pesce, L.} and \textsc{Stephan, L.} (2024).
\newblock How two-layer neural networks learn, one (giant) step at a time.
\newblock \textit{Journal of Machine Learning Research} \textbf{25} 1--65.

\bibitem[{Dudeja and Hsu(2018)}]{dudeja2018learning}
\textsc{Dudeja, R.} and \textsc{Hsu, D.} (2018).
\newblock Learning single-index models in {Gaussian} space.
\newblock In \textit{Conference On Learning Theory}. PMLR.

\bibitem[{Fukumizu and Amari(2000)}]{fukumizu2000local}
\textsc{Fukumizu, K.} and \textsc{Amari, S.-i.} (2000).
\newblock Local minima and plateaus in hierarchical structures of multilayer
  perceptrons.
\newblock \textit{Neural networks} \textbf{13} 317--327.

\bibitem[{Ge et~al.(2021)Ge, Ren, Wang and Zhou}]{ge2021understanding}
\textsc{Ge, R.}, \textsc{Ren, Y.}, \textsc{Wang, X.} and \textsc{Zhou, M.}
  (2021).
\newblock Understanding deflation process in over-parametrized tensor
  decomposition.
\newblock \textit{Advances in Neural Information Processing Systems}
  \textbf{34} 1299--1311.

\bibitem[{Glasgow(2024)}]{glasgow2023sgd}
\textsc{Glasgow, M.} (2024).
\newblock {SGD} finds then tunes features in two-layer neural networks with
  near-optimal sample complexity: A case study in the {XOR} problem.
\newblock In \textit{The Twelfth International Conference on Learning
  Representations}.

\bibitem[{Goldt et~al.(2020)Goldt, Advani, Saxe, Krzakala and
  Zdeborov{\'a}}]{goldt2019dynamics}
\textsc{Goldt, S.}, \textsc{Advani, M.~S.}, \textsc{Saxe, A.~M.},
  \textsc{Krzakala, F.} and \textsc{Zdeborov{\'a}, L.} (2020).
\newblock Dynamics of stochastic gradient descent for two-layer neural networks
  in the teacher-student setup.
\newblock \textit{Journal of Statistical Mechanics: Theory and Experiment}
  124010.

\bibitem[{Hirsch et~al.(1977)Hirsch, Pugh and Shub}]{hirsch1977invariant}
\textsc{Hirsch, M.~W.}, \textsc{Pugh, C.~C.} and \textsc{Shub, M.} (1977).
\newblock \textit{Invariant Manifolds}, vol. 583 of \textit{Lecture Notes in
  Mathematics}.
\newblock Springer.

\bibitem[{Huang et~al.(2026)Huang, Loureiro and {Sarao
  Mannelli}}]{huang2026sharp}
\textsc{Huang, J.}, \textsc{Loureiro, B.} and \textsc{{Sarao Mannelli}, S.}
  (2026).
\newblock Sharp description of local minima in the loss landscape of
  high-dimensional two-layer {ReLU} neural networks.
\newblock In \textit{Proceedings of the 43rd International Conference on
  Machine Learning}.

\bibitem[{Lee et~al.(2024)Lee, Oko, Suzuki and Wu}]{lee2024neural}
\textsc{Lee, J.~D.}, \textsc{Oko, K.}, \textsc{Suzuki, T.} and \textsc{Wu, D.}
  (2024).
\newblock Neural network learns low-dimensional polynomials with {SGD} near the
  information-theoretic limit.
\newblock \textit{Advances in Neural Information Processing Systems}
  \textbf{37} 58716--58756.

\bibitem[{Li et~al.(2020)Li, Ma and Zhang}]{li2020learning}
\textsc{Li, Y.}, \textsc{Ma, T.} and \textsc{Zhang, H.~R.} (2020).
\newblock Learning over-parametrized two-layer neural networks beyond {NTK}.
\newblock In \textit{Conference on learning theory}. PMLR.

\bibitem[{Martin et~al.(2024)Martin, Bach and Biroli}]{martin2024impact}
\textsc{Martin, S.}, \textsc{Bach, F.} and \textsc{Biroli, G.} (2024).
\newblock On the impact of overparameterization on the training of a shallow
  neural network in high dimensions.
\newblock In \textit{International Conference on Artificial Intelligence and
  Statistics}. PMLR.

\bibitem[{Min et~al.(2025)Min, Zhu and Vidal}]{min2025neural}
\textsc{Min, H.}, \textsc{Zhu, Z.} and \textsc{Vidal, R.} (2025).
\newblock Neural collapse under gradient flow on shallow {ReLU} networks for
  orthogonally separable data.
\newblock \textit{Advances in Neural Information Processing Systems}
  \textbf{38} 105551--105590.

\bibitem[{Oko et~al.(2024)Oko, Song, Suzuki and Wu}]{oko2024learning}
\textsc{Oko, K.}, \textsc{Song, Y.}, \textsc{Suzuki, T.} and \textsc{Wu, D.}
  (2024).
\newblock Learning sum of diverse features: computational hardness and
  efficient gradient-based training for ridge combinations.
\newblock In \textit{The Thirty Seventh Annual Conference on Learning Theory}.
  PMLR.

\bibitem[{Oostwal et~al.(2021)Oostwal, Straat and Biehl}]{oostwal2021hidden}
\textsc{Oostwal, E.}, \textsc{Straat, M.} and \textsc{Biehl, M.} (2021).
\newblock Hidden unit specialization in layered neural networks: {ReLU} vs.
  sigmoidal activation.
\newblock \textit{Physica A: Statistical Mechanics and its Applications}
  \textbf{564} 125517.

\bibitem[{Ren and Lee(2025)}]{ren2024learning}
\textsc{Ren, Y.} and \textsc{Lee, J.~D.} (2025).
\newblock Learning orthogonal multi-index models: A fine-grained information
  exponent analysis.
\newblock \textit{Advances in Neural Information Processing Systems}
  \textbf{38} 183843--183894.

\bibitem[{Ren et~al.(2025)Ren, Nichani, Wu and Lee}]{ren2025emergence}
\textsc{Ren, Y.}, \textsc{Nichani, E.}, \textsc{Wu, D.} and \textsc{Lee, J.~D.}
  (2025).
\newblock Emergence and scaling laws in {SGD} learning of shallow neural
  networks.
\newblock \textit{Advances in Neural Information Processing Systems}
  \textbf{38} 42829--42911.

\bibitem[{Riegler and Biehl(1995)}]{riegler1995line}
\textsc{Riegler, P.} and \textsc{Biehl, M.} (1995).
\newblock On-line backpropagation in two-layered neural networks.
\newblock \textit{Journal of Physics A: Mathematical and General} \textbf{28}
  L507--L513.

\bibitem[{Saad and Solla(1995)}]{saad1995line}
\textsc{Saad, D.} and \textsc{Solla, S.~A.} (1995).
\newblock On-line learning in soft committee machines.
\newblock \textit{Physical Review E} \textbf{52} 4225--4243.

\bibitem[{Safran and Shamir(2018)}]{safran2018spurious}
\textsc{Safran, I.} and \textsc{Shamir, O.} (2018).
\newblock Spurious local minima are common in two-layer {ReLU} neural networks.
\newblock In \textit{International conference on machine learning}. PMLR.

\bibitem[{Safran et~al.(2021)Safran, Yehudai and Shamir}]{safran2021effects}
\textsc{Safran, I.~M.}, \textsc{Yehudai, G.} and \textsc{Shamir, O.} (2021).
\newblock The effects of mild over-parameterization on the optimization
  landscape of shallow {ReLU} neural networks.
\newblock In \textit{Conference on Learning Theory}. PMLR.

\bibitem[{{\c{S}}im{\c{s}}ek et~al.(2025){\c{S}}im{\c{s}}ek, Bendjeddou and
  Hsu}]{csimcsek2024learning}
\textsc{{\c{S}}im{\c{s}}ek, B.}, \textsc{Bendjeddou, A.} and \textsc{Hsu, D.}
  (2025).
\newblock Learning {Gaussian} multi-index models with gradient flow: Time
  complexity and directional convergence.
\newblock In \textit{Proceedings of The 28th International Conference on
  Artificial Intelligence and Statistics}, vol. 258 of \textit{Proceedings of
  Machine Learning Research}. PMLR.

\bibitem[{Soltanolkotabi(2017)}]{soltanolkotabi2017learning}
\textsc{Soltanolkotabi, M.} (2017).
\newblock Learning {ReLU}s via gradient descent.
\newblock \textit{Advances in neural information processing systems}
  \textbf{30}.

\bibitem[{Sumner et~al.(1991)Sumner, Jagers, Vowe and
  Anglesio}]{sumner1991inequalities}
\textsc{Sumner, J.~S.}, \textsc{Jagers, A.~A.}, \textsc{Vowe, M.} and
  \textsc{Anglesio, J.} (1991).
\newblock Inequalities involving trigonometric functions.
\newblock \textit{The American Mathematical Monthly} \textbf{98} 264--267.

\bibitem[{Tian(2017)}]{tian2017analytical}
\textsc{Tian, Y.} (2017).
\newblock An analytical formula of population gradient for two-layered {ReLU}
  network and its applications in convergence and critical point analysis.
\newblock In \textit{International Conference on Machine Learning}.

\bibitem[{Wang and Ma(2023)}]{wang2023understanding}
\textsc{Wang, M.} and \textsc{Ma, C.} (2023).
\newblock Understanding multi-phase optimization dynamics and rich nonlinear
  behaviors of {ReLU} networks.
\newblock \textit{Advances in Neural Information Processing Systems}
  \textbf{36} 35654--35747.

\bibitem[{Wu et~al.(2018)Wu, Luo and Lee}]{wu2018no}
\textsc{Wu, C.}, \textsc{Luo, J.} and \textsc{Lee, J.~D.} (2018).
\newblock No spurious local minima in a two hidden unit {ReLU} network.
\newblock In \textit{International Conference on Learning Representations,
  Workshop Track}.

\bibitem[{Xu and Du(2023)}]{xu2023over}
\textsc{Xu, W.} and \textsc{Du, S.} (2023).
\newblock Over-parameterization exponentially slows down gradient descent for
  learning a single neuron.
\newblock In \textit{The Thirty Sixth Annual Conference on Learning Theory}.
  PMLR.

\bibitem[{Yehudai and Shamir(2020)}]{yehudai2020learning}
\textsc{Yehudai, G.} and \textsc{Shamir, O.} (2020).
\newblock Learning a single neuron with gradient methods.
\newblock In \textit{Conference on Learning Theory}. PMLR.

\bibitem[{Zhang et~al.(2026)Zhang, Wang, Fu and Lee}]{zhang2025neural}
\textsc{Zhang, B.}, \textsc{Wang, Z.}, \textsc{Fu, H.} and \textsc{Lee, J.~D.}
  (2026).
\newblock Neural networks learn generic multi-index models near
  information-theoretic limit.
\newblock In \textit{The Fourteenth International Conference on Learning
  Representations}.

\bibitem[{Zhou et~al.(2021)Zhou, Ge and Jin}]{zhou2021local}
\textsc{Zhou, M.}, \textsc{Ge, R.} and \textsc{Jin, C.} (2021).
\newblock A local convergence theory for mildly over-parameterized two-layer
  neural network.
\newblock In \textit{Conference on Learning Theory}. PMLR.

\bibitem[{Zhu et~al.(2025)Zhu, Liu and Cevher}]{zhu2025gradient}
\textsc{Zhu, Z.}, \textsc{Liu, F.} and \textsc{Cevher, V.} (2025).
\newblock How gradient descent balances features: A dynamical analysis for
  two-layer neural networks.
\newblock In \textit{The Thirteenth International Conference on Learning
  Representations}.

\end{thebibliography}
\bibliographystyle{ims}

\newpage
\allowdisplaybreaks
\appendix
\crefalias{section}{appendix}
\crefalias{subsection}{subappendix}
\crefalias{subsubsection}{subsubappendix}
\tableofcontents
\newpage

\section{Technical lemmas}\label{app:technical-lemmas}

This appendix collects the lemmas used in the landscape and training-dynamics proofs. Throughout this appendix, unless stated otherwise, $v_1,v_2\in\RR^d$ are unit teacher directions with $\alpha=\angle(v_1,v_2)\in(0,\pi/2]$, and the angle notation $\theta_{ij}$ and $\gamma$ follows the convention in \eqref{eq:angle_shorthand}. Whenever directions or angles of student weights are used, the corresponding weights are assumed to be nonzero.

\subsection{Gaussian ReLU kernel}\label{subsec:relu-kernel}

We first recall the Gaussian ReLU kernel identity from \citet{safran2018spurious}. For nonzero $u,v\in\RR^d$, let $\theta_{u,v}\coloneqq\angle(u,v)$. Then
\begin{subequations}
\begin{align}
    \EE_{x\sim\cN(0,I_d)} \big[\phi(\inner{u}{x})\phi(\inner{v}{x})\big] &= \|u\|_2\|v\|_2\,l(\theta_{u,v}),
    \label{eq:relu-gaussian-kernel}\\
    l(\theta) &\coloneqq \frac{(\pi-\theta)\cos\theta+\sin\theta}{2\pi}, \qquad \theta\in[0,\pi].
    \label{eq:relu-kernel-profile}
\end{align}
\end{subequations}

\begin{lemma}[Regularity of $L$ and continuation of gradient flow]\label{lem:relu-kernel-regularity}
The Gaussian ReLU kernel in \eqref{eq:relu-gaussian-kernel} is jointly $C^2$ in its two nonzero vector arguments, including at aligned and antipodal directions. Consequently, the population loss $L$ is $C^2$ on the open set where both student weights are nonzero. On this set, gradient flow has a unique local solution, and any solution that stays in a compact subset of this set on $[0,T)$, with $T<\infty$, extends beyond $T$.
\end{lemma}

\begin{proof}[Proof of \cref{lem:relu-kernel-regularity}]
Write
\[
    \mathcal U=(\RR^d\setminus\{0\})^2, \qquad K(u,v)=\|u\|_2\|v\|_2\,l(\angle(u,v)).
\]
For $(u,v)\in\mathcal U$, write
\[
    c=\inner{\bar u}{\bar v},\quad z=\bar u-c\bar v,\quad \tau=\|z\|_2=\sqrt{1-c^2}.
\]
When $|c|<1$, $K(u,v)=\|u\|_2\|v\|_2\,l(\arccos c)$ is smooth. Near $c=\pm1$, the sign of $c$ is fixed and
\[
\begin{aligned}
    \theta=\angle(u,v) &=\begin{cases}
        \arcsin\tau,&c>0,\\
        \pi-\arcsin\tau,&c<0,
      \end{cases}\\[1mm]
    l(\theta) &=\begin{cases}
        c/2+F(\tau),&c>0,\\
        F(\tau),&c<0,
      \end{cases}
    \qquad F(\tau)=\frac{\tau-\sqrt{1-\tau^2}\arcsin\tau}{2\pi}.
\end{aligned}
\]
The function $F$ is odd and analytic near zero, with
\[
\begin{aligned}
    \arcsin\tau&=\tau+\frac{\tau^3}{6} +\frac{3\tau^5}{40}+O(\tau^7),\\
    \sqrt{1-\tau^2}&=1-\frac{\tau^2}{2} -\frac{\tau^4}{8}+O(\tau^6),\\
    F(\tau)&=\frac{\tau^3}{6\pi}+\frac{\tau^5}{15\pi} +O(\tau^7) =\tau^3\psi(\tau^2),
\end{aligned}
\]
where $\psi$ is analytic near zero and $\psi(0)=1/(6\pi)$.

For $H(z)=\|z\|_2^3$ and $z\ne0$,
\[
\begin{aligned}
    \nabla H(z)&=3\|z\|_2z,\\
    \nabla^2H(z)&=3\|z\|_2I_d+\frac{3zz^\top}{\|z\|_2}, &\|\nabla^2H(z)\|_{\rm op}&\le6\|z\|_2.
\end{aligned}
\]
At zero,
\[
    \frac{|H(z)-H(0)|}{\|z\|_2}=\|z\|_2^2\longrightarrow0, \qquad \frac{\|\nabla H(z)\|_2}{\|z\|_2}=3\|z\|_2\longrightarrow0.
\]
Thus $\nabla H(0)=0$ and $\nabla^2H(0)=0$, and $\nabla^2H(z)\to\nabla^2H(0)$ as $z\to0$. Hence $H\in C^2(\RR^d)$, and $F(\|z\|_2)=H(z)\psi(\|z\|_2^2)$ is $C^2$ near $0$. Since $c$, $z$, $\|u\|_2$, and $\|v\|_2$ depend smoothly on $(u,v)\in\mathcal U$, the two endpoint formulas give $K\in C^2(\mathcal U)$. Since
\[
    L(w_1,w_2) =\frac12\sum_{i,j=1}^2K(w_i,w_j) -\sum_{i,j=1}^2K(w_i,v_j) +\frac12\sum_{i,j=1}^2K(v_i,v_j),
\]
it follows that $L\in C^2(\mathcal U)$ and $-\nabla L\in C^1(\mathcal U)$.
Local existence and uniqueness follow from the Picard--Lindel\"of theorem. For continuation, suppose that $w([0,T))\subset\mathcal Q$ for some $T<\infty$ and some compact set $\mathcal Q\subset\mathcal U$, and let $B_{\mathcal Q}=\max_{w\in\mathcal Q}\|\nabla L(w)\|_2<\infty$. For $0\le s<t<T$,
\[
    \|w(t)-w(s)\|_2 \le\int_s^t\|\nabla L(w(a))\|_2\,\diff a \le B_{\mathcal Q}(t-s).
\]
Therefore the limit $w_T=\lim_{t\uparrow T}w(t)$ exists and lies in $\mathcal Q\subset\mathcal U$. The local solution starting at $w_T$ extends $w$ beyond $T$.
\end{proof}

\begin{lemma}[Monotonicity and bounds for $l$]\label{lem:l-monotone-bounds}
The function $l$ defined in \eqref{eq:relu-kernel-profile} is strictly decreasing on $[0,\pi]$ and satisfies $0\leq l(\theta)\leq 1/2$ for every $\theta\in[0,\pi]$.
\end{lemma}

\begin{proof}[Proof of \cref{lem:l-monotone-bounds}]
Differentiating gives
\begin{equation}\label{eq:l-first-derivative}
    l'(\theta) = -\frac{(\pi-\theta)\sin\theta}{2\pi},
\end{equation}
which is negative for $\theta\in(0,\pi)$. Hence $l$ is strictly decreasing on $[0,\pi]$. The endpoint values $l(0)=1/2$ and $l(\pi)=0$ give the claimed bounds.
\end{proof}

\begin{lemma}[Bounds on $l(\theta_1)+l(\theta_2)$]\label{lem:l-two-teacher-angle-bounds}
Let $u\in\RR^d$ be nonzero and write $\theta_j\coloneqq\angle(u,v_j)$ for $j=1,2$. Then
\[
    2l\left(\pi-\frac{\alpha}{2}\right) \leq l(\theta_1)+l(\theta_2) \leq 2l\left(\frac{\alpha}{2}\right).
\]
Moreover, equality in the upper bound holds if and only if $\theta_1=\theta_2=\alpha/2$.
\end{lemma}

\begin{proof}[Proof of \cref{lem:l-two-teacher-angle-bounds}]
We begin with the lower bound. Set
\[
    \vartheta_j\coloneqq \pi-\theta_j=\angle(-\bar u,v_j), \qquad j=1,2.
\]
The triangle inequality for angles gives
\[
    \vartheta_1+\vartheta_2\geq \alpha.
\]
Write
\[
    q(x)\coloneqq \sin x-x\cos x, \qquad l(\pi-x)=\frac{q(x)}{2\pi}, \qquad x\in[0,\pi].
\]
Since $q'(x)=x\sin x\geq0$, the function $q$ is increasing. Let
\[
    a\coloneqq\min\{\vartheta_1,\alpha\}, \qquad b\coloneqq\alpha-a.
\]
Then $a,b\in[0,\alpha]$, $a+b=\alpha$, $a\leq\vartheta_1$, and $b\leq\vartheta_2$. Indeed, if $\vartheta_1<\alpha$, then $b=\alpha-\vartheta_1\leq\vartheta_2$; otherwise $b=0$. Therefore,
\begin{equation}\label{eq:angular-lower-monotonicity}
    q(\vartheta_1)+q(\vartheta_2)\geq q(a)+q(b).
\end{equation}
Moreover,
\[
    q''(x)=\sin x+x\cos x\geq0, \qquad x\in[0,\alpha],
\]
because $\alpha\leq\pi/2$. Thus $q$ is convex on $[0,\alpha]$, and Jensen's inequality yields
\begin{equation}\label{eq:angular-lower-jensen}
    q(a)+q(b) \geq 2q\left(\frac{a+b}{2}\right) = 2q\left(\frac{\alpha}{2}\right).
\end{equation}
Combining \eqref{eq:angular-lower-monotonicity} and \eqref{eq:angular-lower-jensen} gives
\begin{equation}\label{eq:angular-lower-bound}
    l(\theta_1)+l(\theta_2) = l(\pi-\vartheta_1)+l(\pi-\vartheta_2) \geq 2l\left(\pi-\frac{\alpha}{2}\right).
\end{equation}

We now prove the upper bound and identify its equality case.

\paragraph{Case 1:} $\theta_1,\theta_2\in[0,\alpha]$.

Define $l_1$ on $[0,\alpha]$ by
\[
    l_1(x)\coloneqq l(x)+l(\alpha-x).
\]
The triangle inequality gives $\theta_2\geq\alpha-\theta_1$. Since $l$ is decreasing,
\begin{equation}\label{eq:case-one-reduction}
    l(\theta_1)+l(\theta_2) \leq l(\theta_1)+l(\alpha-\theta_1) =l_1(\theta_1).
\end{equation}
Differentiating \eqref{eq:l-first-derivative} gives
\begin{align}
    2\pi l_1''(x) &= \sin x-(\pi-x)\cos x +\sin(\alpha-x)-(\pi-\alpha+x)\cos(\alpha-x)\notag\\
    &= 2\left( \sin\frac{\alpha}{2} -\left(\pi-\frac{\alpha}{2}\right)\cos\frac{\alpha}{2} \right) \cos\left(x-\frac{\alpha}{2}\right) -2\left(x-\frac{\alpha}{2}\right)\sin\frac{\alpha}{2} \sin\left(x-\frac{\alpha}{2}\right).
    \label{eq:l-one-second-derivative}
\end{align}
Because $0<\alpha/2\leq\pi/4$,
\[
    \tan\frac{\alpha}{2} \leq 1 <\frac{3\pi}{4} \leq \pi-\frac{\alpha}{2}.
\]
Hence
\[
    \sin\frac{\alpha}{2} -\left(\pi-\frac{\alpha}{2}\right)\cos\frac{\alpha}{2} <0.
\]
For $x\in[0,\alpha]$, $|x-\alpha/2|\leq\alpha/2\leq\pi/4$, so
\[
    \cos\left(x-\frac{\alpha}{2}\right)>0, \qquad \left(x-\frac{\alpha}{2}\right) \sin\left(x-\frac{\alpha}{2}\right)\geq0.
\]
It follows from \eqref{eq:l-one-second-derivative} that $l_1''(x)<0$ on $[0,\alpha]$.

The symmetry $l_1(x)=l_1(\alpha-x)$ and strict concavity show that $l_1$ has a unique maximum at $x=\alpha/2$. Therefore,
\begin{equation}\label{eq:l-one-maximum}
    l_1(x) \leq l_1\left(\frac{\alpha}{2}\right) = 2l\left(\frac{\alpha}{2}\right), \qquad x\in[0,\alpha],
\end{equation}
with equality only at $x=\alpha/2$. Combining \eqref{eq:case-one-reduction} and \eqref{eq:l-one-maximum} yields
\[
    l(\theta_1)+l(\theta_2) \leq 2l\left(\frac{\alpha}{2}\right).
\]
If equality holds, \eqref{eq:l-one-maximum} forces $\theta_1=\alpha/2$. Equality in \eqref{eq:case-one-reduction}, together with the strict monotonicity of $l$, then forces $\theta_2=\alpha-\theta_1=\alpha/2$. Conversely, $\theta_1=\theta_2=\alpha/2$ attains equality.

\paragraph{Case 2:} Exactly one of $\theta_1,\theta_2$ lies in $(\alpha,\pi]$. By symmetry, assume $\theta_1\in(\alpha,\pi]$ and $\theta_2\in[0,\alpha]$.

The reverse triangle inequality gives $\theta_1-\theta_2\leq\alpha$, so $\theta_2\geq\theta_1-\alpha$. Monotonicity of $l$ then gives
\[
    l(\theta_1)+l(\theta_2) \leq l(\theta_1)+l(\theta_1-\alpha) \leq l(\alpha)+l(0) \leq 2l\left(\frac{\alpha}{2}\right).
\]
The last inequality is \eqref{eq:l-one-maximum} evaluated at $x=0$. Since $\alpha/2>0$, the equality condition in \eqref{eq:l-one-maximum} shows that this inequality is strict.

\paragraph{Case 3:} $\theta_1,\theta_2\in(\alpha,\pi]$.
Monotonicity gives
\[
    l(\theta_1)+l(\theta_2) \leq 2l(\alpha) \leq 2l\left(\frac{\alpha}{2}\right).
\]
The second inequality is strict because $l$ is strictly decreasing and $\alpha>0$.

Thus the upper bound holds in all three cases, with equality if and only if $\theta_1=\theta_2=\alpha/2$. Together with the lower bound \eqref{eq:angular-lower-bound}, this completes the proof.
\end{proof}

\subsection{Closed-form loss and gradient flow}\label{subsec:closed-form-loss}

The identities in this subsection hold for the wider range $\alpha=\angle(v_1,v_2)\in(0,\pi)$. We recall the angles $\theta_{ij}$ and $\gamma$ from \eqref{eq:angle_shorthand}, the kernel profile $l$ from \eqref{eq:relu-kernel-profile}, and the shorthand $\theta_{u,v}=\angle(u,v)$ from \cref{subsec:relu-kernel}. For nonzero $u,v\in\RR^d$, the derivative of the Gaussian ReLU kernel with respect to its first argument is
\begin{align}
    \nabla_u \EE_{x\sim\cN(0,I_d)} \big[\phi(\inner{u}{x})\phi(\inner{v}{x})\big] &= \EE_{x\sim\cN(0,I_d)} \big[\ind(u^\top x\geq 0)\phi(\inner{v}{x})x\big]\notag\\
    &= \frac{\|v\|_2\sin\theta_{u,v}}{2\pi}\bar u + \frac{\pi-\theta_{u,v}}{2\pi}v.
    \label{eq:relu-kernel-gradient}
\end{align}
The Lipschitz continuity of ReLU and Gaussian integrability justify differentiating under the expectation. The hyperplane $\{x:u^\top x=0\}$ has Gaussian measure zero, so using either $>$ or $\geq$ in the indicator gives the same expectation.

Expanding the squared residual and applying \eqref{eq:relu-gaussian-kernel} to each student--student, teacher--teacher, and student--teacher pair gives, for nonzero $w_1,w_2$,
\begin{align}
    L(w_1,w_2) &= \frac{1}{2}\EE_{x\sim\cN(0,I_d)} \left[ \big( \phi(\inner{w_1}{x})+\phi(\inner{w_2}{x}) -\phi(\inner{v_1}{x})-\phi(\inner{v_2}{x}) \big)^2 \right]\notag\\
    &= \frac{1}{4}\big(\|w_1\|_2^2+\|w_2\|_2^2\big) +\|w_1\|_2\|w_2\|_2l(\gamma) -\sum_{i=1}^2\|w_i\|_2\sum_{j=1}^2l(\theta_{ij}) +l(\alpha)+\frac{1}{2}.
    \label{eq:closed-form-loss}
\end{align}
Since $l$ is bounded by \cref{lem:l-monotone-bounds}, the right-hand side of \eqref{eq:closed-form-loss} extends continuously to all $w_1,w_2\in\RR^d$ if every term carrying a factor $\|w_i\|_2$ with $w_i=0$ is set to zero. We use \eqref{eq:closed-form-loss} with this convention whenever one or both student weights vanish.

On the domain $w_1,w_2\ne0$, applying \eqref{eq:relu-kernel-gradient} to the student--teacher and student--student terms in the gradient flow \eqref{eq:population_squared_loss_and_gf} gives
\begin{align}
    \frac{\diff w_1}{\diff t} &= \frac{1}{2\pi} \big( (\pi-\theta_{11})v_1+(\pi-\theta_{12})v_2 +(\sin \theta_{11}+\sin \theta_{12})\bar w_1 \big) - \left( \frac{1}{2} +\frac{\sin\gamma}{2\pi}\frac{\|w_2\|_2}{\|w_1\|_2} \right)w_1 -\frac{\pi-\gamma}{2\pi}w_2,
    \label{eq:gf-w1}\\
    \frac{\diff w_2}{\diff t} &= \frac{1}{2\pi} \big( (\pi-\theta_{21})v_1+(\pi-\theta_{22})v_2 +(\sin \theta_{21}+\sin \theta_{22})\bar w_2 \big) -\frac{\pi-\gamma}{2\pi}w_1 - \left( \frac{1}{2} +\frac{\sin\gamma}{2\pi}\frac{\|w_1\|_2}{\|w_2\|_2} \right)w_2.
    \label{eq:gf-w2}
\end{align}

\begin{lemma}[Direction and norm equations]\label{lem:direction-norm-equations}
Let $w(t)=(w_1(t),w_2(t))$ solve the gradient flow \eqref{eq:population_squared_loss_and_gf}. At each time $t$ for which $w_1(t),w_2(t)\ne0$, the following identities hold:
\begin{align}
    \frac{\diff\bar w_1}{\diff t} &= \frac{1}{2\pi\|w_1\|_2}(I-\bar w_1\bar w_1^\top) \big( (\pi-\theta_{11})v_1+(\pi-\theta_{12})v_2 -(\pi-\gamma)w_2 \big),
    \label{eq:direction-w1}\\
    \frac{\diff\bar w_2}{\diff t} &= \frac{1}{2\pi\|w_2\|_2}(I-\bar w_2\bar w_2^\top) \big( (\pi-\theta_{21})v_1+(\pi-\theta_{22})v_2 -(\pi-\gamma)w_1 \big),
    \label{eq:direction-w2}\\
    \frac{\diff\|w_1\|_2}{\diff t} &= \frac{1}{2\pi} \big( \sin\theta_{11}+\sin\theta_{12} +(\pi-\theta_{11})\cos\theta_{11} +(\pi-\theta_{12})\cos\theta_{12} \big)\notag\\
    &\quad -\frac{1}{2}\|w_1\|_2 -\frac{\sin\gamma+(\pi-\gamma)\cos\gamma}{2\pi}\|w_2\|_2,
    \label{eq:norm-w1}\\
    \frac{\diff\|w_2\|_2}{\diff t} &= \frac{1}{2\pi} \big( \sin\theta_{21}+\sin\theta_{22} +(\pi-\theta_{21})\cos\theta_{21} +(\pi-\theta_{22})\cos\theta_{22} \big)\notag\\
    &\quad -\frac{\sin\gamma+(\pi-\gamma)\cos\gamma}{2\pi}\|w_1\|_2 -\frac{1}{2}\|w_2\|_2.
    \label{eq:norm-w2}
\end{align}
\end{lemma}

\begin{proof}[Proof of \cref{lem:direction-norm-equations}]
Project \eqref{eq:gf-w1} and \eqref{eq:gf-w2} onto the orthogonal complement of each $\bar w_i$ for the direction equations and onto each $\bar w_i$ for the norm equations.
\end{proof}

\begin{lemma}[Bounds on the radial velocity]\label{lem:radial-velocity-bounds}
Let $w(t)=(w_1(t),w_2(t))$ solve the gradient flow \eqref{eq:population_squared_loss_and_gf}. At each time $t$ for which $w_1(t),w_2(t)\ne0$,
\[
    \frac{r_-}{2}-\frac{\|w_1\|_2+\|w_2\|_2}{2} \le \frac{\diff\|w_i\|_2}{\diff t} \le 1-\frac{\|w_i\|_2}{2}, \qquad i=1,2,
\]
where $r_-$ is defined in \eqref{eq:def_r}.
\end{lemma}

\begin{proof}[Proof of \cref{lem:radial-velocity-bounds}]
By \eqref{eq:relu-kernel-profile}, the norm equation \eqref{eq:norm-w1} reads
\[
    \frac{\diff\|w_1\|_2}{\diff t} =l(\theta_{11})+l(\theta_{12})-\frac12\|w_1\|_2-l(\gamma)\|w_2\|_2 .
\]
By \cref{lem:l-two-teacher-angle-bounds} and \eqref{eq:def_r}, $l(\theta_{11})+l(\theta_{12})\ge2l(\pi-\alpha/2)=r_-/2$, and \cref{lem:l-monotone-bounds} gives $l(\theta_{11})+l(\theta_{12})\le1$ and $0\le l(\gamma)\le1/2$. These bounds give the claim for $i=1$; the case $i=2$ follows in the same way from \eqref{eq:norm-w2}.
\end{proof}

\subsection{Hessian of the population loss}

As in \cref{subsec:closed-form-loss}, we consider $\alpha\in(0,\pi)$. We recall $\theta_{w,v}=\angle(w,v)$ and write
\[
    n_{v,w}\coloneqq \bar v-\cos\theta_{w,v}\bar w, \qquad \bar n_{v,w}\coloneqq \frac{n_{v,w}}{\|n_{v,w}\|_2}
\]
for nonzero $w,v\in\RR^d$ with $\theta_{w,v}\in(0,\pi)$.

\begin{lemma}[Hessian of the population loss]\label{lem:relu-two-neuron-hessian}
For nonzero $w_1,w_2$, the Hessian of the population loss $L(w_1,w_2)$ has diagonal blocks
\[
    \frac{\partial^2 L}{\partial w_i^2} = \frac{1}{2}I_d+Q_1(w_i,w_{3-i})-Q_1(w_i,v_1)-Q_1(w_i,v_2), \qquad i=1,2,
\]
and off-diagonal blocks
\[
    \frac{\partial^2 L}{\partial w_i\partial w_{3-i}} = Q_2(w_i,w_{3-i}), \qquad i=1,2,
\]
where
\begin{align*}
    Q_1(w,v) &= \frac{\sin\theta_{w,v}\|v\|_2}{2\pi\|w\|_2} \left( I_d-\bar w\bar w^\top+\bar n_{v,w}\bar n_{v,w}^\top \right),\\
    Q_2(w,v) &= \frac{1}{2\pi} \left( (\pi-\theta_{w,v})I_d +\bar n_{v,w}\bar w^\top+\bar n_{w,v}\bar v^\top \right).
\end{align*}
At collinear pairs, $Q_1$ and $Q_2$ are defined by continuous extension: $Q_1(w,v)=0$ when $\bar w=\pm\bar v$, while $Q_2(w,v)=\frac12 I_d$ when $\bar w=\bar v$ and $Q_2(w,v)=0$ when $\bar w=-\bar v$.
\end{lemma}

\begin{proof}[Proof of \cref{lem:relu-two-neuron-hessian}]
This is the block Hessian formula from Theorem~5 of \citet{safran2018spurious}, specialized to the two-student, two-teacher loss \eqref{eq:population_squared_loss_and_gf}.
\end{proof}

\subsection{Proof of \texorpdfstring{\cref{prop:teacher_leap_profile}}{the teacher leap proposition}}\label{subsec:teacher-leap-proof}

We use the normalized probabilists' Hermite polynomials from \cref{subsec:teacher_leap_complexity},
\[
    h_k(x) \coloneqq \frac{(-1)^k}{\sqrt{k!}}e^{x^2/2} \frac{\diff^k}{\diff x^k}e^{-x^2/2}, \qquad k\geq0.
\]

\begin{proof}[Proof of \cref{prop:teacher_leap_profile}]
Let $a\coloneqq\cos(\alpha/2)$ and $b\coloneqq\sin(\alpha/2)$. Since $\|v_1\|_2=\|v_2\|_2=1$, the directions in \eqref{eq:two-directions} satisfy
\[
    v_1=a\bar v+bv_\Delta,\qquad v_2=a\bar v-bv_\Delta.
\]
With $z_1=\inner{\bar v}{x}$ and $z_2=\inner{v_\Delta}{x}$ as in \cref{subsec:teacher_leap_complexity},
\[
    f^*(x)=g(z_1,z_2)\coloneqq\phi(az_1+bz_2)+\phi(az_1-bz_2).
\]
We expand $g$ as in \eqref{eq:two_coordinate_hermite_expansion} with $(U,W)=(z_1,z_2)$, and write
\[
    \widehat\phi_k\coloneqq \EE_{G\sim\cN(0,1)}[\phi(G)h_k(G)].
\]
Since $a^2+b^2=1$, the Hermite addition formula
\[
    h_k(az_1\pm bz_2) = \sum_{m+n=k}\sqrt{\frac{k!}{m!\,n!}}\,a^m(\pm b)^n h_m(z_1)h_n(z_2)
\]
and the orthonormality of $\{h_m(z_1)h_n(z_2)\}_{m,n\ge0}$ give
\[
    \EE[\phi(az_1\pm bz_2)h_m(z_1)h_n(z_2)] = \widehat\phi_{m+n}\sqrt{\frac{(m+n)!}{m!\,n!}}\,a^m(\pm b)^n.
\]
Therefore
\[
    \widehat g_{m,n} = \widehat\phi_{m+n}\sqrt{\frac{(m+n)!}{m!\,n!}}\,a^m b^n\bigl(1+(-1)^n\bigr).
\]
In particular, $\widehat g_{m,1}=0$ for every $m\ge0$.

It remains to compute $\widehat g_{1,0}$ and $\widehat g_{0,2}$. Since $h_1(x)=x$ and $h_2(x)=(x^2-1)/\sqrt2$,
\[
    \widehat\phi_1=\EE[G^2\ind(G>0)]=\frac12, \qquad \widehat\phi_2 =\frac{1}{\sqrt2}\EE[(G^3-G)\ind(G>0)] =\frac{1}{2\sqrt\pi}.
\]
Thus
\[
    \widehat g_{1,0}=2\widehat\phi_1 a=a\neq0, \qquad \widehat g_{0,2}=2\widehat\phi_2 b^2=\frac{b^2}{\sqrt\pi}\neq0,
\]
where $a>0$ and $b>0$ because $\alpha\in(0,\pi/2]$. Hence, along the ordering $z_1$ then $z_2$, the leap to $z_1$ is $1$, and the leap to $z_2$ is $2$ because $\widehat g_{m,1}=0$ for every $m$ while $\widehat g_{0,2}\neq0$. The leap exponent is therefore $2$.
\end{proof}

\begin{remark}[Unequal output weights]
The argument above uses that the two teacher neurons have equal output weights. If instead
\[
    y=c_1\phi(\inner{v_1}{x})+c_2\phi(\inner{v_2}{x}), \qquad c_1\neq c_2,
\]
then $y=c_1\phi(az_1+bz_2)+c_2\phi(az_1-bz_2)$, and Stein's lemma gives
\[
    \EE[z_2y]=\frac{b}{2}(c_1-c_2)\neq 0.
\]
The splitting coordinate then appears at degree $1$, so the leap-$2$ structure is specific to the symmetric teacher considered here.
\end{remark}

\section{Correlation-loss landscape}\label{app:correlation-loss-proof}

\subsection{Reduction to one student in \texorpdfstring{$\operatorname{span}\{v_1,\ldots,v_k\}$}{the span of the teachers}}

The correlation loss \eqref{eq:correlation-loss} is separable across student neurons, so it is enough to analyze one student on the unit sphere. For a single student $w\in\mathbb S^{d-1}$, \eqref{eq:relu-gaussian-kernel} gives, with $l$ as in \eqref{eq:relu-kernel-profile},
\[
    L_{\mathrm{corr}}(w)=-\sum_{i=1}^k l(\angle(w,v_i)).
\]
We also write \(g(t)=l(\arccos t)\) for \(t\in[-1,1]\), so that \(L_{\mathrm{corr}}(w)=-\sum_{i=1}^k g(\inner{w}{v_i})\) and
\begin{equation}\label{eq:corr-g-def}
    g(t) = \frac{t(\pi-\arccos t)+\sqrt{1-t^2}}{2\pi}.
\end{equation}
A direct computation gives
\begin{equation}\label{eq:corr-g-derivatives}
    g'(t)=\frac{\pi-\arccos t}{2\pi}, \qquad g''(t)=\frac{1}{2\pi\sqrt{1-t^2}}>0 \quad (t\in(-1,1)).
\end{equation}

We next compute \(\inner{\bar v}{v_i}\), where \(\bar v=\sum_i v_i/\|\sum_i v_i\|_2\) is the direction in \cref{thm:corr-loss-global-min}. Since \(\inner{v_i}{v_j}=\cos\alpha\) for \(i\ne j\),
\[
    \bigg\|\sum_{i=1}^k v_i\bigg\|_2^2 = k+k(k-1)\cos\alpha = k(1+(k-1)\cos\alpha), \qquad \bigg\langle v_i,\sum_{j=1}^k v_j\bigg\rangle = 1+(k-1)\cos\alpha.
\]
Therefore the angle \(\rho\coloneqq\angle(\bar v,v_i)\) is the same for every \(i\), with
\begin{equation}\label{eq:corr-rho}
    \cos\rho=\sqrt{\frac{1+(k-1)\cos\alpha}{k}}.
\end{equation}
Under the setting of \cref{thm:corr-loss-global-min}, we have \(1+(k-1)\cos\alpha>0\): this is immediate when \(\alpha\le\pi/2\), and for \(k=2\) it follows from \(\alpha<\pi\). Since also \(\alpha>0\), we have \(\rho\in(0,\pi/2)\). Moreover, \(L_{\mathrm{corr}}(\bar v)=-kg(\cos\rho)\).

\begin{lemma}[Minimizers lie in \texorpdfstring{$\operatorname{span}\{v_1,\ldots,v_k\}$}{the span of the teachers}]\label{lem:corr-teacher-span}
Under the setting of \cref{thm:corr-loss-global-min}, every minimizer of \(L_{\mathrm{corr}}\) on \(\mathbb S^{d-1}\) lies in \(\operatorname{span}\{v_1,\ldots,v_k\}\).
\end{lemma}

\begin{proof}[Proof of \cref{lem:corr-teacher-span}]
Let $w_*$ be a minimizer of $L_{\mathrm{corr}}$ on the unit sphere. At a unit vector \(w\), the gradient of \eqref{eq:correlation-loss} is \(-\sum_{i=1}^k g'(\inner{w}{v_i})v_i\) plus a multiple of \(w\), so the Lagrange multiplier condition gives, for some \(\lambda\in\RR\),
\[
    \sum_{i=1}^k g'(\inner{w_*}{v_i})v_i=\lambda w_*.
\]
The left-hand side lies in \(\operatorname{span}\{v_1,\ldots,v_k\}\), so it suffices to show that \(\lambda\ne0\). If \(\lambda=0\), taking the inner product with \(\sum_{j=1}^k v_j\) gives
\[
    0= \sum_{i=1}^k g'(\inner{w_*}{v_i}) \bigg\langle v_i,\sum_{j=1}^k v_j\bigg\rangle = (1+(k-1)\cos\alpha)\sum_{i=1}^k g'(\inner{w_*}{v_i}).
\]
Because \(1+(k-1)\cos\alpha>0\) and \(g'\ge0\), this forces \(g'(\inner{w_*}{v_i})=0\) for every \(i\). Since \(g'(t)=0\) only at \(t=-1\), we get \(w_*=-v_i\) for all \(i\), which contradicts \(\alpha>0\). Hence \(\lambda\ne0\) and \(w_*\in\operatorname{span}\{v_1,\ldots,v_k\}\).
\end{proof}

It remains to compare \(L_{\mathrm{corr}}(w)\) with \(L_{\mathrm{corr}}(\bar v)\) for unit vectors \(w\in\operatorname{span}\{v_1,\ldots,v_k\}\). We do this for \(0<\alpha\le\pi/2\) in \cref{lem:corr-case-large-rho,lem:corr-case-small-rho}, and for \(k=2\) and \(\pi/2<\alpha<\pi\) in \cref{lem:corr-two-teachers}.

\subsection{The case \texorpdfstring{$\alpha\le\pi/2$}{alpha <= pi/2}}

For \(0<\alpha\le\pi/2\) and general \(k\), we treat \(\rho\ge\pi/4\) and \(\rho\le\pi/4\) separately. The first case relies on the following scalar inequality for \(g\).

\begin{lemma}\label{lem:corr-scalar-ineq}
Let \(\rho\in[\pi/4,\pi/2)\), and let \(g\) be as in \eqref{eq:corr-g-def}. Then
\[
    g(y)\le g(\cos\rho)+\frac{\pi-\rho}{4\pi\cos\rho}(y^2-\cos^2\rho), \qquad y\in[-1,1],
\]
with equality only at \(y=\cos\rho\).
\end{lemma}

\begin{proof}[Proof of \cref{lem:corr-scalar-ineq}]
Write \(y=\cos\theta\) with \(\theta\in[0,\pi]\). Since \(g(\cos\theta)=l(\theta)\), where \(l\) is defined in \eqref{eq:relu-kernel-profile}, the inequality is equivalent to \(R_\rho(\theta)\ge0\), where
\[
    R_\rho(\theta) \coloneqq l(\rho) + \frac{\pi-\rho}{4\pi\cos\rho} (\cos^2\theta-\cos^2\rho) - l(\theta).
\]
Note that \(R_\rho(\rho)=0\). By \eqref{eq:l-first-derivative},
\[
    R_\rho'(\theta) = \frac{\sin\theta}{2\pi}\,H_\rho(\theta), \qquad H_\rho(\theta) \coloneqq (\pi-\theta)-\frac{\pi-\rho}{\cos\rho}\cos\theta.
\]
We show that \(H_\rho<0\) on \((0,\rho)\) and \(H_\rho>0\) on \((\rho,\pi)\).

On \([0,\rho]\), we have \(H_\rho(\rho)=0\) and \(H_\rho(0)=\pi-(\pi-\rho)/\cos\rho\), which is negative if and only if \(\pi\cos\rho<\pi-\rho\). The function \(\rho\mapsto\pi-\rho-\pi\cos\rho\) is positive at \(\rho=\pi/4\), where it equals \(3\pi/4-\pi/\sqrt2\), and its derivative \(-1+\pi\sin\rho\ge-1+\pi/\sqrt2\) is positive on \([\pi/4,\pi/2)\). Hence \(H_\rho(0)<0\). Moreover,
\[
    H_\rho'(\theta)=-1+\frac{\pi-\rho}{\cos\rho}\sin\theta, \qquad H_\rho''(\theta)=\frac{\pi-\rho}{\cos\rho}\cos\theta,
\]
so \(H_\rho\) is convex on \([0,\pi/2]\). Since \(H_\rho(0)<0=H_\rho(\rho)\), convexity gives \(H_\rho<0\) on \((0,\rho)\).

On \((\rho,\pi/2]\), we use \(H_\rho'(\rho)=-1+(\pi-\rho)\tan\rho\). The function \(\rho\mapsto(\pi-\rho)\tan\rho\) equals \(3\pi/4>1\) at \(\rho=\pi/4\) and is increasing on \([\pi/4,\pi/2)\), since
\[
    \frac{\diff}{\diff\rho}\bigl((\pi-\rho)\tan\rho\bigr) = \frac{(\pi-\rho)-\sin\rho\cos\rho}{\cos^2\rho}>0
\]
because \(\pi-\rho>\pi/2>\sin\rho\cos\rho\). Hence \(H_\rho'(\rho)>0\). By convexity, \(H_\rho'\) is nondecreasing on \([0,\pi/2]\), so \(H_\rho'>0\) on \([\rho,\pi/2]\) and \(H_\rho(\theta)>H_\rho(\rho)=0\) for \(\rho<\theta\le\pi/2\).

On \([\pi/2,\pi)\), we have \(\cos\theta\le0\) and \(\pi-\theta>0\), so \(H_\rho(\theta)>0\).

Since \(\sin\theta>0\) on \((0,\pi)\), the function \(R_\rho\) is strictly decreasing on \([0,\rho]\) and strictly increasing on \([\rho,\pi]\). Thus \(R_\rho(\theta)>R_\rho(\rho)=0\) for \(\theta\ne\rho\).
\end{proof}

The case \(\rho\ge\pi/4\) now follows from \cref{lem:corr-scalar-ineq}.

\begin{lemma}\label{lem:corr-case-large-rho}
Suppose that \(0<\alpha\le\pi/2\) and \(\rho\ge\pi/4\), with \(\rho\) as in \eqref{eq:corr-rho}. Then \(L_{\mathrm{corr}}(w)\ge L_{\mathrm{corr}}(\bar v)\) for every unit vector \(w\in\operatorname{span}\{v_1,\ldots,v_k\}\), with equality only at \(w=\bar v\).
\end{lemma}

\begin{proof}[Proof of \cref{lem:corr-case-large-rho}]
Applying \cref{lem:corr-scalar-ineq} to each \(\inner{w}{v_i}\) and summing over \(i\) gives
\[
    L_{\mathrm{corr}}(w)-L_{\mathrm{corr}}(\bar v) = kg(\cos\rho)-\sum_i g(\inner{w}{v_i}) \ge \frac{\pi-\rho}{4\pi\cos\rho}\left(k\cos^2\rho-\sum_i \inner{w}{v_i}^2\right),
\]
with equality only if \(\inner{w}{v_i}=\cos\rho\) for every \(i\). The coefficient \((\pi-\rho)/(4\pi\cos\rho)\) is positive. Moreover, \(\sum_i \inner{w}{v_i}^2=w^\top(\sum_i v_iv_i^\top)w\). The matrix \(\sum_i v_iv_i^\top\) has eigenvalue \(1+(k-1)\cos\alpha=k\cos^2\rho\) in the direction \(\bar v\) and eigenvalue \(1-\cos\alpha\) on the orthogonal complement of \(\bar v\) inside \(\operatorname{span}\{v_1,\ldots,v_k\}\). Since \(\cos\alpha\ge0\), we have \(1-\cos\alpha\le k\cos^2\rho\), and hence \(\sum_i \inner{w}{v_i}^2\le k\cos^2\rho\). Therefore \(L_{\mathrm{corr}}(w)\ge L_{\mathrm{corr}}(\bar v)\).

If \(L_{\mathrm{corr}}(w)=L_{\mathrm{corr}}(\bar v)\), the inequality in the display is an equality, so \(\inner{w}{v_i}=\cos\rho=\inner{\bar v}{v_i}\) for all \(i\). The Gram matrix
\[
    (\inner{v_i}{v_j})_{i,j}=(1-\cos\alpha)I+\cos\alpha\,\mathbf 1\mathbf 1^\top
\]
is positive definite because \(0\le\cos\alpha<1\), and both \(w\) and \(\bar v\) lie in \(\operatorname{span}\{v_1,\ldots,v_k\}\). Hence \(w=\bar v\).
\end{proof}

When \(\rho\le\pi/4\), we use a different argument: we write \(w\) in terms of its angle \(\varphi\) from \(\bar v\) and bound \(L_{\mathrm{corr}}(w)\) separately for \(\varphi\ge\rho\) and \(\varphi\le\rho\).

\begin{lemma}\label{lem:corr-case-small-rho}
Suppose that \(0<\alpha\le\pi/2\) and \(\rho\le\pi/4\), with \(\rho\) as in \eqref{eq:corr-rho}. Then \(L_{\mathrm{corr}}(w)\ge L_{\mathrm{corr}}(\bar v)\) for every unit vector \(w\in\operatorname{span}\{v_1,\ldots,v_k\}\), with equality only at \(w=\bar v\).
\end{lemma}

\begin{proof}[Proof of \cref{lem:corr-case-small-rho}]
Let \(\tilde v_i=(v_i-\cos\rho\,\bar v)/\sin\rho\) be the unit vector in the direction of the component of \(v_i\) orthogonal to \(\bar v\). These vectors satisfy
\begin{equation}\label{eq:corr-case2-simplex-frame}
    v_i=\cos\rho\,\bar v+\sin\rho\,\tilde v_i, \qquad \tilde v_i\perp \bar v,\qquad \|\tilde v_i\|_2=1, \qquad \sum_i \tilde v_i=0, \qquad \inner{\tilde v_i}{\tilde v_j}=-\frac1{k-1}\quad (i\ne j).
\end{equation}
We write
\[
    w=\cos\varphi\,\bar v+\sin\varphi\,u,
\]
where \(\varphi\in[0,\pi]\), \(u\perp\bar v\), \(\|u\|_2=1\), and \(u\in\operatorname{span}\{\tilde v_1,\ldots,\tilde v_k\}\). When \(\sin\varphi=0\), the vector \(u\) is arbitrary, and we fix any such \(u\). Let \(t_i=\inner{u}{\tilde v_i}\) and
\[
    z_i = \inner{w}{v_i} = \cos\rho\cos\varphi+t_i\sin\rho\sin\varphi.
\]
With a slight abuse of notation, we write
\[
    L_{\mathrm{corr}}(\varphi,u) \coloneqq L_{\mathrm{corr}}(\cos\varphi\,\bar v+\sin\varphi\,u) = -\sum_{i=1}^k g(z_i),
\]
where \(g\) is defined in \eqref{eq:corr-g-def}. In particular, \(L_{\mathrm{corr}}(0,u)=L_{\mathrm{corr}}(\bar v)\).

We note that the numbers \(t_i\) satisfy
\[
    \sum_i t_i=0, \qquad \sum_i t_i^2=\frac{k}{k-1}, \qquad -1\le t_i\le1.
\]
Indeed, \(\sum_i t_i=\inner{u}{\sum_i \tilde v_i}=0\), and \(|t_i|\le1\) by Cauchy--Schwarz. For every \(j\), \eqref{eq:corr-case2-simplex-frame} gives
\[
    \bigg\langle\sum_i t_i \tilde v_i,\tilde v_j\bigg\rangle =t_j-\frac{1}{k-1}\sum_{i\ne j}t_i =\frac{k}{k-1}t_j =\bigg\langle\frac{k}{k-1}u,\tilde v_j\bigg\rangle.
\]
Both \(\sum_i t_i \tilde v_i\) and \(u\) lie in \(\operatorname{span}\{\tilde v_1,\ldots,\tilde v_k\}\), so \(\sum_i t_i \tilde v_i=\frac{k}{k-1}u\). Taking the inner product with \(u\) gives \(\sum_i t_i^2=k/(k-1)\).

We first show that \(L_{\mathrm{corr}}(\varphi,u)\) is nondecreasing in \(\varphi\) on \([\rho,\pi]\). We have
\begin{equation}\label{eq:corr-case2-derivative}
    \frac{\partial}{\partial\varphi}L_{\mathrm{corr}}(\varphi,u) = \cos\rho\sin\varphi\sum_i g'(z_i) - \sin\rho\cos\varphi\sum_i t_i g'(z_i).
\end{equation}
Since \(\sum_i t_i=0\), the identity \(\sum_{i<j}(t_i-t_j)\bigl(g'(z_i)-g'(z_j)\bigr) =k\sum_i t_ig'(z_i)-\bigl(\sum_i t_i\bigr)\bigl(\sum_i g'(z_i)\bigr)\) gives
\begin{equation}\label{eq:corr-case2-monotone-sum}
    \sum_i t_i g'(z_i) = \frac1k\sum_{i<j}(t_i-t_j) \bigl(g'(z_i)-g'(z_j)\bigr)\ge0,
\end{equation}
because \(g'\) is increasing by \eqref{eq:corr-g-derivatives} and \(z_i\) is nondecreasing in \(t_i\) (recall \(\sin\rho>0\) and \(\sin\varphi\ge0\)).

If \(\varphi\ge\pi/2\), then \(\cos\varphi\le0\), so both terms in \eqref{eq:corr-case2-derivative} are nonnegative, by \(g'\ge0\) and \eqref{eq:corr-case2-monotone-sum}. If \(\rho\le\varphi<\pi/2\), then \(\cos\varphi>0\) and \(\tan\varphi\ge\tan\rho\), so \(\cot\rho\tan\varphi\ge1\ge t_i\) and
\[
    \frac{\partial}{\partial\varphi}L_{\mathrm{corr}}(\varphi,u) = \sin\rho\cos\varphi \sum_i \left(\cot\rho\tan\varphi-t_i\right) g'(z_i) \ge0.
\]
Hence \(L_{\mathrm{corr}}(\varphi,u)\) is nondecreasing in \(\varphi\) on \([\rho,\pi]\), and
\[
    L_{\mathrm{corr}}(\varphi,u)\geq L_{\mathrm{corr}}(\rho,u), \qquad \varphi\ge\rho.
\]

It remains to consider \(0\le\varphi\le\rho\). We first take \(0\le\varphi<\rho\); the bound \eqref{eq:corr-case2-worst-direction} obtained below then extends to \(\varphi=\rho\) by continuity. Since \(\rho\le\pi/4\), we have \(\rho+\varphi<\pi/2\), and for every \(t\in[-1,1]\),
\[
    0 \le \cos(\rho+\varphi) \le \cos\rho\cos\varphi+t\sin\rho\sin\varphi \le \cos(\rho-\varphi) <1.
\]
Hence the function \(h_\varphi(t)=g(\cos\rho\cos\varphi+t\sin\rho\sin\varphi)\) is \(C^3\) on \([-1,1]\). Differentiating \eqref{eq:corr-g-derivatives} once more gives \(g'''(s)=s/(2\pi(1-s^2)^{3/2})\), which is nonnegative for \(s\in[0,1)\). By the chain rule,
\[
    h_\varphi'''(t) = (\sin\rho\sin\varphi)^3\, g'''(\cos\rho\cos\varphi+t\sin\rho\sin\varphi) \ge0, \qquad t\in[-1,1].
\]

Let \(s_0=-1/(k-1)\in[-1,0)\), and let \(Q_\varphi\) be the quadratic polynomial
\[
    Q_\varphi(t) = h_\varphi(s_0) +h_\varphi'(s_0)(t-s_0) + \frac{ h_\varphi(1)-h_\varphi(s_0)-h_\varphi'(s_0)(1-s_0) }{(1-s_0)^2} (t-s_0)^2,
\]
which satisfies \(Q_\varphi(1)=h_\varphi(1)\), \(Q_\varphi(s_0)=h_\varphi(s_0)\), and \(Q_\varphi'(s_0)=h_\varphi'(s_0)\). For \(t\in[-1,1]\), the Hermite interpolation remainder with a double node at \(s_0\) and a simple node at \(1\) gives
\[
    h_\varphi(t)-Q_\varphi(t) = \frac{h_\varphi'''(\xi)}{6}(t-s_0)^2(t-1)
\]
for some \(\xi\in[-1,1]\). Since \(h_\varphi'''(\xi)\ge0\) and \(t-1\le0\), we have \(h_\varphi\le Q_\varphi\) on \([-1,1]\), and hence \(\sum_i h_\varphi(t_i)\le\sum_i Q_\varphi(t_i)\).

Since \(Q_\varphi\) is quadratic, \(\sum_i Q_\varphi(t_i)\) depends on the \(t_i\) only through \(\sum_i t_i=0\) and \(\sum_i t_i^2=k/(k-1)\). The \(k\) numbers \(1,s_0,\ldots,s_0\) have the same two sums, because \(1+(k-1)s_0=0\) and \(1+(k-1)s_0^2=k/(k-1)\). Therefore
\begin{equation}\label{eq:corr-case2-worst-direction}
    -L_{\mathrm{corr}}(\varphi,u) =\sum_i h_\varphi(t_i) \le Q_\varphi(1)+(k-1)Q_\varphi(s_0) =h_\varphi(1)+(k-1)h_\varphi(s_0) =-L_{\mathrm{corr}}(\varphi,\tilde v_1),
\end{equation}
where the last equality holds because \(u=\tilde v_1\) gives \(t_1=1\) and \(t_j=s_0\) for \(j\ne1\), by \eqref{eq:corr-case2-simplex-frame}. Thus, for each \(\varphi\in[0,\rho]\), \(L_{\mathrm{corr}}(\varphi,\cdot)\) is minimized at \(u=\tilde v_1\).

For \(u=\tilde v_1\), we have \(z_1=\cos(\rho-\varphi)\), and \(z_j=z_2\) for every \(j\ge2\), where
\begin{equation}\label{eq:corr-case2-worst-loss}
    z_2=\cos\rho\cos\varphi-\frac{\sin\rho\sin\varphi}{k-1}, \qquad L_{\mathrm{corr}}(\varphi,\tilde v_1) = -g\bigl(\cos(\rho-\varphi)\bigr)-(k-1)g(z_2).
\end{equation}
We now show that \(L_{\mathrm{corr}}(\varphi,\tilde v_1)>L_{\mathrm{corr}}(\bar v)\) for \(0<\varphi\le\rho\). Since \(0\le\varphi\le\rho\le\pi/4\),
\begin{equation}\label{eq:corr-case2-z2-bounds}
    0\le \cos(\rho+\varphi) \le z_2 \le \cos\rho.
\end{equation}
Since the arccosine is decreasing, \eqref{eq:corr-case2-z2-bounds} gives \(\rho\le\arccos z_2\le\rho+\varphi\). We write \(\delta(\varphi)=\arccos z_2-\rho\), so that \(0\le\delta(\varphi)\le\varphi\). By \eqref{eq:corr-g-derivatives},
\[
    g'\bigl(\cos(\rho-\varphi)\bigr) = \frac{\pi-\rho+\varphi}{2\pi}, \qquad g'(z_2) = \frac{\pi-\rho-\delta(\varphi)}{2\pi}.
\]
Differentiating \eqref{eq:corr-case2-worst-loss} therefore gives
\begin{align}
    \frac{\partial}{\partial\varphi}L_{\mathrm{corr}}(\varphi,\tilde v_1) &= \frac{1}{2\pi}\Big[ -(\pi-\rho+\varphi)\sin(\rho-\varphi) \nonumber\\
    &\qquad+(k-1)\bigl(\pi-\rho-\delta(\varphi)\bigr) \left(\cos\rho\sin\varphi+\frac{\sin\rho\cos\varphi}{k-1}\right) \Big] \nonumber\\
    &= \frac{1}{2\pi}\Big[ \cos\rho\sin\varphi\left[k(\pi-\rho)+\varphi-(k-1)\delta(\varphi)\right] - \sin\rho\cos\varphi\bigl(\varphi+\delta(\varphi)\bigr) \Big].
    \label{eq:corr-case2-worst-derivative}
\end{align}
The bracket in \eqref{eq:corr-case2-worst-derivative} is an affine function of \(\delta(\varphi)\) with slope \(-(k-1)\cos\rho\sin\varphi-\sin\rho\cos\varphi<0\). Since \(\delta(\varphi)\le\varphi\), replacing \(\delta(\varphi)\) by \(\varphi\) gives
\begin{equation}\label{eq:corr-case2-worst-derivative-lower}
    \frac{\partial}{\partial\varphi}L_{\mathrm{corr}}(\varphi,\tilde v_1) \ge \frac{1}{2\pi}\left[ \cos\rho\sin\varphi\left[k(\pi-\rho)-(k-2)\varphi\right] - 2\varphi\sin\rho\cos\varphi \right].
\end{equation}
For \(0<\varphi\le\rho\), dividing by \(\cos\rho\sin\varphi>0\) shows that the right-hand side of \eqref{eq:corr-case2-worst-derivative-lower} is positive whenever
\[
    2\tan\rho\cdot \varphi\cot\varphi < k(\pi-\rho)-(k-2)\varphi.
\]
This inequality holds: since \(\varphi\cot\varphi\le1\) and \(\rho\le\pi/4\), the left-hand side is at most \(2\tan\rho\le2\), while the right-hand side is at least \(k(\pi-\rho)-(k-2)\rho=k\pi-2(k-1)\rho\ge(k+1)\pi/2\ge3\pi/2>2\). Hence \(L_{\mathrm{corr}}(\varphi,\tilde v_1)\) is strictly increasing on \([0,\rho]\), and
\[
    L_{\mathrm{corr}}(\varphi,\tilde v_1) > L_{\mathrm{corr}}(0,\tilde v_1) = L_{\mathrm{corr}}(\bar v), \qquad 0<\varphi\le\rho.
\]

Combining the two ranges with \eqref{eq:corr-case2-worst-direction}, \(L_{\mathrm{corr}}(\varphi,u)\ge L_{\mathrm{corr}}(\varphi,\tilde v_1)>L_{\mathrm{corr}}(\bar v)\) for \(0<\varphi\le\rho\), and \(L_{\mathrm{corr}}(\varphi,u)\ge L_{\mathrm{corr}}(\rho,u)\ge L_{\mathrm{corr}}(\rho,\tilde v_1)>L_{\mathrm{corr}}(\bar v)\) for \(\varphi\ge\rho\). Thus \(L_{\mathrm{corr}}(w)\ge L_{\mathrm{corr}}(\bar v)\), with equality only when \(\varphi=0\), that is, \(w=\bar v\).
\end{proof}

\subsection{The case \texorpdfstring{$k=2$ and $\alpha>\pi/2$}{k = 2 and alpha > pi/2}}

For \(k=2\), we have \(\rho=\alpha/2\). An obtuse \(\alpha\) therefore gives \(\cos\alpha<0\), which is outside the range of \cref{lem:corr-case-large-rho}. In this case \(\operatorname{span}\{v_1,v_2\}=\operatorname{span}\{\bar v,v_\Delta\}\), and we compare \(L_{\mathrm{corr}}\) along its unit circle directly.

\begin{lemma}\label{lem:corr-two-teachers}
Let \(k=2\) and \(\pi/2<\alpha<\pi\), and let \(\alpha_\star\) be as in \eqref{eq:width-two-angle-threshold}.
\begin{enumerate}[label=(\alph*)]
    \item If \(\alpha\le\alpha_\star\), then \(\bar v\) is the unique minimizer of \(L_{\mathrm{corr}}\) on the unit sphere.
    \item If \(\alpha>\alpha_\star\), then \(\bar v\) is not a local minimizer of \(L_{\mathrm{corr}}\) on the unit sphere.
\end{enumerate}
\end{lemma}

\begin{proof}[Proof of \cref{lem:corr-two-teachers}]
Recall from \eqref{eq:two-directions} that \(\bar v\) and \(v_\Delta\) form an orthonormal basis of the teacher plane, with
\[
    v_1=\cos(\alpha/2)\,\bar v+\sin(\alpha/2)\,v_\Delta, \qquad v_2=\cos(\alpha/2)\,\bar v-\sin(\alpha/2)\,v_\Delta.
\]
As in the proof of \cref{lem:corr-case-small-rho}, we write \(L_{\mathrm{corr}}(\varphi,v_\Delta)=L_{\mathrm{corr}}(\cos\varphi\,\bar v+\sin\varphi\,v_\Delta)\). For \(0\le\varphi\le\pi/2\), with \(l\) as in \eqref{eq:relu-kernel-profile},
\[
    L_{\mathrm{corr}}(\varphi,v_\Delta) = -l(|\alpha/2-\varphi|)-l(\alpha/2+\varphi).
\]
For \(0<\varphi<\alpha/2\), direct differentiation gives
\begin{equation}\label{eq:corr-two-teachers-derivative}
    \frac{\partial}{\partial\varphi}L_{\mathrm{corr}}(\varphi,v_\Delta) = \frac{\cos(\alpha/2)\,\sin\varphi}{\pi} \left[ (\pi-\alpha/2)-\tan(\alpha/2)\cdot\varphi\cot\varphi \right].
\end{equation}

We first prove~(a). Since \(\arccos(-t)=\pi-\arccos t\), the formula \eqref{eq:corr-g-def} gives \(g(t)-g(-t)=t/2\) for \(t\in[-1,1]\). Consequently,
\[
    L_{\mathrm{corr}}(-w)-L_{\mathrm{corr}}(w) = \frac12\inner{w}{v_1+v_2} = \cos(\alpha/2)\,\inner{w}{\bar v}.
\]
By \cref{lem:corr-teacher-span}, every minimizer lies in \(\operatorname{span}\{\bar v,v_\Delta\}\). By the identity above, a minimizer cannot have \(\inner{w}{\bar v}<0\), and reflection across \(\bar v\) leaves \(L_{\mathrm{corr}}\) unchanged while changing the sign of \(\inner{w}{v_\Delta}\). It therefore suffices to show that \(L_{\mathrm{corr}}(\varphi,v_\Delta)>L_{\mathrm{corr}}(\bar v)\) for \(0<\varphi\le\pi/2\). The definition \eqref{eq:width-two-angle-threshold} of \(\alpha_\star\), together with the monotonicity of \(x\mapsto\tan x+x\), gives \(\tan(\alpha/2)\leq\pi-\alpha/2\). Since \(\varphi\cot\varphi<1\), the derivative \eqref{eq:corr-two-teachers-derivative} is positive for \(0<\varphi<\alpha/2\). For \(\alpha/2<\varphi\leq\pi/2\),
\[
    \frac{\partial}{\partial\varphi}L_{\mathrm{corr}}(\varphi,v_\Delta) = -l'(\varphi-\alpha/2)-l'(\varphi+\alpha/2)>0
\]
by \eqref{eq:l-first-derivative}, since \(0<\varphi+\alpha/2<\pi\). The one-sided derivatives agree at \(\varphi=\alpha/2\) because \(l'(0)=0\). Hence \(L_{\mathrm{corr}}(\varphi,v_\Delta)>L_{\mathrm{corr}}(\bar v)\) for every \(\varphi\in(0,\pi/2]\).

We now prove~(b). Since \(\alpha>\alpha_\star\), we have \(\tan(\alpha/2)>\pi-\alpha/2\). Differentiating \eqref{eq:corr-two-teachers-derivative} at \(\varphi=0\) gives
\[
    \frac{\partial^2}{\partial\varphi^2}L_{\mathrm{corr}}(\varphi,v_\Delta)\Big|_{\varphi=0} = \frac{(\pi-\alpha/2)\cos(\alpha/2)-\sin(\alpha/2)}{\pi} <0,
\]
while the first derivative vanishes at \(\varphi=0\). Hence \(L_{\mathrm{corr}}(\varphi,v_\Delta)<L_{\mathrm{corr}}(\bar v)\) for small \(\varphi>0\), so \(\bar v\) is not a local minimizer of \(L_{\mathrm{corr}}\).
\end{proof}

\subsection{Proof of \texorpdfstring{\cref{thm:corr-loss-global-min}}{the correlation-loss theorem}}

\begin{proof}[Proof of \cref{thm:corr-loss-global-min}]
Since \(L_{\mathrm{corr}}\) is continuous on the compact unit sphere, it has a minimizer, and by \cref{lem:corr-teacher-span} every minimizer lies in \(\operatorname{span}\{v_1,\ldots,v_k\}\). If \(0<\alpha\le\pi/2\), then \cref{lem:corr-case-large-rho} (when \(\rho\ge\pi/4\)) or \cref{lem:corr-case-small-rho} (when \(\rho\le\pi/4\)) shows that \(\bar v\) is the only unit vector in \(\operatorname{span}\{v_1,\ldots,v_k\}\) attaining the minimum of \(L_{\mathrm{corr}}\). If \(k=2\) and \(\pi/2<\alpha\le\alpha_\star\), the same conclusion is part~(a) of \cref{lem:corr-two-teachers}. In either case, \(\bar v\) is the unique minimizer of \(L_{\mathrm{corr}}\) on the unit sphere. For unit-norm \(w_1,\ldots,w_m\),
\[
    L_{\mathrm{corr}}(w_1,\ldots,w_m) = \sum_{j=1}^m L_{\mathrm{corr}}(w_j) \geq mL_{\mathrm{corr}}(\bar v),
\]
with equality if and only if \(w_j=\bar v\) for every \(j\). Thus \((\bar v,\ldots,\bar v)\) is the unique global minimum.
\end{proof}

\section{Squared-loss landscape}

\subsection{Stationarity and dimension reduction}

The proofs below use the condition \(\alpha\leq\alpha_\star\) in the form \(\sin(\alpha/2)\leq(\pi-\alpha/2)\cos(\alpha/2)\). The following lemma shows that the two are equivalent.

\begin{lemma}[Equivalent angle threshold]\label{lem:squared-loss-angle-threshold}
The equation \(\tan(\alpha/2)=\pi-\alpha/2\) has exactly one solution in \((\pi/2,\pi)\), namely \(\alpha_\star\) in \eqref{eq:width-two-angle-threshold}. Moreover, for every \(\alpha\in(0,\pi)\), the following conditions are equivalent: (a) \(\alpha\leq\alpha_\star\); (b) \(\tan(\alpha/2)\leq\pi-\alpha/2\); (c) \(\sin(\alpha/2)\leq(\pi-\alpha/2)\cos(\alpha/2)\).
\end{lemma}

\begin{proof}[Proof of \cref{lem:squared-loss-angle-threshold}]
The function \(\alpha\mapsto\tan(\alpha/2)+\alpha/2-\pi\) is strictly increasing on \((0,\pi)\), since its derivative is \(\bigl(\sec^2(\alpha/2)+1\bigr)/2>0\). It equals \(1-3\pi/4<0\) at \(\alpha=\pi/2\) and tends to \(+\infty\) as \(\alpha\uparrow\pi\). Thus it has a unique root \(\alpha_\star\in(\pi/2,\pi)\) and is nonpositive exactly on \((0,\alpha_\star]\), so (a) and (b) are equivalent. Since \(\cos(\alpha/2)>0\), multiplying (b) by \(\cos(\alpha/2)\) shows that (b) and (c) are equivalent.
\end{proof}

\subsubsection{Stationary equations}\label{subsubsec:stationary-equations-proof}

We first derive the stationarity equations \eqref{eq:critical_eq} from the gradient-flow equations \eqref{eq:gf-w1}--\eqref{eq:gf-w2}.

\begin{proof}[Proof of \cref{lem:critical_eq_from_flow}]
Since \(\frac{\diff w_i}{\diff t}=-\nabla_{w_i}L\), the pair \((w_1,w_2)\) is stationary if and only if \(\frac{\diff w_1}{\diff t}=\frac{\diff w_2}{\diff t}=0\) in \eqref{eq:gf-w1}--\eqref{eq:gf-w2}. Setting \(\frac{\diff w_1}{\diff t}=0\) in \eqref{eq:gf-w1}, multiplying by \(2\pi\), and using \(w_i=\|w_i\|_2\bar w_i\) gives
\[
    (\pi-\theta_{11})v_1+(\pi-\theta_{12})v_2 +a_1\bar w_1-(\pi-\gamma)\|w_2\|_2\bar w_2=0,
\]
which is the first identity in \eqref{eq:critical_eq}. Setting \(\frac{\diff w_2}{\diff t}=0\) in \eqref{eq:gf-w2} similarly yields the second identity. These steps can be reversed, so \eqref{eq:critical_eq} holds if and only if \((w_1,w_2)\) is stationary.
\end{proof}

\subsubsection{Uniqueness of the out-of-plane root}\label{subsubsec:out-of-plane-root-uniqueness-proof}

The out-of-plane saddles in \cref{thm:squared_loss_landscape} are determined by a root of \(F_\alpha\) in \((0,\pi/2)\), so we next show that this root exists and is unique.

\begin{proof}[Proof of \cref{lem:out-of-plane-root-uniqueness}]
Fix \(\alpha\in(0,\pi)\). By \eqref{eq:out-of-plane-angle}, \(\theta_\alpha(\kappa)=\arccos\bigl(\cos(\alpha/2)\cos\kappa\bigr)\), so for \(0<\kappa<\pi/2\) we have \(\theta_\alpha(\kappa)<\pi/2\) and \(0<\theta_\alpha'(\kappa)=\cos(\alpha/2)\sin\kappa/\sin\theta_\alpha(\kappa)<1\), where the last inequality follows from \(\sin^2\theta_\alpha(\kappa)-\cos^2(\alpha/2)\sin^2\kappa =\sin^2(\alpha/2)>0\). Using \(\sin(2\kappa)=2\sin\kappa\cos\kappa\), we first rewrite \eqref{eq:out-of-plane-radius} as
\[
    r_\alpha(\kappa) =\frac{2\sin\theta_\alpha(\kappa)}{\sin(2\kappa)+2\kappa} =\frac{\sin\theta_\alpha(\kappa)}{\kappa+\sin\kappa\cos\kappa}.
\]
Substituting this expression into \eqref{eq:out-of-plane-root-function} yields
\[
    F_\alpha(\kappa) =\bigl(\pi-\theta_\alpha(\kappa)\bigr)\cos(\alpha/2) -\frac{(\pi-2\kappa)\sin\theta_\alpha(\kappa)\cos\kappa} {\kappa+\sin\kappa\cos\kappa}.
\]
The fraction in the second term is strictly positive because \(0<\kappa<\pi/2\) and \(0<\theta_\alpha(\kappa)<\pi/2\). Dividing by this fraction shows that \(F_\alpha(\kappa)=0\) if and only if \(Q_\alpha(\kappa)=1\), where
\begin{equation}\label{eq:out-of-plane-Q}
\begin{aligned}
    Q_\alpha(\kappa) &= \frac{\cos(\alpha/2)\,\bigl(\pi-\theta_\alpha(\kappa)\bigr)(\kappa+\sin\kappa\cos\kappa)} {(\pi-2\kappa)\sin\theta_\alpha(\kappa)\cos\kappa}\\
    &= \cos(\alpha/2)\, \frac{\pi-\theta_\alpha(\kappa)}{\pi-2\kappa} \frac{\kappa+\sin\kappa\cos\kappa}{\sin\kappa\cos\kappa} \frac{\sin\kappa}{\sin\theta_\alpha(\kappa)}\\
    &= \cos(\alpha/2)\, \frac{\pi-\theta_\alpha(\kappa)}{\pi-2\kappa} \left(1+\frac{2\kappa}{\sin(2\kappa)}\right) \frac{\sin\kappa}{\sin\theta_\alpha(\kappa)}.
\end{aligned}
\end{equation}
We can show that each of the three positive factors after \(\cos(\alpha/2)\) in the last line is strictly increasing in \(\kappa\). For the first factor,
\[
    \frac{\diff}{\diff \kappa}\frac{\pi-\theta_\alpha(\kappa)}{\pi-2\kappa} = \frac{2\bigl(\pi-\theta_\alpha(\kappa)\bigr)-(\pi-2\kappa)\theta_\alpha'(\kappa)} {(\pi-2\kappa)^2} >\frac{2\kappa}{(\pi-2\kappa)^2}>0.
\]
Here \(\pi-2\kappa>0\), \(\theta_\alpha'(\kappa)<1\), and \(\theta_\alpha(\kappa)<\pi/2\) give
\[
\begin{aligned}
    2\bigl(\pi-\theta_\alpha(\kappa)\bigr)-(\pi-2\kappa)\theta_\alpha'(\kappa) >2\bigl(\pi-\theta_\alpha(\kappa)\bigr)-(\pi-2\kappa) =\pi-2\theta_\alpha(\kappa)+2\kappa >2\kappa.
\end{aligned}
\]
For the second factor, \(0<2\kappa<\pi\) implies \(t\sin t>0\) for \(0<t<2\kappa\), so integration by parts gives
\[
    0<\int_0^{2\kappa}t\sin t\,\diff t =\left.(-t\cos t+\sin t)\right|_0^{2\kappa} =\sin(2\kappa)-2\kappa\cos(2\kappa).
\]
Therefore
\[
    \frac{\diff}{\diff \kappa} \left(1+\frac{2\kappa}{\sin(2\kappa)}\right) =\frac{2\bigl(\sin(2\kappa)-2\kappa\cos(2\kappa)\bigr)} {\sin^2(2\kappa)}>0.
\]
For the third factor, the identity \(\sin\kappa/\sin\theta_\alpha(\kappa)=\bigl(1+\sin^2(\alpha/2)\cot^2\kappa\bigr)^{-1/2}\) and \(\sin^2(\alpha/2)>0\) give strict monotonicity. Hence \(Q_\alpha\) in \eqref{eq:out-of-plane-Q} is continuous and strictly increasing on \((0,\pi/2)\).

Finally, as \(\kappa\downarrow0\), we have \(\theta_\alpha(\kappa)\to\alpha/2\), so the third factor tends to \(0\) while the first two stay bounded. As \(\kappa\uparrow\pi/2\), we have \(\theta_\alpha(\kappa)\to\pi/2\), so the first two factors tend to \(+\infty\) while the third tends to \(1\). Hence
\[
    \lim_{\kappa\downarrow0}Q_\alpha(\kappa)=0, \qquad \lim_{\kappa\uparrow\pi/2}Q_\alpha(\kappa)=+\infty.
\]
The intermediate value theorem and strict monotonicity therefore give exactly one solution of \(Q_\alpha(\kappa)=1\), and thus exactly one root of \(F_\alpha\) in \((0,\pi/2)\).
\end{proof}

\subsubsection{Reduction to three dimensions}\label{subsubsec:three-dimensional-reduction-proof}

The teacher terms in \eqref{eq:critical_eq} lie in the teacher plane, so the components of \eqref{eq:critical_eq} orthogonal to this plane involve only the students. These components give the dimension reduction in \cref{prop:3D}.

\begin{proof}[Proof of \cref{prop:3D}]
Projecting the two equations in \eqref{eq:critical_eq} onto the orthogonal complement of $\operatorname{span}\{v_1,v_2\}$ gives $a_1\bar w_1^\perp-(\pi-\gamma)\|w_2\|_2\bar w_2^\perp=0$ and $a_2\bar w_2^\perp-(\pi-\gamma)\|w_1\|_2\bar w_1^\perp=0.$ If $\gamma<\pi$, then $(\pi-\gamma)\|w_2\|_2>0$, so the first identity expresses $\bar w_2^\perp$ as a multiple of $\bar w_1^\perp$, possibly zero. If $\gamma=\pi$, then $\bar w_2=-\bar w_1$, so $\bar w_2^\perp=-\bar w_1^\perp$. In either case $\bar w_1^\perp$ and $\bar w_2^\perp$ are linearly dependent, and hence so are $w_1^\perp$ and $w_2^\perp$. Thus $w_1$ and $w_2$ lie in the span of the teacher plane and at most one further direction, which gives $\dim\operatorname{span}\{v_1,v_2,w_1,w_2\}\leq3$.
\end{proof}

\subsection{Three-dimensional symmetry reduction}

\subsubsection{The teacher force}

The angle between \(\bar v\) and each teacher is \(\alpha/2\), so
\[
    v_1=\cos(\alpha/2)\,\bar v+\sin(\alpha/2)\,v_\Delta, \qquad v_2=\cos(\alpha/2)\,\bar v-\sin(\alpha/2)\,v_\Delta.
\]
For a unit vector \(x\), we define its teacher force by
\[
    T(x) = \bigl(\pi-\angle(x,v_1)\bigr)v_1 + \bigl(\pi-\angle(x,v_2)\bigr)v_2.
\]
The teacher terms in the \(i\)th equation of \eqref{eq:critical_eq} are \(T(\bar w_i)\). The proof of \cref{lem:3Dsym} relies on the following comparison between the directions of \(T(x)\) and \(x\).

\begin{lemma}[Direction of the teacher force]\label{lem:teacher-force-azimuth}
Assume \(0<\alpha\leq\alpha_\star\), and let \(x\in\RR^d\) satisfy \(\|x\|_2=1\). Write $p=\langle x, \bar v\rangle$, $q=\langle x,v_\Delta\rangle$, and $x_P=p\bar v+qv_\Delta.$ Then \(\langle T(x),\bar v\rangle>0\) and \(\operatorname{sgn}\langle T(x),v_\Delta\rangle=\operatorname{sgn}q\). If \(p>0\) and \(q\neq0\), the teacher force lies strictly closer to the bisector than \(x_P\):
\[
    0< \frac{\langle T(x),v_\Delta\rangle}{\langle T(x),\bar v\rangle} <\frac{q}{p} \quad\text{if }q>0, \qquad \frac{q}{p}< \frac{\langle T(x),v_\Delta\rangle}{\langle T(x),\bar v\rangle} <0 \quad\text{if }q<0.
\]
\end{lemma}

\begin{proof}[Proof of \cref{lem:teacher-force-azimuth}]
We write \(H(t)=\pi-\arccos t\) for \(-1\leq t\leq1\). This function is strictly increasing and nonnegative, and it vanishes only at \(t=-1\). Since \(x\) and \(v_j\) are unit vectors, \(\pi-\angle(x,v_j)=H(\langle x,v_j\rangle)\) for \(j=1,2\), and the decomposition of \(v_1,v_2\) in the basis \(\bar v,v_\Delta\) gives
\[
    \langle x,v_1\rangle=p\cos(\alpha/2)+q\sin(\alpha/2), \qquad \langle x,v_2\rangle=p\cos(\alpha/2)-q\sin(\alpha/2).
\]
Substituting into the definition of \(T(x)\) and taking inner products with \(\bar v\) and \(v_\Delta\) yields
\[
\begin{aligned}
    \langle T(x),\bar v\rangle &= \cos(\alpha/2)\,\bigl(H(p\cos(\alpha/2)+q\sin(\alpha/2))+H(p\cos(\alpha/2)-q\sin(\alpha/2))\bigr),\\
    \langle T(x),v_\Delta\rangle &= \sin(\alpha/2)\,\bigl(H(p\cos(\alpha/2)+q\sin(\alpha/2))-H(p\cos(\alpha/2)-q\sin(\alpha/2))\bigr).
\end{aligned}
\]
Since \(\cos(\alpha/2)>0\) and \(H\geq0\), the first quantity is nonnegative. It vanishes only if \(\langle x,v_1\rangle=\langle x,v_2\rangle=-1\), that is, only if \(x=-v_1=-v_2\), which is impossible because \(v_1\neq v_2\). Hence \(\langle T(x),\bar v\rangle>0\). In the second quantity, the two arguments of \(H\) differ by \(2q\sin(\alpha/2)\), and \(\sin(\alpha/2)>0\), so the strict monotonicity of \(H\) shows that \(\langle T(x),v_\Delta\rangle\) has the sign of \(q\).

It remains to prove the bounds on the ratio when \(p>0\) and \(q\neq0\). By reflection across the \(\bar v\)-axis, we may assume \(q>0\). The lower bound then follows from the sign of \(\langle T(x),v_\Delta\rangle\), and since \(\langle T(x),\bar v\rangle>0\) and \(p>0\), the upper bound is equivalent to \(q\langle T(x),\bar v\rangle-p\langle T(x),v_\Delta\rangle>0\). We write \(x_P\) in polar form, \(x_P=\rho(\cos\varphi\,\bar v+\sin\varphi\,v_\Delta)\) with \(0<\rho\leq1\) and \(0<\varphi<\pi/2\), so that \(p=\rho\cos\varphi\) and \(q=\rho\sin\varphi\). Then
\[
    p\cos(\alpha/2)+q\sin(\alpha/2)=\rho\cos(\varphi-\alpha/2), \qquad p\cos(\alpha/2)-q\sin(\alpha/2)=\rho\cos(\varphi+\alpha/2).
\]
Substituting these identities into the formulas for \(\langle T(x),\bar v\rangle\) and \(\langle T(x),v_\Delta\rangle\) and using the angle-addition formulas, we obtain
\begin{align*}
    \frac{q\langle T(x),\bar v\rangle-p\langle T(x),v_\Delta\rangle}{\rho} &=\cos(\alpha/2)\Bigl[ H\bigl(\rho\cos(\varphi-\alpha/2)\bigr) +H\bigl(\rho\cos(\varphi+\alpha/2)\bigr) \Bigr]\sin\varphi \\
    &\quad-\sin(\alpha/2)\Bigl[ H\bigl(\rho\cos(\varphi-\alpha/2)\bigr) -H\bigl(\rho\cos(\varphi+\alpha/2)\bigr) \Bigr]\cos\varphi \\
    &=H\bigl(\rho\cos(\varphi-\alpha/2)\bigr) (\cos(\alpha/2)\sin\varphi-\sin(\alpha/2)\cos\varphi) \\
    &\quad+H\bigl(\rho\cos(\varphi+\alpha/2)\bigr) (\cos(\alpha/2)\sin\varphi+\sin(\alpha/2)\cos\varphi) \\
    &=H\bigl(\rho\cos(\varphi-\alpha/2)\bigr)\sin(\varphi-\alpha/2) + H\bigl(\rho\cos(\varphi+\alpha/2)\bigr)\sin(\varphi+\alpha/2).
\end{align*}
If $\varphi\geq\alpha/2$, both terms on the right are nonnegative and the second is positive, so the expression is positive.

Suppose now that $0<\varphi<\alpha/2$, and set $x_0=\alpha/2-\varphi$ and $y_0=\alpha/2+\varphi$, so that $0<x_0<y_0<\alpha<\pi$. Since $\sin(\varphi-\alpha/2)=-\sin x_0$ and $\sin(\varphi+\alpha/2)=\sin y_0$, the expression is positive exactly when
\begin{equation}
    H(\rho\cos x_0)\sin x_0 < H(\rho\cos y_0)\sin y_0.
    \label{eq:teacher-force-scalar-goal}
\end{equation}
To compare the two sides, we write \(R(\rho)=H(\rho\cos x_0)/H(\rho\cos y_0)\) and \(K(u)=u/\bigl(\sqrt{1-u^2}\,H(u)\bigr)\). For $-1<u<1$,
\[
    K'(u) = \frac{H(u)-u\sqrt{1-u^2}} {H(u)^2(1-u^2)^{3/2}}.
\]
The numerator is positive: for \(u<0\) both of its terms are positive, and for \(u\geq0\) we have \(H(u)\geq\pi/2>1/2\geq u\sqrt{1-u^2}\). Thus \(K\) is strictly increasing. Since \(\cos x_0>\cos y_0\), it follows that
\[
    \frac{\diff}{\diff\rho}\log R(\rho) = \frac{K(\rho\cos x_0)-K(\rho\cos y_0)}{\rho}>0,
\]
and therefore
\begin{equation}
    R(\rho) \leq R(1) = \frac{\pi-x_0}{\pi-y_0}.
    \label{eq:teacher-force-ratio-bound}
\end{equation}

It therefore suffices to show that \((\pi-x_0)\sin x_0<(\pi-y_0)\sin y_0\). Using \(x_0=\alpha/2-\varphi\) and \(y_0=\alpha/2+\varphi\), we have
\[
    (\pi-y_0)\sin y_0-(\pi-x_0)\sin x_0 = 2\bigl( (\pi-\alpha/2)\cos(\alpha/2)\sin\varphi - \varphi\sin(\alpha/2)\cos\varphi \bigr).
\]
By \cref{lem:squared-loss-angle-threshold}, \((\pi-\alpha/2)\cos(\alpha/2)\geq\sin(\alpha/2)\), so the right-hand side is at least \(2\sin(\alpha/2)(\sin\varphi-\varphi\cos\varphi)\), which is positive because \(\tan\varphi>\varphi\). Combined with \eqref{eq:teacher-force-ratio-bound}, this gives \eqref{eq:teacher-force-scalar-goal}, and hence the upper bound for \(q>0\). The case \(q<0\) follows by reflection.
\end{proof}

\subsubsection{Elimination of the splitting coordinate}\label{subsubsec:splitting-coordinate-elimination}

\begin{proof}[Proof of \cref{lem:3Dsym}]
By \cref{prop:3D}, there exists a unit vector \(\xi\perp\operatorname{span}\{\bar v,v_\Delta\}\) such that \(\bar w_1,\bar w_2\in\operatorname{span}\{\bar v,v_\Delta,\xi\}\). We write \(\bar w_i=p_i\bar v+q_iv_\Delta+\eta_i\xi\) and \(T_i=T(\bar w_i)\) for \(i=1,2\). Since one of the students lies outside the teacher plane, at least one \(\eta_i\) is nonzero. By \cref{lem:teacher-force-azimuth},
\begin{equation}
    \langle T_i,\bar v\rangle>0, \qquad \operatorname{sgn}\langle T_i,v_\Delta\rangle=\operatorname{sgn}q_i.
    \label{eq:three-dimensional-force-signs}
\end{equation}
Each \(T_i\) lies in the teacher plane, so it has no \(\xi\)-component.

We first show that \(0<\gamma<\pi\). If \(\gamma=0\), then \(\bar w_1=\bar w_2\), and the first equation in \eqref{eq:critical_eq} becomes \(T_1+(a_1-\pi\|w_2\|_2)\bar w_1=0\). Its \(\xi\)-coordinate shows that the coefficient of \(\bar w_1\) vanishes, so \(T_1=0\), which contradicts \eqref{eq:three-dimensional-force-signs}. If \(\gamma=\pi\), then \(\bar w_2=-\bar w_1\), both \(\eta_i\) are nonzero, and the equations become \(T_i+a_i\bar w_i=0\). Their \(\xi\)-coordinates give \(a_i=0\), and hence again \(T_i=0\). Thus
\begin{equation}
    0<\gamma<\pi.
    \label{eq:three-dimensional-gamma-range}
\end{equation}

We next show that \(\eta_1\eta_2\neq0\). If, for instance, \(\eta_1=0\), then \(\eta_2\neq0\), and the \(\xi\)-coordinate of the first equation in \eqref{eq:critical_eq} gives \((\pi-\gamma)\|w_2\|_2\eta_2=0\), which contradicts \eqref{eq:three-dimensional-gamma-range}. The case \(\eta_2=0\) is the same.

Let \(t=\eta_2/\eta_1\). The \(\xi\)-coordinates of the two equations in \eqref{eq:critical_eq} give \(a_1=(\pi-\gamma)\|w_2\|_2t\) and \(a_2=(\pi-\gamma)\|w_1\|_2/t\), and substituting these values back into \eqref{eq:critical_eq} yields
\begin{align}
    \bar w_2&=t \bar w_1+\frac{1}{(\pi-\gamma)\|w_2\|_2}T_1,\qquad T_2=-\frac{\|w_1\|_2}{t\|w_2\|_2}T_1.
\label{eq:three-dimensional-collinearity}
\end{align}
Since \(T_1\) and \(T_2\) both have positive \(\bar v\)-coordinate, the coefficient in the second identity is positive. Hence \(t<0\), and \(T_2\) is a positive multiple of \(T_1\). By \eqref{eq:three-dimensional-force-signs}, the numbers \(q_1\) and \(q_2\) are then either both zero or both nonzero with the same sign.

It remains to rule out the second possibility. Suppose that \(q_1\) and \(q_2\) are nonzero with the same sign. By reflection across the \(\bar v\)-axis, we may assume \(q_1,q_2>0\), so that \(\langle T_1,v_\Delta\rangle>0\). Taking the \(\bar v\)- and \(v_\Delta\)-coordinates of the first identity in \eqref{eq:three-dimensional-collinearity} gives
\begin{align}
    p_2&=t p_1+\frac{\langle T_1,\bar v\rangle}{(\pi-\gamma)\|w_2\|_2},\quad q_2=t q_1+\frac{\langle T_1,v_\Delta\rangle}{(\pi-\gamma)\|w_2\|_2}.
\label{eq:three-dimensional-coordinate-identities}
\end{align}
If \(p_1\leq0\), the first identity gives \(p_2>0\), because \(t<0\) and \(\langle T_1,\bar v\rangle>0\). After exchanging the students if necessary, we may therefore assume \(p_1>0\). Then \cref{lem:teacher-force-azimuth}, applied to \(\bar w_1\), gives \(\langle T_1,v_\Delta\rangle/\langle T_1,\bar v\rangle<q_1/p_1\), that is, \(q_1\langle T_1,\bar v\rangle-p_1\langle T_1,v_\Delta\rangle>0\). Substituting \eqref{eq:three-dimensional-coordinate-identities}, we obtain
\begin{align}
    q_2\langle T_1,\bar v\rangle-p_2\langle T_1,v_\Delta\rangle &=\langle T_1,\bar v\rangle\left(tq_1+\frac{\langle T_1,v_\Delta\rangle}{(\pi-\gamma)\|w_2\|_2}\right) -\langle T_1,v_\Delta\rangle\left(tp_1+\frac{\langle T_1,\bar v\rangle}{(\pi-\gamma)\|w_2\|_2}\right) \notag\\
    &=t(q_1\langle T_1,\bar v\rangle-p_1\langle T_1,v_\Delta\rangle)<0,
    \label{eq:three-dimensional-second-cross-negative}
\end{align}
where the last inequality uses \(t<0\). Since \(\langle T_1,\bar v\rangle\), \(\langle T_1,v_\Delta\rangle\), and \(q_2\) are positive, this inequality forces \(p_2>0\). We may then apply \cref{lem:teacher-force-azimuth} to \(\bar w_2\). Since \(T_2\) is a positive multiple of \(T_1\), it gives
\[
    \frac{\langle T_1,v_\Delta\rangle}{\langle T_1,\bar v\rangle} =\frac{\langle T_2,v_\Delta\rangle}{\langle T_2,\bar v\rangle} <\frac{q_2}{p_2},
\]
that is, \(q_2\langle T_1,\bar v\rangle-p_2\langle T_1,v_\Delta\rangle>0\). This contradicts \eqref{eq:three-dimensional-second-cross-negative}. Hence \(q_1=q_2=0\), and \(w_1,w_2\in\operatorname{span}\{\bar v,\xi\}\).
\end{proof}

\subsection{Out-of-plane stationary points}\label{subsec:out-of-plane-stationary-points}

By \cref{lem:3Dsym}, a stationary point with a student outside the teacher plane lies in \(\operatorname{span}\{\bar v,\xi\}\). For such points, we first show that the two students are mirror images across the teacher plane, then reduce stationarity to the equation \(F_\alpha(\kappa)=0\), and finally find a direction of negative curvature. Together these steps prove \cref{lem:critical_point_perp_plane}.

\subsubsection{A scalar inequality}

In the proof of \cref{lem:perp_plane_symmetry} below, an out-of-plane stationary point that is not symmetric across the teacher plane leads to two scalar identities. The following inequality shows that these identities are incompatible.

\begin{lemma}[A scalar inequality]\label{lem:scalar_separation_bisector}
Let \(0<t<m<\pi/2\), and put
\[
    X=t+\frac{2(\pi-2m)}{2m-\sin 2m}\bigl(t-\cos^2m\tan t\bigr), \qquad Y=t+\frac{2(\pi-2m)}{2m+\sin 2m}\bigl(t+\sin^2m\cot t\bigr).
\]
If \(\lambda>0\) satisfies \(\lambda X=\pi/2+\arctan(\lambda\tan t/\tan^2m)\), then \(\lambda X>Y/\lambda\).
\end{lemma}

\begin{proof}[Proof of \cref{lem:scalar_separation_bisector}]
We fix \(m\) and write \(a=m-\sin m\cos m\), \(b=m+\sin m\cos m\), \(P(t)=t-\cos^2m\tan t\), and \(Q(t)=t+\sin^2m\cot t\). Then \(a,b>0\), \(X=t+(\pi-2m)P(t)/a\), and \(Y=t+(\pi-2m)Q(t)/b\). We have \(X>0\) and \(Y>0\), since \(P(t)>t-\cos^2t\tan t=t-\sin t\cos t>0\) and \(Q(t)>0\).

We first collect three bounds on \(P\) and \(Q\). Since \(P'(t)=1-\cos^2m\sec^2t>0\) and \(Q'(t)=1-\sin^2m\csc^2t<0\) on \(0<t<m\), and since \(P(m)=a\) and \(Q(m)=b\), we have \(0<P(t)<a\) and \(Q(t)>b\). The product \(P(t)Q(t)\) is also increasing on \((0,m)\), because
\begin{align*}
    \frac{\diff}{\diff t}\bigl(P(t)Q(t)\bigr) &= \frac{\tan^2 m-\tan^2 t}{1+\tan^2 m} \left[ Q(t)-\frac{P(t)}{\tan^2 t} \right] = \frac{\tan^2 m-\tan^2 t}{1+\tan^2 m} \frac{\tan t-t+t\tan^2 t}{\tan^2 t} >0.
\end{align*}
Hence \(P(t)Q(t)<P(m)Q(m)=ab\). Finally, the function \(t(P(t)/a+Q(t)/b)\) equals \(2m\) at \(t=m\), and we claim that it is increasing on \((0,m)\), so that \(t(P(t)/a+Q(t)/b)<2m\). Its derivative is
\begin{equation}\label{eq:scalar_bisector_tPQ_derivative}
    \frac{\diff}{\diff t}\left[t\left(\frac{P(t)}{a}+\frac{Q(t)}{b}\right)\right] = \frac{2t-\cos^2m(\tan t+t\sec^2t)}{a} + \frac{2t+\sin^2m(\cot t-t\csc^2t)}{b}.
\end{equation}
For the first numerator in \eqref{eq:scalar_bisector_tPQ_derivative}, the inequalities $0<t<m<\pi/2$ give $\cos^2m<\cos^2t$. Since $\tan t+t\sec^2t>0$, it follows that
\[
\begin{aligned}
    2t-\cos^2m(\tan t+t\sec^2t) &>2t-\cos^2t(\tan t+t\sec^2t)=t-\sin t\cos t>0.
\end{aligned}
\]

For the second numerator, the inequality $t-\sin t\cos t>0$ gives \(\cot t-t\csc^2t=(\sin t\cos t-t)/\sin^2t<0\). Since $0<\sin^2m<1$, multiplying this negative quantity by $\sin^2m$ makes it strictly larger. Therefore
\[
    2t+\sin^2m(\cot t-t\csc^2t) > 2t+\cot t-t\csc^2t = \frac{\sin t\cos t-t\cos 2t}{\sin^2t} >0.
\]
For the last inequality, the numerator vanishes at $t=0$ and its derivative is $2t\sin 2t>0$ on $(0,\pi/2)$. Both numerators in \eqref{eq:scalar_bisector_tPQ_derivative} are therefore positive, which proves the claim.

Expanding \(XY\) and using \(P(t)Q(t)<ab\) and \(t(P(t)/a+Q(t)/b)<2m\), we obtain
\begin{align*}
    (\pi-m)^2-XY &= (\pi-m)^2 - \left(t+(\pi-2m)\frac{P(t)}{a}\right) \left(t+(\pi-2m)\frac{Q(t)}{b}\right)\\
    &= m^2-t^2 + (\pi-2m)\left[ 2m - t\left(\frac{P(t)}{a}+\frac{Q(t)}{b}\right) \right]+ (\pi-2m)^2\left(1-\frac{P(t)Q(t)}{ab}\right)\\
    &>0.
\end{align*}
Therefore \(\sqrt{XY}<\pi-m\).

We also need the bound
\begin{equation}\label{eq:scalar_bisector_ratio_bound}
    \frac{t(aQ(t)-bP(t))}{ab-P(t)Q(t)}<\frac{\pi}{2},
\end{equation}
in which the denominator is positive and so is the numerator, since \(aQ(t)-bP(t)>ab-bP(t)=b(a-P(t))>0\). Because \(b=m+\frac{1}{2}\sin 2m<\pi/2\), it is enough to show that the left-hand side is less than \(b\), that is,
\[
    b(ab-P(t)Q(t))-t(aQ(t)-bP(t))>0.
\]
A direct expansion gives
\begin{equation}\label{eq:scalar_bisector_positive_expansion}
    b(ab-P(t)Q(t))-t(aQ(t)-bP(t)) = bm^2-at^2-2m\sin^2m\,t\cot t.
\end{equation}
At \(t=m\), the right-hand side of \eqref{eq:scalar_bisector_positive_expansion} is zero. Its derivative with respect to \(t\) is
\begin{equation}\label{eq:scalar_bisector_positive_derivative}
    2t\left[ m\sin^2m\,\frac{t-\sin t\cos t}{t\sin^2t} - a \right].
\end{equation}
For \(0<u<\pi/2\), the function \(u\mapsto (u-\sin u\cos u)/(u\sin^2u)\) is strictly increasing, since
\[
    \frac{\diff}{\diff u} \left[ \frac{u-\sin u\cos u}{u\sin^2u} \right] = \frac{\cos u}{\sin^3u} \left[ \frac{\tan u}{u} + \left(\frac{\sin u}{u}\right)^2 -2 \right] >0
\]
by Wilker's inequality \citep{sumner1991inequalities}. Since \(t<m\), the bracketed factor \(m\sin^2m\,(t-\sin t\cos t)/(t\sin^2t)-a\) in \eqref{eq:scalar_bisector_positive_derivative} is strictly less than \(m\sin^2m\,\frac{m-\sin m\cos m}{m\sin^2m}-a=0\). Thus the derivative in \eqref{eq:scalar_bisector_positive_derivative} is negative on \((0,m)\). Since the right-hand side of \eqref{eq:scalar_bisector_positive_expansion} vanishes at \(t=m\), it is positive for \(0<t<m\), which proves \eqref{eq:scalar_bisector_ratio_bound}.

The conclusion of the lemma will follow from the inequality
\begin{equation}\label{eq:scalar_bisector_arctan_goal}
    \sqrt{XY} + \arctan\left( \frac{\tan^2 m}{\tan t}\sqrt{\frac{X}{Y}} \right) <\pi .
\end{equation}
If \(\sqrt{XY}\le \pi/2\), this holds because the arctangent is less than \(\pi/2\). We may therefore assume \(\sqrt{XY}>\pi/2\). Since \(\sqrt{XY}<\pi-m\), we then have \(m<\pi-\sqrt{XY}<\pi/2\). The tangent function is strictly convex on \((0,\pi/2)\), so \(\tan u>\tan m+\sec^2m\,(u-m)\) for every \(u\in(0,\pi/2)\) with \(u\neq m\). Taking \(u=\pi-\sqrt{XY}\) and using \(\sec^2m=1+\tan^2 m\) gives \(\tan(\pi-\sqrt{XY})>\tan m+(1+\tan^2 m)(\pi-\sqrt{XY}-m)\). Hence \eqref{eq:scalar_bisector_arctan_goal} follows once we show
\begin{equation}\label{eq:scalar_bisector_tangent_target}
    \frac{\tan^2 m}{\tan t}\sqrt{\frac{X}{Y}} \le \tan m+(1+\tan^2 m)(\pi-\sqrt{XY}-m).
\end{equation}
Multiplying this inequality by \(\cos^2m\) and using \(\sqrt{XY}=Y\sqrt{X/Y}\), we can rewrite it as \(\sqrt{X/Y}\,(Y+\sin^2m\cot t)\le\pi-m+\sin m\cos m\). Since \(Y+\sin^2m\cot t\) equals \(t+\sin^2m\cot t\) times \((\pi-m+\sin m\cos m)/(m+\sin m\cos m)\), this inequality is equivalent to
\begin{equation}\label{eq:scalar_bisector_core_ineq}
    X\bigl(t+\sin^2m\cot t\bigr)^2 \le Y(m+\sin m\cos m)^2.
\end{equation}
To prove \eqref{eq:scalar_bisector_core_ineq} when \(\sqrt{XY}>\pi/2\), we fix \(t\) and \(m\) and set, for \(\mu\ge \pi-2m\), \(X_\mu=t+\mu P(t)/a\) and \(Y_\mu=t+\mu Q(t)/b\). At \(\mu=\pi-2m\), these are the original quantities \(X\) and \(Y\), and \eqref{eq:scalar_bisector_core_ineq} reads \(Y_{\pi-2m}b^2-X_{\pi-2m}Q(t)^2\geq0\). The expression \(Y_\mu b^2-X_\mu Q(t)^2\) is affine in \(\mu\), with slope
\[
    \frac{\diff}{\diff \mu} \left( Y_\mu b^2-X_\mu Q(t)^2 \right) = \frac{Q(t)(ab-P(t)Q(t))}{a} >0.
\]
Suppose that \eqref{eq:scalar_bisector_core_ineq} fails. Then this affine expression is negative at \(\mu=\pi-2m\), so it vanishes at exactly one \(\mu_*>\pi-2m\), where \(X_{\mu_*}Q(t)^2=Y_{\mu_*}b^2\). Solving this equation gives \(\mu_*=ta\bigl(Q(t)^2-b^2\bigr)/\bigl(Q(t)(ab-P(t)Q(t))\bigr)\), and therefore
\[
    \sqrt{X_{\mu_*}Y_{\mu_*}} = \frac{b}{Q(t)}\,Y_{\mu_*} = \frac{tb}{Q(t)}+\mu_* = t\,\frac{aQ(t)-bP(t)}{ab-P(t)Q(t)} <\frac{\pi}{2}
\]
by \eqref{eq:scalar_bisector_ratio_bound}. Since \(X_\mu\) and \(Y_\mu\) increase with \(\mu\), this gives \(\sqrt{XY}<\sqrt{X_{\mu_*}Y_{\mu_*}}<\pi/2\), contrary to our assumption. Hence \eqref{eq:scalar_bisector_core_ineq} holds, and so does \eqref{eq:scalar_bisector_tangent_target}. Together with the tangent-line bound, it gives \((\tan^2m/\tan t)\sqrt{X/Y}<\tan(\pi-\sqrt{XY})\), and since \(0<\pi-\sqrt{XY}<\pi/2\), taking arctangents proves \eqref{eq:scalar_bisector_arctan_goal}.

We now derive the lemma from \eqref{eq:scalar_bisector_arctan_goal}. Since \(\arctan r+\arctan(1/r)=\pi/2\) for \(r>0\), the inequality \eqref{eq:scalar_bisector_arctan_goal} can be written as
\[
    \sqrt{XY} < \frac{\pi}{2} + \arctan\left(\frac{\tan t}{\tan^2 m}\sqrt{\frac{Y}{X}}\right),
\]
and dividing by \(\sqrt{Y/X}\) gives
\begin{equation}\label{eq:scalar_bisector_quotient_lower}
    \frac{ \frac{\pi}{2} + \arctan\big(\frac{\tan t}{\tan^2 m}\sqrt{\frac{Y}{X}}\big) }{ \sqrt{Y/X} } >X.
\end{equation}
For fixed \(t\) and \(m\), the function \(r\mapsto\bigl(\pi/2+\arctan(r\tan t/\tan^2 m)\bigr)/r\) on \(r>0\) is strictly decreasing, because its derivative is
\[
    \frac{ \frac{r\tan t/\tan^2 m}{1+(r\tan t/\tan^2 m)^2} - \left(\frac{\pi}{2}+\arctan(r\tan t/\tan^2 m)\right) }{r^2}<0.
\]
By the hypothesis of the lemma, this function takes the value \(X\) at \(r=\lambda\), while by \eqref{eq:scalar_bisector_quotient_lower} its value at \(r=\sqrt{Y/X}\) is larger than \(X\). Hence \(\lambda>\sqrt{Y/X}\), and consequently \(\lambda X>\sqrt{XY}>Y/\lambda\).

\end{proof}

\subsubsection{Mirror symmetry}

We now use \cref{lem:scalar_separation_bisector} to show that the two students are mirror images across the teacher plane.

\begin{lemma}[Out-of-plane symmetry]\label{lem:perp_plane_symmetry}
Let $\xi$ be any unit vector orthogonal to $\operatorname{span}\{v_1,v_2\}$. Assume $w_1,w_2\neq0$ satisfy \eqref{eq:critical_eq}, lie in $\operatorname{span}\{\bar v,\xi\}$, and do not both lie in the teacher plane. Then, after possibly replacing $\xi$ by $-\xi$, there exist $r>0$ and $\kappa\in(0,\pi/2)$ such that
\[
    w_1=r(\cos\kappa\,\bar v+\sin\kappa\,\xi), \qquad w_2=r(\cos\kappa\,\bar v-\sin\kappa\,\xi).
\]
\end{lemma}

\begin{proof}[Proof of \cref{lem:perp_plane_symmetry}]
We use coordinates with respect to the orthonormal basis \(\bar v,v_\Delta,\xi\), so that \(v_1=(\cos(\alpha/2),\sin(\alpha/2),0)\) and \(v_2=(\cos(\alpha/2),-\sin(\alpha/2),0)\). Since \(w_1,w_2\in\operatorname{span}\{\bar v,\xi\}\), we can write \(\bar w_i=(x_i,0,z_i)\). Then \(\langle\bar w_i,v_1\rangle=\langle\bar w_i,v_2\rangle\), so each student makes the same angle with both teachers; we denote this angle by \(\theta_i\). The system \eqref{eq:critical_eq} becomes
\begin{equation}\label{eq:perp_analytic_critical_eq}
    (\pi-\theta_i)(v_1+v_2) + \bigl(2\sin\theta_i-\pi \|w_i\|_2-\sin\gamma\,\|w_j\|_2\bigr)\bar w_i - (\pi-\gamma)\|w_j\|_2\bar w_j=0, \qquad \{i,j\}=\{1,2\}.
\end{equation}
We first show that \(z_1z_2\neq0\). By hypothesis, \(z_1\) and \(z_2\) are not both zero. If \(z_1=0\) and \(z_2\neq0\), then \(\bar w_1\) and \(\bar w_2\) are neither parallel nor antiparallel, so \(0<\gamma<\pi\), and the third coordinate of \eqref{eq:perp_analytic_critical_eq} for \(i=1\) gives \((\pi-\gamma)\|w_2\|_2z_2=0\), a contradiction. The case \(z_2=0\) is the same.

We next show that \(0<\gamma<\pi\). If \(\gamma\) were \(0\) or \(\pi\), then \(\bar w_2=\pm\bar w_1\), and \eqref{eq:perp_analytic_critical_eq} for \(i=1\) would express \((\pi-\theta_1)(v_1+v_2)\), a nonzero multiple of \(\bar v\), as a multiple of \(\bar w_1\). This is impossible because \(z_1\neq0\).

To remove the term in \(\bar w_i\) from the \(i\)th equation, we take its inner product with \((z_i,0,-x_i)\), which is orthogonal to \(\bar w_i\). This gives
\begin{align}
    2\cos(\alpha/2)(\pi-\theta_1)z_1 +(\pi-\gamma)\|w_2\|_2(x_1z_2-z_1x_2)&=0,\notag\\
    2\cos(\alpha/2)(\pi-\theta_2)z_2 -(\pi-\gamma)\|w_1\|_2(x_1z_2-z_1x_2)&=0.
\label{eq:perp_analytic_determinants}
\end{align}
The coefficients \(2\cos(\alpha/2)(\pi-\theta_i)\) and \((\pi-\gamma)\|w_j\|_2\) are positive. The first identity therefore gives \(\operatorname{sgn}(x_1z_2-z_1x_2)=-\operatorname{sgn}z_1\), and the second gives \(\operatorname{sgn}(x_1z_2-z_1x_2)=\operatorname{sgn}z_2\), so \(z_1\) and \(z_2\) have opposite signs. Replacing \(\xi\) by \(-\xi\) if necessary, we may assume \(z_1>0>z_2\) and write \(\bar w_1=(\cos p,0,\sin p)\) and \(\bar w_2=(\cos q,0,-\sin q)\) with \(p,q\in(0,\pi)\). Then \(x_1z_2-z_1x_2=-\sin(p+q)\), which is negative by the first identity in \eqref{eq:perp_analytic_determinants}. Hence \(p+q<\pi\) and \(\gamma=p+q\), and \eqref{eq:perp_analytic_determinants} becomes
\begin{align}
    (\pi-p-q)\sin(p+q)\|w_2\|_2 &= 2\cos(\alpha/2) (\pi-\theta_1)\sin p,\notag\\
    (\pi-p-q)\sin(p+q)\|w_1\|_2 &= 2\cos(\alpha/2) (\pi-\theta_2)\sin q.
\label{eq:perp_analytic_tangent_projection}
\end{align}

The teacher angles satisfy \(\cos\theta_1=\cos(\alpha/2)\cos p\) and \(\cos\theta_2=\cos(\alpha/2)\cos q\). Taking the inner product of the \(i\)th equation in \eqref{eq:perp_analytic_critical_eq} with \(\bar w_i\) gives
\begin{align}
    \pi \|w_1\|_2 +\bigl(\sin(p+q)+(\pi-p-q)\cos(p+q)\bigr)\|w_2\|_2 &= 2\bigl(\sin\theta_1+(\pi-\theta_1)\cos\theta_1\bigr),\notag\\
    \bigl(\sin(p+q)+(\pi-p-q)\cos(p+q)\bigr)\|w_1\|_2+\pi \|w_2\|_2 &= 2\bigl(\sin\theta_2+(\pi-\theta_2)\cos\theta_2\bigr).
\label{eq:perp_analytic_inner_products}
\end{align}
Eliminating \(\|w_1\|_2,\|w_2\|_2\) from \eqref{eq:perp_analytic_inner_products} and \eqref{eq:perp_analytic_tangent_projection} gives
\begin{align}
&(\pi-p-q)\sin(p+q) \bigl(\sin\theta_1+(\pi-\theta_1)\cos\theta_1\bigr) \notag\\
&\qquad= \bigl(\sin(p+q)+(\pi-p-q)\cos(p+q)\bigr) \cos(\alpha/2)(\pi-\theta_1)\sin p +\pi\cos(\alpha/2)(\pi-\theta_2)\sin q,
\label{eq:perp_analytic_scalar_1}
\end{align}
\begin{align}
&(\pi-p-q)\sin(p+q) \bigl(\sin\theta_2+(\pi-\theta_2)\cos\theta_2\bigr) \notag\\
&\qquad= \pi\cos(\alpha/2)(\pi-\theta_1)\sin p +\bigl(\sin(p+q)+(\pi-p-q)\cos(p+q)\bigr) \cos(\alpha/2)(\pi-\theta_2)\sin q.
\label{eq:perp_analytic_scalar_2}
\end{align}

We now show that \eqref{eq:perp_analytic_scalar_1} and \eqref{eq:perp_analytic_scalar_2} force \(p=q\). Suppose instead that \(p\neq q\); by symmetry we may assume \(p>q\), so that
\begin{equation}\label{eq:perp_analytic_pq_range}
    0<\frac{p-q}{2}<\frac{p+q}{2}<\frac{\pi}{2}.
\end{equation}
The formula \eqref{eq:out-of-plane-angle} for \(\theta_\alpha(\kappa)\) makes sense for every \(\kappa\in(0,\pi)\), and \(\theta_1=\theta_\alpha(p)\), \(\theta_2=\theta_\alpha(q)\). Direct differentiation gives
\[
    \theta_\alpha'(\kappa) = \frac{\cos(\alpha/2)\sin\kappa} {\sqrt{1-\cos^2(\alpha/2)\cos^2\kappa}}>0, \qquad 1-\theta_\alpha'(\kappa)^2 = \frac{\sin^2(\alpha/2)} {1-\cos^2(\alpha/2)\cos^2\kappa}>0,
\]
so \(0<\theta_\alpha'(\kappa)<1\) on \((0,\pi)\). By the mean value theorem, there exists $\zeta\in(q,p)$ such that \(\theta_1-\theta_2=\theta_\alpha(p)-\theta_\alpha(q)=\theta_\alpha'(\zeta)(p-q)\). Since $0<\theta_\alpha'(\zeta)<1$ and $p>q$, it follows that
\begin{equation}\label{eq:perp_analytic_teacher_angle_difference_range}
    0<\frac{\theta_1-\theta_2}{2}<\frac{p-q}{2}.
\end{equation}
Moreover, \(\cos\theta_1+\cos\theta_2=\cos(\alpha/2)(\cos p+\cos q)>0\) because \(p+q<\pi\), so \((\theta_1+\theta_2)/2<\pi/2\).

Adding and subtracting the identities \(\cos\theta_1=\cos(\alpha/2)\cos p\) and \(\cos\theta_2=\cos(\alpha/2)\cos q\), we obtain
\begin{equation}\label{eq:perp_analytic_cos_sin_split}
    \begin{aligned}
    \cos\frac{\theta_1+\theta_2}{2} \cos\frac{\theta_1-\theta_2}{2} &= \cos(\alpha/2) \cos\frac{p+q}{2}\cos\frac{p-q}{2},\\
    \sin\frac{\theta_1+\theta_2}{2} \sin\frac{\theta_1-\theta_2}{2} &= \cos(\alpha/2) \sin\frac{p+q}{2}\sin\frac{p-q}{2}.
    \end{aligned}
\end{equation}
Hence
\[
    \tan\frac{\theta_1+\theta_2}{2} \tan\frac{\theta_1-\theta_2}{2} = \tan\frac{p+q}{2}\tan\frac{p-q}{2}.
\]
Since \(0<(\theta_1+\theta_2)/2<\pi/2\), this gives
\begin{equation}\label{eq:perp_analytic_E_geometric}
    \pi-\frac{\theta_1+\theta_2}{2} = \frac{\pi}{2} + \arctan\left( \frac{\tan\frac{\theta_1-\theta_2}{2}} {\tan\frac{p+q}{2}\tan\frac{p-q}{2}} \right).
\end{equation}

We now subtract \eqref{eq:perp_analytic_scalar_2} from \eqref{eq:perp_analytic_scalar_1} and write \(p,q,\theta_1,\theta_2\) in terms of the half-sums \((p+q)/2\), \((\theta_1+\theta_2)/2\) and the half-differences \((p-q)/2\), \((\theta_1-\theta_2)/2\). This gives
\begin{align*}
&(\pi-p-q)\sin(p+q) \Bigl[ 2\cos\frac{\theta_1+\theta_2}{2} \sin\frac{\theta_1-\theta_2}{2} -2\left(\pi-\frac{\theta_1+\theta_2}{2}\right) \sin\frac{\theta_1+\theta_2}{2} \sin\frac{\theta_1-\theta_2}{2} \\
&\qquad -(\theta_1-\theta_2) \cos\frac{\theta_1+\theta_2}{2} \cos\frac{\theta_1-\theta_2}{2} \Bigr] \\
&\quad = \cos(\alpha/2) \bigl(\sin(p+q)+(\pi-p-q)\cos(p+q)-\pi\bigr) \Bigl[ 2\left(\pi-\frac{\theta_1+\theta_2}{2}\right) \cos\frac{p+q}{2}\sin\frac{p-q}{2} \\
&\qquad -(\theta_1-\theta_2) \sin\frac{p+q}{2}\cos\frac{p-q}{2} \Bigr].
\end{align*}
Using \eqref{eq:perp_analytic_cos_sin_split} in this identity and collecting the terms containing \(\pi-(\theta_1+\theta_2)/2\), we get
\[
\begin{aligned}
&\left(\pi-\frac{\theta_1+\theta_2}{2}\right) \bigl(p+q-\sin(p+q)\bigr) \cos\frac{p+q}{2}\tan\frac{p-q}{2} \\
&\quad = \sin\frac{p+q}{2} \biggl[ \bigl(p+q-\sin(p+q)\bigr) \frac{\theta_1-\theta_2}{2} +2(\pi-p-q) \left( \frac{\theta_1-\theta_2}{2} -\cos^2\frac{p+q}{2}\tan\frac{\theta_1-\theta_2}{2} \right) \biggr].
\end{aligned}
\]
Therefore
\begin{equation}\label{eq:perp_analytic_E_difference}
\begin{aligned}
    \pi-\frac{\theta_1+\theta_2}{2} &= \tan\frac{p+q}{2}\cot\frac{p-q}{2} \biggl[ \frac{\theta_1-\theta_2}{2} + \frac{2(\pi-p-q)}{p+q-\sin(p+q)} \left( \frac{\theta_1-\theta_2}{2} -\cos^2\frac{p+q}{2}\tan\frac{\theta_1-\theta_2}{2} \right) \biggr].
\end{aligned}
\end{equation}
Adding \eqref{eq:perp_analytic_scalar_1} and \eqref{eq:perp_analytic_scalar_2} instead, we obtain in the same way
\begin{align*}
&(\pi-p-q)\sin(p+q) \Bigl[ 2\sin\frac{\theta_1+\theta_2}{2} \cos\frac{\theta_1-\theta_2}{2} +2\left(\pi-\frac{\theta_1+\theta_2}{2}\right) \cos\frac{\theta_1+\theta_2}{2} \cos\frac{\theta_1-\theta_2}{2} \\
&\qquad +(\theta_1-\theta_2) \sin\frac{\theta_1+\theta_2}{2} \sin\frac{\theta_1-\theta_2}{2} \Bigr] \\
&\quad = \cos(\alpha/2) \bigl(\sin(p+q)+(\pi-p-q)\cos(p+q)+\pi\bigr) \Bigl[ 2\left(\pi-\frac{\theta_1+\theta_2}{2}\right) \sin\frac{p+q}{2}\cos\frac{p-q}{2} \\
&\qquad -(\theta_1-\theta_2) \cos\frac{p+q}{2}\sin\frac{p-q}{2} \Bigr].
\end{align*}
Using \eqref{eq:perp_analytic_cos_sin_split} again and collecting \(\pi-(\theta_1+\theta_2)/2\), we obtain
\[
\begin{aligned}
&\left(\pi-\frac{\theta_1+\theta_2}{2}\right) \bigl(p+q+\sin(p+q)\bigr) \sin\frac{p+q}{2}\cot\frac{p-q}{2} \\
&\quad = \cos\frac{p+q}{2} \biggl[ \bigl(p+q+\sin(p+q)\bigr) \frac{\theta_1-\theta_2}{2} +2(\pi-p-q) \left( \frac{\theta_1-\theta_2}{2} +\sin^2\frac{p+q}{2}\cot\frac{\theta_1-\theta_2}{2} \right) \biggr].
\end{aligned}
\]
Thus
\begin{equation}\label{eq:perp_analytic_E_sum}
\begin{aligned}
    \pi-\frac{\theta_1+\theta_2}{2} &= \cot\frac{p+q}{2}\tan\frac{p-q}{2} \biggl[ \frac{\theta_1-\theta_2}{2} + \frac{2(\pi-p-q)}{p+q+\sin(p+q)} \\
    &\qquad\qquad\qquad\quad \times \left( \frac{\theta_1-\theta_2}{2} +\sin^2\frac{p+q}{2}\cot\frac{\theta_1-\theta_2}{2} \right) \biggr].
\end{aligned}
\end{equation}

We now apply \cref{lem:scalar_separation_bisector} with \(t=(\theta_1-\theta_2)/2\), \(m=(p+q)/2\), and \(\lambda=\tan\frac{p+q}{2}\big/\tan\frac{p-q}{2}\). By \eqref{eq:perp_analytic_pq_range} and \eqref{eq:perp_analytic_teacher_angle_difference_range}, we have \(0<t<m<\pi/2\). In the notation of the lemma, \eqref{eq:perp_analytic_E_difference} says that \(\pi-(\theta_1+\theta_2)/2=\lambda X\), and \eqref{eq:perp_analytic_E_sum} says that \(\pi-(\theta_1+\theta_2)/2=Y/\lambda\). Combined with \eqref{eq:perp_analytic_E_geometric}, the first identity shows that \(\lambda\) satisfies the hypothesis of the lemma, which then gives \(\lambda X>Y/\lambda\). This contradicts \(\lambda X=\pi-(\theta_1+\theta_2)/2=Y/\lambda\). Hence \(p=q\).

Thus \(\bar w_1=(\cos p,0,\sin p)\) and \(\bar w_2=(\cos p,0,-\sin p)\) with \(0<p<\pi/2\), so that \(\theta_1=\theta_2\) and \(\gamma=2p\). It remains to show that \(\|w_1\|_2=\|w_2\|_2\). Dividing the third coordinates of \eqref{eq:perp_analytic_critical_eq} by \(\sin p\) gives
\[
    \begin{aligned}
    0 &= 2\sin\theta_1-\pi \|w_1\|_2+\bigl((\pi-\gamma)-\sin\gamma\bigr)\|w_2\|_2,\\
    0 &= 2\sin\theta_1-\pi \|w_2\|_2+\bigl((\pi-\gamma)-\sin\gamma\bigr)\|w_1\|_2.
    \end{aligned}
\]
Subtracting these two equations yields \((2\pi-\gamma-\sin\gamma)(\|w_1\|_2-\|w_2\|_2)=0\), and the coefficient is positive because \(0<\gamma<\pi\). Hence \(\|w_1\|_2=\|w_2\|_2\), and the lemma follows with \(r=\|w_1\|_2\) and \(\kappa=p\).

\end{proof}

\subsubsection{Exact scalar characterization}

By \cref{lem:perp_plane_symmetry}, an out-of-plane stationary point is determined by a common norm \(r\) and an angle \(\kappa\) from \(\bar v\). It remains to find the pairs \((r,\kappa)\) that give stationary points. We recall \(\theta_\alpha\), \(r_\alpha\), and \(F_\alpha\) from \eqref{eq:out-of-plane-angle}--\eqref{eq:out-of-plane-root-function}, and that \(\kappa_\alpha\) is the unique root of \(F_\alpha\) in \((0,\pi/2)\) by \cref{lem:out-of-plane-root-uniqueness}.

\begin{proof}[Proof of the characterization in \cref{lem:critical_point_perp_plane}]
By \cref{lem:perp_plane_symmetry}, after possibly swapping the students or replacing \(\xi\) by \(-\xi\), every pair in question has the form
\begin{equation}\label{eq:out-of-plane-mirror-form}
    w_1=r(\cos\kappa\,\bar v+\sin\kappa\,\xi), \qquad w_2=r(\cos\kappa\,\bar v-\sin\kappa\,\xi)
\end{equation}
for some \(r>0\) and \(0<\kappa<\pi/2\), and every such pair lies in \(\operatorname{span}\{\bar v,\xi\}\) outside the teacher plane. For the pair \eqref{eq:out-of-plane-mirror-form}, all four teacher--student angles equal \(\theta_\alpha(\kappa)\) by \eqref{eq:out-of-plane-angle}, and the angle between the students is \(\gamma=2\kappa\). We write
\[
\begin{aligned}
    R(r,\kappa) &= 2\sin\theta_\alpha(\kappa)-r\bigl(\sin(2\kappa)+2\kappa\bigr),\\
    G_\alpha(r,\kappa) &= \bigl(\pi-\theta_\alpha(\kappa)\bigr)\cos(\alpha/2) -(\pi-2\kappa)r\cos\kappa.
\end{aligned}
\]
Substituting \eqref{eq:out-of-plane-mirror-form} into \eqref{eq:critical_eq} reduces its two vector equations to
\begin{align*}
    \bigl(2G_\alpha(r,\kappa)+\cos\kappa\,R(r,\kappa)\bigr)\bar v + \sin\kappa\,R(r,\kappa)\xi &=0, \\
    \bigl(2G_\alpha(r,\kappa)+\cos\kappa\,R(r,\kappa)\bigr)\bar v - \sin\kappa\,R(r,\kappa)\xi &=0.
\end{align*}
Since \(\sin\kappa>0\), the pair is stationary if and only if
\[
    R(r,\kappa)=0, \qquad G_\alpha(r,\kappa)=0.
\]
Since \(\sin(2\kappa)+2\kappa>0\), the first equation means \(r=2\sin\theta_\alpha(\kappa)/(\sin(2\kappa)+2\kappa)=r_\alpha(\kappa)\). With this value of \(r\), the second equation reads
\[
    0 =G_\alpha(r_\alpha(\kappa),\kappa) =\bigl(\pi-\theta_\alpha(\kappa)\bigr)\cos(\alpha/2) -(\pi-2\kappa)r_\alpha(\kappa)\cos\kappa =F_\alpha(\kappa).
\]
By \cref{lem:out-of-plane-root-uniqueness}, the pair is therefore stationary if and only if \(\kappa=\kappa_\alpha\) and \(r=r_\alpha(\kappa_\alpha)\).
\end{proof}

\subsubsection{Negative curvature}

To finish the proof of \cref{lem:critical_point_perp_plane}, we move the two students toward different teachers, along the direction \((v_\Delta,-v_\Delta)\), and show that the loss has negative curvature in this direction.

\begin{proof}[Proof of the strict-saddle claim in \cref{lem:critical_point_perp_plane}]
Fix a stationary point \eqref{eq:out-of-plane-mirror-form} with \(\kappa=\kappa_\alpha\) and \(r=r_\alpha(\kappa_\alpha)\), and consider the perturbation \((v_\Delta,-v_\Delta)\). Since \(v_\Delta\) is orthogonal to \(\bar w_1\) and \(\bar w_2\), the blocks in \cref{lem:relu-two-neuron-hessian} give, for \(i=1,2\),
\begin{align*}
    v_\Delta^\top Q_1(w_i,w_{3-i})v_\Delta &= \frac{\sin(2\kappa)}{2\pi},\\
    v_\Delta^\top Q_1(w_i,v_j)v_\Delta &= \frac{\sin\theta_\alpha(\kappa)}{2\pi r} \left( 1+\frac{\sin^2(\alpha/2)}{\sin^2\theta_\alpha(\kappa)} \right), \qquad j\in\{1,2\},\\
    v_\Delta^\top Q_2(w_1,w_2)v_\Delta &= \frac{\pi-2\kappa}{2\pi}.
\end{align*}
Consequently,
\begin{align*}
    &(v_\Delta,-v_\Delta)^\top \nabla^2L(w_1,w_2) (v_\Delta,-v_\Delta)\\
    &\quad= 1+\frac{\sin(2\kappa)}{\pi} - \frac{2\sin\theta_\alpha(\kappa)}{\pi r} \left( 1+\frac{\sin^2(\alpha/2)}{\sin^2\theta_\alpha(\kappa)} \right) - \frac{\pi-2\kappa}{\pi}.
\end{align*}
Since \(r=r_\alpha(\kappa)\), \eqref{eq:out-of-plane-radius} gives \(2\sin\theta_\alpha(\kappa)/r=\sin(2\kappa)+2\kappa\), and therefore
\[
    (v_\Delta,-v_\Delta)^\top \nabla^2L(w_1,w_2) (v_\Delta,-v_\Delta) = - \frac{\sin(2\kappa)+2\kappa}{\pi} \frac{\sin^2(\alpha/2)}{\sin^2\theta_\alpha(\kappa)} <0.
\]
Thus every point of the out-of-plane family has a strictly negative Hessian direction.
\end{proof}
\subsection{In-plane stationary points}\label{subsec:in-plane-stationary-points}

In the teacher plane, each student is described by its norm and its angle from \(\bar v\), and stationarity reduces to four scalar equations. Solving them for \(0<\alpha\leq\alpha_\star\) proves \cref{lem:critical_points_in_plane}.

\subsubsection{Coordinates in the teacher plane}

For \(\varphi\in[-\pi,\pi)\), write \(e(\varphi)=\cos\varphi\,\bar v+\sin\varphi\,v_\Delta\). Its derivative is the positively oriented unit tangent \(e'(\varphi)=-\sin\varphi\,\bar v+\cos\varphi\,v_\Delta\). For $\varphi,\vartheta\in\RR$, define their circular distance by
\[
    d_{\mathbb S^1}(\varphi,\vartheta) \coloneqq \min_{k\in\mathbb{Z}}|\varphi-\vartheta+2\pi k| =\arccos\bigl(\cos(\varphi-\vartheta)\bigr) \in[0,\pi].
\]
With \(l\) as in \eqref{eq:relu-kernel-profile}, define
\begin{align*}
    \mathcal P(\varphi) &= 2\pi\Bigl( l\bigl(d_{\mathbb S^1}(\varphi,\alpha/2)\bigr) + l\bigl(d_{\mathbb S^1}(\varphi,-\alpha/2)\bigr) \Bigr), \\
    \mathcal P'(\varphi) &=-\bigl(\pi-d_{\mathbb S^1}(\varphi,\alpha/2)\bigr)\sin(\varphi-\alpha/2) -\bigl(\pi-d_{\mathbb S^1}(\varphi,-\alpha/2)\bigr)\sin(\varphi+\alpha/2).
\end{align*}
The function \(\mathcal P\) is \(2\pi\)-periodic, even, and continuously differentiable. In these coordinates, \eqref{eq:critical_eq} splits into a radial and a tangential equation for each student.

\begin{lemma}[Planar stationarity equations]\label{lem:planar-scalar-stationarity}
Let \(0<\alpha<\pi\), let \(w_1,w_2\in\operatorname{span}\{v_1,v_2\}\) be nonzero, and write
\[
    w_i=\|w_i\|_2\, e(\varphi_i), \qquad \varphi_i\in[-\pi,\pi), \quad i=1,2,
\]
and let \(\gamma=d_{\mathbb S^1}(\varphi_1,\varphi_2)\). Then \((w_1,w_2)\) is stationary if and only if
\begin{align}
    \pi \|w_1\|_2+2\pi l(\gamma)\|w_2\|_2 &= \mathcal P(\varphi_1),
    \label{eq:planar-radial-one}\\
    \pi \|w_2\|_2+2\pi l(\gamma)\|w_1\|_2 &= \mathcal P(\varphi_2),
    \label{eq:planar-radial-two}\\
    \mathcal P'(\varphi_1) &= (\pi-\gamma)\|w_2\|_2\sin(\varphi_2-\varphi_1),
    \label{eq:planar-tangential-one}\\
    \mathcal P'(\varphi_2) &= (\pi-\gamma)\|w_1\|_2\sin(\varphi_1-\varphi_2).
    \label{eq:planar-tangential-two}
\end{align}
If in addition \(0<\alpha\leq\alpha_\star\), then \(\mathcal P'(\varphi)<0\) for \(0<\varphi<\pi\) and \(\mathcal P'(\varphi)>0\) for \(-\pi<\varphi<0\). In particular, the zeros of \(\mathcal P'\) are exactly the integer multiples of \(\pi\).
\end{lemma}

\begin{proof}[Proof of \cref{lem:planar-scalar-stationarity}]
The closed-form loss \eqref{eq:closed-form-loss} gives
\[
    2\pi L = \frac{\pi}{2}(\|w_1\|_2^2+\|w_2\|_2^2) +2\pi\|w_1\|_2\|w_2\|_2\,l(\gamma) -\|w_1\|_2\mathcal P(\varphi_1) -\|w_2\|_2\mathcal P(\varphi_2) +\text{constant}.
\]
Differentiating with respect to \(\|w_1\|_2\) and \(\|w_2\|_2\) gives \eqref{eq:planar-radial-one} and \eqref{eq:planar-radial-two}.

For the tangential equations, we write \(B_i=(\pi-\theta_{i1})v_1+(\pi-\theta_{i2})v_2\). The direction equations \eqref{eq:direction-w1}--\eqref{eq:direction-w2} state that \((I-e(\varphi_i)e(\varphi_i)^\top)\bigl(B_i-(\pi-\gamma)w_{3-i}\bigr)=0\). Since \(e'(\varphi_i)\perp e(\varphi_i)\), taking the inner product with \(e'(\varphi_i)\) removes the projection and gives \(\langle e'(\varphi_i),B_i\rangle=(\pi-\gamma)\langle e'(\varphi_i),w_{3-i}\rangle\). To identify the left-hand side with \(\mathcal P'(\varphi_i)\), let \(a=\pm\alpha/2\). If \(\varphi\not\equiv a\) and \(\varphi\not\equiv a+\pi\) modulo \(2\pi\), then \(\cos d_{\mathbb S^1}(\varphi,a)=\cos(\varphi-a)\), and differentiating gives \(\sin d_{\mathbb S^1}(\varphi,a)\,\partial_\varphi d_{\mathbb S^1}(\varphi,a)=\sin(\varphi-a)\). Since \(2\pi l'(t)=-(\pi-t)\sin t\) by \eqref{eq:l-first-derivative} and \(\langle e'(\varphi),e(a)\rangle=-\sin(\varphi-a)\), the chain rule gives
\begin{align*}
    \frac{\partial}{\partial\varphi} 2\pi l\bigl(d_{\mathbb S^1}(\varphi,a)\bigr) &=2\pi l'\bigl(d_{\mathbb S^1}(\varphi,a)\bigr) \frac{\partial}{\partial\varphi}d_{\mathbb S^1}(\varphi,a) =-\bigl(\pi-d_{\mathbb S^1}(\varphi,a)\bigr)\sin(\varphi-a)\\
    &=\bigl(\pi-d_{\mathbb S^1}(\varphi,a)\bigr) \langle e'(\varphi),e(a)\rangle.
\end{align*}
With \(v_1=e(\alpha/2)\), \(v_2=e(-\alpha/2)\), \(\theta_{i1}=d_{\mathbb S^1}(\varphi_i,\alpha/2)\), and \(\theta_{i2}=d_{\mathbb S^1}(\varphi_i,-\alpha/2)\), it follows that
\[
    \mathcal P'(\varphi_i) =(\pi-\theta_{i1})\langle e'(\varphi_i),v_1\rangle +(\pi-\theta_{i2})\langle e'(\varphi_i),v_2\rangle =\langle e'(\varphi_i),B_i\rangle.
\]
Since \(\langle e'(\varphi_i),e(\varphi_{3-i})\rangle=\sin(\varphi_{3-i}-\varphi_i)\), this gives \eqref{eq:planar-tangential-one}--\eqref{eq:planar-tangential-two}.

The computation above assumes that \(e(\varphi_i)\neq\pm v_1,\pm v_2\), because \(\varphi\mapsto d_{\mathbb S^1}(\varphi,a)\) is not differentiable at \(\varphi\equiv a\) and \(\varphi\equiv a+\pi\). The composition \(\varphi\mapsto l(d_{\mathbb S^1}(\varphi,a))\) is nevertheless differentiable at these points. Near \(\varphi\equiv a\) the circular distance is \(|\varphi-a|\), and near \(\varphi\equiv a+\pi\) it is \(\pi-|\varphi-a-\pi|\). Since the formula \eqref{eq:relu-kernel-profile} for \(l\) extends smoothly to \(\RR\) and \(l'(0)=l'(\pi)=0\) by \eqref{eq:l-first-derivative}, both one-sided derivatives vanish. Hence the formula for \(\mathcal P'(\varphi_i)\) holds for every \(\varphi_i\), and the tangential equations remain valid also when \(\gamma=0\) or \(\gamma=\pi\). Conversely, the gradient with respect to each student lies in the teacher plane, and the four equations say that its radial and tangential components vanish. Hence the four equations are equivalent to stationarity.

It remains to determine the sign of \(\mathcal P'\) when \(0<\alpha\leq\alpha_\star\). Since \(\mathcal P\) is even, \(\mathcal P'\) is odd, so it suffices to consider \(0<\varphi<\pi\). For \(0\leq\varphi\leq\alpha/2\), direct expansion gives
\[
    \mathcal P'(\varphi) = 2\sin\varphi\bigl(\varphi\cot\varphi\sin(\alpha/2)-(\pi-\alpha/2)\cos(\alpha/2)\bigr).
\]
For \(0<\varphi\leq\alpha/2\), this is negative, because \(\varphi\cot\varphi<1\) and \(\sin(\alpha/2)\leq(\pi-\alpha/2)\cos(\alpha/2)\) by \cref{lem:squared-loss-angle-threshold}. For \(\alpha/2\leq\varphi\leq\pi-\alpha/2\),
\[
    \mathcal P'(\varphi) = 2\bigl((\alpha/2)\sin(\alpha/2)\cos\varphi-(\pi-\varphi)\cos(\alpha/2)\sin\varphi\bigr).
\]
If \(\varphi\geq\pi/2\), both terms are nonpositive and the second is negative. If \(\alpha/2\leq\varphi<\pi/2\), then \((\pi-\varphi)\tan\varphi>(\alpha/2)\tan(\alpha/2)\) because \(\pi-\varphi>\alpha/2\) and \(\tan\varphi\geq\tan(\alpha/2)\); multiplying by \(\cos(\alpha/2)\cos\varphi>0\) shows again that \(\mathcal P'(\varphi)<0\). Finally, for \(\pi-\alpha/2\leq\varphi\leq\pi\), we put \(x=\pi-\varphi\in[0,\alpha/2]\) and find
\[
    \mathcal P'(\varphi) = -2\bigl((\alpha/2)\cos(\alpha/2)\sin x+x\sin(\alpha/2)\cos x\bigr),
\]
which is negative for \(x>0\) and vanishes at \(x=0\), that is, at \(\varphi=\pi\). Together with \(\mathcal P'(0)=0\) and the oddness of \(\mathcal P'\), this proves the claims about the sign and the zeros of \(\mathcal P'\).
\end{proof}

\subsubsection{Stationary points above the angle threshold}

The assumption \(\alpha\leq\alpha_\star\) in the sign condition cannot be dropped. Above the threshold, \(\mathcal P'\) vanishes inside \((0,\alpha/2)\), which gives the following stationary points.

\begin{proposition}[Additional stationary families above the threshold]\label{prop:extra-planar-stationary-families}
Suppose \(\alpha\in(\alpha_\star,\pi)\). There is a unique \(\varphi_0\in(0,\alpha/2)\) satisfying
\begin{equation}
    \varphi_0\cot\varphi_0=(\pi-\alpha/2)\cot(\alpha/2).
    \label{eq:extra-planar-stationary-angle}
\end{equation}
For \(\sigma=\pm1\) and every \(s_1,s_2>0\) with \(s_1+s_2=\mathcal P(\varphi_0)/\pi\), the pair \((w_1,w_2)=\bigl(s_1e(\sigma\varphi_0),s_2e(\sigma\varphi_0)\bigr)\) is stationary. These families are not among the configurations listed in \cref{lem:critical_points_in_plane}.
\end{proposition}

\begin{proof}[Proof of \cref{prop:extra-planar-stationary-families}]
The function \(\varphi\mapsto\varphi\cot\varphi\) is strictly decreasing on \((0,\alpha/2)\), since its derivative \((\sin\varphi\cos\varphi-\varphi)/\sin^2\varphi\) is negative. Moreover, \(\varphi\cot\varphi\to1\) as \(\varphi\downarrow0\), and \((\alpha/2)\cot(\alpha/2)<(\pi-\alpha/2)\cot(\alpha/2)<1\). Here the first inequality uses \(\alpha<\pi\), and the second is equivalent to \(\tan(\alpha/2)>\pi-\alpha/2\), which holds by \cref{lem:squared-loss-angle-threshold} because \(\alpha>\alpha_\star\). Thus \eqref{eq:extra-planar-stationary-angle} has a unique root \(\varphi_0\in(0,\alpha/2)\).

The distances from \(\varphi_0\) to the two teachers, \(\alpha/2-\varphi_0\) and \(\alpha/2+\varphi_0\), lie in \((0,\pi)\). Since \(l\) is strictly decreasing on \([0,\pi]\) by \eqref{eq:l-first-derivative} and \(l(\pi)=0\), it follows that \(\mathcal P(\varphi_0)>0\), so there are positive \(s_1,s_2\) with the stated sum.

By the formula for \(\mathcal P'\) on \([0,\alpha/2]\) in the proof of \cref{lem:planar-scalar-stationarity},
\[
    \mathcal P'(\varphi) = 2\sin\varphi\sin(\alpha/2) \left(\varphi\cot\varphi-(\pi-\alpha/2)\cot(\alpha/2)\right)
\]
for \(0<\varphi<\alpha/2\), so \(\mathcal P'(\varphi_0)=0\). For the pair in the proposition, \(\gamma=0\), so the tangential equations reduce to \(\mathcal P'(\varphi_0)=0\) and both radial equations reduce to \(\pi(s_1+s_2)=\mathcal P(\varphi_0)\). By \cref{lem:planar-scalar-stationarity}, the pair is stationary. Since \(0<\varphi_0<\alpha/2\), these points are not parallel to \(\pm\bar v\) and differ from \((v_1,v_2)\), so they are not among the configurations in \cref{lem:critical_points_in_plane}.
\end{proof}

\subsubsection{Classification up to the angle threshold}

We now return to \(0<\alpha\leq\alpha_\star\), which we assume for the rest of this subsection. The points \(0\), \(\pm \alpha/2\), and \(\pm(\pi-\alpha/2)\) split \((-\pi,\pi)\) into the open intervals
\[
\begin{array}{lll}
    J_2=(-\pi,-\pi+\alpha/2),& J_1=(-\pi+\alpha/2,-\alpha/2),& J_0=(-\alpha/2,0),\\
    I_0=(0,\alpha/2),& I_1=(\alpha/2,\pi-\alpha/2),& I_2=(\pi-\alpha/2,\pi).
\end{array}
\]
The only symmetries of the loss that we use are the exchange of the two students and the reflection \(\varphi\mapsto-\varphi\) across \(\bar v\); the map \(\varphi\mapsto\varphi+\pi\) is not a symmetry of the loss.

\begin{proof}[Proof of \cref{lem:critical_points_in_plane}]
As in \cref{lem:planar-scalar-stationarity}, we write \(w_i=\|w_i\|_2\,e(\varphi_i)\) with \(\varphi_i\in[-\pi,\pi)\). Both \(r_+\) and \(r_-\) in \eqref{eq:def_r} are positive; for \(r_-\), this uses \(\tan(\alpha/2)>\alpha/2\).

We first treat collinear directions. If \(\gamma=0\), the two directions coincide, and the tangential equations \eqref{eq:planar-tangential-one}--\eqref{eq:planar-tangential-two} give \(\mathcal P'(\varphi_1)=0\), so by \cref{lem:planar-scalar-stationarity} the common direction is \(\bar v\) or \(-\bar v\). Since \(2\pi l(0)=\pi\) and
\[
\begin{aligned}
    \mathcal P(0)&=2\bigl((\pi-\alpha/2)\cos(\alpha/2)+\sin(\alpha/2)\bigr)=\pi r_+,\\
    \mathcal P(\pi)&=2\bigl(\sin(\alpha/2)-(\alpha/2)\cos(\alpha/2)\bigr)=\pi r_-,
\end{aligned}
\]
the radial equations \eqref{eq:planar-radial-one}--\eqref{eq:planar-radial-two} give \(\|w_1\|_2+\|w_2\|_2=r_+\) in the first case and \(\|w_1\|_2+\|w_2\|_2=r_-\) in the second. If \(\gamma=\pi\), the right-hand sides of the tangential equations vanish, so \(\mathcal P'(\varphi_1)=\mathcal P'(\varphi_2)=0\) and the two directions are \(\bar v\) and \(-\bar v\). After exchanging the students, \((\varphi_1,\varphi_2)=(0,-\pi)\), and since \(l(\pi)=0\), the radial equations give \(\|w_1\|_2=r_+\) and \(\|w_2\|_2=r_-\). This determines all stationary points with collinear directions.

\paragraph{Noncollinear directions.} Assume now that \(0<\gamma<\pi\). Multiplying the two tangential equations gives
\[
    \mathcal P'(\varphi_1)\mathcal P'(\varphi_2) = -(\pi-\gamma)^2\|w_1\|_2\|w_2\|_2\sin^2(\varphi_2-\varphi_1)<0.
\]
By \cref{lem:planar-scalar-stationarity}, one direction therefore lies in \((-\pi,0)\) and the other in \((0,\pi)\). After relabeling, we write \(\varphi_-=\varphi_1\in(-\pi,0)\) and \(\varphi_+=\varphi_2\in(0,\pi)\). Since \(\mathcal P'(\varphi_-)>0\), the tangential equation \eqref{eq:planar-tangential-one} gives \((\pi-\gamma)\|w_2\|_2\sin(\varphi_+-\varphi_-)>0\), so \(\sin(\varphi_+-\varphi_-)>0\). As \(0<\varphi_+-\varphi_-<2\pi\), it follows that
\begin{equation}\label{eq:planar-short-gap}
    0<\varphi_+-\varphi_-<\pi.
\end{equation}

Up to the exchange of the students and the reflection \(\varphi\mapsto-\varphi\), \eqref{eq:planar-short-gap} leaves four cases for the intervals containing \(\varphi_-\) and \(\varphi_+\):
\[
\begin{array}{c|c}
    \text{case}&(\varphi_-,\varphi_+)\\ \hline
    1&J_0\times I_0,\\
    2&J_0\times I_1 \quad\text{or, by reflection, }J_1\times I_0,\\
    3&J_2\times I_0 \quad\text{or, by reflection, }J_0\times I_2,\\
    4&J_1\times I_1.
\end{array}
\]
The products \(J_1\times I_2\), \(J_2\times I_1\), and \(J_2\times I_2\) are excluded by \eqref{eq:planar-short-gap}, since on them \(\varphi_+-\varphi_-\geq\pi\). Directions at the endpoints \(\pm\alpha/2\) and \(\pm(\pi-\alpha/2)\) are included in the adjacent cases. We now show that these cases contain no stationary point other than \((w_1,w_2)=(v_1,v_2)\), which lies on the boundary of Case~1.

\paragraph{Case 1: \(\varphi_-\in J_0\) and \(\varphi_+\in I_0\).} We write \(\varphi_+=x\) and \(\varphi_-=-y\) with \(0<x,y\leq\alpha/2\), so that \(\gamma=x+y\), and set \(p=x-y\). Then
\begin{equation}\label{eq:planar-case-one-p-range}
    |p|\leq\gamma\leq\alpha<\pi.
\end{equation}
On \([0,\alpha/2]\),
\[
    \mathcal P(\varphi) = 2\Bigl(\bigl((\pi-\alpha/2)\cos(\alpha/2)+\sin(\alpha/2)\bigr)\cos\varphi +\varphi\sin(\alpha/2)\sin\varphi\Bigr),
\]
and \(\mathcal P'(\varphi)\) is given by the first formula in the proof of \cref{lem:planar-scalar-stationarity}. Here \((\varphi_1,\varphi_2)=(-y,x)\). Since \(\mathcal P\) is even and \((\pi-\gamma)\sin\gamma>0\), the tangential equations give
\[
    \|w_1\|_2=-\frac{\mathcal P'(x)}{(\pi-\gamma)\sin\gamma}, \qquad \|w_2\|_2=-\frac{\mathcal P'(y)}{(\pi-\gamma)\sin\gamma},
\]
and the radial equations read \(\pi\|w_2\|_2+2\pi l(\gamma)\|w_1\|_2=\mathcal P(x)\) and \(\pi\|w_1\|_2+2\pi l(\gamma)\|w_2\|_2=\mathcal P(y)\). Substituting the norms and multiplying by \((\pi-\gamma)\sin\gamma\), we obtain
\begin{align*}
    -\pi\mathcal P'(y)-2\pi l(\gamma)\mathcal P'(x) &= (\pi-\gamma)\sin\gamma\,\mathcal P(x), \\
    -\pi\mathcal P'(x)-2\pi l(\gamma)\mathcal P'(y) &= (\pi-\gamma)\sin\gamma\,\mathcal P(y).
\end{align*}
After division by \(2\), both equations are linear and homogeneous in the nonzero vector \(\bigl((\pi-\alpha/2)\cos(\alpha/2),\sin(\alpha/2)\bigr)\), with coefficient matrix
\[
    M=
    \begin{pmatrix}
        \gamma\sin y+\sin\gamma\sin x & -\gamma(\pi-x)\cos y-\sin\gamma(\pi-y)\cos x\\
        \gamma\sin x+\sin\gamma\sin y & -\gamma(\pi-y)\cos x-\sin\gamma(\pi-x)\cos y
    \end{pmatrix}.
\]
Hence \(\det M=0\). Expanding the determinant and using \(x+y=\gamma\) and \(x-y=p\), we find
\begin{align*}
    \det M &= (\gamma^2-\sin^2\gamma) \bigl((\pi-x)\sin x\cos y-(\pi-y)\sin y\cos x\bigr) \\
    &= \frac{\gamma^2-\sin^2\gamma}{2} \bigl((2\pi-\gamma)\sin p-p\sin\gamma\bigr).
\end{align*}
Since \(\gamma>\sin\gamma>0\), it follows that \((2\pi-\gamma)\sin p=p\sin\gamma\). If \(p\neq0\), this says that \(\sin|p|/|p|=\sin\gamma/(2\pi-\gamma)\). But \(|p|\leq\gamma<\pi\) by \eqref{eq:planar-case-one-p-range}, and \(\sin t/t\) is strictly decreasing on \((0,\pi)\), so \(\sin|p|/|p|\geq\sin\gamma/\gamma>\sin\gamma/(2\pi-\gamma)\), a contradiction. Hence \(p=0\).

Thus \(x=y=\gamma/2\), and either of the two equations above becomes
\[
    (\gamma+\sin\gamma) \left( (\pi-\alpha/2)\cos(\alpha/2)\sin\frac{\gamma}{2} -\sin(\alpha/2)\left(\pi-\frac{\gamma}{2}\right) \cos\frac{\gamma}{2} \right)=0.
\]
Since \(\gamma+\sin\gamma>0\), it follows that \((\pi-\gamma/2)\cot(\gamma/2)=(\pi-\alpha/2)\cot(\alpha/2)\). The function \(\gamma\mapsto(\pi-\gamma/2)\cot(\gamma/2)\) is strictly decreasing on \((0,\pi)\), because its derivative is
\[
    -\frac{ \left(\pi-\frac{\gamma}{2}\right) +\sin\frac{\gamma}{2}\cos\frac{\gamma}{2} }{2\sin^2\frac{\gamma}{2}} <0.
\]
Hence \(\gamma=\alpha\) and \((\varphi_-,\varphi_+)=(-\alpha/2,\alpha/2)\). Since \(-\mathcal P'(\alpha/2)=(\pi-\alpha)\sin\alpha\), the formulas for the norms give \(\|w_1\|_2=\|w_2\|_2=1\). Thus the only stationary point in Case~1, including its boundary, is \((v_1,v_2)\), up to the exchange of the students.

\paragraph{Case 2: \(\varphi_-\in J_0\) and \(\varphi_+\in I_1\).} By the reflection \(\varphi\mapsto-\varphi\), we may assume that \(\varphi_-\in J_0\) and \(\varphi_+\in I_1\), including the endpoints except \(\varphi_+=\pi-\alpha/2\), which is treated in Case~3. We write \(\varphi_1=\varphi_-=\alpha/2-b_0\) and \(\varphi_2=\varphi_+=\alpha/2+a_0\), and set \(s=a_0+b_0\) and \(\rho=\alpha-b_0\). Then \(a_0\geq0\), \(\alpha/2<b_0\leq\alpha\), \(s<\pi\) by \eqref{eq:planar-short-gap}, and \(a_0+\alpha<\pi\) because \(\varphi_+<\pi-\alpha/2\).

First suppose that \(a_0>0\) and \(\rho>0\). The teacher--student angles are \(\theta_{11}=b_0\), \(\theta_{12}=\rho\), \(\theta_{21}=a_0\), and \(\theta_{22}=a_0+\alpha\). Substituting them into \eqref{eq:planar-radial-one}--\eqref{eq:planar-tangential-two}, using the tangential equations to eliminate the radii, and collecting the coefficients of \(\sin b_0\), \(\sin a_0\), \(\sin\rho\), and \(\sin(a_0+\alpha)\), we obtain
\begin{align*}
    0 &= -a_0(2\pi-s)\sin b_0 +s(2\pi-a_0-\alpha)\sin\rho+\sin s\bigl(\rho\sin(a_0+\alpha)-b_0\sin a_0\bigr), \\
    0 &= -b_0(2\pi-s)\sin a_0 +s\rho\sin(a_0+\alpha)+\sin s\bigl((2\pi-a_0-\alpha)\sin\rho-a_0\sin b_0\bigr).
\end{align*}
We write \(U=a_0\sin b_0\), \(V=b_0\sin a_0\), \(M=(2\pi-a_0-\alpha)\sin\rho\), and \(N=\rho\sin(a_0+\alpha)\), which are all positive. Since \(2\pi-a_0-\alpha=2\pi-s-\rho\) and \(a_0+\alpha=s+\rho\), these equations can be written as
\begin{align}
    (2\pi-s)U+\sin s\,V &= sM+\sin s\,N,
    \label{eq:planar-case-two-linear-one}\\
    \sin s\,U+(2\pi-s)V &= \sin s\,M+sN.
    \label{eq:planar-case-two-linear-two}
\end{align}
and adding and subtracting them gives
\begin{align}
    (2\pi-s+\sin s)(U+V) &= (s+\sin s)(M+N),
    \label{eq:planar-case-two-sum}\\
    (2\pi-s-\sin s)(U-V) &= (s-\sin s)(M-N).
    \label{eq:planar-case-two-difference}
\end{align}

Since \(0<\rho<s+\rho<\pi\), we have \(2\pi-s-\rho>s+\rho\), and since \(\sin t/t\) is strictly decreasing on \((0,\pi)\), \(M>(s+\rho)\sin\rho>\rho\sin(s+\rho)=N\). The coefficients in \eqref{eq:planar-case-two-difference} are positive, so \(U>V\), that is, \(\sin b_0/b_0>\sin a_0/a_0\), and hence \(a_0>b_0\). Together with \(b_0>\alpha/2\) and \(s+\rho=a_0+\alpha<\pi\), this gives
\begin{equation}\label{eq:planar-case-two-range}
    0<\rho=\alpha-b_0<b_0<a_0, \qquad b_0<\frac{s}{2}, \qquad \rho<\min\left\{\frac{s}{2},\pi-s\right\}.
\end{equation}

To compare \(U-V\) with \(M-N\), we fix \(s\in(0,\pi)\) and introduce
\[
    R_s(x) = \frac{(s-x)\sin x-x\sin(s-x)} {(s-x)\sin x+x\sin(s-x)}, \qquad Q_s(x) = \frac{(2\pi-s-x)\sin x-x\sin(s+x)} {(2\pi-s-x)\sin x+x\sin(s+x)},
\]
for \(0<x<s/2\) and \(0<x<\min\{s/2,\pi-s\}\), respectively. Both denominators are positive, the second because \(0<s+x<\pi\). By \eqref{eq:planar-case-two-range}, \((U-V)/(U+V)=R_s(b_0)\) and \((M-N)/(M+N)=Q_s(\rho)\), so dividing \eqref{eq:planar-case-two-difference} by \eqref{eq:planar-case-two-sum} gives
\begin{equation}\label{eq:planar-case-two-ratio}
    R_s(b_0)=\Lambda_s Q_s(\rho), \qquad \Lambda_s = \frac{s-\sin s}{s+\sin s} \frac{2\pi-s+\sin s}{2\pi-s-\sin s}.
\end{equation}
As \(x\downarrow0\), \(R_s(x)\) and \(Q_s(x)\) tend to \(R_s(0)=(s-\sin s)/(s+\sin s)\) and \(Q_s(0)=(2\pi-s-\sin s)/(2\pi-s+\sin s)\), so \(\Lambda_s=R_s(0)/Q_s(0)\).

We next show that
\begin{align}
    R_s(x) &< \frac{s-x}{s}R_s(0), &&0<x<s/2,
    \label{eq:planar-case-two-R-bound}\\
    Q_s(x) &\geq \frac{s-x}{s}Q_s(0), &&0<x<\min\{s/2,\pi-s\}.
    \label{eq:planar-case-two-Q-bound}
\end{align}
For the first bound, we write \(A(t)=\log(t/\sin t)\), which is positive for \(0<t<\pi\), and \(\psi(t)=(t-\sin t)/\bigl(t(t+\sin t)\bigr)\). Dividing the numerator and the denominator of \(R_s(x)\) by \(\sin x\sin(s-x)\) gives
\[
    R_s(x) = \tanh\left(\frac{A(s-x)-A(x)}{2}\right) < \tanh\left(\frac{A(s-x)}{2}\right) = (s-x)\psi(s-x).
\]
To see that \(\psi\) is nondecreasing on \((0,\pi)\), we compute
\[
    \psi'(t) = \frac{\sin^2t+2t\sin t-t^2(1+2\cos t)}{t^2(t+\sin t)^2}.
\]
The numerator vanishes at \(t=0\) and has derivative \(2\bigl(t^2\sin t-(1+\cos t)(t-\sin t)\bigr)\geq0\), because \(t-\sin t=\int_0^t(1-\cos u)\,\diff u\leq t(1-\cos t)\) implies \((1+\cos t)(t-\sin t)\leq t\sin^2t\leq t^2\sin t\). Hence \(R_s(x)<(s-x)\psi(s)=\frac{s-x}{s}R_s(0)\), which is \eqref{eq:planar-case-two-R-bound}.

For the second bound, we write \(B(t)=\log\bigl((2\pi-t)/\sin t\bigr)\) for \(0<t<\pi\). In the same way, \(Q_s(x)=\tanh\bigl((B(s+x)-A(x))/2\bigr)\), and we first show that
\begin{equation}\label{eq:planar-case-two-log-comparison}
    B(s+x)-A(x)\geq\frac{s-x}{s}B(s).
\end{equation}
With \(A(0)=0\), the function \(H_s(x)=B(s+x)-A(x)-\frac{s-x}{s}B(s)\) vanishes at \(x=0\), and
\[
    H_s'(x) = \frac{B(s)}{s} - \left(\frac1x-\cot x\right) - \left(\cot(s+x)+\frac{1}{2\pi-s-x}\right).
\]
The function \(t\mapsto1/t-\cot t\) is increasing on \((0,\pi)\), since its derivative \(1/\sin^2t-1/t^2\) is positive, and \(t\mapsto\cot t+1/(2\pi-t)\) is decreasing, since its derivative \(-1/\sin^2t+1/(2\pi-t)^2\) is negative. Because \(x<s/2\) and \(s+x>s\),
\[
    H_s'(x) \geq \frac{B(s)}{s} - \left(\frac2s-\cot\frac{s}{2}\right) - \left(\cot s+\frac1{2\pi-s}\right).
\]
Since \(\cot(s/2)=1/\sin s+\cot s\), the right-hand side is nonnegative as soon as
\begin{equation}\label{eq:planar-case-two-B-lower}
    B(s)\geq2-\frac{s}{\sin s}+\frac{s}{2\pi-s}.
\end{equation}
Applying \(\log z\geq2(z-1)/(z+1)\) for \(z>1\) with \(z=(2\pi-s)/\sin s\), we see that \eqref{eq:planar-case-two-B-lower} follows from
\begin{equation}\label{eq:planar-case-two-polynomial}
    s\bigl((2\pi-s)^2-\sin^2s\bigr) \geq 4(2\pi-s)\sin^2s.
\end{equation}
If \(0<s\leq1\), then \(\sin s\leq s\) and \((2\pi-s)/s\geq2\pi-1>2+\sqrt5\), which implies \((2\pi-s)^2-s^2\geq4(2\pi-s)s\) and hence \eqref{eq:planar-case-two-polynomial}. If \(1\leq s<\pi\), then \(\sin^2s\leq1\), and it is enough to show that \(\chi(s)=s\bigl((2\pi-s)^2-1\bigr)-4(2\pi-s)\) is positive. Since \(\chi''(s)=6s-8\pi<0\) on \([1,\pi]\), this follows from \(\chi(1)=4(\pi^2-3\pi+1)>0\) and \(\chi(\pi)=\pi(\pi^2-5)>0\). Thus \eqref{eq:planar-case-two-B-lower} holds, so \(H_s'\geq0\), and \eqref{eq:planar-case-two-log-comparison} follows from \(H_s(0)=0\). Since \(\tanh\) is increasing, and concave on \([0,\infty)\) with \(\tanh0=0\),
\[
    Q_s(x) \geq \tanh\left(\frac{s-x}{s}\frac{B(s)}{2}\right) \geq \frac{s-x}{s}\tanh\left(\frac{B(s)}{2}\right) = \frac{s-x}{s}Q_s(0),
\]
which is \eqref{eq:planar-case-two-Q-bound}.

Since \(0<\rho<b_0\) by \eqref{eq:planar-case-two-range}, the two bounds give
\[
    R_s(b_0) < \frac{s-b_0}{s}R_s(0) < \frac{s-\rho}{s}R_s(0) = \Lambda_s\frac{s-\rho}{s}Q_s(0) \leq \Lambda_s Q_s(\rho),
\]
which contradicts \eqref{eq:planar-case-two-ratio}. Hence there is no stationary point with \(a_0>0\) and \(\rho>0\).

It remains to treat the endpoints. If \(\rho=0\) and \(a_0>0\), then \(M=N=0\), and \eqref{eq:planar-case-two-linear-one}--%
\eqref{eq:planar-case-two-linear-two} say that \((U,V)\) lies in the kernel of the matrix with diagonal entries \(2\pi-s\) and off-diagonal entries \(\sin s\). This matrix is invertible because \(2\pi-s>\sin s\), so \(U=V=0\), which contradicts \(a_0>0\). If \(a_0=0\), then \(\varphi_+=\alpha/2\) lies in the closure of Case~1, where the only stationary point is \((v_1,v_2)\). Thus Case~2 contributes no further stationary point.

\paragraph{Case 3: \(\varphi_-\in J_2\) and \(\varphi_+\in I_0\).} By reflection, we may assume that \(\varphi_-\in J_2\) and \(\varphi_+\in I_0\), and we write \(\varphi_+=p\) and \(\varphi_-=-\pi+q\) with \(0\leq p,q\leq\alpha/2\). Then \(\varphi_+-\varphi_-=\pi-(q-p)\), so \eqref{eq:planar-short-gap} gives \(p<q\), and \(\gamma=\pi-\delta\) with \(\delta=q-p\in(0,\alpha/2]\). Substituting these directions into \eqref{eq:planar-radial-one}--\eqref{eq:planar-tangential-two} and eliminating the two radii, we obtain
\begin{equation}\label{eq:planar-case-three-residual}
    0=\sin\delta\,A_p+\delta B_{p,q}-\pi C_q,
\end{equation}
where
\[
\begin{aligned}
    A_p &= 2p\sin(\alpha/2)\cos p - 2(\pi-\alpha/2)\cos(\alpha/2)\sin p,\\
    B_{p,q} &= 2(\pi-\alpha/2)\cos(\alpha/2)\sin q - 2p\sin(\alpha/2)\cos q + 2\sin(\alpha/2)\cos p\sin\delta,\\
    C_q &= 2q\sin(\alpha/2)\cos q + \alpha\cos(\alpha/2)\sin q.
\end{aligned}
\]
By \cref{lem:squared-loss-angle-threshold}, \(\sin(\alpha/2)\leq(\pi-\alpha/2)\cos(\alpha/2)\). Since \(\tan p>p\) for \(0<p<\pi/2\), it follows that \(A_p=2p\cos p\bigl(\sin(\alpha/2)-(\pi-\alpha/2)\cos(\alpha/2)\tan p/p\bigr)<0\) for \(p>0\), while \(A_0=0\). Thus \(A_p\leq0\). Next, the middle term of \(B_{p,q}\) is nonpositive. Dropping it and using \(\delta\leq q\leq\alpha/2\), \(\cos p\leq1\), and \(\sin\delta\leq\sin q<\pi\cos q\), where the last inequality holds because \(\tan q\leq\tan(\alpha/2)\leq\pi-\alpha/2<\pi\), we obtain
\[
\begin{aligned}
    \delta B_{p,q} &\leq 2\delta(\pi-\alpha/2)\cos(\alpha/2)\sin q + 2\delta\sin(\alpha/2)\cos p\sin\delta\\
    &< \pi\alpha\cos(\alpha/2)\sin q + 2\pi q\sin(\alpha/2)\cos q = \pi C_q.
\end{aligned}
\]
Hence the right-hand side of \eqref{eq:planar-case-three-residual} is negative, a contradiction. Thus Case~3 contains no stationary point.

\paragraph{Case 4: \(\varphi_-\in J_1\) and \(\varphi_+\in I_1\).} We write \(\varphi_+=\alpha/2+a_0\) and \(\varphi_-=-\alpha/2-b_0\) with \(a_0,b_0\geq0\), so that \(\gamma=\alpha+a_0+b_0<\pi\) and the teacher--student angles are \(\theta_{11}=b_0+\alpha\), \(\theta_{12}=b_0\), \(\theta_{21}=a_0\), and \(\theta_{22}=a_0+\alpha\). Substituting them into the equations of \cref{lem:planar-scalar-stationarity} and eliminating the two radii, we obtain
\begin{align*}
    0 &= -\sin\gamma \bigl( (\gamma-a_0)\sin a_0 + (\gamma-a_0-\alpha)\sin(a_0+\alpha) \bigr) \\
    &\quad -(2\pi-\gamma) \bigl( a_0\sin(b_0+\alpha) + (a_0+\alpha)\sin b_0 \bigr).
\end{align*}
Both terms on the right are nonpositive, and the second is negative when \(a_0,b_0>0\). Hence there is no stationary point with \(a_0,b_0>0\). If exactly one of \(a_0\) and \(b_0\) vanishes, the point lies in the closure of Case~2, and if both vanish, it is \((v_1,v_2)\).

Together with the collinear directions, Cases~1--4 show that every stationary point is one of the configurations listed in \cref{lem:critical_points_in_plane}.

\paragraph{Stationarity and curvature of the listed points.} Conversely, every listed configuration is stationary. The pair \((v_1,v_2)\) reproduces the teacher, so the loss, which is nonnegative, vanishes there; it is therefore a global minimum. For two students along \(\bar v\), we have \(\gamma=0\) and \(\mathcal P'(0)=0\), and since \(2\pi l(0)=\pi\) and \(\mathcal P(0)=\pi r_+\), the equations of \cref{lem:planar-scalar-stationarity} hold exactly when \(\|w_1\|_2+\|w_2\|_2=r_+\). The same computation for two students along \(-\bar v\) gives \(\|w_1\|_2+\|w_2\|_2=r_-\). For the pair \((r_+\bar v,-r_-\bar v)\), we have \(\mathcal P'(0)=\mathcal P'(\pi)=0\) and \(l(\pi)=0\), and the radial equations hold with the norms \(r_+\) and \(r_-\).

It remains to find a direction of negative curvature at each of these saddles. If both students are parallel to \(\bar v\), or both are parallel to \(-\bar v\), the Hessian in \cref{lem:relu-two-neuron-hessian}, restricted to perturbations \((t_1v_\Delta,t_2v_\Delta)\), is
\[
    H_{v_\Delta} =
    \begin{pmatrix}
        \frac12-\frac{2\sin(\alpha/2)}{\pi \|w_1\|_2}&\frac12\\[1mm] \frac12&\frac12-\frac{2\sin(\alpha/2)}{\pi \|w_2\|_2}
    \end{pmatrix},
\]
The direction \((1,-1)\), which moves the two students apart along \(v_\Delta\), gives
\[
    (1,-1)H_{v_\Delta}(1,-1)^\top = -\frac{2\sin(\alpha/2)}{\pi} \left(\frac1{\|w_1\|_2}+\frac1{\|w_2\|_2}\right)<0.
\]
At \((r_+\bar v,-r_-\bar v)\), the off-diagonal Hessian block vanishes by the continuous extension in \cref{lem:relu-two-neuron-hessian}, and the perturbation \((0,v_\Delta)\) of the second student gives
\[
    \frac12-\frac{2\sin(\alpha/2)}{\pi r_-} = -\frac{\sin(\alpha/2)+(\alpha/2)\cos(\alpha/2)}{2\bigl(\sin(\alpha/2)-(\alpha/2)\cos(\alpha/2)\bigr)} <0.
\]
Thus every listed configuration other than \((v_1,v_2)\) has a strictly negative Hessian direction.
\end{proof}

\subsection{Completion of the squared-loss landscape theorem}

The statements of \cref{prop:3D,lem:3Dsym,lem:critical_point_perp_plane,lem:critical_points_in_plane} assume \(0<\alpha\leq\pi/2\), but their proofs in the preceding subsections use only \(0<\alpha\leq\alpha_\star\), with \(\alpha_\star\) as in \eqref{eq:width-two-angle-threshold}. These results therefore hold for every \(\alpha\in(0,\alpha_\star]\), and they give the following extension of \cref{thm:squared_loss_landscape}.

\begin{theorem}[Squared-loss landscape up to the angle threshold]\label{thm:squared-loss-landscape-nonacute}
Let \(0<\alpha\leq\alpha_\star\), and let \((w_1,w_2)\) with \(w_1,w_2\neq0\) be a stationary point of the squared loss \eqref{eq:population_squared_loss_and_gf}. After possibly swapping the students, exactly one of the following holds:
\begin{enumerate}[label=(\arabic*)]
    \item \(w_1=v_1\) and \(w_2=v_2\). This is a global minimum.
    \item The pair belongs to one of the three in-plane families
    \[
    \begin{gathered}
        w_1=s_1\bar v,\quad w_2=s_2\bar v,\quad s_1,s_2>0,\quad s_1+s_2=r_+,\\
        w_1=-s_1\bar v,\quad w_2=-s_2\bar v,\quad s_1,s_2>0,\quad s_1+s_2=r_-,\\
        w_1=r_+\bar v,\quad w_2=-r_-\bar v.
    \end{gathered}
    \]
    Every such point has a strictly negative Hessian direction.
    \item Necessarily \(d\geq3\), and there is a unit vector \(\xi\perp\operatorname{span}\{v_1,v_2\}\) such that
    \[
    \begin{aligned}
        w_1 &= r_\alpha(\kappa_\alpha) (\cos\kappa_\alpha\,\bar v+\sin\kappa_\alpha\,\xi),\\
        w_2 &= r_\alpha(\kappa_\alpha) (\cos\kappa_\alpha\,\bar v-\sin\kappa_\alpha\,\xi).
    \end{aligned}
    \]
    Every such point has a strictly negative Hessian direction.
\end{enumerate}
Consequently, every local minimum of the squared loss with \(w_1\neq0\) and \(w_2\neq0\) is a global minimum.
\end{theorem}

\begin{proof}[Proof of \cref{thm:squared-loss-landscape-nonacute}]
By \cref{lem:critical_eq_from_flow}, the pair satisfies \eqref{eq:critical_eq}, so by \cref{prop:3D} the vectors \(v_1,v_2,w_1,w_2\) span a space of dimension at most three. If \(w_1,w_2\in\operatorname{span}\{v_1,v_2\}\), \cref{lem:critical_points_in_plane} shows that the pair is \((v_1,v_2)\) or belongs to one of the three families in~(2), and that each point of these families has a strictly negative Hessian direction.

Otherwise, \(d\geq3\), and there is a unit vector \(\xi\perp\operatorname{span}\{v_1,v_2\}\) such that \(w_1,w_2\in\operatorname{span}\{\bar v,v_\Delta,\xi\}\). By \cref{lem:3Dsym}, \(w_1,w_2\in\operatorname{span}\{\bar v,\xi\}\), and \cref{lem:critical_point_perp_plane} shows that the pair has the form in~(3) and has a strictly negative Hessian direction. Since this pair does not lie in \(\operatorname{span}\{v_1,v_2\}\), the three cases are disjoint.

Conversely, the same lemmas show that every configuration in (1)--(3) is stationary. The pair \((v_1,v_2)\) has zero loss, and every other stationary point has a strictly negative Hessian direction. Since a local minimum with \(w_1\neq0\) and \(w_2\neq0\) is stationary, it is a global minimum.
\end{proof}

\begin{proof}[Proof of \cref{thm:squared_loss_landscape}]
Since \(\pi/2<\alpha_\star\), the theorem is the restriction of \cref{thm:squared-loss-landscape-nonacute} to \(0<\alpha\leq\pi/2\).
\end{proof}

\section{Global convergence and boundary avoidance}\label{sec:global-convergence}

For the proof of \cref{thm:generic_global_convergence}, we fix \(d\ge2\) and unit teachers with \(\alpha=\angle(v_1,v_2)\in(0,\pi/2]\), and we write \(\Omega=\{(w_1,w_2)\in(\RR^d)^2:w_1\ne0,\ w_2\ne0\}\), its boundary \(\partial\Omega=\{(w_1,w_2):w_1=0\text{ or }w_2=0\}\), \(\mathcal G=\{(v_1,v_2),(v_2,v_1)\}\) and \(r_i=\|w_i\|_2\). A set is null if it has Lebesgue measure zero; on a manifold, this is understood in local coordinates. For the student output \(f\) in \eqref{eq:two_relu_student}, the cross term \(2\EE[\phi(\langle w_1,x\rangle)\phi(\langle w_2,x\rangle)]\) in \(\EE[f(x)^2]\) is nonnegative, and \(\EE[\phi(\langle w_i,x\rangle)^2]=r_i^2/2\) by \eqref{eq:relu-gaussian-kernel}. Since \(f^2\le2(f^*)^2+2(f-f^*)^2\), we obtain, for all \(w_1,w_2\in\RR^d\),
\begin{equation}\label{eq:coercive-loss-bound}
    \frac12\bigl(r_1^2+r_2^2\bigr) \le \EE[f(x)^2] \le 2\EE[f^*(x)^2]+4L(w_1,w_2).
\end{equation}
Since \(L\) does not increase along gradient flow, every trajectory stays in a bounded set determined by its initial loss.

\subsection{Almost-everywhere avoidance of zero weights}

For \(w(0)\in\Omega\), let \([0,T)\) be the maximal interval on which the solution of gradient flow exists in \(\Omega\); it is well defined by \cref{lem:relu-kernel-regularity}, and \(T\in(0,\infty]\) depends on \(w(0)\). A trajectory may leave \(\Omega\) in finite time, and it may approach \(\partial\Omega\) as \(t\to\infty\). We show that both happen only for a null set of initializations. The first lemma shows that the two weights cannot tend to zero together.

\begin{lemma}[Lower bound on \(\|w_1\|_2+\|w_2\|_2\)]\label{lem:coercive_bound_and_no_double_zero}
For every \(w(0)\in\Omega\) and \(0\le t<T\),
\[
    r_1(t)+r_2(t) \ge r_-+\bigl(r_1(0)+r_2(0)-r_-\bigr)e^{-t} \ge\min\{r_1(0)+r_2(0),\,r_-\},
\]
where \(r_-\) is defined in \eqref{eq:def_r}.
\end{lemma}

\begin{proof}[Proof of \cref{lem:coercive_bound_and_no_double_zero}]
By \cref{lem:radial-velocity-bounds}, \(\dot r_i\ge(r_--r_1-r_2)/2\) for \(i=1,2\). The sum therefore satisfies \(\frac{\diff}{\diff t}(r_1+r_2)\ge r_--(r_1+r_2)\), and Gr\"onwall's inequality gives the first inequality. The second holds because the middle expression is a convex combination of \(r_1(0)+r_2(0)\) and \(r_-\).
\end{proof}

\subsubsection{Polar coordinates near \texorpdfstring{\(\{w_1=0\}\)}{w1=0}}

By symmetry it suffices to study trajectories along which \(w_1\) tends to zero. We use the coordinates \((r_1,\bar w_1,w_2)\), where \(w_1=r_1\bar w_1\), and set
\[
    \Psi_{w_2}(\bar w_1) =l(\theta_{11})+l(\theta_{12})-r_2\,l(\gamma),
\]
with \(\theta_{11},\theta_{12},\gamma\) as in \eqref{eq:angle_shorthand}. Substituting \(w_1=r_1\bar w_1\) into \eqref{eq:closed-form-loss}, with its convention at \(w_1=0\), gives
\begin{equation}\label{eq:zero-face-exact-expansion}
    L(r_1\bar w_1,w_2)=L(0,w_2)-r_1\Psi_{w_2}(\bar w_1)+\frac{r_1^2}{4}.
\end{equation}
For a nonzero vector \(b\), \eqref{eq:l-first-derivative} shows that the spherical gradient of \(u\mapsto l(\angle(u,b))\) at \(\bar w_1\) is
\begin{equation}\label{eq:sphere-gradient-l}
    \nabla_{\mathbb S^{d-1}}\bigl[l(\angle(\cdot,b))\bigr](\bar w_1) =p\bigl(\angle(\bar w_1,b)\bigr)P_{\bar w_1}^\perp\bar b, \qquad p(\theta)=\frac{\pi-\theta}{2\pi}.
\end{equation}
With this formula, \eqref{eq:norm-w1} and \eqref{eq:direction-w1} become the first two equations below, and differentiating \eqref{eq:zero-face-exact-expansion} in \(w_2\) gives the third:
\begin{equation}\label{eq:zero-face-polar-flow}
\begin{aligned}
    \dot r_1&=\Psi_{w_2}(\bar w_1)-\frac{r_1}{2},\\
    \dot{\bar w}_1&=\frac{1}{r_1}\nabla_{\mathbb S^{d-1}}\Psi_{w_2}(\bar w_1),\\
    \dot w_2&=F_b(w_2)+r_1\nabla_{w_2}\Psi_{w_2}(\bar w_1),
\end{aligned}
\end{equation}
where \(F_b(w_2)=-\nabla_{w_2}L(0,w_2)\); thus \(\dot w_2=F_b(w_2)\) is the gradient flow of \(L\) restricted to \(\{w_1=0\}\).

The factor \(1/r_1\) in the second equation is removed by the rescaled time \(\tau(t)=\int_0^t\diff t'/r_1(t')\). With \({}'=\diff/\diff\tau\), the triple \((r_1,\bar w_1,w_2)\) solves \((r_1,\bar w_1,w_2)'=X(r_1,\bar w_1,w_2)\), where the vector field \(X\) is \(r_1\) times the right-hand side of \eqref{eq:zero-face-polar-flow}; its components are \(r_1(\Psi_{w_2}(\bar w_1)-r_1/2)\), \(\nabla_{\mathbb S^{d-1}}\Psi_{w_2}(\bar w_1)\) and \(r_1(F_b(w_2)+r_1\nabla_{w_2}\Psi_{w_2}(\bar w_1))\). These formulas make sense for every \(r_1\in\RR\) and every \(\bar w_1\in\mathbb S^{d-1}\). By \cref{lem:relu-kernel-regularity}, the kernel in \eqref{eq:relu-gaussian-kernel} is \(C^2\) in its two nonzero arguments, so \(\Psi_{w_2}(\bar w_1)\) and \(L(0,w_2)\) are \(C^2\) in \((\bar w_1,w_2)\) for \(w_2\ne0\). Hence \(X\) is a \(C^1\) vector field on \(\RR\times\mathbb S^{d-1}\times(\RR^d\setminus\{0\})\), and \(X(0,\bar w_1,w_2)=(0,\nabla_{\mathbb S^{d-1}}\Psi_{w_2}(\bar w_1),0)\).

The next lemma describes a trajectory whose first weight tends to zero.

\begin{lemma}[Behavior as \(w_1\) tends to zero]\label{lem:finite_collision_limit_constraints}
Let \(w(0)\in\Omega\), and suppose that \(r_1(t)\to0\) as \(t\uparrow T\). If \(T=\infty\), suppose also that \(\int_0^\infty\|\dot w(t)\|_2\,\diff t<\infty\). Then \(w_2(t)\) converges to some \(w_{2,*}\ne0\), \(\tau(t)\to\infty\), and \(\Psi_{w_2(t)}(\bar w_1(t))\) converges to some \(\psi_*\le0\), with \(\psi_*=0\) if \(T=\infty\). Every limit point \(u_*\) of \(\bar w_1(t)\) as \(t\uparrow T\) satisfies \(\nabla_{\mathbb S^{d-1}}\Psi_{w_{2,*}}(u_*)=0\) and \(\Psi_{w_{2,*}}(u_*)=\psi_*\).
\end{lemma}

\begin{proof}[Proof of \cref{lem:finite_collision_limit_constraints}]
By \eqref{eq:coercive-loss-bound} and \cref{lem:coercive_bound_and_no_double_zero}, there are \(0<c_0\le C_0\) and \(t_0<T\) such that \(r_1(t)\le C_0\) and \(c_0\le r_2(t)\le C_0\) for \(t_0\le t<T\). On the compact set \(\{|r_1|\le C_0,\ c_0\le\|w_2\|_2\le C_0\}\), the field \(X\), its derivative, and the right-hand sides of the first and third equations in \eqref{eq:zero-face-polar-flow} are bounded.

If \(T<\infty\), then \(\dot w_2\) and \(\dot r_1\) are bounded on \([t_0,T)\). Hence \(w_2(t)\) converges, and \(r_1(t)\le C(T-t)\) for some \(C\), so that \(\tau(t)-\tau(t_0)\ge C^{-1}\log\frac{T-t_0}{T-t}\to\infty\). If \(T=\infty\), then \(w_2(t)\) converges because its path has finite length, and \(\tau(t)-\tau(t_0)\ge(t-t_0)/C_0\to\infty\). In both cases \(\|w_{2,*}\|_2\ge c_0\), and \(\int_{\tau(t_0)}^\infty\|w_2'(\tau)\|_2\,\diff\tau =\int_{t_0}^T\|\dot w_2(t)\|_2\,\diff t<\infty\).

Let \(\psi(\tau)=\Psi_{w_2(\tau)}(\bar w_1(\tau))\). Since \(\bar w_1'=\nabla_{\mathbb S^{d-1}}\Psi_{w_2}(\bar w_1)\), we have \(\psi'=\|\nabla_{\mathbb S^{d-1}}\Psi_{w_2}(\bar w_1)\|_2^2 +D_{w_2}\Psi_{w_2}(\bar w_1)[w_2']\). The last term is integrable, because \(D_{w_2}\Psi\) is bounded on the compact set above and \(w_2'\) is integrable. Since \(\psi\) is bounded, \(\psi\) minus the integral of this term is bounded and nondecreasing. Hence \(\psi\) has a limit \(\psi_*\), and \(\int^\infty\|\nabla_{\mathbb S^{d-1}}\Psi_{w_2}(\bar w_1)\|_2^2\,\diff\tau<\infty\). The derivative of \(\tau\mapsto\|\nabla_{\mathbb S^{d-1}}\Psi_{w_2}(\bar w_1)\|_2^2\) is bounded, because \(\bar w_1'\) and \(w_2'\) are bounded and \(\nabla_{\mathbb S^{d-1}}\Psi\) is \(C^1\). Barbalat's lemma therefore gives \(\nabla_{\mathbb S^{d-1}}\Psi_{w_2}(\bar w_1)\to0\), and continuity gives the two conditions at every limit point \(u_*\).

Finally, \(\dot r_1=\Psi_{w_2}(\bar w_1)-r_1/2\) by \eqref{eq:zero-face-polar-flow}. If \(\psi_*>0\), then \(\dot r_1\ge\psi_*/4\) for \(t\) close to \(T\), which contradicts \(r_1(t)\to0\). If \(T=\infty\) and \(\psi_*<0\), then \(\dot r_1\le\psi_*/2\) for all large \(t\), so \(r_1\) reaches zero in finite time, which contradicts \(T=\infty\).
\end{proof}

\subsubsection{Equilibria of \texorpdfstring{\(X\)}{X} with \texorpdfstring{\(r_1=0\)}{r1=0}}

Let \(\mathcal Z_0\) be the set of points \((0,\bar w_1,w_2)\) with \(w_2\ne0\), \(\nabla_{\mathbb S^{d-1}}\Psi_{w_2}(\bar w_1)=0\) and \(\Psi_{w_2}(\bar w_1)\le0\). These points are equilibria of \(X\), and by \cref{lem:finite_collision_limit_constraints}, for a trajectory as in that lemma, the limit points of \((r_1,\bar w_1,w_2)\) lie in \(\mathcal Z_0\). We show that the divergence of \(X\) is positive on \(\mathcal Z_0\) except at two points, so that trajectories staying near the other points form a null set.

\begin{lemma}[A sine inequality at equilibria of \(X\)]\label{lem:nonpositive_boundary_sine_inequality}
Let \((0,\bar w_1,w_2)\in\mathcal Z_0\). Then \(r_2\sin\gamma\le\sin\theta_{11}+\sin\theta_{12}\), with equality if and only if \(\Psi_{w_2}(\bar w_1)=0\) and
\begin{equation}\label{eq:missing-teacher-cusps}
    (\bar w_1,w_2)=(-v_1,v_2) \quad\text{or}\quad (\bar w_1,w_2)=(-v_2,v_1).
\end{equation}
\end{lemma}

\begin{proof}[Proof of \cref{lem:nonpositive_boundary_sine_inequality}]
For \(\theta\in(0,\pi)\), define \(Q\) by \(Q(p(\theta))=l(\theta)/\sin\theta\), with \(p\) as in \eqref{eq:sphere-gradient-l}. By \eqref{eq:relu-kernel-profile}, \(Q(p(\theta))/p(\theta)=\cot\theta+1/(\pi-\theta)\). The function \(l(\theta)/\sin\theta\) has derivative \(-((\pi-\theta)+\sin\theta\cos\theta)/(2\pi\sin^2\theta)\), which is negative because \(\frac12|\sin2\theta|=\frac12|\sin2(\pi-\theta)|<\pi-\theta\), and second derivative \((1+(\pi-\theta)\cot\theta)/(\pi\sin^2\theta)=2l(\theta)/\sin^3\theta\), which is positive because \(l>0\) on \([0,\pi)\) by \cref{lem:l-monotone-bounds}. The derivative of \(\cot\theta+1/(\pi-\theta)\) is \(1/(\pi-\theta)^2-1/\sin^2\theta\), which is negative because \(\sin\theta=\sin(\pi-\theta)<\pi-\theta\). Since \(p\) is affine and decreasing, \(Q\) is strictly increasing and strictly convex on \((0,1/2)\), and \(Q(p)/p\) is strictly increasing there.

By \eqref{eq:sphere-gradient-l}, the condition \(\nabla_{\mathbb S^{d-1}}\Psi_{w_2}(\bar w_1)=0\) reads
\begin{equation}\label{eq:boundary-critical-tangent-balance}
    p(\theta_{11})P_{\bar w_1}^\perp v_1+p(\theta_{12})P_{\bar w_1}^\perp v_2 =r_2p(\gamma)P_{\bar w_1}^\perp\bar w_2,
\end{equation}
and the norm of the right-hand side is \(r_2p(\gamma)\sin\gamma\).

Since \(Q\) is defined only on \((0,1/2)\), we first treat the case in which \(\bar w_1\) is aligned with a teacher, say \(\theta_{11}=0\), so that \(\theta_{12}=\alpha\). Taking norms in \eqref{eq:boundary-critical-tangent-balance} gives \(p(\alpha)\sin\alpha=r_2p(\gamma)\sin\gamma\), so \(\gamma\in(0,\pi)\). Since \(\Psi_{w_2}(\bar w_1)\le0\) and \(l(0)=1/2\),
\[
    \frac12+l(\alpha) \le r_2l(\gamma) =p(\alpha)\sin\alpha\,\frac{Q(p(\gamma))}{p(\gamma)}, \qquad l(\alpha)=p(\alpha)\sin\alpha\,\frac{Q(p(\alpha))}{p(\alpha)}.
\]
Hence \(Q(p(\gamma))/p(\gamma)>Q(p(\alpha))/p(\alpha)\), so \(p(\gamma)>p(\alpha)\) and \(r_2\sin\gamma=p(\alpha)\sin\alpha/p(\gamma)<\sin\alpha =\sin\theta_{11}+\sin\theta_{12}\).

We may therefore assume that \(\theta_{11},\theta_{12}\ne0\). Let \(S=\sin\theta_{11}+\sin\theta_{12}\). Then \(S>0\), since \(S=0\) would force \(\theta_{11}=\theta_{12}=\pi\) and hence \(v_1=v_2\). A teacher with \(\theta_{1j}=\pi\) contributes nothing to \(\Psi_{w_2}(\bar w_1)\) or to \eqref{eq:boundary-critical-tangent-balance}, because \(l(\pi)=p(\pi)=0\), and we omit it from the sums below.

If the left-hand side of \eqref{eq:boundary-critical-tangent-balance} vanishes, so does the right-hand side, and \(\gamma\in\{0,\pi\}\). The case \(\gamma=\pi\) is impossible, because then \(\Psi_{w_2}(\bar w_1)=l(\theta_{11})+l(\theta_{12})>0\). Hence \(\gamma=0\) and \(r_2\sin\gamma=0<S\).

Otherwise, let \(\hat p\) be the norm of the left-hand side of \eqref{eq:boundary-critical-tangent-balance} divided by \(S\). The triangle inequality gives \(0<\hat p\le(p(\theta_{11})\sin\theta_{11}+p(\theta_{12})\sin\theta_{12})/S<1/2\), where the last inequality holds because \(p(\theta_{1j})<1/2\) for \(\theta_{1j}\ne0\). Taking norms in \eqref{eq:boundary-critical-tangent-balance} gives \(r_2p(\gamma)\sin\gamma=S\hat p>0\), so \(\gamma\in(0,\pi)\). By Jensen's inequality for the convex function \(Q\) with weights \(\sin\theta_{1j}/S\), and then by the monotonicity of \(Q\),
\begin{equation}\label{eq:boundary-jensen-chain}
    l(\theta_{11})+l(\theta_{12}) =S\sum_{j=1}^2\frac{\sin\theta_{1j}}{S}\,Q(p(\theta_{1j})) \ge SQ\Bigl(\sum_{j=1}^2\frac{\sin\theta_{1j}}{S}\,p(\theta_{1j})\Bigr) \ge SQ(\hat p).
\end{equation}
On the other hand, \(\Psi_{w_2}(\bar w_1)\le0\) gives
\[
    l(\theta_{11})+l(\theta_{12}) \le r_2l(\gamma) =r_2p(\gamma)\sin\gamma\,\frac{Q(p(\gamma))}{p(\gamma)} =S\hat p\,\frac{Q(p(\gamma))}{p(\gamma)}.
\]
Combining the two bounds, \(Q(\hat p)/\hat p\le Q(p(\gamma))/p(\gamma)\). Hence \(\hat p\le p(\gamma)\), and \(r_2\sin\gamma=S\hat p/p(\gamma)\le S\), which proves the inequality.

It remains to identify the case of equality. If \(r_2\sin\gamma=S\), then \(\hat p=p(\gamma)\), so equality holds in \eqref{eq:boundary-jensen-chain}, in the triangle inequality, and in \(\Psi_{w_2}(\bar w_1)\le0\). If both \(\sin\theta_{11}\) and \(\sin\theta_{12}\) were positive, strict convexity of \(Q\) would give \(\theta_{11}=\theta_{12}\), and equality in the triangle inequality would make \(P_{\bar w_1}^\perp v_1\) and \(P_{\bar w_1}^\perp v_2\) positive multiples of each other. These projections have the same norm, and \(\langle\bar w_1,v_1\rangle=\langle\bar w_1,v_2\rangle\), so \(v_1=v_2\), a contradiction. Hence one sine vanishes, say \(\sin\theta_{11}=0\); since \(\theta_{11}\ne0\), we have \(\theta_{11}=\pi\) and \(\bar w_1=-v_1\). Then \(S=\sin\theta_{12}>0\) and \(\hat p=p(\theta_{12})\), so \(p(\gamma)=p(\theta_{12})\) gives \(\gamma=\theta_{12}\), and \(r_2\sin\gamma=S\) gives \(r_2=1\). Now \eqref{eq:boundary-critical-tangent-balance} gives \(P_{\bar w_1}^\perp v_2=P_{\bar w_1}^\perp\bar w_2\), and \(\langle\bar w_1,v_2\rangle=\cos\theta_{12}=\cos\gamma=\langle\bar w_1,\bar w_2\rangle\), so \(w_2=v_2\). The case \(\sin\theta_{12}=0\) gives \((\bar w_1,w_2)=(-v_2,v_1)\).

Conversely, \(\Psi_{v_2}(u)=l(\angle(u,v_1))\) for every \(u\in\mathbb S^{d-1}\), because the two terms involving \(v_2\) cancel. This function is nonnegative and vanishes at \(u=-v_1\), so both \(\Psi_{v_2}(-v_1)\) and its spherical gradient vanish there. At \((\bar w_1,w_2)=(-v_1,v_2)\), we have \(r_2\sin\gamma=\sin(\pi-\alpha) =\sin\theta_{11}+\sin\theta_{12}\). The point \((-v_2,v_1)\) is symmetric.
\end{proof}

\begin{lemma}[Positive divergence of \(X\)]\label{lem:nondegenerate_boundary_volume_expansion}
Let \(\mathcal Z\) be the set \(\mathcal Z_0\) without the two points with \((\bar w_1,w_2)\) as in \eqref{eq:missing-teacher-cusps}. Then \(\operatorname{div}X>0\) on \(\mathcal Z\), where the divergence and the volume on \(\RR\times\mathbb S^{d-1}\times(\RR^d\setminus\{0\})\) are those of the product metric. Consequently, every compact \(\mathcal K\subset\mathcal Z\) has a relatively compact open neighborhood \(U\) such that the set \(\mathcal T(U)\) of points of \(U\) whose forward orbit under \(X\) stays in \(U\) is null.
\end{lemma}

\begin{proof}[Proof of \cref{lem:nondegenerate_boundary_volume_expansion}]
At \(r_1=0\), the derivative of the first component of \(X\) with respect to \(r_1\) is \(\Psi_{w_2}(\bar w_1)\), the divergence of the second component on the sphere is \(\tr\Hess_{\mathbb S^{d-1}}\Psi_{w_2}(\bar w_1)\), and the third component has the factor \(r_1\), so its divergence in \(w_2\) vanishes. To compute the trace, let \(b\) be a unit vector, \(u\in\mathbb S^{d-1}\) and \(\theta=\angle(u,b)\). By \eqref{eq:l-first-derivative}, the first and second derivatives of \(l(\arccos(\cdot))\) at \(\cos\theta\) are \((\pi-\theta)/(2\pi)\) and \(1/(2\pi\sin\theta)\). Hence the spherical Hessian of \(u\mapsto l(\angle(u,b))\) is \(\frac{1}{2\pi\sin\theta}P_u^\perp bb^\top P_u^\perp -\frac{(\pi-\theta)\cos\theta}{2\pi}P_u^\perp\), and its trace is \(\frac{\sin\theta}{2\pi}-(d-1)\frac{(\pi-\theta)\cos\theta}{2\pi} =\frac{d\sin\theta}{2\pi}-(d-1)l(\theta)\). This identity extends to \(\theta\in\{0,\pi\}\) by continuity, since \(\Psi\) is \(C^2\). Applying it to the three terms of \(\Psi_{w_2}\), we obtain
\begin{equation}\label{eq:boundary-resolved-divergence}
    \operatorname{div}X(0,\bar w_1,w_2) =\frac d{2\pi}\bigl(\sin\theta_{11}+\sin\theta_{12}-r_2\sin\gamma\bigr) -(d-2)\Psi_{w_2}(\bar w_1).
\end{equation}
On \(\mathcal Z\), the last term is nonnegative because \(d\ge2\) and \(\Psi_{w_2}(\bar w_1)\le0\), and the first term is positive by \cref{lem:nonpositive_boundary_sine_inequality}, since the points of equality are excluded. Hence \(\operatorname{div}X>0\) on \(\mathcal Z\).

Let \(\mathcal K\subset\mathcal Z\) be compact. By continuity, there are a relatively compact open neighborhood \(U\) of \(\mathcal K\) and \(c>0\) such that \(\operatorname{div}X\ge c\) on \(U\). Orbits that stay in \(U\) exist for all \(\tau\ge0\), because \(U\) is relatively compact. For each integer \(N\ge1\), the set of points of \(U\) whose orbit stays in \(U\) on \([0,N]\) is open, by continuous dependence on initial conditions, and \(\mathcal T(U)\) is the intersection of these sets. Hence \(\mathcal T(U)\) is Borel. Let \(\Phi_\tau\) be the flow of \(X\). For \(\tau>0\), \(\Phi_\tau\) maps \(\mathcal T(U)\) injectively into \(U\), and by Liouville's formula its Jacobian determinant on \(\mathcal T(U)\) is \(\exp\int_0^\tau(\operatorname{div}X)\circ\Phi_{\tau'}\,\diff\tau'\ge e^{c\tau}\). Hence \(\operatorname{vol}(U)\ge e^{c\tau}\operatorname{vol}(\mathcal T(U))\), and letting \(\tau\to\infty\) gives \(\operatorname{vol}(\mathcal T(U))=0\).
\end{proof}

We now turn this lemma into a statement about the original variables.

\begin{lemma}\label{lem:nondegenerate-approach-null}
There are null sets \(\mathcal B_1,\mathcal B_2,\ldots\subset\Omega\) with the following property. Let \(w\) be as in \cref{lem:finite_collision_limit_constraints}, and suppose that no limit point of \((\bar w_1(t),w_2(t))\) as \(t\uparrow T\) is one of the two points in \eqref{eq:missing-teacher-cusps}. Then there is \(n\) such that \(w(t)\in\mathcal B_n\) for all \(t\) sufficiently close to \(T\).
\end{lemma}

\begin{proof}[Proof of \cref{lem:nondegenerate-approach-null}]
For \(n\ge1\), let \(\mathcal K_n\) be the set of points \((0,\bar w_1,w_2)\in\mathcal Z\) with \(1/n\le\|w_2\|_2\le n\) whose distance from the two excluded points is at least \(1/n\). Each \(\mathcal K_n\) is closed and bounded, hence compact. Let \(U_n\) be the neighborhood of \(\mathcal K_n\) given by \cref{lem:nondegenerate_boundary_volume_expansion}. The map \((r_1,\bar w_1,w_2)\mapsto(r_1\bar w_1,w_2)\) is smooth between manifolds of dimension \(2d\), so it maps the null set \(\mathcal T(U_n)\) onto a null set; let \(\mathcal B_n\) be its intersection with \(\Omega\).

Let \(w\) be as in the lemma. By \cref{lem:finite_collision_limit_constraints}, the set of limit points of \((r_1,\bar w_1,w_2)\) as \(\tau\to\infty\) is a compact subset of \(\mathcal Z\). It has positive distance from the two excluded points, so it lies in some \(\mathcal K_n\). Since \((r_1,\bar w_1,w_2)\) stays bounded, it approaches this limit set as \(\tau\to\infty\), and hence lies in \(U_n\) for all large \(\tau\). It therefore lies in \(\mathcal T(U_n)\) for all large \(\tau\), and \(w(t)\in\mathcal B_n\) for all \(t\) close to \(T\).
\end{proof}

\subsubsection{The equilibria \texorpdfstring{\((0,-v_1,v_2)\) and \((0,-v_2,v_1)\)}{(0,-v1,v2) and (0,-v2,v1)}}\label{subsubsec:boundary-cusps}

At the two points in \eqref{eq:missing-teacher-cusps}, both terms in \eqref{eq:boundary-resolved-divergence} vanish, so \cref{lem:nondegenerate_boundary_volume_expansion} does not apply. We treat \((\bar w_1,w_2)=(-v_1,v_2)\); exchanging \(v_1\) and \(v_2\) gives the other point. We show that a trajectory reaching this point in finite time does so at explicit rates, and that, after a blow-up of the variables by these rates, it lies on a manifold of dimension less than \(2d\).

We write \(\bar w_1=\exp_{-v_1}z\) with \(z\in v_1^\perp\), where \(\exp_{-v_1}\) is the exponential map of the sphere at \(-v_1\), whose tangent space there is \(v_1^\perp\), and we let \(\mathsf G(z)\) be the matrix of the round metric in these normal coordinates. Then \(\mathsf G(0)=I\), the first derivatives of \(\mathsf G\) vanish at \(z=0\), and \(\mathsf G(z)z=z\) by the Gauss lemma. In rescaled time, the angular equation reads \(z'=\mathsf G(z)^{-1}\nabla_z\Psi_{w_2}(\exp_{-v_1}z)\), and therefore \(\langle z,z'\rangle=\langle z,\nabla_z\Psi_{w_2}(\exp_{-v_1}z)\rangle\). We write \(K(u,v)=\|u\|_2\|v\|_2l(\angle(u,v))\) for the kernel in \eqref{eq:relu-gaussian-kernel}, so that \(\Psi_{w_2}(\bar w_1)=K(\bar w_1,v_1)+K(\bar w_1,v_2)-K(\bar w_1,w_2)\). Since \(\angle(\exp_{-v_1}z,v_1)=\pi-\|z\|_2\), the first term equals \(l(\pi-\|z\|_2)=(\sin\|z\|_2-\|z\|_2\cos\|z\|_2)/(2\pi) =\|z\|_2^3a_0(\|z\|_2^2)\), and its gradient is \(\|z\|_2z\,a_1(\|z\|_2^2)\), where \(a_0\) and \(a_1\) are real analytic with \(a_0(0)=1/(6\pi)\) and \(a_1(0)=1/(2\pi)\). By \eqref{eq:relu-kernel-gradient},
\[
    G_0\coloneqq\nabla_{w_2}K(-v_1,w_2)\big|_{w_2=v_2} =\frac{\sin\alpha\,v_2-\alpha v_1}{2\pi}, \qquad F_b(v_2)=\frac{\sin\alpha\,v_2+(\pi-\alpha)v_1}{2\pi},
\]
and a direct computation gives \(G_0^\top F_b(v_2)=-\nu_\alpha\), where
\[
    \nu_\alpha =\frac{\alpha(\pi-\alpha)-\sin^2\alpha -(\pi-2\alpha)\sin\alpha\cos\alpha}{4\pi^2}.
\]
The numerator vanishes at \(\alpha=0\) and has derivative \(2(\pi-2\alpha)\sin^2\alpha\), which is positive on \((0,\pi/2)\); hence \(\nu_\alpha>0\) for \(\alpha\in(0,\pi/2]\). Since \(F_b(v_2)\ne0\), the flow \(\dot w_2=F_b(w_2)\) passes through \(v_2\) with nonzero velocity, and we first straighten its orbit near \(v_2\).

\begin{lemma}[Coordinates near \(v_2\)]\label{lem:cusp-polynomial-boundary-coordinates}
There are a cubic polynomial curve \(\Gamma\) with \(\Gamma(0)=v_2\), \(\Gamma'(0)=-F_b(v_2)\), \(\Gamma''(0)=DF_b(v_2)F_b(v_2)\) and \(\Gamma'(s)=-F_b(\Gamma(s))+O(s^3)\) as \(s\downarrow0\), and a linear isometry \(J:\RR^{d-1}\to F_b(v_2)^\perp\), such that \(\chi(s,y)=\Gamma(s)+Jy\) is a real-analytic diffeomorphism from a neighborhood of \((0,0)\) onto a neighborhood of \(v_2\). Writing \(w_2=\chi(s,y)\) along the trajectory, the third equation in \eqref{eq:zero-face-polar-flow} reads
\begin{equation}\label{eq:missing-teacher-flow-box}
    \dot s=-1+R_s(s,y)+r_1B_s(s,y,z), \qquad \dot y=R_y(s,y)+r_1B_y(s,y,z),
\end{equation}
where \(B=(B_s,B_y)\) is real analytic near the origin, \(R=(R_s,R_y)\) is \(C^1\), and, for some \(C,\delta>0\), \(\|R(s,y)\|_2\le C(s^3+\|y\|_2)\) whenever \(0\le s\le\delta\) and \(\|y\|_2\le\delta\). Moreover, for every \(\Lambda<\infty\), after decreasing \(\delta\), there is a real-analytic function \(\widehat R\) on a neighborhood of \(\{0\}\times\{\hat y:\|\hat y\|_2\le\Lambda\}\) such that \(R(s,s^3\hat y)=s^3\widehat R(s,\hat y)\) for \(0\le s\le\delta\) and \(\|\hat y\|_2\le\Lambda\).
\end{lemma}

\begin{proof}[Proof of \cref{lem:cusp-polynomial-boundary-coordinates}]
For \(w_2=v_2+s\omega\) with \(s>0\) and \(P_{v_2}^\perp\omega\ne0\), \eqref{eq:relu-kernel-gradient} gives
\[
    \nabla_{w_2}K(w_2,v_2) =\frac{v_2}{2} +\frac1{2\pi}\left( \frac{s\|P_{v_2}^\perp\omega\|_2}{\|w_2\|_2^2}\,w_2 -\arctan\!\left(\frac{s\|P_{v_2}^\perp\omega\|_2}{1+s\langle v_2,\omega\rangle}\right)v_2 \right).
\]
The right-hand side is real analytic in \((s,\omega)\) near \((0,-F_b(v_2))\), because \(P_{v_2}^\perp F_b(v_2)=\frac{\pi-\alpha}{2\pi}P_{v_2}^\perp v_1\ne0\). Replacing \(\nabla_{w_2}K(w_2,v_2)\) by it in \(F_b(w_2)=-w_2/2+\nabla_{w_2}K(w_2,v_1)+\nabla_{w_2}K(w_2,v_2)\) defines a real-analytic function \(\widetilde F(s,\omega)\) near \((0,-F_b(v_2))\) with \(\widetilde F(s,\omega)=F_b(v_2+s\omega)\) for \(s>0\) and \(\widetilde F(0,\omega)=F_b(v_2)\) for all \(\omega\). In particular, \(D_\omega\widetilde F(0,\cdot)\) and \(D_\omega^2\widetilde F(0,\cdot)\) vanish, and \(\partial_s\widetilde F(0,\omega)=DF_b(v_2)\omega\), since \(F_b\) is \(C^1\) by \cref{lem:relu-kernel-regularity}.

Define
\[
    \omega_1=-\frac12\partial_s\widetilde F(0,-F_b(v_2)), \qquad \omega_2=-\frac13\Bigl(\frac12\partial_s^2\widetilde F(0,-F_b(v_2)) +D_\omega\partial_s\widetilde F(0,-F_b(v_2))[\omega_1]\Bigr),
\]
and \(\Gamma(s)=v_2+s\,\omega(s)\) with \(\omega(s)=-F_b(v_2)+s\omega_1+s^2\omega_2\). Expanding \(\widetilde F(s,\omega(s))\) in \(s\), and using that \(D_\omega\widetilde F(0,\cdot)\) and \(D_\omega^2\widetilde F(0,\cdot)\) vanish, we find that the coefficients of \(1\), \(s\) and \(s^2\) in \(\widetilde F(s,\omega(s))+\Gamma'(s)\) vanish. This proves \(\Gamma'(s)=-F_b(\Gamma(s))+O(s^3)\), and \(\Gamma''(0)=2\omega_1=DF_b(v_2)F_b(v_2)\).

Fix a linear isometry \(J:\RR^{d-1}\to F_b(v_2)^\perp\). The matrix \([\Gamma'(s)\ J]\) is invertible for small \(s\), because \(\Gamma'(0)\) is orthogonal to the range of \(J\). By the inverse function theorem, \(\chi\) is a real-analytic diffeomorphism near \((0,0)\). Define
\[
\begin{aligned}
    R(s,y)&=[\Gamma'(s)\ J]^{-1}\bigl(F_b(\chi(s,y))+\Gamma'(s)\bigr),\\
    B(s,y,z)&=[\Gamma'(s)\ J]^{-1} \nabla_{w_2}\Psi_{w_2}(\exp_{-v_1}z)\big|_{w_2=\chi(s,y)}.
\end{aligned}
\]
Since \(\dot w_2=\Gamma'(s)\dot s+J\dot y\), the third equation in \eqref{eq:zero-face-polar-flow} becomes \eqref{eq:missing-teacher-flow-box}. The function \(R\) is \(C^1\) because \(F_b\) is. The function \(B\) is real analytic, because \(\nabla_{w_2}\Psi_{w_2}(\bar w_1)=-\nabla_{w_2}K(\bar w_1,w_2)\) and the angle between \(\exp_{-v_1}z\) and \(\chi(s,y)\) stays near \(\pi-\alpha\in(0,\pi)\), where \(K\) is analytic. By the property of \(\Gamma\), \(R(s,0)=O(s^3)\), and \(R\) is Lipschitz in \(y\); this gives the bound on \(R\).

Finally, for bounded \(\hat y\) we have \(\chi(s,s^3\hat y)=v_2+s\bigl(\omega(s)+s^2J\hat y\bigr)\). Since \(D_\omega\widetilde F(0,\cdot)\) and \(D_\omega^2\widetilde F(0,\cdot)\) vanish, the coefficients of \(1\), \(s\) and \(s^2\) in \(\widetilde F(s,\omega(s)+s^2J\hat y)+\Gamma'(s)\) do not depend on \(\hat y\), and so they vanish. Dividing this analytic function by \(s^3\) and multiplying by \([\Gamma'(s)\ J]^{-1}\) defines \(\widehat R(s,\hat y)\), which satisfies \(R(s,s^3\hat y)=s^3\widehat R(s,\hat y)\) for \(s>0\).
\end{proof}

Along a trajectory that reaches \((0,-v_1,v_2)\) at time \(T\), we will see in \cref{lem:missing_teacher_cusp_rates} that \(\Psi_{w_2}(\bar w_1)\) behaves like \(-\nu_\alpha(T-t)\) and \(r_1\) like \(\nu_\alpha(T-t)^2/2\), that \(s\) is close to \(T-t\), and that \(\|z\|_2\) and \(\|y\|_2\) are at most of order \((T-t)^{1/2}\) and \((T-t)^3\). We therefore blow up the variables by these powers. When \(\alpha=\pi/2\), the constant term in the angular equation of the blow-up vanishes and \(\|z\|_2\) is at most of order \(T-t\), so this case uses different powers. The next lemma gives the expansions and the equations in the blow-up variables.

\begin{lemma}[Blow-up near \((0,-v_1,v_2)\)]\label{lem:cusp-weighted-c1-extension}
Let \(h(z)=D_{w_2}K(\exp_{-v_1}z,w_2)[F_b(v_2)]\big|_{w_2=v_2}\) and \(\beta_0=B_y(0,0,0)\). Then \(\nabla h(0)=-\frac{(\pi-2\alpha)\sin\alpha}{4\pi^2}(v_2-\cos\alpha\,v_1)\), and if \(\alpha=\pi/2\), then \(\nabla^2h(0)=\frac1{16}I+\frac1{4\pi^2}v_2v_2^\top\), where \(I\) is the identity on \(v_1^\perp\). For bounded \(\hat y\) and \((z,s)\) near \((0,0)\),
\begin{equation}\label{eq:cusp-scalar-vector-analytic-decomposition}
    \Psi_{\chi(s,s^3\hat y)}(\exp_{-v_1}z)=l(\pi-\|z\|_2)+s\,g(z,s,\hat y),
\end{equation}
where \(g\) is real analytic, \(g(z,0,\hat y)=h(z)\), and, if \(\alpha=\pi/2\), \(\partial_s\nabla_zg(0,0,\hat y)=v_2/(16\pi^2)\).

Fix \(\Lambda<\infty\), and let \(\|\hat z\|_2,|\hat r|,\|\hat y\|_2\le\Lambda\). Suppose first that \(0<\alpha<\pi/2\). If we set \(s=\zeta^2\), \(z=\zeta\hat z\), \(r_1=\zeta^4\hat r\) and \(y=\zeta^6\hat y\) with \(\zeta>0\), and change time by \(\diff\sigma=\zeta\,\diff\tau\), the rescaled-time equations for \((s,z,r_1,y)\) become
\begin{equation}\label{eq:acute-cusp-weighted-system}
\begin{aligned}
    \frac{\diff\zeta}{\diff\sigma}&=-\tfrac12\zeta^2\hat r+\zeta^6E_1, &\frac{\diff\hat z}{\diff\sigma}&=q_1(\hat z)+\zeta E_2,\\
    \frac{\diff\hat r}{\diff\sigma}&=\zeta\hat r(2\hat r-\nu_\alpha)+\zeta^2E_3, &\frac{\diff\hat y}{\diff\sigma}&=\zeta(\hat r^2\beta_0+3\hat r\hat y)+\zeta^2E_4,
\end{aligned}
\end{equation}
where \(q_1(\hat z)=\|\hat z\|_2\hat z/(2\pi)+\nabla h(0)\). If \(\alpha=\pi/2\), we set instead \(z=s\hat z\), \(r_1=s^2\hat r\) and \(y=s^3\hat y\) with \(s>0\), and \(\diff\sigma=s\,\diff\tau\); this gives
\begin{equation}\label{eq:right-angle-cusp-weighted-system}
\begin{aligned}
    \frac{\diff s}{\diff\sigma}&=-s\hat r+s^3E_1, &\frac{\diff\hat z}{\diff\sigma}&=q_2(\hat z,\hat r)+sE_2,\\
    \frac{\diff\hat r}{\diff\sigma}&=\hat r(2\hat r-\nu_\alpha)+sE_3, &\frac{\diff\hat y}{\diff\sigma}&=\hat r^2\beta_0+3\hat r\hat y+sE_4,
\end{aligned}
\end{equation}
where \(q_2(\hat z,\hat r)=\|\hat z\|_2\hat z/(2\pi)+(\nabla^2h(0)+\hat rI)\hat z +v_2/(16\pi^2)\). In the first case, the remainders \(E_1,\dots,E_4\) extend to \(C^1\) functions on a neighborhood of \(\{\zeta=0\}\), including negative values of \(\zeta\); in the second case, the same holds with \(s\) in place of \(\zeta\).
\end{lemma}

\begin{proof}[Proof of \cref{lem:cusp-weighted-c1-extension}]
By \eqref{eq:relu-kernel-gradient}, \(h(z)=\bigl(\sin\theta_{12}\langle v_2,F_b(v_2)\rangle +(\pi-\theta_{12})\langle\bar w_1,F_b(v_2)\rangle\bigr)/(2\pi)\), where \(\bar w_1=\exp_{-v_1}z\). At \(z=0\), \(\bar w_1=-v_1\) and \(\theta_{12}=\pi-\alpha\), so \(h(0)=G_0^\top F_b(v_2)=-\nu_\alpha\). Since \(\bar w_1=-v_1+z+O(\|z\|_2^2)\) and \(\cos\theta_{12}=-\cos\alpha+\langle z,v_2\rangle+O(\|z\|_2^2)\), the differential of \(\theta_{12}\) at \(z=0\) is \(-\langle v_2,\diff z\rangle/\sin\alpha\). Differentiating \(h\) at \(z=0\), with \(\langle v_2,F_b(v_2)\rangle=(\sin\alpha+(\pi-\alpha)\cos\alpha)/(2\pi)\) and \(\langle v_1,F_b(v_2)\rangle=((\pi-\alpha)+\sin\alpha\cos\alpha)/(2\pi)\), gives the formula for \(\nabla h(0)\). When \(\alpha=\pi/2\), expanding \(h\) to second order in the same way gives \(h(z)=-\nu_\alpha+\|z\|_2^2/32+\langle v_2,z\rangle^2/(8\pi^2)+O(\|z\|_2^3)\), and hence the formula for \(\nabla^2h(0)\).

The terms \(K(\exp_{-v_1}z,v_2)-K(\exp_{-v_1}z,\chi(s,s^3\hat y))\) of \(\Psi_{\chi(s,s^3\hat y)}(\exp_{-v_1}z)\) are real analytic near \((z,s)=(0,0)\), because both angles stay near \(\pi-\alpha\), and they vanish at \(s=0\). Hence they equal \(s\,g(z,s,\hat y)\) with \(g\) real analytic, and \(g(z,0,\hat y)=-D_{w_2}K(\exp_{-v_1}z,v_2)[\Gamma'(0)]=h(z)\). Since \(\hat y\) enters \(\chi(s,s^3\hat y)\) only at order \(s^3\), a second-order expansion in \(s\) gives
\[
    \partial_sg(z,0,\hat y) =-\frac12\Bigl(D_{w_2}K(\exp_{-v_1}z,v_2)[\Gamma''(0)] +D_{w_2}^2K(\exp_{-v_1}z,v_2)[F_b(v_2),F_b(v_2)]\Bigr).
\]
When \(\alpha=\pi/2\), we have \(\Gamma''(0)=DF_b(v_2)F_b(v_2)=(\frac1{4\pi}-\frac18)v_1-\frac1{4\pi}v_2\), and at \(z=0\) the gradients in \(z\) of the two terms in the bracket are \(-v_2/(8\pi^2)\) and \(0\); hence \(\partial_s\nabla_zg(0,0,\hat y)=v_2/(16\pi^2)\).

By the expansion of \(l(\pi-\|z\|_2)\) given at the beginning of this subsection and \(\mathsf G(z)^{-1}z=z\), the angular equation is \(z'=\|z\|_2z\,a_1(\|z\|_2^2) +s\,\mathsf G(z)^{-1}\nabla_zg(z,s,\hat y)\), with \(y=s^3\hat y\). Because \(\mathsf G(0)=I\) and \(D\mathsf G(0)=0\), the function \(\mathsf G^{-1}\nabla_zg\) equals \(\nabla h(0)\) at \((0,0,\hat y)\), and its derivatives in \(z\) and \(s\) there are \(\nabla^2h(0)\) and \(\partial_s\nabla_zg(0,0,\hat y)\).

Suppose that \(0<\alpha<\pi/2\), and substitute \(z=\zeta\hat z\) and \(s=\zeta^2\). The functions \(g(\zeta\hat z,\zeta^2,\hat y)-g(0,0,\hat y)\) and \(\mathsf G^{-1}\nabla_zg(\zeta\hat z,\zeta^2,\hat y)-\nabla h(0)\) are analytic and vanish at \(\zeta=0\), so they are \(\zeta\) times analytic functions; also, \(a_1(\zeta^2\|\hat z\|_2^2)-1/(2\pi)\) is \(\zeta^2\) times an analytic function, and \(\hat z\mapsto\|\hat z\|_2^3\) and \(\hat z\mapsto\|\hat z\|_2\hat z\) are \(C^1\). Hence \(\Psi_{w_2}(\bar w_1)=\zeta^2(-\nu_\alpha+\zeta k_1)\) and \(z'=\zeta^2(q_1(\hat z)+\zeta k_2)\), where \(k_1\) and \(k_2\) are \(C^1\) functions of \((\zeta,\hat z,\hat y)\) on a neighborhood of \(\zeta=0\). Suppose now that \(\alpha=\pi/2\), and substitute \(z=s\hat z\). Then \(\mathsf G^{-1}\nabla_zg(s\hat z,s,\hat y)\) vanishes at \(s=0\), and its derivative in \(s\) there is \(\nabla^2h(0)\hat z+v_2/(16\pi^2)\). Taylor's formula with integral remainder gives, in the same way, \(\Psi_{w_2}(\bar w_1)=s(-\nu_\alpha+sk_1)\) and \(z'=s^2(q_2(\hat z,0)+sk_2)\), with \(k_1\) and \(k_2\) \(C^1\) on a neighborhood of \(s=0\).

By \eqref{eq:zero-face-polar-flow}, \eqref{eq:missing-teacher-flow-box} and \cref{lem:cusp-polynomial-boundary-coordinates}, the remaining rescaled-time equations are
\[
    s'=r_1\bigl(-1+s^3\widehat R_s+r_1B_s\bigr), \qquad y'=r_1\bigl(s^3\widehat R_y+r_1B_y\bigr), \qquad r_1'=r_1\bigl(\Psi_{w_2}(\bar w_1)-r_1/2\bigr).
\]
We substitute the blow-up variables and divide by \(\zeta\) to pass to the time \(\sigma\) (by \(s\) when \(\alpha=\pi/2\)), and we write \(B_y=\beta_0+\zeta\cdot(\text{analytic})\) (or \(B_y=\beta_0+s\cdot(\text{analytic})\) when \(\alpha=\pi/2\)). For example, \(s=\zeta^2\) gives \(2\zeta\zeta'=-\zeta^4\hat r(1-\zeta^6\widehat R_s-\zeta^4\hat rB_s)\), so \(\diff\zeta/\diff\sigma=\zeta'/\zeta=-\frac12\zeta^2\hat r+\zeta^6E_1\). The other equations follow in the same way from \(\hat z'=(z'-\zeta'\hat z)/\zeta\), \(\hat r'=(r_1'-4\zeta^3\zeta'\hat r)/\zeta^4\) and \(\hat y'=(y'-6\zeta^5\zeta'\hat y)/\zeta^6\), and their analogues for \(\alpha=\pi/2\); the term \(\hat r\hat z\) in \(q_2\) comes from \(-s'\hat z/s\). Each \(E_i\) is a sum of products of analytic functions, \(k_1\), \(k_2\), and polynomials in the blow-up variables. These expressions remain defined and \(C^1\) for \(\zeta\le0\) (for \(s\le0\) when \(\alpha=\pi/2\)), which gives the stated extension.
\end{proof}

\begin{lemma}[Rates of approach to \((0,-v_1,v_2)\)]\label{lem:missing_teacher_cusp_rates}
Let \(w(0)\in\Omega\) with \(T<\infty\), and suppose that \(r_1(t)\to0\), \(\bar w_1(t)\to-v_1\) and \(w_2(t)\to v_2\) as \(t\uparrow T\). Then, as \(t\uparrow T\),
\[
    \frac{s(t)}{T-t}\to1, \qquad \frac{r_1(t)}{(T-t)^2}\to\frac{\nu_\alpha}{2}, \qquad \frac{y(t)}{(T-t)^3}\to-\frac{\nu_\alpha}{6}\beta_0,
\]
and there are \(M<\infty\) and \(t_0<T\) such that \(\|z(t)\|_2\le M\sqrt{T-t}\) for \(t_0\le t<T\). If \(\alpha=\pi/2\), then also \(\|z(t)\|_2\le M(T-t)\) for \(t_0\le t<T\).
\end{lemma}

\begin{proof}[Proof of \cref{lem:missing_teacher_cusp_rates}]
Throughout the proof, \(t\) is close to \(T\), and \(C\) denotes a constant that may change from line to line. By the third equation in \eqref{eq:zero-face-polar-flow}, \(\dot w_2\to F_b(v_2)\), because \(r_1\to0\) and \(\nabla_{w_2}\Psi\) is bounded near \((-v_1,v_2)\); so \((w_2(t)-v_2)/(T-t)\to-F_b(v_2)\) and \(\|w_2(t)-v_2\|_2\le C(T-t)\).

The difference \(K(\exp_{-v_1}z,v_2)-K(\exp_{-v_1}z,w_2)\) is real analytic near \((z,w_2)=(0,v_2)\), vanishes at \(w_2=v_2\), and has gradient \(-G_0\) in \(w_2\) at \((0,v_2)\). Together with the expansion of \(l(\pi-\|z\|_2)\) given at the beginning of this subsection, this gives \(\delta>0\) and \(C_1<\infty\) such that, for \(\|z\|_2\le\delta\) and \(\|w_2-v_2\|_2\le\delta\),
\begin{equation}\label{eq:missing-teacher-coarse-expansion}
\begin{aligned}
    \Bigl|\Psi_{w_2}(\exp_{-v_1}z)-\frac{\|z\|_2^3}{6\pi} +G_0^\top(w_2-v_2)\Bigr| &\le C_1\bigl(\|z\|_2^5+\|z\|_2\|w_2-v_2\|_2+\|w_2-v_2\|_2^2\bigr),\\
    \Bigl\|\nabla_z\Psi_{w_2}(\exp_{-v_1}z)-\frac{\|z\|_2z}{2\pi}\Bigr\|_2 &\le C_1\bigl(\|w_2-v_2\|_2+\|z\|_2^3\bigr).
\end{aligned}
\end{equation}

By the first inequality, \(\Psi_{w_2(t)}(\bar w_1(t))\ge-C(T-t)\), since \(\|z\|_2^3/(6\pi)-C_1\|z\|_2^5\ge0\) for small \(\|z\|_2\). Solving the first equation in \eqref{eq:zero-face-polar-flow} backward from \(r_1(T)=0\) gives
\begin{equation}\label{eq:missing-teacher-terminal-radial-identity}
    r_1(t)=-\int_t^Te^{(t'-t)/2}\Psi_{w_2(t')}(\bar w_1(t'))\,\diff t',
\end{equation}
and hence \(0<r_1(t)\le C(T-t)^2\). By \eqref{eq:missing-teacher-flow-box}, \(\dot s\to-1\), and \(s(t)\to0\) because \(w_2\to v_2=\chi(0,0)\). Integrating \eqref{eq:missing-teacher-flow-box} from \(t\) to \(T\), and using the bound on \(R\) in \cref{lem:cusp-polynomial-boundary-coordinates} and the boundedness of \(B\), we obtain
\[
    |s(t)-(T-t)|+\|y(t)\|_2 \le C\int_t^T\bigl(r_1(t')+(T-t')^3+\|y(t')\|_2\bigr)\,\diff t'.
\]
Since \(r_1(t')\le C(T-t')^2\), Gr\"onwall's inequality gives \(|s(t)-(T-t)|+\|y(t)\|_2\le C(T-t)^3\). In particular, \(s(t)/(T-t)\to1\), and \((T-t)/2\le s(t)\le2(T-t)\).

Next we bound \(\|z\|_2\). By the second inequality in \eqref{eq:missing-teacher-coarse-expansion} and the identity for \(\langle z,z'\rangle\) given at the beginning of this subsection, the upper right Dini derivative of \(\|z\|_2\) in rescaled time satisfies \(D^+_\tau\|z\|_2\ge c_2\|z\|_2^2-C_2s\) for some \(c_2,C_2>0\). Let \(m_1=\|z\|_2/\sqrt{s}\). Since \(s'=-r_1(1-R_s-r_1B_s)\le0\), the factor \(s^{-1/2}\) does not decrease, so \(D^+_\tau m_1\ge\sqrt{s}\,(c_2m_1^2-C_2)\). Choose \(M_1\) with \(c_2M_1^2\ge2C_2\). If \(m_1\ge M_1\) at some time \(t_1\), then \(m_1\) stays above \(M_1\), and \(D^+_\tau m_1\ge(c_2/2)\sqrt{s}\,m_1^2\) gives
\[
    \frac1{m_1(t)} \le\frac1{m_1(t_1)}-\frac{c_2}{2}\int_{t_1}^t\frac{\sqrt{s(t')}}{r_1(t')}\,\diff t'.
\]
Since \(\sqrt{s}/r_1\ge C^{-1}(T-t)^{-3/2}\), the integral tends to infinity as \(t\uparrow T\), and the right-hand side becomes negative, which is impossible. Hence \(m_1<M_1\), and \(\|z\|_2\le M_1\sqrt{s}\le M_1\sqrt{2(T-t)}\).

Dividing the first inequality in \eqref{eq:missing-teacher-coarse-expansion} by \(T-t\), and using \(\|z\|_2^3/(T-t)\le C\sqrt{T-t}\to0\) and \((w_2-v_2)/(T-t)\to-F_b(v_2)\), we obtain \(\Psi_{w_2}(\bar w_1)/(T-t)\to G_0^\top F_b(v_2)=-\nu_\alpha\). Substituting \(t'=t+\lambda(T-t)\) in \eqref{eq:missing-teacher-terminal-radial-identity} gives
\[
    \frac{r_1(t)}{(T-t)^2} =-\int_0^1e^{\lambda(T-t)/2}\, \frac{\Psi_{w_2(t')}(\bar w_1(t'))}{T-t}\,\diff\lambda.
\]
For fixed \(\lambda\in[0,1)\), the quotient in the integrand equals \((1-\lambda)\Psi_{w_2(t')}(\bar w_1(t'))/(T-t')\), which tends to \(-(1-\lambda)\nu_\alpha\), and these quotients are bounded. Dominated convergence gives \(r_1(t)/(T-t)^2\to\nu_\alpha\int_0^1(1-\lambda)\,\diff\lambda =\nu_\alpha/2\). Finally, \(y(t)=-\int_t^T\bigl(R_y(s,y)+r_1B_y(s,y,z)\bigr)\,\diff t'\). The integral of \(R_y\) is \(O((T-t)^4)\) by the bound on \(R\). In the integral of \(r_1B_y\), the substitution \(t'=t+\lambda(T-t)\), the limit of \(r_1/(T-t)^2\), the limit \(B_y\to\beta_0\), and dominated convergence give \(y(t)/(T-t)^3\to-\frac{\nu_\alpha}{2}\beta_0\int_0^1(1-\lambda)^2\,\diff\lambda =-\frac{\nu_\alpha}{6}\beta_0\).

It remains to treat \(\alpha=\pi/2\), where \(\nabla h(0)=0\). Write \(y=s^3\hat y\); by the bounds above, \(\hat y\) is bounded. By \eqref{eq:cusp-scalar-vector-analytic-decomposition} and Taylor's formula for \(\nabla_zg\) at \((0,0,\hat y)\),
\[
    \nabla_z\Psi_{w_2}(\exp_{-v_1}z) =\|z\|_2z\,a_1(\|z\|_2^2) +s\nabla^2h(0)z+\frac{s^2}{16\pi^2}v_2 +O\bigl(s\|z\|_2^2+s^2\|z\|_2+s^3\bigr),
\]
uniformly in \(\hat y\). Since \(\nabla^2h(0)\) is positive definite and \(\|z\|_2\le M_1\sqrt{s}\), this gives \(D^+_\tau\|z\|_2\ge c_3\|z\|_2^2-C_3s^2\) for some \(c_3,C_3>0\). Let \(m_2=\|z\|_2/s\). As before, \(D^+_\tau m_2\ge s(c_3m_2^2-C_3)\), and \(\int_{t_1}^ts(t')/r_1(t')\,\diff t'\to\infty\) because \(s/r_1\ge C^{-1}(T-t)^{-1}\). The argument used for \(m_1\) gives \(\|z\|_2\le M_2s\le2M_2(T-t)\). The lemma holds with \(M\) the largest of the constants obtained.
\end{proof}

\begin{proposition}[Trajectories reaching \((0,-v_1,v_2)\)]\label{prop:missing_teacher_cusp_collision_null}
There is a null set \(\mathcal B\subset\Omega\) with the following property. If \(w(0)\in\Omega\), \(T<\infty\), and \(r_1(t)\to0\), \(\bar w_1(t)\to-v_1\) and \(w_2(t)\to v_2\) as \(t\uparrow T\), then \(w(t)\in\mathcal B\) for all \(t\) sufficiently close to \(T\).
\end{proposition}

\begin{proof}[Proof of \cref{prop:missing_teacher_cusp_collision_null}]
Suppose first that \(0<\alpha<\pi/2\), and consider the blow-up variables of \eqref{eq:acute-cusp-weighted-system}. By \cref{lem:missing_teacher_cusp_rates}, as \(t\uparrow T\) we have \(\zeta=\sqrt{s}\to0\), \(\hat r=r_1/s^2\to\nu_\alpha/2\), \(\hat y=y/s^3\to-\nu_\alpha\beta_0/6\), and \(\hat z=z/\sqrt{s}\) stays bounded. Moreover, \(\diff\sigma/\diff t=\sqrt{s}/r_1\) behaves like \((2/\nu_\alpha)(T-t)^{-3/2}\), so \(\sigma\to\infty\). We apply \cref{lem:cusp-weighted-c1-extension} with \(\Lambda\) larger than the bounds on \(\|\hat z\|_2\), \(|\hat r|\) and \(\|\hat y\|_2\) along the trajectory. Then \(Y=(\zeta,\hat z,\hat r,\hat y)\) is a bounded solution of \eqref{eq:acute-cusp-weighted-system} for all large \(\sigma\), with \(\zeta(\sigma)\to0\).

Since \((\pi-2\alpha)\sin\alpha>0\), we have \(\nabla h(0)\ne0\), and \(q_1=\nabla V_1\) for the strictly convex function \(V_1(\hat z)=\|\hat z\|_2^3/(6\pi)+\langle\nabla h(0),\hat z\rangle\), whose only critical point is \(\hat z_1^*=-(2\pi/\|\nabla h(0)\|_2)^{1/2}\,\nabla h(0)\). Along a solution of \(\diff\hat z/\diff\sigma=q_1(\hat z)\), the derivative of \(\frac12\|\hat z-\hat z_1^*\|_2^2\) is \(\langle\hat z-\hat z_1^*,\nabla V_1(\hat z)-\nabla V_1(\hat z_1^*)\rangle\). By the monotonicity of \(\nabla V_1\) along rays from \(\hat z_1^*\) and the compactness of the unit sphere, this derivative has a positive lower bound on \(\{\|\hat z-\hat z_1^*\|_2\ge a\}\) for each \(a>0\). Hence no solution with \(\hat z(\sigma')\ne\hat z_1^*\) at some \(\sigma'\) is defined and bounded on \([\sigma',\infty)\), so a solution that is defined and bounded on \(\RR\) is constant, equal to \(\hat z_1^*\).

Let \(\sigma_n\to\infty\). The translates \(Y(\sigma_n+\cdot)\) are bounded and have bounded derivatives, so a subsequence converges uniformly on compact intervals to a solution of \eqref{eq:acute-cusp-weighted-system} on \(\RR\) with \(\zeta\equiv0\). Its \(\hat z\)-component is a bounded solution of \(\diff\hat z/\diff\sigma=q_1(\hat z)\) on \(\RR\), so it equals \(\hat z_1^*\). Therefore \(\hat z(\sigma)\to\hat z_1^*\), and \(Y(\sigma)\) converges to \(Y_1^*=(0,\hat z_1^*,\nu_\alpha/2,-\nu_\alpha\beta_0/6)\). At \(Y_1^*\), the partial derivatives of the right-hand sides of the \(\zeta\)-, \(\hat r\)- and \(\hat y\)-equations vanish: the first is \(O(\zeta^2)\), and the other two are \(\zeta\) times functions that vanish at \(Y_1^*\), plus \(O(\zeta^2)\). In the order \((\zeta,\hat r,\hat y,\hat z)\), the Jacobian at \(Y_1^*\) is therefore block lower triangular, with a zero block of size \(d+1\) and the block \(Dq_1(\hat z_1^*)=(2\pi)^{-1}(\|\hat z_1^*\|_2I+\hat z_1^*(\hat z_1^*)^\top/\|\hat z_1^*\|_2)\), which is positive definite. By the \(C^1\) center-stable manifold theorem \citep{hirsch1977invariant}, there is a \(C^1\) manifold \(\mathcal M_1\) of dimension \(d+1\) through \(Y_1^*\) that contains every forward orbit staying in a small neighborhood of \(Y_1^*\). Hence \(Y(\sigma)\in\mathcal M_1\) for all large \(\sigma\).

Now let \(\alpha=\pi/2\), so that \(\nabla h(0)=0\), and consider the blow-up variables of \eqref{eq:right-angle-cusp-weighted-system}. By \cref{lem:missing_teacher_cusp_rates}, \(s\to0\), \(\hat r=r_1/s^2\to\nu_\alpha/2\), \(\hat y=y/s^3\to-\nu_\alpha\beta_0/6\), and \(\hat z=z/s\) stays bounded. Moreover, \(\diff\sigma/\diff t=s/r_1\) behaves like \((2/\nu_\alpha)(T-t)^{-1}\), so \(\sigma\to\infty\). At \(s=0\) and \(\hat r=\nu_\alpha/2\), the \(\hat z\)-equation is \(\diff\hat z/\diff\sigma=\nabla V_2(\hat z)\) for the strongly convex function
\[
    V_2(\hat z)=\frac{\|\hat z\|_2^3}{6\pi} +\frac12\Bigl\langle\hat z,\Bigl(\nabla^2h(0)+\frac{\nu_\alpha}{2}I\Bigr)\hat z\Bigr\rangle +\frac{\langle v_2,\hat z\rangle}{16\pi^2},
\]
whose only critical point \(\hat z_2^*\) is nonzero because \(\nabla V_2(0)=v_2/(16\pi^2)\). The argument of the acute case shows that \(Y=(s,\hat z,\hat r,\hat y)\) converges to \(Y_2^*=(0,\hat z_2^*,\nu_\alpha/2,-\nu_\alpha\beta_0/6)\). In the order \((s,\hat r,\hat z,\hat y)\), the Jacobian at \(Y_2^*\) is block lower triangular with diagonal blocks \(-\nu_\alpha/2\), \(\nu_\alpha\), \(D_{\hat z}q_2(\hat z_2^*,\nu_\alpha/2)\) and \((3\nu_\alpha/2)I\), where
\[
    D_{\hat z}q_2(\hat z_2^*,\nu_\alpha/2) =\frac1{2\pi}\Bigl(\|\hat z_2^*\|_2I+\frac{\hat z_2^*(\hat z_2^*)^\top}{\|\hat z_2^*\|_2}\Bigr) +\nabla^2h(0)+\frac{\nu_\alpha}{2}I
\]
is positive definite. Hence \(Y_2^*\) has one stable direction and \(2d-1\) unstable directions, and its stable manifold \(\mathcal M_2\), a \(C^1\) curve \citep{hirsch1977invariant}, contains \(Y(\sigma)\) for all large \(\sigma\).

Finally, we return to \(w\). In the two cases, the maps
\[
    (\zeta,\hat z,\hat r,\hat y)\mapsto \bigl(\zeta^4\hat r\exp_{-v_1}(\zeta\hat z), \chi(\zeta^2,\zeta^6\hat y)\bigr), \qquad (s,\hat z,\hat r,\hat y)\mapsto \bigl(s^2\hat r\exp_{-v_1}(s\hat z),\chi(s,s^3\hat y)\bigr),
\]
defined for \(\zeta>0\) and \(s>0\) respectively, are \(C^1\) and send \(Y(\sigma)\) to \(w(t)\). The parts of \(\mathcal M_1\) with \(\zeta>0\) and of \(\mathcal M_2\) with \(s>0\) are \(C^1\) manifolds of dimensions \(d+1\) and \(1\), both less than \(2d\) since \(d\ge2\). Their images under these maps are therefore null in \((\RR^d)^2\). Let \(\mathcal B\) be the intersection of the union of the two images with \(\Omega\). By the above, \(w(t)\in\mathcal B\) for all \(t\) sufficiently close to \(T\).
\end{proof}

\subsubsection{Avoidance of \texorpdfstring{\(\partial\Omega\)}{the boundary}}

The results above give null sets that contain \(w(t)\) for all \(t\) close to \(T\). The next lemma turns them into null sets of initial points.

\begin{lemma}[Preimages of null sets under the flow]\label{lem:null-pullback}
Let \(\mathcal A\subset\Omega\) be a null set. The set of \(w(0)\in\Omega\) for which \(w(t)\in\mathcal A\) for all \(t\) in some nonempty open subinterval of \([0,T)\) is null.
\end{lemma}

\begin{proof}[Proof of \cref{lem:null-pullback}]
Such a trajectory lies in \(\mathcal A\) at some rational time \(t_*\ge0\). Since \(L\in C^2(\Omega)\) by \cref{lem:relu-kernel-regularity}, the time-\(t_*\) map of gradient flow is a \(C^1\) diffeomorphism from the open set of initial points with \(T>t_*\) onto an open subset of \(\Omega\), with the backward flow as its inverse. Hence the set of \(w(0)\) with \(w(t_*)\in\mathcal A\) is null, and the union over rational \(t_*\) is null.
\end{proof}

\begin{proposition}[Avoidance of \(\partial\Omega\) in finite time]\label{prop:finite_zero_weight_collision_basin_null}
The set of \(w(0)\in\Omega\) with \(T<\infty\) is null.
\end{proposition}

\begin{proof}[Proof of \cref{prop:finite_zero_weight_collision_basin_null}]
Let \(w(0)\in\Omega\) with \(T<\infty\). For \(i=1,2\),
\[
    \|\nabla_{w_i}L(w)\|_2 =\bigl\|\EE\bigl[(f(x)-f^*(x))\ind(\langle w_i,x\rangle>0)\,x\bigr]\bigr\|_2 \le\bigl(\EE(f-f^*)^2\bigr)^{1/2}\bigl(\EE\|x\|_2^2\bigr)^{1/2} \le\sqrt{2dL(w(0))}
\]
along the trajectory, so \(w(t)\) is Lipschitz on \([0,T)\) and has a limit \(w(T)\). By \cref{lem:relu-kernel-regularity}, \(w(T)\notin\Omega\), since otherwise the solution would extend beyond \(T\); and \(w(T)\ne(0,0)\) by \cref{lem:coercive_bound_and_no_double_zero}. Exchanging the two students if necessary, we may assume that \(w_1(T)=0\ne w_2(T)\). Then \cref{lem:finite_collision_limit_constraints} applies.

If no limit point of \((\bar w_1(t),w_2(t))\) is one of the two points in \eqref{eq:missing-teacher-cusps}, then by \cref{lem:nondegenerate-approach-null}, \(w(t)\in\mathcal B_n\) for some \(n\) and all \(t\) close to \(T\). Otherwise, say, \((-v_1,v_2)\) is a limit point. Then \(w_2(T)=v_2\), and the limit \(\psi_*\) in \cref{lem:finite_collision_limit_constraints} equals \(\Psi_{v_2}(-v_1)=0\). Since \(\Psi_{v_2}(u)=l(\angle(u,v_1))\) vanishes only at \(u=-v_1\) by \cref{lem:l-monotone-bounds}, every limit point of \(\bar w_1(t)\) equals \(-v_1\). Hence \(\bar w_1(t)\to-v_1\), and by \cref{prop:missing_teacher_cusp_collision_null}, \(w(t)\) lies in its null set \(\mathcal B\) for all \(t\) close to \(T\). The point \((-v_2,v_1)\) is handled in the same way, with the teachers exchanged.

In all cases, \(w(t)\) lies, for all \(t\) close to \(T\), in one of countably many null sets: the sets \(\mathcal B_n\), the set \(\mathcal B\) of \cref{prop:missing_teacher_cusp_collision_null} and its analogue with the teachers exchanged, and the images of all these sets under the exchange of the two students. \Cref{lem:null-pullback} completes the proof.
\end{proof}

A trajectory that stays in \(\Omega\) for all time can still approach \(\partial\Omega\) as \(t\to\infty\). We first show that every such trajectory converges, and then that its limit lies on \(\partial\Omega\) only for a null set of initializations.

\begin{proposition}[Convergence of global solutions]\label{prop:bounded_global_interior_kl_limit}
Let \(w(0)\in\Omega\) with \(T=\infty\). Then the trajectory has finite length, it converges to some \(w_\infty\in(\RR^d)^2\), and \(0\in\partial L(w_\infty)\), where \(\partial L\) is the limiting (Mordukhovich) subdifferential. In particular, \(\nabla L(w_\infty)=0\) if \(w_\infty\in\Omega\).
\end{proposition}

\begin{proof}[Proof of \cref{prop:bounded_global_interior_kl_limit}]
Since \(L(w(t))\) is nonincreasing and nonnegative, it has a limit \(L_\infty\), and since
\begin{equation}\label{eq:global-flow-dissipation}
    \frac{\diff}{\diff t}L(w(t))=-\|\nabla L(w(t))\|_2^2,
\end{equation}
we have \(\int_0^\infty\|\nabla L(w(t))\|_2^2\,\diff t<\infty\). By \eqref{eq:coercive-loss-bound}, the trajectory is bounded. Hence there are \(t_n\to\infty\) with \(\nabla L(w(t_n))\to0\) and \(w(t_n)\to w_\infty\) for some \(w_\infty\).

For nonzero arguments, the kernel \(K\) is built from norms, inner products, algebraic operations and \(\arccos\) on \([-1,1]\), whose graph is the inverse graph of \(\cos\) on \([0,\pi]\). Extending it by zero where an argument vanishes preserves subanalyticity, and the bound \(0\le K(u,v)\le\|u\|_2\|v\|_2/2\), which follows from \cref{lem:l-monotone-bounds}, gives continuity there. Hence \(L\) is continuous and globally subanalytic on \((\RR^d)^2\). On \(\Omega\), \(\partial L(w)=\{\nabla L(w)\}\), because \(L\) is \(C^1\) there. By continuity, \(L(w_\infty)=L_\infty\), and since \(\nabla L(w(t_n))\in\partial L(w(t_n))\) tends to zero while \(L(w(t_n))\to L(w_\infty)\), the closedness of the limiting subdifferential gives \(0\in\partial L(w_\infty)\).

Because \(w_\infty\) may lie on \(\partial\Omega\), where \(L\) is not differentiable, we use the nonsmooth \L ojasiewicz inequality for continuous subanalytic functions \citep[Theorem~3.1]{bolte2007lojasiewicz}. It gives \(r_0,C,\delta_0>0\) and \(\mu\in[0,1)\) such that
\[
    |L(w)-L_\infty|^\mu\le C\operatorname{dist}\bigl(0,\partial L(w)\bigr) \qquad\text{whenever }\|w-w_\infty\|_2<r_0 \text{ and }0<|L(w)-L_\infty|<\delta_0.
\]
If \(L(w(t))=L_\infty\) at some finite time, then by \eqref{eq:global-flow-dissipation} the trajectory is constant from that time on. Otherwise, whenever \(\|w(t)-w_\infty\|_2<r_0\) and \(L(w(t))-L_\infty<\delta_0\),
\[
    -\frac{\diff}{\diff t}\bigl(L(w(t))-L_\infty\bigr)^{1-\mu} =(1-\mu)\frac{\|\nabla L(w(t))\|_2^2}{(L(w(t))-L_\infty)^\mu} \ge\frac{1-\mu}{C}\|\dot w(t)\|_2.
\]
Choose \(n\) so large that \(\|w(t_n)-w_\infty\|_2<r_0/2\), \(L(w(t_n))-L_\infty<\delta_0\), and \(\frac{C}{1-\mu}(L(w(t_n))-L_\infty)^{1-\mu}<r_0/2\). If the trajectory left the ball \(\{\|w-w_\infty\|_2<r_0\}\) after \(t_n\), integrating the last inequality up to the first exit time would show that its length after \(t_n\) is less than \(r_0/2\), which is impossible. Hence it stays in the ball, and the same inequality gives \(\int_{t_n}^\infty\|\dot w(t)\|_2\,\diff t<\infty\). The trajectory therefore converges, and its limit is \(w_\infty\).
\end{proof}

\begin{proposition}[Avoidance of \(\partial\Omega\) as \(t\to\infty\)]\label{prop:asymptotic_zero_face_basin_null}
The set of \(w(0)\in\Omega\) with \(T=\infty\) whose trajectory converges to a point of \(\partial\Omega\) is null.
\end{proposition}

\begin{proof}[Proof of \cref{prop:asymptotic_zero_face_basin_null}]
Let \(w\) be such a trajectory, with limit \(w_\infty\in\partial\Omega\). By \cref{lem:coercive_bound_and_no_double_zero}, \(w_\infty\ne(0,0)\); exchanging the two students if necessary, we may assume that \(w_\infty=(0,w_{2,*})\) with \(w_{2,*}\ne0\). By \cref{prop:bounded_global_interior_kl_limit}, the trajectory has finite length, so \cref{lem:finite_collision_limit_constraints} applies with \(T=\infty\). Moreover, by the third equation in \eqref{eq:zero-face-polar-flow}, \(\dot w_2(t)\to F_b(w_{2,*})\); since \(w_2(t)\) converges, this limit is zero. On the other hand, \(F_b(v_2)\ne0\) by the formula in \cref{subsubsec:boundary-cusps}, and \(F_b(v_1)\ne0\) by symmetry. Hence \(w_{2,*}\notin\{v_1,v_2\}\), no limit point of \((\bar w_1(t),w_2(t))\) is one of the two points in \eqref{eq:missing-teacher-cusps}, and by \cref{lem:nondegenerate-approach-null}, \(w(t)\in\mathcal B_n\) for some \(n\) and all large \(t\). Together with the images of the sets \(\mathcal B_n\) under the exchange of the two students, \cref{lem:null-pullback} completes the proof.
\end{proof}

\subsection{Saddle avoidance and global convergence}\label{proof:thm-generic-global-convergence}

It remains to show that trajectories converge to strict saddles in \(\Omega\) only for a null set of initializations.

\begin{proposition}[Avoidance of strict saddles]\label{prop:nonzero_strict_saddle_basin_null}
The set of \(w(0)\in\Omega\) with \(T=\infty\) whose trajectory converges to a strict saddle of \(L\) in \(\Omega\), that is, a stationary point in \(\Omega\) at which \(\nabla^2L\) has a negative eigenvalue, is null.
\end{proposition}

\begin{proof}[Proof of \cref{prop:nonzero_strict_saddle_basin_null}]
For \(n\ge1\), choose a function in \(C_c^\infty(\Omega)\) equal to one on a neighborhood of \(\{w:\|w\|_2\le n,\ \min(r_1,r_2)\ge1/n\}\), and let \(L_n\) be its product with \(L\), extended by zero outside \(\Omega\). Then \(L_n\) is \(C^2\) on \((\RR^d)^2\) by \cref{lem:relu-kernel-regularity}, and its gradient has compact support, so its gradient flow is defined for all time. By \citet[Theorem~28]{bah2022learning}, the set \(\mathcal S_n\) of initial points whose \(L_n\)-gradient flow converges to a strict saddle of \(L_n\) is null.

Suppose that a trajectory converges to a strict saddle \(w_\infty\in\Omega\) of \(L\). Then there are \(n\) and \(t_1\) such that \(w_\infty\) and \(w(t)\) for all \(t\ge t_1\) lie in the open set where \(L_n=L\). Hence \(w_\infty\) is a strict saddle of \(L_n\), and for every \(t\ge t_1\) the \(L_n\)-gradient flow starting at \(w(t)\) follows the original trajectory, so \(w(t)\in\mathcal S_n\). \Cref{lem:null-pullback}, applied to \(\bigcup_n\mathcal S_n\cap\Omega\), completes the proof.
\end{proof}

\begin{proof}[Proof of \cref{thm:generic_global_convergence}]
The initializations with \(w_1(0)=0\) or \(w_2(0)=0\) form a null set, so we consider \(w(0)\in\Omega\). Outside the null set of \cref{prop:finite_zero_weight_collision_basin_null}, \(T=\infty\), that is, the solution exists for all \(t\ge0\) and both weights stay nonzero. By \cref{prop:bounded_global_interior_kl_limit}, it converges to some \(w_\infty\) with \(0\in\partial L(w_\infty)\). Outside the null set of \cref{prop:asymptotic_zero_face_basin_null}, \(w_\infty\in\Omega\), and then \(\nabla L(w_\infty)=0\). By \cref{thm:squared_loss_landscape}, every stationary point in \(\Omega\) outside \(\mathcal G\) has a strictly negative Hessian direction, that is, it is a strict saddle, and by \cref{prop:nonzero_strict_saddle_basin_null} convergence to a strict saddle happens only on a null set. Hence, for almost every initialization, the solution exists for all \(t\ge0\), both weights stay nonzero, and \(w(t)\) converges to \((v_1,v_2)\) or \((v_2,v_1)\).
\end{proof}

\section{Approach to the bisector saddle}

As in \cref{lem:phase1_bisector_saddle}, we consider \(\alpha=\angle(v_1,v_2)\in(0,\pi/2]\). We recall the set \(\Omega\) of pairs \((w_1,w_2)\) with \(w_1,w_2\ne0\) and the notation \(r_i=\|w_i\|_2\) from \cref{sec:global-convergence}. For \(w(0)\in\Omega\), we write \([0,T_{\max})\) for the maximal interval on which the gradient-flow solution exists in \(\Omega\); it is well defined by \cref{lem:relu-kernel-regularity}. By \eqref{eq:def_r} and \eqref{eq:relu-kernel-profile}, \(r_-=4l(\pi-\alpha/2)\) and \(r_+=4l(\alpha/2)\), so \cref{lem:l-monotone-bounds} gives \(0<r_-<2\) and \(r_+\le2\). For unit vectors \(p,q\in\RR^d\), we also use
\begin{equation}\label{eq:phase-one-angle-chord}
    \angle(p,q)=2\arcsin\frac{\|p-q\|_2}{2}\le\frac\pi2\|p-q\|_2,
\end{equation}
which holds because \(\arcsin\) is convex on \([0,1]\), with \(\arcsin0=0\) and \(\arcsin1=\pi/2\).

\subsection{Norm bounds}

\begin{lemma}\label{lem:phase-one-norm-upper}
Let \(w(0)\in\Omega\). For \(0\le t<T_{\max}\) and \(i=1,2\), we have \(r_i(t)\le r_i(0)+t\). If, in addition, \(r_i(0)\le2\), then \(r_i(t)\le2\) for \(0\le t<T_{\max}\).
\end{lemma}

\begin{proof}[Proof of \cref{lem:phase-one-norm-upper}]
By \cref{lem:radial-velocity-bounds}, \(\dot r_i\le1-r_i/2\le1\), and integration gives the first bound. The same inequality reads \(\frac{\diff}{\diff t}(r_i-2)\le-(r_i-2)/2\), so Gr\"onwall's inequality gives \(r_i(t)-2\le(r_i(0)-2)e^{-t/2}\le0\).
\end{proof}

\begin{lemma}\label{lem:phase-one-short-time-radial}
Let \(w(0)\in\Omega\) and \(T>0\) satisfy \(\max_{i=1,2}r_i(0)+T\le r_-/4\). Then \(T_{\max}>T\), and
\[
    r_i(0)+\frac{r_-}{4}\,t\le r_i(t)\le r_i(0)+t, \qquad 0\le t\le T,\quad i=1,2.
\]
\end{lemma}

\begin{proof}[Proof of \cref{lem:phase-one-short-time-radial}]
Let \(0\le t<\min\{T,T_{\max}\}\). By \cref{lem:phase-one-norm-upper}, \(r_i(t)\le r_i(0)+t\le r_-/4\) for \(i=1,2\), so \cref{lem:radial-velocity-bounds} gives \(\dot r_i(t)\ge r_-/2-(r_1(t)+r_2(t))/2\ge r_-/4\). Integration proves the two bounds for these \(t\). If \(T_{\max}\le T\), the bounds keep \(w(t)\), \(0\le t<T_{\max}\), in the compact subset of \(\Omega\) on which \(r_i(0)\le\|w_i\|_2\le r_-/4\) for \(i=1,2\). By \cref{lem:relu-kernel-regularity}, the solution then extends beyond \(T_{\max}\), which is impossible. Hence \(T_{\max}>T\), and the bounds hold at \(t=T\) by continuity.
\end{proof}

\subsection{Alignment with the bisector and balance of the norms}

The alignment of \(\bar w_i\) with \(\bar v\) is driven by the component along \(\bar v\) of the projected teacher force in \eqref{eq:direction-w1}.

\begin{lemma}\label{lem:phase-one-teacher-force-alignment}
For a unit vector \(q\in\RR^d\), let
\[
    \psi(q)\coloneqq \Bigl\langle \bar v, P_q^\perp\bigl((\pi-\angle(q,v_1))v_1+(\pi-\angle(q,v_2))v_2\bigr) \Bigr\rangle
\]
be the component along \(\bar v\) of the projected teacher force at \(q\). Then \(\psi(q)\ge0\), with equality if and only if \(q=\bar v\) or \(q=-\bar v\).
\end{lemma}

\begin{proof}[Proof of \cref{lem:phase-one-teacher-force-alignment}]
We write \(\theta_j=\angle(q,v_j)\) for \(j=1,2\). Since \(\bar v=(v_1+v_2)/(2\cos(\alpha/2))\), we have \(\langle\bar v,v_j\rangle=\cos(\alpha/2)\) and \(\langle\bar v,q\rangle=(\cos\theta_1+\cos\theta_2)/(2\cos(\alpha/2))\). Expanding \(P_q^\perp=I-qq^\top\) therefore gives
\[
    2\cos\frac{\alpha}{2}\,\psi(q) =2\cos^2\frac{\alpha}{2}\,(2\pi-\theta_1-\theta_2) -(\cos\theta_1+\cos\theta_2) \bigl((\pi-\theta_1)\cos\theta_1+(\pi-\theta_2)\cos\theta_2\bigr).
\]
By symmetry we may assume \(\theta_1\ge\theta_2\). We set \(x=\theta_1+\theta_2\) and \(y=\theta_1-\theta_2\). The triangle inequality for angles, applied to \(q,v_1,v_2\) and to \(-q,v_1,v_2\), gives \(\alpha\le x\le2\pi-\alpha\) and \(0\le y\le\alpha\le\pi/2\). Since \(\cos\theta_1+\cos\theta_2=2\cos(x/2)\cos(y/2)\) and
\[
    (\pi-\theta_1)\cos\theta_1+(\pi-\theta_2)\cos\theta_2 =(2\pi-x)\cos\frac x2\cos\frac y2+y\sin\frac x2\sin\frac y2,
\]
the identity above becomes
\begin{equation}\label{eq:phase-one-psi-xy}
    2\cos\frac{\alpha}{2}\,\psi(q) =2(2\pi-x)\Bigl(\cos^2\frac{\alpha}{2}-\cos^2\frac x2\cos^2\frac y2\Bigr) -\frac y2\sin x\sin y .
\end{equation}
As \(x/2\in[\alpha/2,\pi-\alpha/2]\), we have \(\cos^2(x/2)\le\cos^2(\alpha/2)\), so the right-hand side is at least
\[
    2(2\pi-x)\cos^2\frac{\alpha}{2}\sin^2\frac y2-\frac y2\sin x\sin y =2\sin\frac y2\Bigl((2\pi-x)\cos^2\frac{\alpha}{2}\sin\frac y2 -\frac y2\cos\frac y2\sin x\Bigr).
\]
Since \(0\le y/2\le\pi/4\) and \(\tan u\ge u\) for \(0\le u<\pi/2\), we have \(\sin(y/2)\ge(y/2)\cos(y/2)\), and hence
\begin{equation}\label{eq:phase-one-psi-lower}
    2\cos\frac{\alpha}{2}\,\psi(q) \ge\frac y2\sin y\, \Bigl((2\pi-x)\cos^2\frac{\alpha}{2}-\sin x\Bigr).
\end{equation}
For \(x\le\pi\), the last factor is at least \(\pi\cos^2(\alpha/2)-1\ge\pi/2-1>0\), because \(\cos^2(\alpha/2)\ge1/2\). For \(\pi<x\le2\pi-\alpha\), it is positive because \(\sin x<0\) and \(2\pi-x>0\). Since \(\sin y\ge0\), it follows that \(\psi(q)\ge0\).

If \(q=\pm\bar v\), then \(\theta_1=\theta_2\), so the teacher force is a multiple of \(v_1+v_2\) and \(P_q^\perp\) annihilates it; thus \(\psi(q)=0\). Conversely, suppose that \(\psi(q)=0\). Then \eqref{eq:phase-one-psi-lower} forces \(y\sin y=0\), that is, \(y=0\), and \eqref{eq:phase-one-psi-xy} reduces to \(2(2\pi-x)(\cos^2(\alpha/2)-\cos^2(x/2))=0\). Since \(x<2\pi\), this gives \(x=\alpha\) or \(x=2\pi-\alpha\). With \(y=0\), we have \(\langle q,\bar v\rangle=\cos(x/2)/\cos(\alpha/2)\), which equals \(1\) in the first case and \(-1\) in the second. Hence \(q=\bar v\) or \(q=-\bar v\).
\end{proof}

\begin{lemma}\label{lem:phase-one-uniform-teacher-drift}
For every \(\eta\in(0,\pi/2)\) and \(\varepsilon\in(0,1]\), there is \(\delta=\delta(\alpha,\eta,\varepsilon)>0\), independent of \(d\), such that \(\psi(q)\ge\delta\) for every unit vector \(q\in\RR^d\) with \(\|q-\bar v\|_2\ge\varepsilon\) and \(\angle(q,\bar v)\le\pi-\eta\).
\end{lemma}

\begin{proof}[Proof of \cref{lem:phase-one-uniform-teacher-drift}]
Let \(a=\langle q,\bar v\rangle\) and \(b=\langle q,v_\Delta\rangle\), so that \(a^2+b^2\le1\). By \eqref{eq:two-directions}, \(v_1=\cos(\alpha/2)\bar v+\sin(\alpha/2)v_\Delta\) and \(v_2=\cos(\alpha/2)\bar v-\sin(\alpha/2)v_\Delta\), so \(\cos\angle(q,v_1)=a\cos(\alpha/2)+b\sin(\alpha/2)\) and \(\cos\angle(q,v_2)=a\cos(\alpha/2)-b\sin(\alpha/2)\). Expanding \(P_q^\perp=I-qq^\top\) in the definition of \(\psi\) shows that \(\psi(q)=\tilde\psi(a,b)\) for a continuous function \(\tilde\psi\) on the closed unit disk \(\mathcal D\subset\RR^2\) that depends only on \(\alpha\). The two conditions on \(q\) read \(a\le1-\varepsilon^2/2\) and \(a\ge-\cos\eta\). They define a compact set \(\mathcal K\subset\mathcal D\), which is nonempty because \(\varepsilon\le1\). Every point of \(\mathcal D\) is the pair \((a,b)\) of some unit vector in \(\RR^3\), so \cref{lem:phase-one-teacher-force-alignment}, applied in \(\RR^3\) to two teachers at angle \(\alpha\), shows that \(\tilde\psi\) is nonnegative on \(\mathcal D\) and vanishes only at \((1,0)\) and \((-1,0)\). Neither point lies in \(\mathcal K\). Hence \(\delta\coloneqq\min_{\mathcal K}\tilde\psi\) is positive, and it depends only on \(\alpha\), \(\eta\) and \(\varepsilon\).
\end{proof}

\begin{lemma}\label{lem:phase-one-direction-error}
Let \(w(0)\in\Omega\), and let \(E_i\coloneqq\frac12\|\bar w_i-\bar v\|_2^2\) for \(i=1,2\). For \(0\le t<T_{\max}\),
\begin{equation}\label{eq:phase-one-direction-error-single}
    \dot E_i \le-\frac{\psi(\bar w_i)}{2\pi r_i}+\frac{r_{3-i}}{2r_i} \le\frac{r_{3-i}}{2r_i}, \qquad i=1,2,
\end{equation}
and
\begin{equation}\label{eq:phase-one-direction-error-sum}
    \frac{\diff}{\diff t}(E_1+E_2) \le2\max_{j=1,2}\frac{r_{3-j}}{r_j}\,(E_1+E_2).
\end{equation}
\end{lemma}

\begin{proof}[Proof of \cref{lem:phase-one-direction-error}]
Since \(E_i=1-\langle\bar w_i,\bar v\rangle\), \eqref{eq:direction-w1} and \eqref{eq:direction-w2} give
\[
    \dot E_i =-\frac{\psi(\bar w_i)}{2\pi r_i} +\frac{\pi-\gamma}{2\pi}\,\frac{r_{3-i}}{r_i}\, \bigl\langle\bar v,P_{\bar w_i}^\perp\bar w_{3-i}\bigr\rangle .
\]
The first inequality in \eqref{eq:phase-one-direction-error-single} follows from \(0\le\pi-\gamma\le\pi\) and \(|\langle\bar v,P_{\bar w_i}^\perp\bar w_{3-i}\rangle|\le1\), and the second from \cref{lem:phase-one-teacher-force-alignment}. For \eqref{eq:phase-one-direction-error-sum}, let \(e_i=\bar w_i-\bar v\). Since \(P_{\bar w_i}^\perp\bar w_i=0\), we have \(P_{\bar w_i}^\perp\bar v=-P_{\bar w_i}^\perp e_i\) and \(P_{\bar w_i}^\perp\bar w_{3-i}=P_{\bar w_i}^\perp(e_{3-i}-e_i)\). Since \(P_{\bar w_i}^\perp\) is an orthogonal projection, the Cauchy--Schwarz inequality gives
\[
    \bigl|\bigl\langle\bar v,P_{\bar w_i}^\perp\bar w_{3-i}\bigr\rangle\bigr| =\bigl|\bigl\langle P_{\bar w_i}^\perp e_i, P_{\bar w_i}^\perp(e_{3-i}-e_i)\bigr\rangle\bigr| \le\|e_i\|_2\bigl(\|e_1\|_2+\|e_2\|_2\bigr).
\]
Dropping the nonpositive first term and summing over \(i\), we obtain
\[
    \frac{\diff}{\diff t}(E_1+E_2) \le\frac12\max_{j=1,2}\frac{r_{3-j}}{r_j}\, \bigl(\|e_1\|_2+\|e_2\|_2\bigr)^2 \le\max_{j=1,2}\frac{r_{3-j}}{r_j}\, \bigl(\|e_1\|_2^2+\|e_2\|_2^2\bigr),
\]
which is \eqref{eq:phase-one-direction-error-sum}.
\end{proof}

While the weights are small, the factor \(1/r_i\) in the first term of \eqref{eq:phase-one-direction-error-single} makes the alignment fast. The next lemma shows that it takes place on a time scale proportional to \(r_{\mathrm{init}}\).

\begin{lemma}\label{lem:phase-one-initial-alignment}
Let \(\eta\in(0,\pi/2)\), \(0<m\le M\) and \(\varepsilon\in(0,1]\). There are constants \(\Lambda,r_0>0\), depending only on \(\alpha,\eta,m,M,\varepsilon\) and not on \(d\), with the following property. If \(w(0)\in\Omega\) satisfies
\[
    r_{\mathrm{init}}\le r_0, \qquad mr_{\mathrm{init}}\le r_i(0)\le Mr_{\mathrm{init}}, \qquad \angle(\bar w_i(0),\bar v)\le\pi-\eta, \qquad i=1,2,
\]
then \(T_{\max}>\Lambda r_{\mathrm{init}}\) and \(\|\bar w_i(\Lambda r_{\mathrm{init}})-\bar v\|_2<\varepsilon\) for \(i=1,2\).
\end{lemma}

\begin{proof}[Proof of \cref{lem:phase-one-initial-alignment}]
Let \(\delta=\delta(\alpha,\eta,\varepsilon/2)\) be the constant of \cref{lem:phase-one-uniform-teacher-drift}, and set
\[
    \Lambda\coloneqq M\bigl(e^{6\pi/\delta}-1\bigr), \qquad B\coloneqq\frac{M+\Lambda}{2m}, \qquad r_0\coloneqq\min\left\{\frac{r_-}{4(M+\Lambda)}, \frac{\delta}{4\pi B(M+\Lambda)},\frac{\varepsilon^2}{8B\Lambda}\right\}.
\]
Since \(\max_ir_i(0)+\Lambda r_{\mathrm{init}}\le(M+\Lambda)r_{\mathrm{init}}\le r_-/4\), \cref{lem:phase-one-short-time-radial} gives \(T_{\max}>\Lambda r_{\mathrm{init}}\) and
\[
    mr_{\mathrm{init}}\le r_i(t)\le Mr_{\mathrm{init}}+t \le(M+\Lambda)r_{\mathrm{init}}, \qquad 0\le t\le \Lambda r_{\mathrm{init}},\quad i=1,2.
\]
In particular, \(r_{3-i}/r_i\le2B\) on this interval, and \cref{lem:phase-one-direction-error} gives
\begin{equation}\label{eq:phase-one-initial-alignment-energy}
    \dot E_i\le-\frac{\psi(\bar w_i)}{2\pi r_i}+B\le B, \qquad 0\le t\le \Lambda r_{\mathrm{init}},\quad i=1,2.
\end{equation}

Fix \(i\), and suppose that \(\|\bar w_i(t)-\bar v\|_2\ge\varepsilon/2\) for all \(t\in[0,\Lambda r_{\mathrm{init}}]\). At every such \(t\) with \(\angle(\bar w_i(t),\bar v)\le\pi-\eta\), \cref{lem:phase-one-uniform-teacher-drift}, the bound \(r_i\le(M+\Lambda)r_{\mathrm{init}}\) and the choice of \(r_0\) give
\[
    \dot E_i \le-\frac{\delta}{2\pi(M+\Lambda)r_{\mathrm{init}}}+B\le-B<0 .
\]
Let \(t_*\) be the supremum of the times \(t\le \Lambda r_{\mathrm{init}}\) such that \(\angle(\bar w_i,\bar v)\le\pi-\eta\) on \([0,t]\). By continuity, the angle condition holds on \([0,t_*]\), so \(E_i\) is decreasing there and \(\dot E_i(t_*)<0\). The angle \(\angle(\bar w_i,\bar v)=\arccos(1-E_i)\) increases with \(E_i\). If \(t_*<\Lambda r_{\mathrm{init}}\), then \(E_i(t)<E_i(t_*)\le E_i(0)\) for \(t\) slightly larger than \(t_*\), and the angle condition holds beyond \(t_*\), contradicting the definition of \(t_*\). Hence \(t_*=\Lambda r_{\mathrm{init}}\), and \(\psi(\bar w_i)\ge\delta\) on the whole interval. The first bound in \eqref{eq:phase-one-initial-alignment-energy}, together with \(r_i(t)\le Mr_{\mathrm{init}}+t\), now gives
\[
    E_i(0)-E_i(\Lambda r_{\mathrm{init}}) \ge\int_0^{\Lambda r_{\mathrm{init}}} \left(\frac{\delta}{2\pi(Mr_{\mathrm{init}}+t)}-B\right)\diff t =\frac{\delta}{2\pi}\log\left(1+\frac\Lambda M\right)-B\Lambda r_{\mathrm{init}} \ge3-\frac{\varepsilon^2}{8}>2 .
\]
This is impossible, since \(E_i\ge0\) and \(E_i(0)=1-\cos\angle(\bar w_i(0),\bar v)\le1+\cos\eta<2\). Hence there is \(t_i\in[0,\Lambda r_{\mathrm{init}}]\) with \(\|\bar w_i(t_i)-\bar v\|_2<\varepsilon/2\), that is, \(E_i(t_i)<\varepsilon^2/8\). The second bound in \eqref{eq:phase-one-initial-alignment-energy} and the choice of \(r_0\) then give \(E_i(\Lambda r_{\mathrm{init}})\le E_i(t_i)+B\Lambda r_{\mathrm{init}}<\varepsilon^2/4\), so \(\|\bar w_i(\Lambda r_{\mathrm{init}})-\bar v\|_2 =\sqrt{2E_i(\Lambda r_{\mathrm{init}})}<\varepsilon\).
\end{proof}

After this interval, the norms need a time of order one to grow. As long as the ratio \(r_{3-i}/r_i\) stays bounded, \eqref{eq:phase-one-direction-error-sum} allows the direction errors to grow only by a bounded factor over such a time.

\begin{lemma}\label{lem:phase-one-alignment-persistence}
Let \(C\ge1\), \(\eta\in(0,\pi/2)\) and \(T_\alpha\coloneqq r_-/8\). For every \(\varepsilon\in(0,1]\) and \(T>T_\alpha\), there is \(r_0=r_0(\alpha,C,\eta,\varepsilon,T)>0\), independent of \(d\), with the following property. If the initialization is \(C\)-balanced, \(0<r_{\mathrm{init}}\le r_0\), and \(\angle(\bar w_i(0),-\bar v)\ge\eta\) for \(i=1,2\), then \(T_{\max}>T\) and
\[
    \|\bar w_i(t)-\bar v\|_2<\varepsilon, \qquad |r_1(t)-r_2(t)|<\varepsilon, \qquad T_\alpha\le t\le T,\quad i=1,2.
\]
\end{lemma}

\begin{proof}[Proof of \cref{lem:phase-one-alignment-persistence}]
Let
\[
    K\coloneqq\exp\left(C+\frac{256\,T}{r_-^2}\right), \qquad \varepsilon_*\coloneqq\frac{\min\{\varepsilon,r_-/2\}}{1+2\pi T}.
\]
Let \(\Lambda\) and \(r_{\mathrm a}\) be the constants of \cref{lem:phase-one-initial-alignment} for \(m=2/(C+1)\), \(M=2\), the margin \(\eta\) and the tolerance \(\varepsilon_*/(2\sqrt K)\), and set
\[
    r_0\coloneqq\min\left\{r_{\mathrm a},\frac{r_-}{32},\frac{T_\alpha}{\Lambda}, \frac{\min\{\varepsilon,r_-/2\}}{2(2+\Lambda)}\right\}.
\]

Let \(0<r_{\mathrm{init}}\le r_0\). Since \(r_1(0)+r_2(0)=2r_{\mathrm{init}}\) and \(\max_ir_i(0)\le C\min_ir_i(0)\), we have \(2r_{\mathrm{init}}/(C+1)\le r_i(0)\le2r_{\mathrm{init}}\). Since \(r_{\mathrm{init}}\le r_-/32\), we get \(\max_ir_i(0)+T_\alpha\le r_-/16+r_-/8<r_-/4\), so \cref{lem:phase-one-short-time-radial} gives \(T_{\max}>T_\alpha\) and \(r_i(0)+r_-t/4\le r_i(t)\le r_i(0)+t\) for \(0\le t\le T_\alpha\). Hence \(r_i(T_\alpha)\ge r_-^2/32\). Moreover, \(r_{3-i}(t)/r_i(t)\le(r_{3-i}(0)+t)/(r_i(0)+r_-t/4)\). The right-hand side is a monotone function of \(t\) that equals \(r_{3-i}(0)/r_i(0)\le C\) at \(t=0\) and tends to \(4/r_-\) as \(t\to\infty\). Hence
\begin{equation}\label{eq:phase-one-early-ratio}
    \frac{r_{3-i}(t)}{r_i(t)}\le\max\left\{C,\frac4{r_-}\right\}, \qquad 0\le t\le T_\alpha,\quad i=1,2.
\end{equation}
The angular hypothesis is equivalent to \(\angle(\bar w_i(0),\bar v)\le\pi-\eta\), so \cref{lem:phase-one-initial-alignment} applies. At \(\tau_{\mathrm a}\coloneqq \Lambda r_{\mathrm{init}}\le T_\alpha\), it gives \(\|\bar w_i(\tau_{\mathrm a})-\bar v\|_2<\varepsilon_*/(2\sqrt K)\), and hence \(E_1(\tau_{\mathrm a})+E_2(\tau_{\mathrm a})<\varepsilon_*^2/(4K)\), with \(E_i\) as in \cref{lem:phase-one-direction-error}.

Let \(T'\coloneqq\min\{T,T_{\max}\}\), and let \(\tau\) be the supremum of the times \(t\in[T_\alpha,T')\) such that \(r_1,r_2>r_-^2/64\) on \([T_\alpha,t]\). Since \(r_i(T_\alpha)\ge r_-^2/32\), we have \(\tau>T_\alpha\). Since \(r_i(0)\le2r_{\mathrm{init}}\le2\), \cref{lem:phase-one-norm-upper} gives \(r_i\le2\) on \([0,T_{\max})\), so \(r_{3-i}/r_i\le128/r_-^2\) on \([T_\alpha,\tau)\). Since \(r_-<2\) and \(C\ge1\), we have \(2T_\alpha\max\{C,4/r_-\}=\max\{Cr_-/4,1\}\le C\), and therefore
\[
    \int_{\tau_{\mathrm a}}^{t}2\max_{j=1,2}\frac{r_{3-j}}{r_j}\diff u \le2T_\alpha\max\Bigl\{C,\frac4{r_-}\Bigr\}+\frac{256\,T}{r_-^2} \le\log K, \qquad \tau_{\mathrm a}\le t<\tau .
\]
Hence \eqref{eq:phase-one-direction-error-sum}, \eqref{eq:phase-one-early-ratio} and Gr\"onwall's inequality give \(E_1(t)+E_2(t)\le K(E_1(\tau_{\mathrm a})+E_2(\tau_{\mathrm a})) <\varepsilon_*^2/4\) for \(\tau_{\mathrm a}\le t<\tau\), and therefore
\begin{equation}\label{eq:phase-one-persistence-directions}
    \|\bar w_i(t)-\bar v\|_2<\varepsilon_*, \qquad \tau_{\mathrm a}\le t<\tau,\quad i=1,2.
\end{equation}

We next bound \(r_1-r_2\). By \eqref{eq:norm-w1}, \eqref{eq:norm-w2} and \eqref{eq:relu-kernel-profile},
\[
    \frac{\diff}{\diff t}(r_1-r_2) =\sum_{j=1}^2\bigl(l(\theta_{1j})-l(\theta_{2j})\bigr) -\Bigl(\frac12-l(\gamma)\Bigr)(r_1-r_2).
\]
Since \(l(\gamma)\le1/2\) by \cref{lem:l-monotone-bounds}, solving this linear equation gives
\[
    |r_1(t)-r_2(t)|\le|r_1(0)-r_2(0)| +\int_0^t\Bigl|\sum_{j=1}^2\bigl(l(\theta_{1j})-l(\theta_{2j})\bigr)\Bigr| \diff u .
\]
The integrand is at most \(1\), because \(0\le l\le1/2\). It is also at most \(\gamma\), because \(|l'|\le1/2\) by \eqref{eq:l-first-derivative} and \(|\theta_{1j}-\theta_{2j}|\le\gamma\) by the triangle inequality for angles. On \([\tau_{\mathrm a},\tau)\), \eqref{eq:phase-one-angle-chord} and \eqref{eq:phase-one-persistence-directions} give \(\gamma\le\angle(\bar w_1,\bar v)+\angle(\bar w_2,\bar v)<\pi\varepsilon_*\). Using the bound \(1\) on \([0,\tau_{\mathrm a}]\), \(|r_1(0)-r_2(0)|\le2r_{\mathrm{init}}\) and the choice of \(\varepsilon_*\) and \(r_0\), we obtain
\begin{equation}\label{eq:phase-one-persistence-balance}
    |r_1(t)-r_2(t)|\le(2+\Lambda)r_{\mathrm{init}}+\pi\varepsilon_*T <\min\left\{\varepsilon,\frac{r_-}2\right\}, \qquad 0\le t<\tau .
\end{equation}

Suppose that \(\tau<T'\). Then \(r_i(\tau)=r_-^2/64\) for some \(i\), while \(r_i>r_-^2/64\) on \([T_\alpha,\tau)\). By \eqref{eq:phase-one-persistence-balance} and continuity, \(r_{3-i}(\tau)\le r_-^2/64+r_-/2\), and \cref{lem:radial-velocity-bounds} gives \(\dot r_i(\tau)\ge r_-/2-r_-^2/64-r_-/4>0\), because \(r_-<2\). This contradicts \(r_i>r_i(\tau)\) on \([T_\alpha,\tau)\), so \(\tau=T'\). If \(T_{\max}\le T\), then \(\min\{r_i(0),r_-^2/64\}\le r_i\le2\) on \([0,T_{\max})\) for \(i=1,2\), so the trajectory stays in a compact subset of \(\Omega\); \cref{lem:relu-kernel-regularity} then extends the solution beyond \(T_{\max}\), which is impossible. Hence \(T_{\max}>T\) and \(\tau=T\). By continuity, \eqref{eq:phase-one-persistence-balance} also holds at \(t=T\), and \(\|\bar w_i(T)-\bar v\|_2\le\varepsilon_*\). Since \(\varepsilon_*<\varepsilon\) and \(\tau_{\mathrm a}\le T_\alpha\), this proves the lemma.
\end{proof}

\subsection{Proofs of \texorpdfstring{\cref{lem:phase1_bisector_saddle,cor:phase-one-gaussian-initialization}}{the Phase 1 results}}\label{proof:phase-one}

\begin{proof}[Proof of \cref{lem:phase1_bisector_saddle}]
If the conclusion holds for some \(\rho\), it holds with the same \(t_0\) and \(r_*\) for every larger \(\rho\), so we may assume \(\rho\le1\). Let \(T_\alpha=r_-/8\) be as in \cref{lem:phase-one-alignment-persistence}, and set
\[
    t_0\coloneqq T_\alpha+\log\frac{16}{\rho}, \qquad \varepsilon\coloneqq\frac{\rho}{20}, \qquad r_*\coloneqq\min\{r_0(\alpha,C,\eta,\varepsilon,t_0),1\},
\]
where \(r_0\) is given by \cref{lem:phase-one-alignment-persistence}. Then \(t_0\) depends only on \(\alpha\) and \(\rho\), and \(r_*\) only on \(\alpha,C,\eta,\rho\). Suppose that \(0<r_{\mathrm{init}}\le r_*\). By \cref{lem:phase-one-alignment-persistence}, the solution exists on \([0,t_0]\), both weights are nonzero there, and the bounds of that lemma hold on \([T_\alpha,t_0]\). Moreover, \(r_i(0)\le2r_{\mathrm{init}}\le2\), so \(r_i\le2\) on \([0,t_0]\) by \cref{lem:phase-one-norm-upper}.

Adding \eqref{eq:norm-w1} and \eqref{eq:norm-w2}, and using \(r_+=4l(\alpha/2)\) and \eqref{eq:relu-kernel-profile}, we obtain
\[
    \frac{\diff}{\diff t}(r_1+r_2)+(r_1+r_2)-r_+ =\sum_{i,j=1}^2\Bigl(l(\theta_{ij})-l\Bigl(\frac\alpha2\Bigr)\Bigr) +\Bigl(\frac12-l(\gamma)\Bigr)(r_1+r_2).
\]
We bound the right-hand side on \([T_\alpha,t_0]\). Since \(|l'|\le1/2\) by \eqref{eq:l-first-derivative} and \(|\theta_{ij}-\alpha/2|\le\angle(\bar w_i,\bar v)\) by the triangle inequality for angles, \eqref{eq:phase-one-angle-chord} gives \(|l(\theta_{ij})-l(\alpha/2)|<\pi\varepsilon/4\). In the same way, \(\gamma<\pi\varepsilon\), and hence \(0\le\frac12-l(\gamma)=l(0)-l(\gamma)<\pi\varepsilon/2\). Since \(r_1+r_2\le4\), the right-hand side is less than \(3\pi\varepsilon\) in absolute value. Solving the linear equation for \(r_1+r_2-r_+\) on \([T_\alpha,t_0]\), and using \(|r_1(T_\alpha)+r_2(T_\alpha)-r_+|\le4\), we get \(|r_1(t_0)+r_2(t_0)-r_+|\le4e^{-(t_0-T_\alpha)}+3\pi\varepsilon =\rho/4+3\pi\varepsilon\). Recall from \cref{sec:training_dynamics} that \(r_s=r_+/2\) and \(w_s=r_s\bar v\). Hence \(r_i-r_s=\frac12(r_1+r_2-r_+)\pm\frac12(r_1-r_2)\), and therefore \(|r_i(t_0)-r_s|<\rho/8+3\pi\varepsilon/2+\varepsilon/2\). Since \(w_i-w_s=(r_i-r_s)\bar w_i+r_s(\bar w_i-\bar v)\) and \(r_s\le1\), it follows that
\[
    \|w(t_0)-w_{\mathrm{saddle}}\|_2 \le\sum_{i=1}^2\|w_i(t_0)-w_s\|_2 <\frac\rho4+(3\pi+3)\varepsilon<\rho .
\]
Thus \(t_{\rm in}=t_0\) has the required property.
\end{proof}

\begin{proof}[Proof of \cref{cor:phase-one-gaussian-initialization}]
Let \(\lambda\coloneqq\min\{1/2,(C^2-1)/(C^2+1)\}>0\). For each \(i\), \(\|w_i(0)\|_2^2/\sigma^2\) has the \(\chi^2\) distribution with \(d\) degrees of freedom, whose moment generating function is \((1-2s)^{-d/2}\) for \(s<1/2\). Markov's inequality applied to \(\exp(s\|w_i(0)\|_2^2/\sigma^2)\), with \(s=\lambda/(2(1+\lambda))\) and with \(s=-\lambda/(2(1-\lambda))\), gives
\[
\begin{aligned}
    \PP\bigl\{\|w_i(0)\|_2^2\ge(1+\lambda)\sigma^2d\bigr\} &\le\exp\Bigl(-\frac d2\bigl(\lambda-\log(1+\lambda)\bigr)\Bigr) \le e^{-d\lambda^2/8},\\
    \PP\bigl\{\|w_i(0)\|_2^2\le(1-\lambda)\sigma^2d\bigr\} &\le\exp\Bigl(-\frac d2\bigl(-\lambda-\log(1-\lambda)\bigr)\Bigr) \le e^{-d\lambda^2/4}.
\end{aligned}
\]
The last inequalities hold because, for \(0<\lambda\le1/2\),
\[
    \lambda-\log(1+\lambda)=\int_0^\lambda\frac{u}{1+u}\diff u \ge\frac{\lambda^2}4, \qquad -\lambda-\log(1-\lambda)=\int_0^\lambda\frac{u}{1-u}\diff u \ge\frac{\lambda^2}2 .
\]
Let \(\mathcal E\) be the event that \((1-\lambda)\sigma^2d<\|w_i(0)\|_2^2<(1+\lambda)\sigma^2d\) and \(\angle(\bar w_i(0),-\bar v)\ge\eta\) for \(i=1,2\). If the norm condition holds but the angular condition fails for some \(i\), then \(\langle w_i(0),\bar v\rangle<-\|w_i(0)\|_2\cos\eta <-\sigma\sqrt{(1-\lambda)d}\,\cos\eta\). Here \(\langle w_i(0),\bar v\rangle/\sigma\) is standard normal, and a standard normal \(Z\) satisfies \(\PP\{Z\le-a\}\le e^{-a^2/2}\) for \(a\ge0\). A union bound therefore gives
\[
    \PP(\mathcal E^{c}) \le4e^{-d\lambda^2/8}+2e^{-(1-\lambda)d\cos^2\eta/2} \le6e^{-cd}, \qquad c\coloneqq\min\left\{\frac{\lambda^2}8, \frac{(1-\lambda)\cos^2\eta}2\right\}.
\]
On \(\mathcal E\), both initial weights are nonzero, and the choice of \(\lambda\) gives
\[
    \max_i\|w_i(0)\|_2 \le\sqrt{\frac{1+\lambda}{1-\lambda}}\,\min_i\|w_i(0)\|_2 \le C\min_i\|w_i(0)\|_2, \qquad r_{\mathrm{init}}<\sigma\sqrt{(1+\lambda)d}\le2\sigma\sqrt d\le r_* .
\]
Hence \cref{lem:phase1_bisector_saddle} applies on \(\mathcal E\) and gives the conclusion.
\end{proof}

\section{Exclusion of saddles and global convergence}\label{sec:signed-transverse-exclusion}

As in \cref{thm:squared_loss_landscape}, we consider \(\alpha=\angle(v_1,v_2)\in(0,\pi/2]\). We recall the set \(\Omega\) of pairs \((w_1,w_2)\) with \(w_1,w_2\ne0\) and the notation \(r_i=\|w_i\|_2\) from \cref{sec:global-convergence}. For \(w(0)\in\Omega\), we write \([0,T_{\max})\) for the maximal interval on which the gradient-flow solution exists in \(\Omega\); it is well defined by \cref{lem:relu-kernel-regularity}. We use the coordinates \(z_{i,j}\) of \cref{subsubsec:homo_init} with respect to the orthonormal basis \(\{\bar v,v_\Delta,u_3,\ldots,u_d\}\).

Since the teachers are unit vectors at angle \(\alpha\), \eqref{eq:two-directions} gives \(v_1=\cos(\alpha/2)\,\bar v+\sin(\alpha/2)\,v_\Delta\) and \(v_2=\cos(\alpha/2)\,\bar v-\sin(\alpha/2)\,v_\Delta\). Hence
\begin{equation}\label{eq:signed-transverse-angle-difference}
    \cos\theta_{i1}-\cos\theta_{i2} =\langle\bar w_i,v_1-v_2\rangle =\frac{2\sin(\alpha/2)}{r_i}\,z_{i,2}, \qquad i=1,2.
\end{equation}
Since \(\arccos\) is strictly decreasing, \(\theta_{i2}-\theta_{i1}\) has the sign of \(z_{i,2}\) and vanishes with it. Taking inner products of \eqref{eq:gf-w1} and \eqref{eq:gf-w2} with \(\bar v\), \(v_\Delta\) and \(u_j\), and noting that the terms in \(\bar w_i\) and \(w_i\) there combine to \(a_iw_i/(2\pi r_i)\), with \(a_i\) as in \cref{lem:critical_eq_from_flow}, we obtain, on \(\Omega\),
\begin{equation}\label{eq:structured-escape-component-ode}
\begin{aligned}
    \dot z_{1,1} &= \frac1{2\pi}\left[ (2\pi-\theta_{11}-\theta_{12})\cos\frac{\alpha}{2} +\frac{a_1}{r_1}z_{1,1}-(\pi-\gamma)z_{2,1} \right],\\
    \dot z_{2,1} &= \frac1{2\pi}\left[ (2\pi-\theta_{21}-\theta_{22})\cos\frac{\alpha}{2} +\frac{a_2}{r_2}z_{2,1}-(\pi-\gamma)z_{1,1} \right],\\
    \dot z_{1,2} &= \frac1{2\pi}\left[ (\theta_{12}-\theta_{11})\sin\frac{\alpha}{2} +\frac{a_1}{r_1}z_{1,2}-(\pi-\gamma)z_{2,2} \right],\\
    \dot z_{2,2} &= \frac1{2\pi}\left[ (\theta_{22}-\theta_{21})\sin\frac{\alpha}{2} +\frac{a_2}{r_2}z_{2,2}-(\pi-\gamma)z_{1,2} \right],\\
    \dot z_{1,j} &= \frac1{2\pi}\left[ \frac{a_1}{r_1}z_{1,j}-(\pi-\gamma)z_{2,j} \right], \qquad 3\le j\le d,\\
    \dot z_{2,j} &= \frac1{2\pi}\left[ \frac{a_2}{r_2}z_{2,j}-(\pi-\gamma)z_{1,j} \right], \qquad 3\le j\le d.
\end{aligned}
\end{equation}

By \eqref{eq:signed-transverse-angle-difference}, the first term in the equation for \(z_{i,2}\), which comes from the teachers, vanishes together with \(z_{i,2}\). At a zero of one planar splitting coordinate, the sign of its derivative is therefore determined by the planar splitting coordinate of the other student.

\begin{lemma}[Signs of the planar splitting coordinates]\label{lem:planar-split-sign-preservation}
Let \(w(0)\in\Omega\) with \(z_{1,2}(0)z_{2,2}(0)<0\). Then \(z_{1,2}(t)z_{2,2}(t)<0\) for \(0\le t<T_{\max}\).
\end{lemma}

\begin{proof}[Proof of \cref{lem:planar-split-sign-preservation}]
Exchanging the two students if necessary, we may assume that \(z_{1,2}(0)>0>z_{2,2}(0)\). By \eqref{eq:signed-transverse-angle-difference}, the right-hand sides of the third and fourth lines of \eqref{eq:structured-escape-component-ode} vanish on the subspace where \(z_{1,2}=z_{2,2}=0\), so the vector field of the flow is tangent to this subspace. The field is locally Lipschitz on \(\Omega\) by \cref{lem:relu-kernel-regularity}, and uniqueness of solutions in both time directions shows that a trajectory that meets the subspace lies in it at all times. Since \(z_{1,2}(0)\ne0\), the trajectory never meets it.

Suppose now that \(z_{1,2}(t)z_{2,2}(t)=0\) for some \(t<T_{\max}\), and let \(t_1\) be the first such time. By the previous paragraph, only one of the two coordinates vanishes at \(t_1\). We treat the case \(z_{1,2}(t_1)=0>z_{2,2}(t_1)\); the other is symmetric. Then \(\gamma(t_1)<\pi\), since otherwise \(w_2(t_1)\) would be a negative multiple of \(w_1(t_1)\) and \(z_{2,2}(t_1)\) would vanish as well. As \(\theta_{11}=\theta_{12}\) at \(t_1\), the third line of \eqref{eq:structured-escape-component-ode} gives \(\dot z_{1,2}(t_1)=-(\pi-\gamma(t_1))z_{2,2}(t_1)/(2\pi)>0\). This is impossible: \(z_{1,2}\) is positive on \([0,t_1)\) and vanishes at \(t_1\), so \(\dot z_{1,2}(t_1)\le0\).
\end{proof}

Under the hypothesis of \cref{lem:planar-split-sign-preservation}, \(z_{1,2}\) and \(z_{2,2}\) keep opposite signs, and \(\theta_{i2}-\theta_{i1}\) has the sign of \(z_{i,2}\). The third and fourth lines of \eqref{eq:structured-escape-component-ode} then give, for \(0\le t<T_{\max}\),
\begin{equation}\label{eq:signed-transverse-planar-abs}
    \frac{\diff}{\diff t}|z_{i,2}| =\frac1{2\pi}\left[ |\theta_{i2}-\theta_{i1}|\sin\frac{\alpha}{2} +\frac{a_i}{r_i}|z_{i,2}| +(\pi-\gamma)|z_{3-i,2}| \right], \qquad i=1,2.
\end{equation}
The equations for the out-of-plane coordinates in \eqref{eq:structured-escape-component-ode} involve the same factors \(a_i/r_i\) and \(\pi-\gamma\) as \eqref{eq:signed-transverse-planar-abs}, but no term coming from the teachers.

\begin{lemma}[Comparison of out-of-plane and planar coordinates]\label{lem:signed-transverse-cone-invariance}
Let \(w(0)\) satisfy \cref{assumption:homogeneous_initialization}. Then \(|z_{i,j}(t)|\le C_\perp|z_{i,2}(t)|\) and \(|z_{1,j}(t)-z_{2,j}(t)|\le C_\perp|z_{1,2}(t)-z_{2,2}(t)|\) for \(0\le t<T_{\max}\), \(i\in\{1,2\}\) and \(3\le j\le d\).
\end{lemma}

\begin{proof}[Proof of \cref{lem:signed-transverse-cone-invariance}]
By \cref{lem:planar-split-sign-preservation}, \eqref{eq:signed-transverse-planar-abs} holds on \([0,T_{\max})\). Fix \(j\) with \(3\le j\le d\). For \(i=1,2\), we set \(p_i^\pm\coloneqq C_\perp|z_{i,2}|\mp z_{i,j}\). The first claim says that all four functions \(p_i^\pm\) are nonnegative, and by \cref{assumption:homogeneous_initialization} this holds at \(t=0\). By \eqref{eq:signed-transverse-planar-abs} and the last two lines of \eqref{eq:structured-escape-component-ode},
\[
    \dot p_i^\pm =\frac1{2\pi}\left[ C_\perp|\theta_{i2}-\theta_{i1}|\sin\frac{\alpha}{2} +\frac{a_i}{r_i}p_i^\pm +(\pi-\gamma)p_{3-i}^\mp \right].
\]
By \cref{lem:planar-split-sign-preservation} and \eqref{eq:signed-transverse-angle-difference}, \(\theta_{i2}\ne\theta_{i1}\). Hence, if all four functions are nonnegative and one of them vanishes, its derivative is positive. Let \(t_2\) be the supremum of the times \(t<T_{\max}\) such that all four functions are nonnegative on \([0,t]\), and suppose that \(t_2<T_{\max}\). Then one of them is negative at times arbitrarily close to \(t_2\) from above. By continuity, it vanishes at \(t_2\) while the others are nonnegative there, so its derivative at \(t_2\) is positive and it is positive just after \(t_2\), which is a contradiction. This proves the first claim. Since \(z_{1,2}\) and \(z_{2,2}\) have opposite signs, \(|z_{1,2}|+|z_{2,2}|=|z_{1,2}-z_{2,2}|\), and the second claim follows from \(|z_{1,j}-z_{2,j}|\le|z_{1,j}|+|z_{2,j}|\).
\end{proof}

\begin{proof}[Proof of \cref{thm:homogeneous_excludes_out_of_plane}]
The first claim is the second estimate of \cref{lem:signed-transverse-cone-invariance}. If \(d=2\), the family \(\mathcal S_\perp\) is empty by \cref{thm:squared_loss_landscape}, so we may assume that \(d\ge3\). As in \cref{prop:3D}, we write \(w_i^\perp=\sum_{j=3}^dz_{i,j}u_j\) for the component of \(w_i\) orthogonal to the teacher plane. Since \(u_3,\ldots,u_d\) are orthonormal, \cref{lem:signed-transverse-cone-invariance} gives
\begin{equation}\label{eq:signed-transverse-cone}
    \|w_1^\perp(t)-w_2^\perp(t)\|_2 \le C_\perp\sqrt{d-2}\, |z_{1,2}(t)-z_{2,2}(t)|.
\end{equation}

Let \(\widehat w\in\mathcal S_\perp\). By \eqref{eq:out-of-plane-saddle-form}, \(\widehat w_1\) and \(\widehat w_2\) are orthogonal to \(v_\Delta\), and \(\|\widehat w_1^\perp-\widehat w_2^\perp\|_2 =2r_\alpha(\kappa_\alpha)\sin\kappa_\alpha>0\). By the Cauchy--Schwarz inequality,
\[
\begin{aligned}
    |z_{1,2}(t)-z_{2,2}(t)| &=\bigl|\bigl\langle(w_1(t)-\widehat w_1)-(w_2(t)-\widehat w_2), v_\Delta\bigr\rangle\bigr|\\
    &\le\|w_1(t)-\widehat w_1\|_2+\|w_2(t)-\widehat w_2\|_2 \le\sqrt2\,\|w(t)-\widehat w\|_2,
\end{aligned}
\]
and since orthogonal projection does not increase norms, the same bound holds with \(w_i(t)-\widehat w_i\) replaced by \(w_i^\perp(t)-\widehat w_i^\perp\). The triangle inequality and \eqref{eq:signed-transverse-cone} now give
\[
    2r_\alpha(\kappa_\alpha)\sin\kappa_\alpha \le\|w_1^\perp(t)-w_2^\perp(t)\|_2+\sqrt2\,\|w(t)-\widehat w\|_2 \le\sqrt2\bigl(1+C_\perp\sqrt{d-2}\bigr)\|w(t)-\widehat w\|_2.
\]
Hence the theorem holds with \(c=\sqrt2\,r_\alpha(\kappa_\alpha)\sin\kappa_\alpha/(1+C_\perp\sqrt{d-2})\).
\end{proof}

Every in-plane saddle in \eqref{eq:in-plane-saddle-form} has zero planar splitting coordinates, while near each of them \eqref{eq:signed-transverse-planar-abs} makes at least one of these coordinates grow.

\begin{lemma}[Avoidance of in-plane saddles]\label{lem:signed-in-plane-saddle-exclusion}
Let \(w(0)\in\Omega\) with \(z_{1,2}(0)z_{2,2}(0)<0\) and \(T_{\max}=\infty\). Then \(w(t)\) does not converge to an in-plane saddle in \eqref{eq:in-plane-saddle-form}.
\end{lemma}

\begin{proof}[Proof of \cref{lem:signed-in-plane-saddle-exclusion}]
Suppose that \(w(t)\) converges to such a saddle \(w_\infty\). Both components of \(w_\infty\) are multiples of \(\bar v\), so \(z_{1,2}(t)\) and \(z_{2,2}(t)\) tend to zero. By \cref{lem:planar-split-sign-preservation}, \eqref{eq:signed-transverse-planar-abs} holds for all \(t\ge0\).

First suppose that \(w_\infty\) has one of the first two forms in \eqref{eq:in-plane-saddle-form}, that is, \(w_\infty=(s_1\bar v,s_2\bar v)\) or \(w_\infty=(-s_1\bar v,-s_2\bar v)\) with \(s_1,s_2>0\). At \(w_\infty\), we have \(\gamma=0\) and \(\sin\theta_{ij}=\sin(\alpha/2)\) for all \(i,j\), so \(a_i/r_i+\pi-\gamma=2\sin(\alpha/2)/s_i\). By continuity, \(a_i/r_i+\pi-\gamma\ge\sin(\alpha/2)/s_i\) for all sufficiently large \(t\). Adding the two equations in \eqref{eq:signed-transverse-planar-abs} and dropping the nonnegative terms that contain \(|\theta_{i2}-\theta_{i1}|\), we obtain, for such \(t\),
\[
    \frac{\diff}{\diff t}\bigl(|z_{1,2}|+|z_{2,2}|\bigr) \ge\frac1{2\pi}\sum_{i=1}^2 \left(\frac{a_i}{r_i}+\pi-\gamma\right)|z_{i,2}| \ge\frac{\sin(\alpha/2)}{2\pi\max\{s_1,s_2\}} \bigl(|z_{1,2}|+|z_{2,2}|\bigr).
\]
Thus \(|z_{1,2}|+|z_{2,2}|\) grows exponentially for large \(t\), contradicting its convergence to zero.

Otherwise, \(w_\infty\) has the third form in \eqref{eq:in-plane-saddle-form}. Let \(i\) be the index of the student whose limit is \(-r_-\bar v\). At \(w_\infty\), we have \(\gamma=\pi\) and \(\sin\theta_{i1}=\sin\theta_{i2}=\sin(\alpha/2)\), so \(a_i(w_\infty)=2\sin(\alpha/2)-\pi r_-\), which is positive because \(r_-<(2/\pi)\sin(\alpha/2)\) by \eqref{eq:def_r}. By continuity, \(a_i/r_i\ge a_i(w_\infty)/(2r_-)\) for all sufficiently large \(t\), and \eqref{eq:signed-transverse-planar-abs} gives \(\frac{\diff}{\diff t}|z_{i,2}|\ge a_i(w_\infty)|z_{i,2}|/(4\pi r_-)\). Again \(|z_{i,2}|\) grows exponentially, contradicting its convergence to zero.
\end{proof}

The conclusion of the next corollary is contained in \cref{thm:generic_global_convergence}. What is new is the proof, in which \cref{thm:homogeneous_excludes_out_of_plane} and \cref{lem:signed-in-plane-saddle-exclusion} take the place of the stable-manifold argument of \cref{prop:nonzero_strict_saddle_basin_null}.

\begin{corollary}[Global convergence under \cref{assumption:homogeneous_initialization}]\label{cor:signed-transverse-global-convergence}
For Lebesgue-almost every initialization satisfying \cref{assumption:homogeneous_initialization}, the solution exists for all \(t\ge0\), both student weights remain nonzero, and \(w(t)\) converges to \((v_1,v_2)\) or \((v_2,v_1)\).
\end{corollary}

\begin{proof}[Proof of \cref{cor:signed-transverse-global-convergence}]
As in the proof of \cref{thm:generic_global_convergence}, outside the null set of \cref{prop:finite_zero_weight_collision_basin_null} we have \(T_{\max}=\infty\), and by \cref{prop:bounded_global_interior_kl_limit} the solution converges to some \(w_\infty\). Outside the null set of \cref{prop:asymptotic_zero_face_basin_null}, \(w_\infty\in\Omega\), and then \(\nabla L(w_\infty)=0\). Letting \(t\to\infty\) in \cref{thm:homogeneous_excludes_out_of_plane} shows that \(w_\infty\notin\mathcal S_\perp\), and \cref{lem:signed-in-plane-saddle-exclusion} shows that \(w_\infty\) is not an in-plane saddle. By \cref{thm:squared_loss_landscape}, \(w_\infty\) is \((v_1,v_2)\) or \((v_2,v_1)\).
\end{proof}

\section{Approach to the bisector saddle under structured initialization}

We prove \cref{thm:componentwise_thin_escape} in this appendix and the next. Here we follow the trajectory from the initialization until it comes close to the bisector saddle; \cref{subsubsec:structured-escape-local-passage} studies the escape from the saddle and completes the proof.

\subsection{Coordinates and estimates near the bisector}

As in \cref{thm:componentwise_thin_escape}, we consider \(\alpha\in(0,\pi/2]\). We use the set \(\Omega\), the norms \(r_i=\|w_i\|_2\), the maximal time \(T_{\max}\) and the coordinates \(z_{i,j}\) as in \cref{sec:signed-transverse-exclusion}, and, as in \cref{prop:3D}, we write \(w_i^\perp=\sum_{j=3}^dz_{i,j}u_j\) for the component of \(w_i\) orthogonal to the teacher plane. We describe the pair \((w_1,w_2)\) by
\begin{equation}\label{eq:structured-escape-coordinates}
\begin{aligned}
    R&\coloneqq\frac{z_{1,1}+z_{2,1}}2, &\delta&\coloneqq\frac{z_{1,1}-z_{2,1}}2, &a&\coloneqq z_{1,2}-z_{2,2},\\
    Y&\coloneqq\bigl(\|w_1^\perp\|_2^2+\|w_2^\perp\|_2^2\bigr)^{1/2}, &N_\perp&\coloneqq\bigl(z_{1,2}^2+z_{2,2}^2+Y^2\bigr)^{1/2}.
\end{aligned}
\end{equation}
Here \(R\) and \(\delta\) are the mean and half the difference of the two bisector coordinates, \(|a|=|\langle w_1-w_2,v_\Delta\rangle|\) is the planar splitting amplitude, and \(Y\) is the norm of the out-of-plane part of \((w_1,w_2)\). Since \(N_\perp^2=\sum_{i=1}^2\sum_{j=2}^dz_{i,j}^2\), the time \(T_\varepsilon\) defined before \cref{thm:componentwise_thin_escape} is the first time at which \(N_\perp\) reaches \(\varepsilon\).

If \(z_{1,2}(0)>0>z_{2,2}(0)\), then \cref{lem:planar-split-sign-preservation} gives
\begin{equation}\label{eq:structured-escape-planar-signs-global}
    z_{1,2}(t)>0>z_{2,2}(t), \qquad a(t)=|z_{1,2}(t)|+|z_{2,2}(t)|>0, \qquad 0\le t<T_{\max}.
\end{equation}
In this case \(a(0)=a_{\mathrm{init}}\) by \eqref{eq:initial-planar-splitting-amplitude}, and, since \(a=|z_{1,2}|+|z_{2,2}|\), we have \(|z_{i,2}|\le a\), \(\|w_i^\perp\|_2\le Y\) and \(N_\perp\le a+Y\).

We recall the rate \(\mu_\alpha=2\sin(\alpha/2)/(\pi r_s)\) defined before \cref{thm:componentwise_thin_escape}. By \eqref{eq:def_r} and \(r_s=r_+/2\), we have \(\pi r_s=(\pi-\alpha/2)\cos(\alpha/2)+\sin(\alpha/2)\). Hence \(\mu_\alpha<1\) is equivalent to \(\tan(\alpha/2)<\pi-\alpha/2\), which holds because \(\tan(\alpha/2)\le1<3\pi/4\le\pi-\alpha/2\) for \(\alpha\le\pi/2\). Thus
\begin{equation}\label{eq:structured-escape-mu-range}
    0<\mu_\alpha<1.
\end{equation}

Three identities follow from \eqref{eq:structured-escape-component-ode}. First, by the definition of \(a_i\) in \cref{lem:critical_eq_from_flow},
\[
    \frac{a_i}{r_i}\,z_{i,1} =(\sin\theta_{i1}+\sin\theta_{i2})\frac{z_{i,1}}{r_i} -\pi z_{i,1}-\sin\gamma\,\frac{r_{3-i}}{r_i}\,z_{i,1}.
\]
We substitute this into the first two lines of \eqref{eq:structured-escape-component-ode} and use \((2\pi-\alpha)\cos(\alpha/2)+2\sin(\alpha/2)=2\pi r_s\), which is \eqref{eq:def_r}. Since \(z_{1,1}+z_{2,1}=2R\), we obtain, for \(i=1,2\),
\begin{equation}\label{eq:structured-escape-radial-identity}
\begin{aligned}
    2\pi\bigl(\dot z_{i,1}-(r_s-R)\bigr) &=(\alpha-\theta_{i1}-\theta_{i2})\cos\frac\alpha2 +\Bigl((\sin\theta_{i1}+\sin\theta_{i2})\frac{z_{i,1}}{r_i} -2\sin\frac\alpha2\Bigr)\\
    &\quad+\gamma z_{3-i,1}-\sin\gamma\,\frac{r_{3-i}}{r_i}\,z_{i,1}.
\end{aligned}
\end{equation}
Second, at each time at which \(z_{1,2}>0>z_{2,2}\), the two equations \eqref{eq:signed-transverse-planar-abs} hold, and their sum gives
\begin{equation}\label{eq:structured-escape-a-identity}
    2\pi\dot a =\sum_{i=1}^2\Bigl[ |\theta_{i2}-\theta_{i1}|\sin\frac\alpha2 +\Bigl(\frac{a_i}{r_i}+\pi-\gamma\Bigr)|z_{i,2}| \Bigr].
\end{equation}
Third, the last two lines of \eqref{eq:structured-escape-component-ode} read \(2\pi\dot w_1^\perp=(a_1/r_1)w_1^\perp-(\pi-\gamma)w_2^\perp\) and \(2\pi\dot w_2^\perp=(a_2/r_2)w_2^\perp-(\pi-\gamma)w_1^\perp\). They give
\begin{equation}\label{eq:structured-escape-Y-identity}
\begin{aligned}
    \pi\frac{\diff}{\diff t}Y^2 &=\frac{a_1}{r_1}\|w_1^\perp\|_2^2+\frac{a_2}{r_2}\|w_2^\perp\|_2^2 -2(\pi-\gamma)\langle w_1^\perp,w_2^\perp\rangle\\
    &\le\Bigl(\max_{i=1,2}\frac{a_i}{r_i}+\pi-\gamma\Bigr)Y^2;
\end{aligned}
\end{equation}
the inequality in the second line follows from \(0\le\pi-\gamma\le\pi\) and \(-2\langle w_1^\perp,w_2^\perp\rangle\le Y^2\). The equations for \(w_1^\perp\) and \(w_2^\perp\) are linear with scalar coefficients, so \(w_1^\perp(t)\) and \(w_2^\perp(t)\) remain in \(\operatorname{span}\{w_1^\perp(0),w_2^\perp(0)\}\). In particular, \(Y\equiv0\) if \(Y(0)=0\). The estimates below involve the out-of-plane components only through \(Y\) and the norms \(\|w_i^\perp\|_2\), which is why their constants do not depend on \(d\).

The bisector makes the angle \(\alpha/2\) with both teachers, so a student with \(z_{i,1}>0\) that is close to the bisector axis \(\RR\bar v\) has both teacher angles close to \(\alpha/2\). We write \(\rho_i\coloneqq(z_{i,2}^2+\|w_i^\perp\|_2^2)^{1/2}\) for the distance of \(w_i\) from \(\RR\bar v\). The next lemma quantifies this in terms of the ratio \(\rho_i/z_{i,1}\).

\begin{lemma}[Angles near the bisector]\label{lem:structured-escape-angles}
Fix \(\alpha\in(0,\pi/2]\) and set \(C_{\mathrm{angle}}\coloneqq6/\sin^5(3\alpha/8)\). Let \(d\ge2\) and \(i\in\{1,2\}\), and suppose that \(z_{i,1}>0\) and \(\rho_i\le(\alpha/8)\,z_{i,1}\), with \(\rho_i\) as above. Then \(\angle(\bar w_i,\bar v)\le\rho_i/z_{i,1}\), \(\theta_{i1},\theta_{i2}\in[3\alpha/8,5\alpha/8]\), \(0\le r_i-z_{i,1}\le\rho_i^2/(2z_{i,1})\), and
\begin{align}
    \Bigl|\theta_{i2}-\theta_{i1}-\frac{2z_{i,2}}{z_{i,1}}\Bigr| &\le C_{\mathrm{angle}}\, \frac{|z_{i,2}|\,\rho_i^2}{z_{i,1}^3},
    \label{eq:structured-escape-angle-difference}\\
    |\theta_{i1}+\theta_{i2}-\alpha| +\Bigl|(\sin\theta_{i1}+\sin\theta_{i2})\frac{z_{i,1}}{r_i} -2\sin\frac\alpha2\Bigr| &\le C_{\mathrm{angle}}\, \frac{\rho_i^2}{z_{i,1}^2}.
    \label{eq:structured-escape-angle-sum}
\end{align}
\end{lemma}

\begin{proof}[Proof of \cref{lem:structured-escape-angles}]
Since \(z_{i,1}>0\), the angle \(\angle(\bar w_i,\bar v)\) is acute. Its tangent is \(\rho_i/z_{i,1}\), which is at most \(\alpha/8\). As \(\arctan y\le y\) for \(y\ge0\), this gives the first claim. Since \(\angle(\bar v,v_1)=\angle(\bar v,v_2)=\alpha/2\), the triangle inequality for angles gives \(|\theta_{ij}-\alpha/2|\le\angle(\bar w_i,\bar v)\le\alpha/8\), which is the second claim. Since \(r_i\ge z_{i,1}\), the identity \(r_i-z_{i,1}=\rho_i^2/(r_i+z_{i,1})\) gives the third claim and
\[
    0\le1-\frac{z_{i,1}}{r_i} \le\frac{\rho_i^2}{2z_{i,1}^2}.
\]

As recalled at the start of \cref{sec:signed-transverse-exclusion}, \(v_1=\cos(\alpha/2)\,\bar v+\sin(\alpha/2)\,v_\Delta\) and \(v_2=\cos(\alpha/2)\,\bar v-\sin(\alpha/2)\,v_\Delta\). Hence \(\cos\theta_{i1}=m_i+\Delta_i\) and \(\cos\theta_{i2}=m_i-\Delta_i\), where
\[
    m_i\coloneqq\cos\frac\alpha2\,\frac{z_{i,1}}{r_i}, \qquad \Delta_i\coloneqq\sin\frac\alpha2\,\frac{z_{i,2}}{r_i}.
\]
By the second claim, the numbers \(m_i\pm\Delta_i\), their midpoint \(m_i\), and \(\cos(\alpha/2)\) all lie in the interval \([\cos(5\alpha/8),\cos(3\alpha/8)]\subset(0,1)\). For \(y\) in this interval, \(1-y^2\ge\sin^2(3\alpha/8)\), and therefore
\[
    |\arccos'(y)|=\frac1{(1-y^2)^{1/2}}\le\frac1{\sin(3\alpha/8)}, \qquad |\arccos''(y)|=\frac{y}{(1-y^2)^{3/2}}\le\frac1{\sin^3(3\alpha/8)},
\]
and \(|\arccos'''(y)|=(1+2y^2)(1-y^2)^{-5/2}\le3/\sin^5(3\alpha/8)\). We also note that \(|\Delta_i|\le|z_{i,2}|/z_{i,1}\), since \(r_i\ge z_{i,1}\).

For the difference, \(\theta_{i2}-\theta_{i1} =\arccos(m_i-\Delta_i)-\arccos(m_i+\Delta_i) =\int_{m_i-\Delta_i}^{m_i+\Delta_i}(-\arccos'(y))\diff y\). The error of the midpoint rule on an interval of length \(2|\Delta_i|\) is at most \((2|\Delta_i|)^3/24\) times the maximum of the second derivative of the integrand, so
\[
    \bigl|\theta_{i2}-\theta_{i1}+2\Delta_i\arccos'(m_i)\bigr| \le\frac{|\Delta_i|^3}{\sin^5(3\alpha/8)} \le\frac{|z_{i,2}|^3}{z_{i,1}^3\sin^5(3\alpha/8)}.
\]
Since \(r_i^2(1-m_i^2) =z_{i,1}^2\sin^2(\alpha/2)+\rho_i^2\), we have
\[
    -2\Delta_i\arccos'(m_i) =\frac{2z_{i,2}}{z_{i,1}} \Bigl(1+\frac{\rho_i^2} {z_{i,1}^2\sin^2(\alpha/2)}\Bigr)^{-1/2}.
\]
As \(0\le1-(1+y)^{-1/2}\le y/2\) for \(y\ge0\) and \(\sin(\alpha/2)\ge\sin(3\alpha/8)\), it follows that
\[
    \Bigl|\theta_{i2}-\theta_{i1}-\frac{2z_{i,2}}{z_{i,1}}\Bigr| \le\Bigl(\frac1{\sin^2(3\alpha/8)}+\frac1{\sin^5(3\alpha/8)}\Bigr) \frac{|z_{i,2}|\,\rho_i^2}{z_{i,1}^3},
\]
which implies \eqref{eq:structured-escape-angle-difference}.

For the sum, Taylor's formula with the Lagrange remainder gives
\[
    \bigl|\arccos(m_i+\Delta_i)+\arccos(m_i-\Delta_i)-2\arccos m_i\bigr| \le\frac{\Delta_i^2}{\sin^3(3\alpha/8)} \le\frac{z_{i,2}^2}{z_{i,1}^2\sin^3(3\alpha/8)}.
\]
Moreover, \(0\le\cos(\alpha/2)-m_i\le1-z_{i,1}/r_i\), so the mean value theorem gives
\[
    |2\arccos m_i-\alpha| =2\Bigl|\arccos m_i-\arccos\cos\frac\alpha2\Bigr| \le\frac{2(1-z_{i,1}/r_i)}{\sin(3\alpha/8)} \le\frac{\rho_i^2}{z_{i,1}^2\sin(3\alpha/8)}.
\]
Together, the last two bounds give \(|\theta_{i1}+\theta_{i2}-\alpha| \le2\rho_i^2/(z_{i,1}^2\sin^3(3\alpha/8))\).

Finally, \(\sin\theta_{i1}+\sin\theta_{i2} =2\sin\bigl((\theta_{i1}+\theta_{i2})/2\bigr) \cos\bigl((\theta_{i2}-\theta_{i1})/2\bigr)\). Since the sine is \(1\)-Lipschitz and \(0\le1-\cos y\le y^2/2\),
\[
    |\sin\theta_{i1}+\sin\theta_{i2}-2\sin(\alpha/2)| \le|\theta_{i1}+\theta_{i2}-\alpha|+\frac{(\theta_{i2}-\theta_{i1})^2}4.
\]
By the mean value theorem, \(|\theta_{i2}-\theta_{i1}|\le2|\Delta_i|/\sin(3\alpha/8) \le2|z_{i,2}|/(z_{i,1}\sin(3\alpha/8))\). The right side above is therefore at most \(3\rho_i^2/(z_{i,1}^2\sin^3(3\alpha/8))\). We multiply by \(z_{i,1}/r_i\le1\) and add \(2\sin(\alpha/2)(1-z_{i,1}/r_i)\le \rho_i^2/z_{i,1}^2\). This bounds the second term on the left of \eqref{eq:structured-escape-angle-sum} by \(4\rho_i^2/(z_{i,1}^2\sin^3(3\alpha/8))\), and \eqref{eq:structured-escape-angle-sum} follows.
\end{proof}

On the bisector axis, where \(N_\perp=0\) and \(z_{1,1},z_{2,1}>0\), both teacher angles equal \(\alpha/2\), \(\gamma=0\) and \(r_i=z_{i,1}\), so \eqref{eq:structured-escape-radial-identity} gives \(\dot z_{1,1}=\dot z_{2,1}=r_s-R\): so \(R\) relaxes to \(r_s\) and \(\delta\) stays constant. The next lemma shows that near the axis this remains true up to quadratic errors, while \(a\) grows at the rate \(\mu_\alpha r_s/R\) and \(Y\) at most at the rate \(\mu_\alpha r_s/(2R)\). At \(R=r_s\) these are the two unstable rates of \cref{lem:hessian_cal}.

\begin{lemma}[Equations of motion near the bisector]\label{lem:structured-escape-near-bisector}
Fix \(\alpha\in(0,\pi/2]\) and set \(\nu_0\coloneqq\min\{1/8,7\alpha/64\}\). There is \(C_{\mathrm{bis}}\ge1\), depending only on \(\alpha\), such that the following holds for every \(d\ge2\). Let \(w(0)\in\Omega\), and let \(t\in[0,T_{\max})\) be a time at which \(z_{1,2}>0>z_{2,2}\), \(0<R\le2r_s\) and \(a+Y+|\delta|\le\nu_0R\). Then \(7R/8\le z_{i,1}\le r_i\le9R/8\) for \(i=1,2\), and
\begin{align*}
    |\dot R-(r_s-R)|+|\dot\delta| &\le C_{\mathrm{bis}}\,\frac{a^2+Y^2}{R^2},\\
    \Bigl|\dot a-\frac{2\sin(\alpha/2)}{\pi R}\,a\Bigr| &\le C_{\mathrm{bis}}\Bigl(\frac{a+Y}R+\frac{|\delta|}{R^2} +\frac{a^2+Y^2}{R^3}\Bigr)a,\\
    \frac{\diff}{\diff t}Y^2 &\le\Bigl[\frac{2\sin(\alpha/2)}{\pi R} +C_{\mathrm{bis}}\Bigl(\frac{a+Y}R+\frac{|\delta|}{R^2} +\frac{a^2+Y^2}{R^3}\Bigr)\Bigr]Y^2.
\end{align*}
\end{lemma}

\begin{proof}[Proof of \cref{lem:structured-escape-near-bisector}]
In this proof, \(C_\alpha\) denotes a constant depending only on \(\alpha\), whose value may change from line to line. Since \(|\delta|+a+Y\le\nu_0R\le R/8\), \(|z_{i,2}|\le a\) and \(\|w_i^\perp\|_2\le Y\),
\[
    \frac78R\le R-|\delta|\le z_{i,1}\le r_i \le z_{i,1}+|z_{i,2}|+\|w_i^\perp\|_2\le R+|\delta|+a+Y\le\frac98R,
\]
which is the first claim. Moreover, \(\rho_i\le a+Y\le\nu_0R \le(8\nu_0/7)z_{i,1}\le(\alpha/8)z_{i,1}\), so \cref{lem:structured-escape-angles} applies to both students, with \(\rho_i^2\le a^2+Y^2\). By that lemma and the triangle inequality for angles,
\[
    \gamma\le\angle(\bar w_1,\bar v)+\angle(\bar w_2,\bar v) \le\frac{2(a+Y)}{7R/8}=\frac{16(a+Y)}{7R}\le\frac{16\nu_0}7.
\]

We first prove the bound for \(\dot R\) and \(\dot\delta\). Let \(i\in\{1,2\}\) and \(j=3-i\). By \cref{lem:structured-escape-angles}, the first two terms on the right of \eqref{eq:structured-escape-radial-identity} are at most \(C_\alpha(a^2+Y^2)/R^2\) in absolute value. We write the remaining term as
\[
    \gamma z_{j,1}-\sin\gamma\,\frac{r_j}{r_i}\,z_{i,1} =\gamma\Bigl(z_{j,1}-\frac{r_j}{r_i}\,z_{i,1}\Bigr) +(\gamma-\sin\gamma)\frac{r_j}{r_i}\,z_{i,1}.
\]
Since \(z_{j,1}r_i-r_jz_{i,1}=z_{j,1}(r_i-z_{i,1})-z_{i,1}(r_j-z_{j,1})\), \cref{lem:structured-escape-angles} and the first claim give \(|z_{j,1}-r_jz_{i,1}/r_i|\le C_\alpha(a^2+Y^2)/R\). Also \(0\le\gamma-\sin\gamma\le\gamma^3/6\), \(r_jz_{i,1}/r_i\le r_j\le9R/8\) and \((a+Y)^2\le2(a^2+Y^2)\). With the bound for \(\gamma\), we obtain
\[
    \Bigl|\gamma z_{j,1}-\sin\gamma\,\frac{r_j}{r_i}\,z_{i,1}\Bigr| \le C_\alpha\frac{(a+Y)(a^2+Y^2)}{R^2} \le C_\alpha\frac{a^2+Y^2}{R^2},
\]
where the last step uses \(a+Y\le\nu_0R\le2\nu_0r_s\). Hence \(|\dot z_{i,1}-(r_s-R)|\le C_\alpha(a^2+Y^2)/R^2\) for \(i=1,2\), and the first bound follows from \(\dot R=(\dot z_{1,1}+\dot z_{2,1})/2\) and \(\dot\delta=(\dot z_{1,1}-\dot z_{2,1})/2\).

For \(a\), we estimate the two terms in \eqref{eq:structured-escape-a-identity}. By \eqref{eq:structured-escape-angle-difference},
\[
    \Bigl||\theta_{i2}-\theta_{i1}|-\frac{2|z_{i,2}|}{z_{i,1}}\Bigr| \le C_\alpha\frac{|z_{i,2}|\,(a^2+Y^2)}{R^3}.
\]
By the definition of \(a_i\) in \cref{lem:critical_eq_from_flow},
\[
    \frac{a_i}{r_i}+\pi-\gamma =\frac{\sin\theta_{i1}+\sin\theta_{i2}}{r_i} -\sin\gamma\,\frac{r_j}{r_i}-\gamma.
\]
Dividing \eqref{eq:structured-escape-angle-sum} by \(z_{i,1}\) shows that the first term on the right differs from \(2\sin(\alpha/2)/z_{i,1}\) by at most \(C_\alpha(a^2+Y^2)/R^3\), and the other two terms are at most \(\gamma(r_j/r_i+1)\le(16(a+Y)/(7R))(9/7+1)\le6(a+Y)/R\) in absolute value. Finally, \(z_{i,1}=R\pm\delta\) gives \(|1/z_{i,1}-1/R|\le8|\delta|/(7R^2)\). Therefore
\[
\begin{aligned}
    \Bigl|\frac{a_i}{r_i}+\pi-\gamma-\frac{2\sin(\alpha/2)}R\Bigr| &\le C_\alpha\Bigl(\frac{a+Y}R+\frac{|\delta|}{R^2}+\frac{a^2+Y^2}{R^3}\Bigr),\\
    \Bigl||\theta_{i2}-\theta_{i1}|-\frac{2|z_{i,2}|}R\Bigr| &\le C_\alpha\Bigl(\frac{|\delta|}{R^2}+\frac{a^2+Y^2}{R^3}\Bigr) |z_{i,2}|.
\end{aligned}
\]
Substituting these bounds into \eqref{eq:structured-escape-a-identity} and using \(|z_{1,2}|+|z_{2,2}|=a\) proves the second bound.

The bound for \(a_i/r_i+\pi-\gamma\) holds for both \(i\). Together with \eqref{eq:structured-escape-Y-identity}, it proves the third bound.
\end{proof}

\subsection{Alignment with the bisector}

At a small initialization, the ratios \(a/R\), \(Y/R\) and \(|\delta|/R\) need not be small, so \cref{lem:structured-escape-near-bisector} does not yet apply. By \cref{lem:phase-one-initial-alignment}, the directions \(\bar w_i\) first align with \(\bar v\) in a time of order \(r_{\mathrm{init}}\). Afterwards \(R\) grows at a rate close to \(r_s\), while \(a\), \(Y\) and \(\delta\) change little. The next lemma follows the trajectory until \(R\) is a large multiple of \(r_{\mathrm{init}}\), when all three ratios are small.

\begin{lemma}[Alignment with the bisector]\label{lem:structured-escape-alignment}
Fix \(\alpha\in(0,\pi/2]\), \(C\ge1\) and \(\nu>0\). There are \(r_0\in(0,1]\), \(B>1\) and \(C_{\mathrm{align}}\ge1\), depending only on \(\alpha\), \(C\) and \(\nu\), with \(Br_0<r_s/4\), such that the following holds for every \(d\ge2\). Suppose that the initialization is \(C\)-balanced, that \(0<r_{\mathrm{init}}\le r_0\), and that \(z_{1,1}(0),z_{2,1}(0)>0\) and \(z_{1,2}(0)>0>z_{2,2}(0)\). Then there is a first time \(\tau_0\in[0,T_{\max})\) at which \(R(\tau_0)=Br_{\mathrm{init}}\); it satisfies \(\tau_0\le C_{\mathrm{align}}r_{\mathrm{init}}\),
\[
    C_{\mathrm{align}}^{-1}a_{\mathrm{init}}\le a(t) \le C_{\mathrm{align}}a_{\mathrm{init}}, \qquad Y(t)+|\delta(t)|\le C_{\mathrm{align}}r_{\mathrm{init}}, \qquad 0\le t\le\tau_0,
\]
and \(a(\tau_0)+Y(\tau_0)+|\delta(\tau_0)|\le\nu R(\tau_0)\).
\end{lemma}

\begin{proof}[Proof of \cref{lem:structured-escape-alignment}]
We first collect bounds that hold on a time interval of length proportional to \(r_{\mathrm{init}}\). Since \(r_1(0)+r_2(0)=2r_{\mathrm{init}}\) and the initialization is \(C\)-balanced, \(2r_{\mathrm{init}}\le(1+C)\min_ir_i(0)\), so \(r_{\mathrm{init}}/C\le r_i(0)\le2r_{\mathrm{init}}\). Also \(a_{\mathrm{init}}\le r_1(0)+r_2(0)=2r_{\mathrm{init}}\), \(Y(0)\le r_1(0)+r_2(0)=2r_{\mathrm{init}}\) and \(|\delta(0)|\le r_{\mathrm{init}}\). Let \(\lambda>0\), and suppose that \(r_{\mathrm{init}}\le r_-/(4(2+\lambda))\), with \(r_-\) as in \eqref{eq:def_r}. Since \(r_-=4l(\pi-\alpha/2)<2\) by \eqref{eq:relu-kernel-profile} and \cref{lem:l-monotone-bounds}, this gives \(r_{\mathrm{init}}<1/4\). Then \(\max_ir_i(0)+\lambda r_{\mathrm{init}}\le r_-/4\), and \cref{lem:phase-one-short-time-radial} gives \(T_{\max}>\lambda r_{\mathrm{init}}\) and \(r_{\mathrm{init}}/C\le r_i(t)\le(2+\lambda)r_{\mathrm{init}}\) for \(0\le t\le\lambda r_{\mathrm{init}}\). On this interval, \(|a_i|\le2+\pi r_i+r_{3-i}\) and \(r_{\mathrm{init}}<1\) give \(|a_i|/r_i\le(2C+\pi+C(2+\lambda))/r_{\mathrm{init}}\). Let \(J\) be the reflection that fixes \(\bar v,u_3,\ldots,u_d\) and maps \(v_\Delta\) to \(-v_\Delta\), so that \(Jv_2=v_1\). The triangle inequality for angles and \eqref{eq:phase-one-angle-chord} give
\[
    |\theta_{i2}-\theta_{i1}| =|\angle(J\bar w_i,v_1)-\angle(\bar w_i,v_1)| \le\angle(J\bar w_i,\bar w_i) \le\frac\pi2\|J\bar w_i-\bar w_i\|_2 =\frac{\pi|z_{i,2}|}{r_i}.
\]
By \eqref{eq:structured-escape-a-identity} and \eqref{eq:structured-escape-Y-identity}, these bounds give \(|\dot a|\le3C(1+\lambda)a/r_{\mathrm{init}}\) and \(\frac{\diff}{\diff t}Y^2\le4C(1+\lambda)Y^2/r_{\mathrm{init}}\). Since \(r_i\le(2+\lambda)r_{\mathrm{init}}\le r_-/4<1/2\), we have \(|a_i|\le2+(\pi+1)/2\), and the first two lines of \eqref{eq:structured-escape-component-ode} give \(2\pi|\dot z_{i,1}|\le2\pi+|a_i|+\pi r_{3-i}<4\pi\). Hence \(|\dot\delta|\le2\). Integrating, we find that, for \(0\le t\le\lambda r_{\mathrm{init}}\),
\begin{equation}\label{eq:structured-escape-short-time}
    e^{-3C(1+\lambda)\lambda}a_{\mathrm{init}}\le a(t) \le e^{3C(1+\lambda)\lambda}a_{\mathrm{init}}, \quad Y(t)\le2e^{2C(1+\lambda)\lambda}r_{\mathrm{init}}, \quad |\delta(t)|\le r_{\mathrm{init}}+2t.
\end{equation}

Next we consider times at which both directions are close to \(\bar v\). Let
\[
    \varepsilon_{\mathrm a}\coloneqq\min\Bigl\{\nu_0,\, \frac{r_s}{4C_{\mathrm{angle}}+4r_s},\, \frac{\nu r_s}{4C_{\mathrm{angle}}+12r_s}\Bigr\},
\]
with \(\nu_0\) as in \cref{lem:structured-escape-near-bisector}, and suppose that at some time \(\|\bar w_i-\bar v\|_2\le\varepsilon_{\mathrm a}\) for \(i=1,2\) and \(R\le r_s/4\). Then \(z_{i,1}/r_i=\langle\bar w_i,\bar v\rangle =1-\|\bar w_i-\bar v\|_2^2/2\ge7/8\) and \(\rho_i/r_i =\|\bar w_i-\langle\bar w_i,\bar v\rangle\bar v\|_2\le\varepsilon_{\mathrm a}\), so \(\rho_i\le(8\varepsilon_{\mathrm a}/7)z_{i,1} \le(\alpha/8)z_{i,1}\). By \cref{lem:structured-escape-angles}, the first two terms on the right of \eqref{eq:structured-escape-radial-identity} are at most \(2C_{\mathrm{angle}}\varepsilon_{\mathrm a}^2\) in absolute value, and the triangle inequality for angles gives \(\gamma\le16\varepsilon_{\mathrm a}/7\). Since \(z_{j,1}\le2R\) and \(r_j\le(8/7)z_{j,1}\le16R/7\) for \(j=1,2\), the last two terms are at most \(\gamma(z_{3-i,1}+r_{3-i})\le11\varepsilon_{\mathrm a} R\). We use \(R\le r_s/4\) and the choice of \(\varepsilon_{\mathrm a}\) for the first bound below, and we sum \(|z_{i,2}|+\|w_i^\perp\|_2\le\sqrt2\,\varepsilon_{\mathrm a} r_i\) over \(i\) for the second:
\begin{equation}\label{eq:structured-escape-aligned}
    |\dot z_{i,1}-(r_s-R)|\le(C_{\mathrm{angle}}+r_s)\varepsilon_{\mathrm a} \le\frac{r_s}4, \qquad a+Y\le\sqrt2\,\varepsilon_{\mathrm a}(r_1+r_2)\le4\varepsilon_{\mathrm a} R.
\end{equation}
Consequently \(\dot R\ge r_s/2\) and \(|\dot\delta|\le(C_{\mathrm{angle}}+r_s)\varepsilon_{\mathrm a}\) at such times.

Let \(\Lambda\) and \(r_{\mathrm a}\) be the constants of \cref{lem:phase-one-initial-alignment} for \(\eta=\pi/3\), \(m=1/C\), \(M=2\) and \(\varepsilon=\varepsilon_{\mathrm a}/\sqrt2\). We set
\[
    B\coloneqq3+\Lambda+\frac{2(1+2\Lambda)}\nu, \qquad \lambda_1\coloneqq\Lambda+\frac{4B}{r_s}, \qquad r_0\coloneqq\min\Bigl\{r_{\mathrm a},\, \frac{r_-}{4(2+\lambda_1)},\, \frac{\varepsilon_{\mathrm a}^2}{4C(2+\lambda_1)\lambda_1}\Bigr\},
\]
and \(C_{\mathrm{align}}\coloneqq(3+2\lambda_1)e^{3C(1+\lambda_1)\lambda_1}\). These depend only on \(\alpha\), \(C\) and \(\nu\), and \(C_{\mathrm{align}}\ge\lambda_1\).

We note that \(Br_0<Br_-/(4\lambda_1)\le r_-r_s/16<r_s/4\), because \(\lambda_1\ge4B/r_s\) and \(r_-<2\). Suppose now that \(r_{\mathrm{init}}\le r_0\). With \(T_1\coloneqq\lambda_1r_{\mathrm{init}}\), \eqref{eq:structured-escape-short-time} holds with \(\lambda=\lambda_1\) on \([0,T_1]\). Since \(\angle(\bar w_i(0),\bar v)<\pi/2\), \cref{lem:phase-one-initial-alignment} applies, and at \(t_{\mathrm a}\coloneqq\Lambda r_{\mathrm{init}}\) we have \(E_i(t_{\mathrm a})<\varepsilon_{\mathrm a}^2/4\), where \(E_i=\frac12\|\bar w_i-\bar v\|_2^2\) as in \cref{lem:phase-one-direction-error}. By \eqref{eq:phase-one-direction-error-single}, \(\dot E_i\le r_{3-i}/(2r_i) \le C(2+\lambda_1)/2\) on \([0,T_1]\), so the choice of \(r_0\) gives \(E_i<\varepsilon_{\mathrm a}^2/4+\varepsilon_{\mathrm a}^2/8<\varepsilon_{\mathrm a}^2/2\) on \([t_{\mathrm a},T_1]\). Thus \(\|\bar w_i-\bar v\|_2<\varepsilon_{\mathrm a}\) on \([t_{\mathrm a},T_1]\).

For \(0\le t\le t_{\mathrm a}\), the bound \(r_i(t)\le r_i(0)+t\) of \cref{lem:phase-one-short-time-radial} gives \(R\le(r_1+r_2)/2\le(2+\Lambda)r_{\mathrm{init}}<Br_{\mathrm{init}}\), so \(R\) does not reach \(Br_{\mathrm{init}}\) before \(t_{\mathrm a}\). As long as \(t_{\mathrm a}\le t\le T_1\) and \(R\le Br_{\mathrm{init}}\), we have \(\|\bar w_i-\bar v\|_2<\varepsilon_{\mathrm a}\) and \(R\le r_s/4\), so \(\dot R\ge r_s/2\). If \(R\) stayed below \(Br_{\mathrm{init}}\) on \([0,T_1]\), we would obtain \(R(T_1)\ge R(t_{\mathrm a})+(r_s/2)(T_1-t_{\mathrm a}) =R(t_{\mathrm a})+2Br_{\mathrm{init}}\), where \(R(t_{\mathrm a})>0\) because \(z_{i,1}\ge(7/8)r_i\) there. This is impossible, so \(\tau_0\) exists, \(t_{\mathrm a}<\tau_0\le T_1\le C_{\mathrm{align}}r_{\mathrm{init}}\), and \(\tau_0-t_{\mathrm a}\le2(Br_{\mathrm{init}}-R(t_{\mathrm a}))/r_s \le2Br_{\mathrm{init}}/r_s\). The bounds for \(a\), \(Y\) and \(\delta\) on \([0,\tau_0]\) follow from \eqref{eq:structured-escape-short-time} with \(\lambda=\lambda_1\).

At \(\tau_0\), \eqref{eq:structured-escape-aligned} gives \(a+Y\le4\varepsilon_{\mathrm a} R(\tau_0)\). Integrating the bound for \(\dot\delta\) from \(t_{\mathrm a}\), and using \(|\delta(t_{\mathrm a})|\le(1+2\Lambda)r_{\mathrm{init}}\), we obtain
\[
    \frac{a(\tau_0)+Y(\tau_0)+|\delta(\tau_0)|}{R(\tau_0)} \le4\varepsilon_{\mathrm a}+\frac{1+2\Lambda}B +\frac{2(C_{\mathrm{angle}}+r_s)\varepsilon_{\mathrm a}}{r_s} \le\frac\nu2+\frac\nu2,
\]
by the choices of \(\varepsilon_{\mathrm a}\) and \(B\).
\end{proof}

\subsection{Growth of the bisector coordinates up to \texorpdfstring{\(r_s/2\)}{rs/2}}

Once the ratios are small, \cref{lem:structured-escape-near-bisector} applies. Then \(R\) grows at a rate close to \(r_s-R\), \(\delta\) changes little, and, since \(\mu_\alpha r_s\int\diff t/R\approx\mu_\alpha\log R\), \(a\) grows like \(R^{\mu_\alpha}\) while \(Y\) grows at most like \(R^{\mu_\alpha/2}\). Since \(\mu_\alpha<1\) by \eqref{eq:structured-escape-mu-range}, the ratios \(a/R\) and \(Y/R\) decrease, and the trajectory stays near the bisector until \(R\) reaches \(r_s/2\).

\begin{lemma}[Growth of the bisector coordinates]\label{lem:structured-escape-growth}
Fix \(\alpha\in(0,\pi/2]\), and let \(\nu_0\) be as in \cref{lem:structured-escape-near-bisector}. There are \(\nu\in(0,\nu_0]\) and \(C_{\mathrm{scale}}\ge1\), depending only on \(\alpha\), such that the following holds for every \(d\ge2\). Let \(w(0)\in\Omega\) satisfy \(z_{1,2}(0)>0>z_{2,2}(0)\), \(0<R_0<r_s/2\) and \(a_0+Y_0+|\delta_0|\le\nu R_0\), where \(R_0\), \(a_0\), \(Y_0\) and \(\delta_0\) denote the values of \(R\), \(a\), \(Y\) and \(\delta\) at \(t=0\). Then there is a first time \(t_*\in[0,T_{\max})\) at which \(R(t_*)=r_s/2\); it satisfies \(0<t_*<2\),
\[
    C_{\mathrm{scale}}^{-1}a_0\Bigl(\frac{R(t)}{R_0}\Bigr)^{\mu_\alpha} \le a(t)\le C_{\mathrm{scale}}a_0\Bigl(\frac{R(t)}{R_0}\Bigr)^{\mu_\alpha}, \qquad Y(t)\le C_{\mathrm{scale}}Y_0\Bigl(\frac{R(t)}{R_0}\Bigr)^{\mu_\alpha/2}, \qquad 0\le t\le t_*,
\]
and
\[
    |\delta(t_*)| \le C_{\mathrm{scale}}\Bigl(|\delta_0|+\frac{a_0^2+Y_0^2}{R_0} +a(t_*)^2\log\frac1{R_0}\Bigr).
\]
\end{lemma}

\begin{proof}[Proof of \cref{lem:structured-escape-growth}]
Let \(C_{\mathrm{bis}}\) be the constant of \cref{lem:structured-escape-near-bisector}, and let \(\nu_1\coloneqq\min\{\nu_0,(r_s/(4C_{\mathrm{bis}}))^{1/2}\}\). We first derive bounds that hold for every \(\nu\le\nu_1/4\); the value of \(\nu\), together with a constant \(Q\ge2\), is fixed afterwards.

Let \(\tau\) be the supremum of the times \(s\in[0,T_{\max})\) such that \(R<r_s/2\) and \(a+Y+|\delta|<\nu_1R\) on \([0,s]\). Both conditions hold at \(t=0\), so \(\tau>0\). On \([0,\tau)\), we have \(R>0\), because \(a>0\) and \(a+Y+|\delta|<\nu_1R\), and the signs \eqref{eq:structured-escape-planar-signs-global} hold, so \cref{lem:structured-escape-near-bisector} applies. Its first bound and \((a^2+Y^2)/R^2\le\nu_1^2\) give \(\dot R\ge r_s-R-r_s/4\ge r_s/4\). Hence \(R\) is increasing, \(R\ge R_0\), and \(t<4(r_s/2-R_0)/r_s<2\) for \(t<\tau\). Moreover \(7R_0/8\le r_i\le9r_s/16\) on \([0,\tau)\). If \(\tau=T_{\max}\), the solution would stay in a compact subset of \(\Omega\) on a bounded interval, which \cref{lem:relu-kernel-regularity} excludes. Thus \(\tau<\min\{2,T_{\max}\}\), the estimates above extend to \([0,\tau]\) by continuity, and at \(\tau\) either \(R=r_s/2\) or \(a+Y+|\delta|=\nu_1R\).

Since \(\dot R\ge r_s/4\), we have \(\diff s\le(4/r_s)\diff R\), so an integral in time of a nonnegative quantity over \([0,t]\subset[0,\tau]\) is at most \(4/r_s\) times the corresponding integral in \(R\) over \([R_0,R(t)]\). Moreover, by the first bound of \cref{lem:structured-escape-near-bisector},
\[
    \Bigl|\frac1{\dot R}-\frac1{r_s-R}\Bigr| =\frac{|\dot R-(r_s-R)|}{\dot R\,(r_s-R)} \le\frac{8C_{\mathrm{bis}}}{r_s^2}\,\frac{a^2+Y^2}{R^2}.
\]
Since \(1/(R(r_s-R))=(1/R+1/(r_s-R))/r_s\) and \(0\le\log((r_s-R_0)/(r_s-R))\le\log2\) for \(R\le r_s/2\), it follows that
\begin{equation}\label{eq:structured-escape-log-R}
    \Bigl|\mu_\alpha r_s\int_0^t\frac{\diff s}{R}-\mu_\alpha\log\frac{R(t)}{R_0}\Bigr| \le\mu_\alpha\log2+\frac{8C_{\mathrm{bis}}}{r_s} \int_{R_0}^{R(t)}\frac{a^2+Y^2}{R^3}\,\diff R.
\end{equation}

Let \(\tau_Q\) be the supremum of the times \(s\in[0,\tau]\) such that, on \([0,s]\),
\begin{equation}\label{eq:structured-escape-growth-bootstrap}
    a\le Qa_0\Bigl(\frac R{R_0}\Bigr)^{\mu_\alpha}, \qquad Y\le QY_0\Bigl(\frac R{R_0}\Bigr)^{\mu_\alpha/2}.
\end{equation}
If \(Y_0>0\), both inequalities are strict at \(t=0\). If \(Y_0=0\), then \(Y\equiv0\) and the second inequality holds throughout, while the first is strict at \(t=0\). In either case \(\tau_Q>0\), and if \(\tau_Q<\tau\), then, by continuity, the first inequality, or the second with \(Y_0>0\), is an equality at \(\tau_Q\). On \([0,\tau_Q]\), since \(R\ge R_0\), \(\mu_\alpha<1\) and \(a_0,Y_0\le\nu R_0\), \eqref{eq:structured-escape-growth-bootstrap} gives
\[
    \int_{R_0}^{R(t)}\frac{a^2+Y^2}{R^3}\,\diff R \le Q^2\Bigl(\frac{a_0^2}{R_0^{2\mu_\alpha}}\int_{R_0}^\infty R^{2\mu_\alpha-3}\diff R +\frac{Y_0^2}{R_0^{\mu_\alpha}}\int_{R_0}^\infty R^{\mu_\alpha-3}\diff R\Bigr) \le\frac{Q^2\nu^2}{1-\mu_\alpha}.
\]
Similarly, \(R^{2\mu_\alpha-2}\le R_0^{2\mu_\alpha-2}\) and \(R^{\mu_\alpha-2}\le R_0^{\mu_\alpha-2}\) for \(R\ge R_0\) give \(\int_{R_0}^{R(t)}(a^2+Y^2)R^{-2}\diff R\le2Q^2\nu^2R(t)\), and the first bound of \cref{lem:structured-escape-near-bisector} yields
\begin{equation}\label{eq:structured-escape-delta-growth}
    |\delta(t)| \le|\delta_0|+\frac{4C_{\mathrm{bis}}}{r_s} \int_{R_0}^{R(t)}\frac{a^2+Y^2}{R^2}\,\diff R \le\Bigl(\nu+\frac{8C_{\mathrm{bis}}Q^2\nu^2}{r_s}\Bigr)R(t).
\end{equation}
For the term with \(\delta\) in \cref{lem:structured-escape-near-bisector}, we use \(|\delta(s)|\le|\delta_0|+\int_0^s|\dot\delta|\) and exchange the order of integration. Since \(\int_{s'}^t\diff s/R(s)^2\le(4/r_s)\int_{R(s')}^\infty\diff R/R^2 =4/(r_sR(s'))\), we find
\[
    \int_0^t\frac{|\delta|}{R^2}\,\diff s \le\frac4{r_s}\Bigl(\frac{|\delta_0|}{R_0} +\int_0^t\frac{|\dot\delta|}{R}\,\diff s\Bigr) \le\frac4{r_s}\Bigl(\nu+C_{\mathrm{bis}} \int_0^t\frac{a^2+Y^2}{R^3}\,\diff s\Bigr).
\]
Combining these bounds with \eqref{eq:structured-escape-log-R}, the second and third bounds of \cref{lem:structured-escape-near-bisector}, \((a+Y)/R<\nu_1\le1\) and \(t<2\), we obtain, for \(0\le t\le\tau_Q\),
\[
    \Bigl|\log\frac{a(t)}{a_0}-\mu_\alpha\log\frac{R(t)}{R_0}\Bigr| \le\Gamma_1+\Gamma_2(\nu+Q^2\nu^2), \qquad Y(t)^2\le Y_0^2\Bigl(\frac{R(t)}{R_0}\Bigr)^{\mu_\alpha} e^{\Gamma_1+\Gamma_2(\nu+Q^2\nu^2)},
\]
where \(\Gamma_1\coloneqq\mu_\alpha\log2+2C_{\mathrm{bis}}\) and \(\Gamma_2\) depends only on \(\alpha\); the second bound follows from Gr\"onwall's inequality for \(Y^2\). We now choose \(Q\coloneqq2e^{\Gamma_1+1}\) and then \(\nu\in(0,\nu_1/4]\) so small that
\[
    \Gamma_2(\nu+Q^2\nu^2)\le1, \qquad (2Q+1)\nu+\frac{8C_{\mathrm{bis}}Q^2\nu^2}{r_s}\le\frac{\nu_1}2.
\]
Then \(a\le(Q/2)a_0(R/R_0)^{\mu_\alpha}\) and \(Y\le(Q/2)Y_0(R/R_0)^{\mu_\alpha/2}\) on \([0,\tau_Q]\), so at \(\tau_Q\) the first inequality in \eqref{eq:structured-escape-growth-bootstrap} is strict, and so is the second when \(Y_0>0\). Hence \(\tau_Q=\tau\). Moreover, \(a_0(R/R_0)^{\mu_\alpha}\le\nu R\) and \(Y_0(R/R_0)^{\mu_\alpha/2}\le\nu R\), so \eqref{eq:structured-escape-delta-growth} gives \(a+Y+|\delta|\le\nu_1R/2\) on \([0,\tau]\). At \(\tau\), therefore, \(R=r_s/2\), and \(t_*=\tau\in(0,2)\).

The bounds for \(a\) and \(Y\) hold with \(C_{\mathrm{scale}}\ge Q\). For \(\delta(t_*)\), we use the first inequality in \eqref{eq:structured-escape-delta-growth} with \eqref{eq:structured-escape-growth-bootstrap}. The contribution of \(Y^2\) to the integral is at most \(Q^2Y_0^2/((1-\mu_\alpha)R_0)\). For the contribution of \(a^2\), if \(\mu_\alpha<1/2\), then \(a_0^2R_0^{-2\mu_\alpha}\int_{R_0}^{\infty}R^{2\mu_\alpha-2}\diff R=a_0^2/((1-2\mu_\alpha)R_0)\). If \(\mu_\alpha\ge1/2\), then \(R^{2\mu_\alpha-2}\le1/R\) for \(R\le r_s/2\), because \(r_s=2l(\alpha/2)\le1\) by \eqref{eq:def_r}, \eqref{eq:relu-kernel-profile} and \cref{lem:l-monotone-bounds}. Hence \(a_0^2R_0^{-2\mu_\alpha}\int_{R_0}^{r_s/2}R^{2\mu_\alpha-2}\diff R \le a_0^2R_0^{-2\mu_\alpha}\log(1/R_0)\). In this case, the lower bound for \(a\) at \(t_*\) gives \(a_0R_0^{-\mu_\alpha}\le(Q/2)(2/r_s)^{\mu_\alpha}a(t_*) \le(Q/r_s)a(t_*)\). Hence the bound for \(\delta(t_*)\) holds with \(C_{\mathrm{scale}}\) large enough, depending only on \(\alpha\).
\end{proof}

\subsection{Entry into a neighborhood of the bisector saddle}

We fix \(c_0\coloneqq r_s/64\) and consider the segment
\[
    \mathcal C_{c_0}\coloneqq\{((r_s+s)\bar v,(r_s-s)\bar v):|s|\le c_0\}
\]
of the first family in \eqref{eq:in-plane-saddle-form}; it contains \(w_{\mathrm{saddle}}\). For \(|s|\le c_0\), \(\|w_1-(r_s+s)\bar v\|_2^2+\|w_2-(r_s-s)\bar v\|_2^2 =(R+\delta-r_s-s)^2+(R-\delta-r_s+s)^2+N_\perp^2\). Hence
\begin{equation}\label{eq:structured-escape-distance}
    \operatorname{dist}\bigl((w_1,w_2),\mathcal C_{c_0}\bigr)^2 =2(R-r_s)^2+2\min_{|s|\le c_0}(\delta-s)^2+N_\perp^2.
\end{equation}

The analysis in \cref{subsubsec:structured-escape-local-passage} starts once the trajectory is within a fixed distance \(\varepsilon_{\mathrm{entry}}\) of \(\mathcal C_{c_0}\). By \cref{lem:structured-escape-growth}, the quantities \(a\), \(Y\) and \(\delta\) are small when \(R\) reaches \(r_s/2\). Afterwards \(R\) relaxes exponentially to \(r_s\), and over the bounded time this takes, \(a\) and \(Y\) change at most by bounded factors and \(\delta\) changes by \(O(a^2+Y^2)\). The next lemma combines this with \cref{lem:structured-escape-alignment,lem:structured-escape-growth}.

\begin{lemma}[Entry into a neighborhood of the bisector saddle]\label{lem:structured-escape-entry}
Fix \(\alpha\in(0,\pi/2]\), \(C\ge1\) and \(\varepsilon_{\mathrm{entry}}\in(0,c_0]\). There are \(c_{\mathrm{entry}}\in(0,1]\) and \(C_{\mathrm{entry}}\ge1\), depending only on \(\alpha\), \(C\) and \(\varepsilon_{\mathrm{entry}}\), such that the following holds for every \(d\ge2\). Suppose that the initialization is \(C\)-balanced, that \(z_{1,1}(0),z_{2,1}(0)>0\), \(z_{1,2}(0)>0>z_{2,2}(0)\) and \(a_{\mathrm{init}}\ge r_{\mathrm{init}}/(C\sqrt d)\), and that
\begin{equation}\label{eq:structured-escape-entry-smallness}
    r_{\mathrm{init}}\sqrt d\le c_{\mathrm{entry}}, \qquad r_{\mathrm{init}}\sqrt d\, \Bigl(1+\log\frac{\sqrt d}{r_{\mathrm{init}}}+\log d\Bigr) \le c_{\mathrm{entry}}.
\end{equation}
We write \(\varepsilon_{\mathrm{init}}\) for the left side of the second inequality. Then there is a first time \(t_{\mathrm{entry}}\in[0,T_{\max})\) at which \(\operatorname{dist}((w_1,w_2),\mathcal C_{c_0}) \le\varepsilon_{\mathrm{entry}}\). Moreover \(t_{\mathrm{entry}}\le C_{\mathrm{entry}}\), and \(a_{\mathrm e}\coloneqq a(t_{\mathrm{entry}})\) and \(q_{\mathrm e}\coloneqq r_{\mathrm{init}}\sqrt d +a_{\mathrm e}^2(1+\log(1/r_{\mathrm{init}}))\) satisfy
\begin{equation}\label{eq:structured-escape-entry-data}
\begin{gathered}
    C_{\mathrm{entry}}^{-1}a_{\mathrm{init}}r_{\mathrm{init}}^{-\mu_\alpha} \le a_{\mathrm e} \le C_{\mathrm{entry}}a_{\mathrm{init}}r_{\mathrm{init}}^{-\mu_\alpha}, \qquad Y(t_{\mathrm{entry}})^2\le C_{\mathrm{entry}}r_{\mathrm{init}}\sqrt d\, a_{\mathrm e},\\
    |\delta(t_{\mathrm{entry}})|\le C_{\mathrm{entry}}q_{\mathrm e}, \qquad \sup_{0\le t\le t_{\mathrm{entry}}}N_\perp(t) +q_{\mathrm e}\log\frac1{a_{\mathrm e}} \le C_{\mathrm{entry}}\varepsilon_{\mathrm{init}}^{1-\mu_\alpha}.
\end{gathered}
\end{equation}
\end{lemma}

\begin{proof}[Proof of \cref{lem:structured-escape-entry}]
Let \(\nu\) and \(C_{\mathrm{scale}}\) be the constants of \cref{lem:structured-escape-growth}, let \(r_0\), \(B\) and \(C_{\mathrm{align}}\) be the constants of \cref{lem:structured-escape-alignment} for \(\alpha\), \(C\) and this \(\nu\), and let \(C_{\mathrm{bis}}\) be the constant of \cref{lem:structured-escape-near-bisector}. The constants \(C_1,\ldots,C_6\) below depend only on \(\alpha\), \(C\) and \(\varepsilon_{\mathrm{entry}}\). We require \(c_{\mathrm{entry}}\le r_0\); a further condition is imposed below. Since \(c_{\mathrm{entry}}\le1\), the first inequality in \eqref{eq:structured-escape-entry-smallness} gives \(r_{\mathrm{init}}\le r_{\mathrm{init}}\sqrt d\le1\); hence \(\log(\sqrt d/r_{\mathrm{init}})\ge0\) and \(r_{\mathrm{init}}\sqrt d\le\varepsilon_{\mathrm{init}}\le c_{\mathrm{entry}}\).

We first follow the trajectory until \(R\) reaches \(r_s/2\). By \cref{lem:structured-escape-alignment}, the time \(\tau_0\) exists, and \(R(\tau_0)=Br_{\mathrm{init}}\le Br_0<r_s/4\). We apply \cref{lem:structured-escape-growth} to the solution \(t\mapsto w(\tau_0+t)\); its hypotheses hold by \eqref{eq:structured-escape-planar-signs-global} and \cref{lem:structured-escape-alignment}. Let \(t_1\coloneqq\tau_0+t_*\) be the first time after \(\tau_0\) at which \(R=r_s/2\), and let \(R_0=Br_{\mathrm{init}}\), \(a_0\), \(Y_0\), \(\delta_0\) be the values of \(R\), \(a\), \(Y\), \(\delta\) at \(\tau_0\). Then \(C_{\mathrm{align}}^{-1}a_{\mathrm{init}}\le a_0\le C_{\mathrm{align}}a_{\mathrm{init}}\) and \(Y_0+|\delta_0|\le C_{\mathrm{align}}r_{\mathrm{init}}\). We also use \(a_{\mathrm{init}}\le r_1(0)+r_2(0)=2r_{\mathrm{init}}\), \(a_{\mathrm{init}}\ge r_{\mathrm{init}}/(C\sqrt d)\), \(r_s/(2R_0)=r_s/(2Br_{\mathrm{init}})\) and \(r_s/(2B)<1\). On \([0,\tau_0]\) we bound \(a\) and \(Y\) by \cref{lem:structured-escape-alignment}, and on \([\tau_0,t_1]\) we use \((R/R_0)^{\mu_\alpha}\le(r_s/(2Br_{\mathrm{init}}))^{\mu_\alpha} \le r_{\mathrm{init}}^{-\mu_\alpha}\). For \(\delta(t_1)\), we use \((a_0^2+Y_0^2)/R_0\le5C_{\mathrm{align}}^2r_{\mathrm{init}}/B\) and \(\log(1/R_0)\le\log(1/r_{\mathrm{init}})\). With these facts, \cref{lem:structured-escape-growth} gives
\begin{equation}\label{eq:structured-escape-at-rs-half}
\begin{gathered}
    C_1^{-1}a_{\mathrm{init}}r_{\mathrm{init}}^{-\mu_\alpha}\le a(t_1) \le C_1a_{\mathrm{init}}r_{\mathrm{init}}^{-\mu_\alpha}, \qquad \frac{Y(t_1)^2}{a(t_1)}\le C_{\mathrm{scale}}^3\frac{Y_0^2}{a_0} \le C_1r_{\mathrm{init}}\sqrt d,\\
    a(t)+Y(t)\le C_1r_{\mathrm{init}}^{1-\mu_\alpha}\quad(0\le t\le t_1), \qquad |\delta(t_1)|\le C_1\Bigl(r_{\mathrm{init}} +a(t_1)^2\Bigl(1+\log\frac1{r_{\mathrm{init}}}\Bigr)\Bigr).
\end{gathered}
\end{equation}
Since \(\sup_{0<r\le1}r^{1-\mu_\alpha}(1+\log(1/r))\le1/(1-\mu_\alpha)\), these bounds give
\begin{equation}\label{eq:structured-escape-small-at-rs-half}
    a(t_1)+Y(t_1)+|\delta(t_1)|\le C_2r_{\mathrm{init}}^{1-\mu_\alpha}.
\end{equation}

Next we follow the trajectory from \(t_1\) to \(t_{\mathrm{entry}}\). We consider the region where \(r_s/4<R<2r_s\) and \(a+Y+|\delta|<\nu_0r_s/4\). There \cref{lem:structured-escape-near-bisector} applies, and it gives
\[
    |\dot R-(r_s-R)|+|\dot\delta|\le\frac{16C_{\mathrm{bis}}}{r_s^2}(a^2+Y^2), \qquad |\dot a|\le C_3a, \qquad \frac{\diff}{\diff t}Y^2\le C_3Y^2,
\]
with \(C_3\coloneqq(3+C_{\mathrm{bis}}(1+r_s))/r_s\), because \(2\sin(\alpha/2)/(\pi R)\le3/r_s\), \((a+Y)/R\le1\), \(|\delta|/R^2\le4\nu_0/r_s\) and \((a^2+Y^2)/R^3\le4\nu_0^2/r_s\) in this region. Let \(T_2\coloneqq\log(2r_s/\varepsilon_{\mathrm{entry}})>0\), \(C_4\coloneqq e^{C_3T_2}(1+16C_{\mathrm{bis}}r_s^{-2}(1+T_2)e^{C_3T_2})\) and \(S\coloneqq a(t_1)+Y(t_1)+|\delta(t_1)|\). We now impose the remaining condition on \(c_{\mathrm{entry}}\): \(C_2c_{\mathrm{entry}}^{1-\mu_\alpha}\le\min\{1,\nu_0r_s/(8C_4), \varepsilon_{\mathrm{entry}}/(8C_4)\}\). By \eqref{eq:structured-escape-small-at-rs-half} and \(r_{\mathrm{init}}\le c_{\mathrm{entry}}\), \(S\) satisfies the same bound, so \(w(t_1)\) lies in the region.

Let \(\tau_1\) be the supremum of the times \(t\in[t_1,T_{\max})\) with \(t\le t_1+T_2\) such that the trajectory lies in the region on \([t_1,t]\). For \(t_1\le t<\tau_1\), Gr\"onwall's inequality gives \(a(t)+Y(t)\le e^{C_3T_2}(a(t_1)+Y(t_1))\) and \(e^{-C_3T_2}a(t_1)\le a(t)\le e^{C_3T_2}a(t_1)\). Integrating the bound for \(\dot\delta\), and applying the variation-of-constants formula to \(\frac{\diff}{\diff t}(R-r_s)=-(R-r_s)+(\dot R-(r_s-R))\) with \(R(t_1)-r_s=-r_s/2\), we obtain
\[
    |\delta(t)-\delta(t_1)|+ \Bigl|R(t)-r_s+\frac{r_s}2e^{-(t-t_1)}\Bigr| \le\frac{16C_{\mathrm{bis}}}{r_s^2}(1+T_2)e^{2C_3T_2} \bigl(a(t_1)+Y(t_1)\bigr)^2.
\]
Since \(S\le1\), the last two bounds give, for \(t_1\le t<\tau_1\),
\begin{equation}\label{eq:structured-escape-transfer}
    a(t)+Y(t)+|\delta(t)|\le C_4S, \qquad \Bigl|R(t)-r_s+\frac{r_s}2e^{-(t-t_1)}\Bigr|\le C_4S^2.
\end{equation}
Hence, on \([t_1,\tau_1)\), we have \(a+Y+|\delta|\le\nu_0r_s/8\) and \(3r_s/8\le R\le9r_s/8\), since \(C_4S^2\le C_4S\le r_s/8\). By \cref{lem:structured-escape-near-bisector}, these bounds also give \(21r_s/64\le r_i\le81r_s/64\). If \(\tau_1=T_{\max}\), the solution would therefore stay in a compact subset of \(\Omega\) on the bounded interval \([t_1,T_{\max})\), which \cref{lem:relu-kernel-regularity} excludes. Hence \(\tau_1<T_{\max}\), the bounds extend to \(\tau_1\) by continuity, and they keep the trajectory inside the region at \(\tau_1\). Thus \(\tau_1=t_2\coloneqq t_1+T_2\), and \eqref{eq:structured-escape-transfer} holds on \([t_1,t_2]\).

At \(t_2\), \(\frac{r_s}2e^{-T_2}=\varepsilon_{\mathrm{entry}}/4\), so \(|R-r_s|\le\varepsilon_{\mathrm{entry}}/4+\varepsilon_{\mathrm{entry}}/8\). By \eqref{eq:structured-escape-distance} with \(s=0\) and \(N_\perp\le a+Y\),
\[
    \operatorname{dist}\bigl((w_1,w_2),\mathcal C_{c_0}\bigr)^2 \le2(R-r_s)^2+2\delta^2+(a+Y)^2 \le2\Bigl(\frac{3\varepsilon_{\mathrm{entry}}}8\Bigr)^2 +3\Bigl(\frac{\varepsilon_{\mathrm{entry}}}8\Bigr)^2 <\varepsilon_{\mathrm{entry}}^2.
\]
Before \(t_1\), we have \(R<r_s/2\), by the definitions of \(\tau_0\) and \(t_*\), and \eqref{eq:structured-escape-distance} gives \(\operatorname{dist}((w_1,w_2),\mathcal C_{c_0})>r_s/\sqrt2 >\varepsilon_{\mathrm{entry}}\). Hence \(t_{\mathrm{entry}}\) exists and \(t_1\le t_{\mathrm{entry}}\le t_2\le C_{\mathrm{align}}+2+T_2\).

It remains to prove \eqref{eq:structured-escape-entry-data}. Let \(C_5\coloneqq C_1C_4^2\). On \([t_1,t_{\mathrm{entry}}]\), \(a\) changes at most by the factor \(e^{C_3T_2}\le C_4\), \(Y^2\) increases at most by this factor, and \(a+Y\le C_4(a(t_1)+Y(t_1))\). With \eqref{eq:structured-escape-at-rs-half}, this gives the bounds for \(a_{\mathrm e}\) and \(Y(t_{\mathrm{entry}})^2\) with \(C_5\) in place of \(C_{\mathrm{entry}}\), and the bound \(N_\perp\le a+Y\le C_5r_{\mathrm{init}}^{1-\mu_\alpha}\) on \([0,t_{\mathrm{entry}}]\). By the bound for \(\dot\delta\) above, \(|\delta(t_{\mathrm{entry}})-\delta(t_1)|\le C_4(a(t_1)^2+Y(t_1)^2)\), and \(Y(t_1)^2\le C_1r_{\mathrm{init}}\sqrt d\,a(t_1)\le C_1r_{\mathrm{init}}\sqrt d\) because \(a(t_1)\le S\le1\). Together with \eqref{eq:structured-escape-at-rs-half} and \(a(t_1)\le e^{C_3T_2}a_{\mathrm e}\le C_4a_{\mathrm e}\), this gives the bound for \(\delta(t_{\mathrm{entry}})\).

For the last bound in \eqref{eq:structured-escape-entry-data}, we note that \(a_{\mathrm e}\le C_4S\le1\) and that \(r_{\mathrm{init}}^{-\mu_\alpha}\ge1\), so \(a_{\mathrm e}\ge C_5^{-1}a_{\mathrm{init}} \ge r_{\mathrm{init}}/(CC_5\sqrt d)\). Hence, with \(C_6\coloneqq1+\log(CC_5)\), \(\log(1/a_{\mathrm e})\le\log(CC_5)+\log(\sqrt d/r_{\mathrm{init}}) \le C_6(1+\log(\sqrt d/r_{\mathrm{init}})+\log d)\), and \(r_{\mathrm{init}}\sqrt d\log(1/a_{\mathrm e})\le C_6\varepsilon_{\mathrm{init}}\). Since \(r_{\mathrm{init}}\sqrt d\le1\), we have \(d\le r_{\mathrm{init}}^{-2}\) and \(1+\log(\sqrt d/r_{\mathrm{init}})+\log d\le4(1+\log(1/r_{\mathrm{init}}))\). With \(a_{\mathrm e}\le2C_5r_{\mathrm{init}}^{1-\mu_\alpha}\), it follows that
\[
    a_{\mathrm e}^2\Bigl(1+\log\frac1{r_{\mathrm{init}}}\Bigr) \log\frac1{a_{\mathrm e}} \le16C_5^2C_6\,r_{\mathrm{init}}^{2-2\mu_\alpha} \Bigl(1+\log\frac1{r_{\mathrm{init}}}\Bigr)^2 \le\frac{64C_5^2C_6}{(1-\mu_\alpha)^2}\,r_{\mathrm{init}}^{1-\mu_\alpha},
\]
where we used \(\sup_{0<r\le1}r^{1-\mu_\alpha}(1+\log(1/r))^2\le4/(1-\mu_\alpha)^2\). Since \(r_{\mathrm{init}}\le\varepsilon_{\mathrm{init}}\le1\), the terms \(C_5r_{\mathrm{init}}^{1-\mu_\alpha}\), \(C_6\varepsilon_{\mathrm{init}}\) and the last bound are all at most a constant times \(\varepsilon_{\mathrm{init}}^{1-\mu_\alpha}\). This proves \eqref{eq:structured-escape-entry-data} with \(C_{\mathrm{entry}}\) large enough.
\end{proof}

\section{Logarithmic escape from the bisector saddle}

\subsection{Hessian at the bisector saddle}\label{subsec:bisector-saddle-hessian}

\begin{proof}[Proof of \cref{lem:hessian_cal}]
At \(w_{\mathrm{saddle}}\), both students equal \(w_s=r_s\bar v\). By the continuous extension in \cref{lem:relu-two-neuron-hessian}, the off-diagonal blocks \(Q_2(w_s,w_s)\) of the Hessian equal \(\frac12I_d\), and the term \(Q_1(w_s,w_s)\) in the diagonal blocks vanishes. As recalled at the start of \cref{sec:signed-transverse-exclusion}, \(v_1=\cos(\alpha/2)\bar v+\sin(\alpha/2)v_\Delta\) and \(v_2=\cos(\alpha/2)\bar v-\sin(\alpha/2)v_\Delta\). Hence the angle between \(w_s\) and each teacher is \(\alpha/2\), \(\bar n_{v_1,w_s}=v_\Delta\) and \(\bar n_{v_2,w_s}=-v_\Delta\), so \(Q_1(w_s,v_1)=Q_1(w_s,v_2)=(\mu_\alpha/4)(I_d-\bar v\bar v^\top +v_\Delta v_\Delta^\top)\), where \(\mu_\alpha=2\sin(\alpha/2)/(\pi r_s)\) is the rate defined before \cref{thm:componentwise_thin_escape}. Both diagonal blocks therefore equal
\[
    A\coloneqq\frac12I_d-\frac{\mu_\alpha}2 \bigl(I_d-\bar v\bar v^\top+v_\Delta v_\Delta^\top\bigr).
\]
The vectors \(\bar v\), \(v_\Delta\) and \(u_3,\ldots,u_d\) are eigenvectors of \(A\), with eigenvalues \(\frac12\), \(\frac12-\mu_\alpha\) and \(\frac12-\mu_\alpha/2\), respectively, because \(I_d-\bar v\bar v^\top+v_\Delta v_\Delta^\top\) maps them to \(0\), \(2v_\Delta\) and \(u_j\). The Hessian maps \((x,x)\) to \(((A+\frac12I_d)x,(A+\frac12I_d)x)\) and \((x,-x)\) to \(((A-\frac12I_d)x,-(A-\frac12I_d)x)\). Hence the vectors \((x,\pm x)\) with \(x\in\{\bar v,v_\Delta,u_3,\ldots,u_d\}\) are eigenvectors of the Hessian, with the following eigenvalues:
\[
\begin{array}{c|c|c}
    x&\text{eigenvalue of }(x,x)&\text{eigenvalue of }(x,-x)\\ \hline
    \bar v&1&0\\
    v_\Delta&1-\mu_\alpha&-\mu_\alpha\\
    u_j,\ 3\le j\le d&1-\mu_\alpha/2&-\mu_\alpha/2
\end{array}
\]
By \eqref{eq:structured-escape-mu-range}, \(\mu_\alpha<1\), so the eigenvalues in the middle column are positive. The \(2d\) vectors \((x,\pm x)\) are orthogonal, so they form a basis of \((\RR^d)^2\), and the table gives the negative, positive and zero eigenspaces stated in \cref{lem:hessian_cal}.
\end{proof}

\subsection{Proof of \texorpdfstring{\cref{thm:componentwise_thin_escape}}{the logarithmic escape theorem}}\label{subsubsec:structured-escape-local-passage}

By \cref{lem:structured-escape-entry}, the trajectory comes within a small fixed distance of the segment \(\mathcal C_{c_0}\) after a bounded time, with a small planar splitting amplitude \(a_{\mathrm e}\). Afterwards \(a\) grows at a rate close to \(\mu_\alpha\) until \(N_\perp\) reaches \(\varepsilon\), which takes a time of order \(\mu_\alpha^{-1}\log(1/a_{\mathrm e})\). Over so long a time, a constant error in the growth rate would change the escape time by an amount that grows with \(\log(1/a_{\mathrm e})\), so we need errors that vanish at the saddle. Once the rate \(2\sin(\alpha/2)/(\pi R)\) in \cref{lem:structured-escape-near-bisector} is compared with \(\mu_\alpha\), the bounds of that lemma take this form near \(w_{\mathrm{saddle}}\).

\begin{lemma}[Equations of motion near the bisector saddle]\label{lem:structured-escape-local-saddle-estimates}
Fix \(\alpha\in(0,\pi/2]\), and let \(\nu_0\) be as in \cref{lem:structured-escape-near-bisector}. There is \(C_{\mathrm{local}}\ge1\), depending only on \(\alpha\), such that the following holds for every \(d\ge2\). Let \(w(0)\in\Omega\), and let \(t\in[0,T_{\max})\) be a time at which \(z_{1,2}>0>z_{2,2}\) and
\begin{equation}\label{eq:structured-escape-local-domain}
    |R-r_s|+|\delta|+a+Y\le\nu_0r_s.
\end{equation}
Then \(r_s/2\le r_i\le2r_s\) for \(i=1,2\), and
\begin{align}
    |\dot a-\mu_\alpha a| &\le C_{\mathrm{local}}\bigl(|R-r_s|+|\delta|+a+Y\bigr)a,
    \label{eq:structured-escape-local-a}\\
    \frac{\diff}{\diff t}Y^2 &\le\bigl(\mu_\alpha+C_{\mathrm{local}}(|R-r_s|+|\delta|+a+Y)\bigr)Y^2,
    \label{eq:structured-escape-local-Y}\\
    |\dot R-(r_s-R)|+|\dot\delta| &\le C_{\mathrm{local}}(a^2+Y^2).
    \label{eq:structured-escape-local-radial}
\end{align}
\end{lemma}

\begin{proof}[Proof of \cref{lem:structured-escape-local-saddle-estimates}]
Let \(C_{\mathrm{bis}}\) be the constant of \cref{lem:structured-escape-near-bisector}. We first check the hypotheses of that lemma. Since \(\nu_0\le1/8\), \eqref{eq:structured-escape-local-domain} gives \(7r_s/8\le R\le9r_s/8\). If \(R\ge r_s\), then \(a+Y+|\delta|\le\nu_0r_s\le\nu_0R\); if \(R<r_s\), then \(a+Y+|\delta|\le\nu_0r_s-(r_s-R)\le\nu_0R\), because \(\nu_0\le1\). Hence the lemma applies. Its first claim gives \(r_s/2\le7R/8\le r_i\le9R/8\le2r_s\), and its first bound, together with \(R\ge7r_s/8\), gives \eqref{eq:structured-escape-local-radial} whenever \(C_{\mathrm{local}}\ge2C_{\mathrm{bis}}/r_s^2\).

For the other two bounds, we compare the error terms of \cref{lem:structured-escape-near-bisector} with \(|R-r_s|+|\delta|+a+Y\). Since \(R\ge7r_s/8\) and \(a+Y\le r_s/8\),
\[
    \frac{a+Y}R+\frac{|\delta|}{R^2}+\frac{a^2+Y^2}{R^3} \le\frac{2(a+Y)}{r_s}+\frac{2|\delta|}{r_s^2}+\frac{a+Y}{4r_s^2} \le3\bigl(1+r_s^{-2}\bigr)\bigl(|\delta|+a+Y\bigr),
\]
where the last step uses \(1/r_s\le1+r_s^{-2}\). Moreover, \(2\sin(\alpha/2)/(\pi R)=\mu_\alpha r_s/R\), and \(\mu_\alpha<1\) by \eqref{eq:structured-escape-mu-range}, so
\[
    \Bigl|\frac{2\sin(\alpha/2)}{\pi R}-\mu_\alpha\Bigr| =\frac{\mu_\alpha|R-r_s|}R \le\frac{2|R-r_s|}{r_s} \le2\bigl(1+r_s^{-2}\bigr)|R-r_s|.
\]
Substituting these two bounds into the second and third bounds of \cref{lem:structured-escape-near-bisector} gives \eqref{eq:structured-escape-local-a} and \eqref{eq:structured-escape-local-Y} with \(C_{\mathrm{local}}\coloneqq4C_{\mathrm{bis}}(1+r_s^{-2})\), since \(C_{\mathrm{bis}}\ge1\). This constant also satisfies the requirement above.
\end{proof}

For each \(\alpha\in(0,\pi/2]\), with \(\nu_0\) and \(C_{\mathrm{local}}\) as in \cref{lem:structured-escape-near-bisector} and \cref{lem:structured-escape-local-saddle-estimates}, we set \(\zeta\coloneqq\min\{\nu_0r_s/8,\mu_\alpha/(16C_{\mathrm{local}})\}\) and \(\varepsilon_{\mathrm{entry}}\coloneqq\sqrt2\,\zeta/100\), which depend only on \(\alpha\). This choice is made so that, whenever \(|R-r_s|+|\delta|+a+Y\le4\zeta\), the inequality \eqref{eq:structured-escape-local-domain} holds and \eqref{eq:structured-escape-local-a} gives \(\dot a\ge\mu_\alpha a/2\). Since \(\nu_0\le1/8\), we have \(\varepsilon_{\mathrm{entry}}<\zeta\le r_s/64=c_0\), so \cref{lem:structured-escape-entry} applies with this \(\varepsilon_{\mathrm{entry}}\). Below, \(t_{\mathrm{entry}}\), \(a_{\mathrm e}\) and \(q_{\mathrm e}\) are as in that lemma.

\begin{lemma}[Escape from a neighborhood of the bisector saddle]\label{lem:structured-escape-local-passage}
Fix \(\alpha\in(0,\pi/2]\) and \(C\ge1\), and let \(c_{\mathrm{entry}}\) and \(C_{\mathrm{entry}}\) be the constants of \cref{lem:structured-escape-entry} for \(\alpha\), \(C\) and \(\varepsilon_{\mathrm{entry}}\). There is \(\varepsilon_0>0\), depending only on \(\alpha\) and \(C\), such that for every \(\varepsilon\in(0,\varepsilon_0]\) there is \(c_{\mathrm{pass}}\in(0,c_{\mathrm{entry}}]\), depending only on \(\alpha\), \(C\) and \(\varepsilon\), with the following property. For every \(d\ge2\) and every initialization that satisfies the hypotheses of \cref{lem:structured-escape-entry} with \(c_{\mathrm{pass}}\) in place of \(c_{\mathrm{entry}}\), the time \(T_\varepsilon\) is finite and
\begin{equation}\label{eq:structured-escape-local-passage-time}
    \Bigl|T_\varepsilon-t_{\mathrm{entry}} -\frac1{\mu_\alpha}\log\frac1{a_{\mathrm e}}\Bigr| \le\frac{C_{\mathrm{local}}+\log(4/\varepsilon)}{\mu_\alpha}.
\end{equation}
\end{lemma}

\begin{proof}[Proof of \cref{lem:structured-escape-local-passage}]
We set \(C_Y\coloneqq4C_{\mathrm{entry}}e^{4C_{\mathrm{local}}}\), \(C_\delta\coloneqq8(C_{\mathrm{entry}} +C_{\mathrm{local}}(1+2C_Y)/\mu_\alpha)\) and \(\varepsilon_0\coloneqq\zeta/(1+2C_\delta\zeta)\). Then \(C_\delta\ge16C_{\mathrm{local}}\), since \(C_Y\ge4\) and \(\mu_\alpha<1\), and, for \(0<\varepsilon\le\varepsilon_0\),
\begin{equation}\label{eq:structured-escape-local-eps}
    2C_\delta\varepsilon^2+\varepsilon\le\zeta.
\end{equation}
We fix \(\varepsilon\in(0,\varepsilon_0]\). Since \(\mu_\alpha<1\), we have \(C_{\mathrm{local}}\zeta\le\mu_\alpha/16<1/16\); as \(C_{\mathrm{local}}\ge1\), also \(\varepsilon\le\zeta<1/16\). We choose \(c_{\mathrm{pass}}\in(0,c_{\mathrm{entry}}]\) so small that
\[
    \beta\coloneqq2C_{\mathrm{entry}}c_{\mathrm{pass}}^{1-\mu_\alpha} \le\min\Bigl\{\frac{\varepsilon}{8C_Y},\frac{\zeta}{2C_\delta}, \frac{\zeta^2}{8C_Y}\Bigr\}.
\]

Now let \(d\ge2\), and consider an initialization that satisfies the hypotheses of \cref{lem:structured-escape-entry} with \(c_{\mathrm{pass}}\) in place of \(c_{\mathrm{entry}}\). Since \(c_{\mathrm{pass}}\le c_{\mathrm{entry}}\), that lemma applies, and \(\varepsilon_{\mathrm{init}}\le c_{\mathrm{pass}}\). The first inequality in \eqref{eq:structured-escape-entry-smallness} gives \(r_{\mathrm{init}}\le r_{\mathrm{init}}\sqrt d\le c_{\mathrm{pass}}\le1\), and \(c_{\mathrm{pass}}\le c_{\mathrm{pass}}^{1-\mu_\alpha}\le\beta/2\). Since \(a_{\mathrm{init}}\le r_1(0)+r_2(0)=2r_{\mathrm{init}}\), \eqref{eq:structured-escape-entry-data} gives \(a_{\mathrm e}\le2C_{\mathrm{entry}}r_{\mathrm{init}}^{1-\mu_\alpha}\le\beta\) and \(\sup_{0\le t\le t_{\mathrm{entry}}}N_\perp(t) +q_{\mathrm e}\log(1/a_{\mathrm e})\le C_{\mathrm{entry}}\varepsilon_{\mathrm{init}}^{1-\mu_\alpha} \le\beta/2\). Since \(\beta\le\varepsilon/(8C_Y)<1/e\), we have \(\log(1/a_{\mathrm e})\ge1\). Hence
\begin{equation}\label{eq:structured-escape-local-entry-seed}
    a_{\mathrm e}\le\beta, \qquad r_{\mathrm{init}}\sqrt d\le\frac\beta2, \qquad q_{\mathrm e}\le q_{\mathrm e}\log\frac1{a_{\mathrm e}}\le\frac\beta2, \qquad \sup_{0\le t\le t_{\mathrm{entry}}}N_\perp(t)\le\frac\beta2.
\end{equation}
By \eqref{eq:structured-escape-entry-data}, \(Y(t_{\mathrm{entry}})^2\le C_{\mathrm{entry}}r_{\mathrm{init}}\sqrt d\, a_{\mathrm e}\) and \(|\delta(t_{\mathrm{entry}})|\le C_{\mathrm{entry}}q_{\mathrm e}\). Since the distance from \((w_1,w_2)\) to \(\mathcal C_{c_0}\) is at most \(\varepsilon_{\mathrm{entry}}\) at \(t_{\mathrm{entry}}\), \eqref{eq:structured-escape-distance} gives \(|R(t_{\mathrm{entry}})-r_s|\le\varepsilon_{\mathrm{entry}}/\sqrt2 =\zeta/100\).

By \eqref{eq:structured-escape-planar-signs-global}, \(z_{1,2}>0>z_{2,2}\) on \([0,T_{\max})\), so \cref{lem:structured-escape-local-saddle-estimates} applies wherever \eqref{eq:structured-escape-local-domain} holds. Let \(\tau\) be the supremum of the times \(t\in[t_{\mathrm{entry}},T_{\max})\) such that, on \([t_{\mathrm{entry}},t]\), we have \(a<2\varepsilon\), the inequality \eqref{eq:structured-escape-local-domain} holds strictly, and
\begin{equation}\label{eq:structured-escape-local-comparison-bootstrap}
    Y^2<C_Yr_{\mathrm{init}}\sqrt d\,a, \qquad |\delta|<C_\delta(q_{\mathrm e}+a^2), \qquad |R-r_s|<\zeta, \qquad \int_{t_{\mathrm{entry}}}^t|R(u)-r_s|\,\diff u<1.
\end{equation}
These conditions hold at \(t_{\mathrm{entry}}\): \(a_{\mathrm e}\le\beta<2\varepsilon\); the first two inequalities follow from the bounds at \(t_{\mathrm{entry}}\) above, because \(C_Y,C_\delta>C_{\mathrm{entry}}\); and \(|R(t_{\mathrm{entry}})-r_s|\le\zeta/100\). Moreover, by \eqref{eq:structured-escape-local-entry-seed} and the choice of \(\beta\), we have \(C_{\mathrm{entry}}q_{\mathrm e}\le C_\delta\beta/16\le\zeta/32\), \(a_{\mathrm e}\le\zeta/32\) and \(Y(t_{\mathrm{entry}})^2\le C_{\mathrm{entry}}\beta^2/2\le C_Y\beta/8 \le\zeta^2/64\), so the left side of \eqref{eq:structured-escape-local-domain} is less than \(\zeta<\nu_0r_s\) at \(t_{\mathrm{entry}}\). Hence \(\tau>t_{\mathrm{entry}}\).

We now show that, on \([t_{\mathrm{entry}},\tau)\), each of these conditions other than \(a<2\varepsilon\) holds with room to spare. On this interval, \(a<2\varepsilon\le1\), so \eqref{eq:structured-escape-local-entry-seed} and \eqref{eq:structured-escape-local-comparison-bootstrap} give \(Y^2\le C_Y\beta\le\zeta^2/8\) and \(|\delta|\le C_\delta\beta/2+4C_\delta\varepsilon^2\). Since \(C_\delta\beta/2\le\zeta/4\) and, by \eqref{eq:structured-escape-local-eps}, \(4C_\delta\varepsilon^2+2\varepsilon\le2\zeta\), it follows that \(|R-r_s|+|\delta|+a+Y\le\zeta+(\zeta/4+2\zeta)+\zeta/2\le4\zeta\). Since \(16C_{\mathrm{local}}\zeta\le\mu_\alpha\), \eqref{eq:structured-escape-local-a} gives
\begin{equation}\label{eq:structured-escape-local-a-growth}
    \dot a\ge\frac{\mu_\alpha}2a \qquad\text{on }[t_{\mathrm{entry}},\tau).
\end{equation}
Hence \(a\) is strictly increasing, and \(\diff t\le2\diff a/(\mu_\alpha a)\). Changing the variable of integration from \(t\) to \(a\), and using \(a\le1\) and \(Y^2\le C_Yr_{\mathrm{init}}\sqrt d\,a\), we obtain, for \(t_{\mathrm{entry}}\le t<\tau\),
\begin{equation}\label{eq:structured-escape-local-integral-control}
\begin{gathered}
    t-t_{\mathrm{entry}}\le\frac2{\mu_\alpha}\log\frac1{a_{\mathrm e}}, \qquad \int_{t_{\mathrm{entry}}}^ta\,\diff u\le\frac{2a(t)}{\mu_\alpha}, \qquad \int_{t_{\mathrm{entry}}}^ta^2\,\diff u\le\frac{a(t)^2}{\mu_\alpha},\\
    \int_{t_{\mathrm{entry}}}^tY^2\,\diff u \le\frac{2C_Y}{\mu_\alpha}r_{\mathrm{init}}\sqrt d\,a(t), \qquad \int_{t_{\mathrm{entry}}}^tY\,\diff u \le\frac4{\mu_\alpha}\bigl(C_Yr_{\mathrm{init}}\sqrt d\,a(t)\bigr)^{1/2}.
\end{gathered}
\end{equation}

Integrating \eqref{eq:structured-escape-local-radial} and using \eqref{eq:structured-escape-local-integral-control} together with \(r_{\mathrm{init}}\sqrt d\,a\le r_{\mathrm{init}}\sqrt d\le q_{\mathrm e}\), we obtain
\[
    |\delta(t)| \le C_{\mathrm{entry}}q_{\mathrm e} +\frac{C_{\mathrm{local}}}{\mu_\alpha} \bigl(a(t)^2+2C_Yr_{\mathrm{init}}\sqrt d\,a(t)\bigr) \le\frac{C_\delta}8\bigl(q_{\mathrm e}+a(t)^2\bigr).
\]
By \eqref{eq:structured-escape-local-comparison-bootstrap}, \eqref{eq:structured-escape-local-integral-control} and \eqref{eq:structured-escape-local-entry-seed},
\[
    \int_{t_{\mathrm{entry}}}^t(|\delta|+a+Y)\,\diff u \le\frac1{\mu_\alpha}\Bigl(2C_\delta q_{\mathrm e}\log\frac1{a_{\mathrm e}} +C_\delta a(t)^2+2a(t)+4\bigl(C_Y\beta\bigr)^{1/2}\Bigr) \le\frac{13\zeta}{2\mu_\alpha} <\frac12,
\]
where we used \(C_\delta\beta\le\zeta/2\), \(4C_\delta\varepsilon^2+4\varepsilon\le4\zeta\), \((C_Y\beta)^{1/2}\le\zeta/2\) and \(16\zeta\le\mu_\alpha\). Since \(a>0\), \(Y^2/a\) is continuously differentiable, and \eqref{eq:structured-escape-local-a} and \eqref{eq:structured-escape-local-Y} give
\[
    \frac{\diff}{\diff t}\frac{Y^2}a \le2C_{\mathrm{local}}\bigl(|R-r_s|+|\delta|+a+Y\bigr)\frac{Y^2}a.
\]
By Gr\"onwall's inequality, the last integral bound and the bound \(\int_{t_{\mathrm{entry}}}^t|R-r_s|\,\diff u<1\), we obtain
\[
    \frac{Y(t)^2}{a(t)} \le\frac{Y(t_{\mathrm{entry}})^2}{a_{\mathrm e}}e^{3C_{\mathrm{local}}} \le C_{\mathrm{entry}}e^{3C_{\mathrm{local}}}r_{\mathrm{init}}\sqrt d <\frac{C_Y}4r_{\mathrm{init}}\sqrt d.
\]
Applying the variation-of-constants formula to \(\frac{\diff}{\diff t}(R-r_s)=-(R-r_s)+(\dot R-(r_s-R))\) and using \eqref{eq:structured-escape-local-radial}, we obtain
\[
    |R(t)-r_s| \le e^{-(t-t_{\mathrm{entry}})}\frac\zeta{100} +C_{\mathrm{local}}\int_{t_{\mathrm{entry}}}^te^{-(t-u)} \bigl(a(u)^2+Y(u)^2\bigr)\,\diff u \le\frac\zeta{100}+C_{\mathrm{local}}\bigl(4\varepsilon^2+C_Y\beta\bigr) <\frac\zeta2,
\]
because \(32C_{\mathrm{local}}\varepsilon^2\le2C_\delta\varepsilon^2\le\zeta\) and \(C_{\mathrm{local}}C_Y\beta\le C_{\mathrm{local}}\zeta^2/8\le\zeta/128\). Integrating the middle expression over \([t_{\mathrm{entry}},t]\) and exchanging the order of integration, we find in the same way
\[
    \int_{t_{\mathrm{entry}}}^t|R(u)-r_s|\,\diff u \le\frac\zeta{100} +\frac{C_{\mathrm{local}}}{\mu_\alpha} \bigl(a(t)^2+2C_Yr_{\mathrm{init}}\sqrt d\,a(t)\bigr) \le\frac\zeta{100}+\frac{C_{\mathrm{local}}}{\mu_\alpha} \bigl(4\varepsilon^2+2C_Y\beta\bigr) <\frac12,
\]
by the same two bounds and \(16\zeta\le\mu_\alpha\). Finally, with the bounds just proved,
\[
    |R-r_s|+|\delta|+a+Y \le\frac\zeta2 +\Bigl(\frac{C_\delta\beta}{16}+\frac{C_\delta}2\varepsilon^2+2\varepsilon\Bigr) +\frac\zeta2 \le\frac\zeta2+\Bigl(\frac\zeta{32}+2\zeta\Bigr)+\frac\zeta2 <\frac{\nu_0r_s}2.
\]

These bounds are stronger than the conditions defining \(\tau\), other than \(a<2\varepsilon\). By \cref{lem:structured-escape-local-saddle-estimates}, also \(r_s/2\le r_i\le2r_s\) on \([t_{\mathrm{entry}},\tau)\), and \eqref{eq:structured-escape-local-a-growth} gives \(\tau-t_{\mathrm{entry}}\le(2/\mu_\alpha)\log(2\varepsilon/a_{\mathrm e})\). If \(\tau=T_{\max}\), the solution would stay in a compact subset of \(\Omega\) on a bounded interval, which \cref{lem:relu-kernel-regularity} excludes. Hence \(\tau<T_{\max}\). By continuity the stronger bounds hold at \(\tau\), so \(a(\tau)=2\varepsilon\).

Since \(a_{\mathrm e}\le\beta<\varepsilon/4\), there is a first time \(t_-\in(t_{\mathrm{entry}},\tau)\) at which \(a=\varepsilon/4\). On \([t_{\mathrm{entry}},t_-]\), the bound for \(Y^2/a\) and \(C_Y\beta\le\varepsilon/8\) give \(Y^2\le C_Y\beta\varepsilon/32\le(\varepsilon/16)^2\), so \(N_\perp\le a+Y\le5\varepsilon/16\). Together with \eqref{eq:structured-escape-local-entry-seed}, this shows that \(N_\perp<\varepsilon\) on \([0,t_-]\), so \(T_\varepsilon\ge t_-\). At \(\tau\), we have \(N_\perp^2\ge z_{1,2}^2+z_{2,2}^2\ge a^2/2=2\varepsilon^2\), so \(T_\varepsilon\le\tau<\infty\).

It remains to estimate \(t_-\) and \(\tau\). By \eqref{eq:structured-escape-local-a} and the integral bounds above, \(\int_{t_{\mathrm{entry}}}^\tau|\dot a/a-\mu_\alpha|\,\diff u \le C_{\mathrm{local}}\). Since \(a(t_-)=\varepsilon/4\) and \(a(\tau)=2\varepsilon\), it follows that
\[
    \frac1{\mu_\alpha}\Bigl(\log\frac\varepsilon4-C_{\mathrm{local}}\Bigr) \le t_--t_{\mathrm{entry}}-\frac1{\mu_\alpha}\log\frac1{a_{\mathrm e}} \le\tau-t_{\mathrm{entry}}-\frac1{\mu_\alpha}\log\frac1{a_{\mathrm e}} \le\frac1{\mu_\alpha}\bigl(\log(2\varepsilon)+C_{\mathrm{local}}\bigr).
\]
As \(t_-\le T_\varepsilon\le\tau\) and \(\log(2\varepsilon)<0\), this proves \eqref{eq:structured-escape-local-passage-time}.
\end{proof}

\begin{proof}[Proof of \cref{thm:componentwise_thin_escape}]
Let \(\varepsilon_0\) be as in \cref{lem:structured-escape-local-passage} for the given \(\alpha\) and \(C\), fix \(\varepsilon\in(0,\varepsilon_0]\), and let \(c_{\mathrm{pass}}\) be as in that lemma. We show that the theorem holds with \(c\coloneqq c_{\mathrm{pass}}/2\) (so that \(c<1\)) and
\[
    K\coloneqq C_{\mathrm{entry}} +\frac{C_{\mathrm{local}}+\log(4/\varepsilon)+\log C_{\mathrm{entry}}} {\mu_\alpha},
\]
which depend only on \(\alpha\), \(C\) and \(\varepsilon\). Let \(d\ge2\), and consider an initialization as in the theorem. Exchanging the two students if necessary, we may assume that \(z_{1,2}(0)>0>z_{2,2}(0)\); this changes none of \(T_\varepsilon\), \(r_{\mathrm{init}}\) and \(a_{\mathrm{init}}\). Since \(c\le c_{\mathrm{pass}}\), the initialization then satisfies the hypotheses of \cref{lem:structured-escape-entry} with \(c_{\mathrm{pass}}\) in place of \(c_{\mathrm{entry}}\). By \cref{lem:structured-escape-local-passage}, \(T_\varepsilon\) is finite and satisfies \eqref{eq:structured-escape-local-passage-time}. By \cref{lem:structured-escape-entry}, \(0\le t_{\mathrm{entry}}\le C_{\mathrm{entry}}\), and \eqref{eq:structured-escape-entry-data} gives \(|\log(1/a_{\mathrm e})-\log(r_{\mathrm{init}}^{\mu_\alpha}/a_{\mathrm{init}})| \le\log C_{\mathrm{entry}}\). These three bounds give \(|T_\varepsilon-\mu_\alpha^{-1}\log(r_{\mathrm{init}}^{\mu_\alpha}/a_{\mathrm{init}})| \le K\).
\end{proof}

\end{document}